\documentclass[12pt]{article}

\usepackage{times}
\usepackage[a4paper, total={6in, 8in}]{geometry}

\usepackage{amsmath,amsfonts, amssymb,amsthm,mathrsfs,bm,mathtools}
\usepackage[numbers]{natbib}

\usepackage{enumitem}

\usepackage{tikz}
\usetikzlibrary{decorations.pathreplacing,calligraphy}
\usepackage{tikz-cd}

\usepackage{hyperref}
\usepackage{url}
\usepackage{graphicx}
\usepackage{mathabx}

\usepackage[export]{adjustbox}
\usepackage{multirow}
\usepackage{rotating}
\usepackage{framed}
\usepackage{booktabs}
\usepackage{accents}
\usepackage{algorithm}

\newcommand{\ubar}[1]{\underaccent{\bar}{#1}}

\DeclareMathOperator*{\argmax}{arg\,max}
\DeclareMathOperator*{\argmin}{arg\,min}

\newcommand{\Hi}{\mathcal{H}}
\newcommand{\cH}{\Hi}
\newcommand{\cG}{\mathcal{G}}

\newcommand{\R}{\mathbb{R}}
\newcommand{\dR}{{\mathbb{R}'}}
\providecommand{\abs}[1]{\left\lvert#1\right\rvert}
\providecommand{\norm}[1]{\left\lVert#1\right\rVert}

\providecommand{\inner}[2]{\left\langle #1,#2\right\rangle}

\newcommand{\cch}{\text{\upshape{cch}\,}}
\newcommand{\ch}{\text{\upshape{ch}\,}}
\newcommand{\aff}{\text{\upshape{aff}\,}}
\newcommand{\caff}{\overline{\text{\upshape{aff}}}\,}

\newcommand{\spn}{\text{\upshape{span}\,}}
\newcommand{\cspn}{\overline{\text{\upshape{span}}}\,}
\newcommand{\diam}{\text{\upshape{width}\,}}
\newcommand{\diamh}{\text{\upshape{width}}_h}
\newcommand{\diame}{\text{\upshape{width}}}
\newcommand{\vol}{\text{\upshape{vol}}}

\newcommand{\tr}{\text{\upshape{tr}\,}}
\newcommand{\intr}{\text{\upshape{int}\,}}
\newcommand{\restr}{\!\!\upharpoonright\!}
\DeclareMathOperator*{\infd}{inf\vphantom{p}} 
\newcommand{\mean}{\mathfrak{m}}
\newcommand{\E}{E}
\newcommand{\Var}{\text{Var}}
\newcommand{\X}{\mathcal{X}}
\newcommand{\Y}{\mathcal{Y}}
\newcommand{\Cov}{\tilde{\mathfrak{C}}}
\newcommand{\CovCap}{\wideparen{\tilde{\mathfrak{C}}}{\vphantom{\mathfrak{C}}}^{\otimes,(n)}_c}
\newcommand{\CovCapn}{\wideparen{\tilde{\mathfrak{C}}}{\vphantom{\mathfrak{C}}}^{\otimes,S_n}_c}
\newcommand{\CovCapntau}{\wideparen{\tilde{\mathfrak{C}}}{\vphantom{\mathfrak{C}}}^{S_n}_{c,\tau}}

\newcommand{\pStart}{{\bfseries P}\hspace{-0.23cm}{\bfseries P}\hspace{-0.23cm}{\bfseries P} \enspace}
\newcommand{\pEnd}{{\bfseries Q}\hspace{-0.28cm}{\bfseries Q}\hspace{-0.28cm}{\bfseries Q}}

\newcommand{\bn}{\bm{\|}}
\newcommand{\bl}{\bm{\langle}}
\newcommand{\br}{\bm{\rangle}_{2}}
\newcommand{\wh}{\widehat}

\newtheorem{prop}{Proposition}

\newtheorem{lemma}{Lemma}
\newtheorem{example}{Example}
\newtheorem{proposition}{Proposition}
\newtheorem{theorem}{Theorem}
\newtheorem{corollary}{Corollary}
\newtheorem{remark}{Remark}

\title{Compressed Empirical Measures \\ (In Finite Dimensions)}
\author{Steffen Gr\"unew\"alder\\
	University of York \\
Department of Mathematics \\
York, UK 
}
\date{}

\begin{document}

\maketitle

\begin{abstract}
We study approaches for compressing the empirical measure in the context of finite dimensional reproducing kernel Hilbert spaces (RKHSs).
In this context, the empirical measure is contained within a natural convex set and can be approximated using convex optimization methods.
Such an approximation gives rise to a coreset of data points. 
A key quantity that controls how large such a coreset has to be is the size of the largest ball around the empirical measure that is contained within the empirical convex set. The bulk of our work is concerned with deriving high probability lower bounds on the size of such a ball under various conditions and in various settings: we show how conditions on the density of the data and the kernel function can be used to infer such lower bounds; we further develop an approach that uses a lower bound on the smallest eigenvalue of a covariance operator to provide lower bounds on the size of such a ball; we extend the approach to  approximate covariance operators and we show how it can be used in the context of kernel ridge regression. We also derive compression guarantees when standard algorithms like the conditional gradient method are used and we discuss variations of such algorithms to improve the runtime of these standard algorithms. We conclude with a construction of an infinite dimensional RKHS for which the compression is poor, highlighting some of the difficulties one faces when trying to move to infinite dimensional RKHSs.
\end{abstract}

\newpage

\tableofcontents

\section{Introduction}
Many methods in machine learning and statistics make use of the empirical measure which is effectively a representation of the data. 
Reducing the number of points on which the empirical measure is supported, while preserving most of the information that is necessary for inference, can result in a significant speed-up of algorithms without sacrificing accuracy. We study the question of how to compress the empirical measure while preserving information in the context of \textit{finite dimensional reproducing kernel Hilbert spaces (RKHSs)}. To give an overview of our results it is useful to introduce the key objects of our investigation. We are generally concerned with data taking values in some set $\X$. Often we will assume this set to be compact. We then look at a kernel function $k$ defined on $\X$ and the corresponding RKHS $\cH$. For various results, it is useful to assume that the functions in $\cH$ are continuous or even Lipschitz-continuous. Our main interest lies in the unknown distribution $P$ of data $X_1,\ldots, X_n$ where we assume throughout that $X_1,\ldots,X_n$ are independent and identically distributed. We adopt a common convention from the empirical process theory literature and will denote by $Pf$ the integral $\int f(x) \,dP(x)$ whenever $f \in \mathcal{L}^1(\X,P)$. Since $P$ is unknown it is common to use the empirical measure $P_n$ as a surrogate, where $P_n f = (1/n) \sum_{i=1}^n f(X_i)$. There is a useful interplay between the measures $P$ and $P_n$ and RKHSs. Whenever $k(X_1,\cdot)$ is Bochner-integrable with respect to $P$ we can define $\mean = \int k(x,\cdot) \, dP(x) \in \cH$ and it follows that
\[
 \langle \mean,h \rangle = Ph, \text{ for all } h \in \cH.
\]
Similarly, by defining $\mean_n = (1/n) \sum_{i=1}^n k(X_i,\cdot)$ we have that $\langle \mean_n,h \rangle = P_n h$ for all $h\in \cH$.
Our aim in this paper is to find an element $\bar{\mean}_n$ such that 
\[
 \| \bar{\mean}_n - \mean_n \| \approx \| \mean_n - \mean\|
\]
to guarantee that $\| \bar{\mean}_n - \mean \|$ is of the same order as $\|\mean_n - \mean\|$ and $\bar{\mean}_n$ can be used in place of $\mean_n$ without sacrificing significant accuracy in applications.

To gain such an approximation $\bar{\mean}_n$, we make use of another fortunate circumstance. The element $\mean$ does not only lie in $\cH$ but within the convex set 
\[
	C = \cch \{k(x,\cdot) : x\in \X\},
\]
where $\cch$ denotes the closed convex hull. This is useful because the extremes of $C$ are contained within the set $\{ k(x,\cdot) : x \in \X\}$ and often we can reduce the study of $C$ to studying the interaction between $k(x,\cdot)$  and functions $h\in \cH$. For instance, the width of $C$ in a direction $h \in \cH, \|h \| = 1$, is 
\[
\diam_h(C) = \sup_{x \in \X} \langle k(x,\cdot), h \rangle - \inf_{x \in \X} \langle k(x,\cdot), h \rangle = \sup_{x\in \X} h(x)  - \inf_{x \in \X} h(x).
\]
The set $\{k(x,\cdot) : x\in \X\}$ is usually infinite and not directly useful for algorithms. However, when using $\mean_n$, we have another convex set in $\cH$ that is usable, that is the empirical convex set $C_n = \ch\{ k(X_i,\cdot) : i \leq n\}$ which contains $\mean _n$. The extremes of $C_n$ are contained within the finite set $\{k(X_i,\cdot) : i \leq n\}$.   

\begin{figure}[t]
\begin{center}
\begin{tikzpicture}[scale=1]
\draw (0,0) -- (4,0);
\node at (4.2,-0.2) {$\mathbb{R}$};
\foreach \Point in {(0.2,-0.04), (1.3,-0.04),  (3.5,-0.04)}{
    \node at \Point {\textbullet};
};
\foreach \Point in { (0.5,-0.04), (2,-0.04)}{
    \node[text=red] at \Point {\textbullet};
};
\foreach \Point in {(3.5,3),  (5,4), (6,2)}{
    \node at \Point {\textbullet};
};
\foreach \Point in { (4,3.5),  (5.7,3.2)}{
    \node[text=red] at \Point {\textbullet};
};
\draw[dashed]  (3.5,3) --  (4,3.5) -- (5,4) --  (5.7,3.2) --  (6,2)-- (3.5,3); 
\node at (5.5,4.5) {$\mathcal{H}$};
\node at (4.8,3.1) {\textbullet};
\node at (5.15,3.1) {$\mean_n$};
\node[text=red] at (4.85,3.35) {\textbullet};
\node (1) at (1,0.5) {};
\node (2) at (3,2.7) {};
\node (3) at (5,2) {};
\node (4) at (3.2,0.3) {};
\path[->]
(1) edge[bend left] node [left] {} (2)
(3) edge[bend left] node [left] {} (4);
\node at (12,2) {$\bigoplus$};
\foreach \Point in {(10.5,0), (8.8,2), (10.3,4),  (11,3.75), (11.1,0.6)}{
    \node at \Point {\textbullet};};
\draw  (10.5,0) -- (8.8,2) -- (10.3,4)   -- (11,3.75) --  (11.1,0.6) --  (10.5,0); 
\node at (10.6,2.7) {$\mathfrak{C}_{y,n}$}; 
\node at (10.3,2.4) {\textbullet};
\foreach \Point in { (8.8,2),  (11,3.75), (9.92,2.83)}{
    \node[text=red] at \Point {\textbullet};};
\node at (9.5,4.5) {$\widehat{\cH \odot \cH}$};
\node at (13.5,4.34) {$\dR \otimes \cH$};
\foreach \Point in {(12.5,1.3), (13.3,3), (13.8,2.75), (14,1),  (13.1,0.2)}{
    \node at \Point {\textbullet};};
\draw (12.5,1.3) -- (13.3,3) -- (13.8,2.75) -- (14,1) --  (13.1,0.2) -- (12.5,1.3);
\foreach \Point in { (12.5,1.3),  (14,1), (13.25,1.15)}{
    \node[text=red] at \Point {\textbullet};};
\node at (13.5,1.85) {$\mathfrak{m}_{y,n}$}; 
\node at (13.31,1.5) {\textbullet};
\node (5) at (10.5,0) {};
\node (6) at (13.1,0.2) {};
\node (7) at (8.8,2) {}; 
\node (8) at (12.5,1.3) {};
\node (9) at (10.3,4) {};
\node (10) at (13.8,2.76) {};
\node (11) at (11,3.75) {}; 
\node (12) at (14,1) {};
\node (13) at (11.1,0.6) {};
\node (14) at (13.3,3) {};
\path [dotted,gray] (5) edge[bend right] node [left] {} (6);
\path [dotted,gray] (7) edge node [left] {} (8);
\path [dotted,gray] (9) edge[bend left] node [left] {} (10);
\path [dotted,gray] (11) edge node [left] {} (12);
\path [dotted,gray] (13) edge node [left] {} (14);
\node at (0,4.5) {\textit{(i)}};
\node at (8,4.5) {\textit{(ii)}};
\end{tikzpicture}
\end{center}
\caption{\label{fig:coreset}(i) The figure depicts how a subset or coreset of the sample is selected: the data is embedded in $\cH$ by using the kernel function of $\cH$. An approximation algorithm is then applied to the convex polytope in $\cH$ to find an approximation of $\mean$ that uses only a few extremes of the convex polytope. The pre-images of these extremes are the sample points that are selected as the coreset. (ii) For most statistical problems approximating $\mean$ itself is insufficient and one has to approximate closely related quantities. In the case of least-squares regression, one has to approximate the operator $\mathfrak{C}_{y,n} \in \widehat{\cH \odot \cH}$ (see Section \ref{sec:prelim} and  Section \ref{sec:related_approx} for the definitions), which is closely related to the empirical covariance operator, and a `weighted' mean embedding  $\mean_{y,n} \in \dR \otimes \cH$. It is often of interest to approximate $\mathfrak{C}_{y,n}$ and $\mean_{y,n}$ simultaneously, for instance, when building a coreset for least-squares regression. This can be achieved by considering the direct sum $(\widehat{\cH \odot \cH}) \oplus (\dR \otimes \cH)$  and a `direct sum' of the convex polytopes in the two spaces. The relation between the extremes of the convex polytopes is highlighted in the figure through the dotted lines; i.e. an algorithm will select a pair that is connected by a dotted line and by selecting such a pair of extremes the approximation of both the covariance and mean element will change.}
\end{figure}

Standard techniques like the conditional gradient method or the kernel herding algorithm are directly applicable to approximate $\mean_n$ by convex combinations of $\{k(X_i,\cdot) : i \leq n\}$. The kernel herding algorithm generates an approximation of the form $(1/l) \sum_{i=1}^l k(X_{\iota_i},\cdot)$, where $\iota: \{1,\ldots,l\} \to \{1,\ldots,n\}$ is some selection of data points and $l \leq n$. The data points $X_{\iota(1)},\ldots, X_{\iota(l)}$ themselves can be seen as a coreset for the data set. This approach is visualized in Figure \ref{fig:coreset}.(i). The conditional gradient method does not provide such an average but an arbitrary convex combination of the points $k(X_1,\cdot), \ldots, k(X_n,\cdot)$ and cannot be used directly to find a coreset. However, a coreset is often not necessary and many algorithms can work directly with an approximation of $\mean$ or related quantities; we demonstrate this in Section \ref{sec:intro_application} and Section \ref{sec:applications}. The advantage of the conditional gradient method when compared to the kernel herding algorithm is that it usually leads to a vastly superior compression of the data. 

\begin{figure}[t]
\begin{tikzpicture}
\node at (6.5,0){\includegraphics[scale=0.7,trim=2cm 8cm 1cm 1cm,clip]{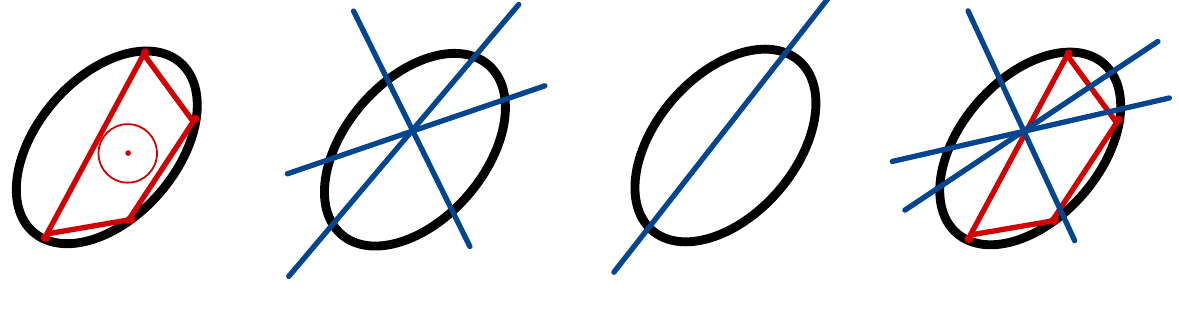}};
\node at (0,2.5) {\textit{(i)}};
\node at (3.5,2.5) {\textit{(ii)}};
\node at (7.5,2.5) {\textit{(iii)}};
\node at (11,2.5) {\textit{(iv)}};
\node at (0.7,1.5) {$C$}; 
\node[text=red] at (0.4,0.4) {$C_n$}; 
\node[text=red] (1) at (1.6,-1) {$\mean_n$}; 
\node (2) at (0.9,0.06) {};
\draw [thin,dashed,red,shorten <= 0.01cm, shorten >= 0.01cm] (1) to[out=90,in=0] (2);
\node at (3.4,1.7) {$h_1$}; 
\node at (5.1,1.7) {$h_2$}; 
\node at (5.75,1.15) {$h_3$}; 
\node at (7.8,-0.1) {\textbullet};
\node at (8.1,-0.1) {$\mean$};
\draw [decorate,decoration={calligraphic brace,amplitude=3pt}, line width=1.25pt] (7.22,-0.73) -- (7.7,-0.1);
\draw [decorate,decoration={calligraphic brace,amplitude=3pt}, line width=1.25pt] (7.82,0.05) -- (8.73,1.2);
\node at (7.3,-0.2) {$a$};
\node at (8,0.8) {$b$};
\node at (11.4,1.7) {$h_1$}; 
\node at (13.1,1.7) {$h_2$}; 
\node at (13.1,0.9) {$h_3$}; 
\end{tikzpicture}
\caption{\label{fig:overview_ball}The figure summarizes some of the key questions we address in this paper: \textit{(i)} This is the central question in this paper; `how large a ball exists within the empirical convex set $C_n$ around $\mean_n$?' \textit{(ii)} This question can be addressed by first controlling the width of $C$ itself in different directions $h_1,h_2,\ldots \in \cH$. The width in such a direction $h$ is the size of the projection of $C$ on the span of the function $h \in \cH$. Lower bounds on the width that hold simultaneously for all relevant $h$ translate to the existence of a ball in $C$; furthermore, the size of the ball is directly related to the lower bounds on the width. \textit{(iii)}  We need not just any  ball in $C$ but one that is centered at $\mean$. Now, generally, $\mean$ can lie close to the boundary and no large ball around it might exist. However, under certain natural conditions, it can be ruled out that $\mean$ will lie too closely to the boundary. In particular, under these conditions, we can control the ratio of $a/b$ for the segments shown in the figure. Controlling this ratio for all relevant $h\in \cH$ allows us to show that there exists a ball around $\mean$ in $C$. \textit{(iv)} To translate this back to $C_n$ and $\mean_n$ we are making use of empirical process theory to control the convergence of $C_n \to C$ and $\mean_n \to \mean$ which allows us to lower bound the size of a ball around $\mean_n$ in $C_n$ with high probability. Similarly to \textit{(ii)} we control the convergence per direction $h$ and then use high probability guarantees that hold simultaneously for all relevant $h$.}
\end{figure}

Crucially, the performance of these techniques depends on the size of the largest ball in $C_n$ that can be centered at $\mean_n$. The existence of such a ball is in itself already of major importance for the performance of the techniques and 
is known as \textit{Slater's condition}. In this paper, our main focus lies on the derivation of high probability lower bounds on the size of such a ball around $\mean_n$ within $C_n$. Figure \ref{fig:overview_ball} outlines our approach. In \textit{(i)} the setting is shown with $\mean_n \in C_n \subset C$ and the largest ball around $\mean_n$ in $C_n$ is drawn. One of the main difficulties is that both $\mean_n$ and $C_n$ are stochastic and change with the sample. 
We sidestep this difficulty by analyzing $C$ and $\mean$, and relating the empirical quantities $C_n$ and $\mean_n$ to $C$ and $\mean$.  Standard techniques from empirical process theory suffice to control the deviations between the empirical versions and their population limits (Figure \ref{fig:overview_ball}.\textit{(iv)}). There are at least two useful approaches to control the size of the largest ball around $\mean$ within $C$. The first approach is sketched in  \textit{(ii)} and \textit{(iii)}:   
first, we lower bound the width of $C$ uniformly over a range of `directions' $h$ in $\cH$ (Figure \ref{fig:overview_ball}.\textit{(ii)}). Then we determine how centered $\mean$ lies within $C$ in each direction $h$ (Figure \ref{fig:overview_ball}.\textit{(iii)}). Combining these two arguments, we can derive a lower bound on the size of the largest ball around $\mean$ in $C$. The second approach is quite different in that it does not try to control the width of the set $C$ explicitly. Instead, it uses the spectrum of the covariance operator to derive lower bounds on the largest ball around $\mean$ in $C$. In particular, a simple argument using the Paley-Zygmund inequality goes a long way and leads to lower bounds that are controlled by the smallest non-zero eigenvalue of the centered covariance operator. 


\subsection{Lower bounding the width of $C$}
When trying to control the width of $C$ the first thing one notices is that 
we seem to know relatively little about $C$. Even the RKHS $\cH$ itself is usually only accessed through $k$ and we do not  have easy access to a basis of $\cH$. So it might come as a surprise that there is a relatively simple way to access the width of $C$. The key to bounding the width is that  
\begin{equation*} 
	\diamh(C) = 2 \inf_{c\in \mathbb{R}} \|h - c \bm 1\|_\infty,
\end{equation*}
where $h\in \cH, \|h\|=1$, and $\bm 1$ is the constant function that is equal to $1$ everywhere. This relationship holds because 
\begin{align*}
\sup_{g \in C} \langle h,g \rangle 
-\infd_{g\in C} \langle h,g \rangle = \sup_{x  \in \X} \langle h, k(x,\cdot) \rangle 
-\infd_{x \in \X} \langle h, k(x,\cdot) \rangle = \sup_{x\in \X} h(x) - \infd_{x \in \X} h(x).
\end{align*}
The relevance of this equality is that it reduces the problem of measuring the width to the problem of measuring how well constant functions can be approximated by functions in the RKHS. The question of how well certain functions can be approximated by RKHS functions is well understood when the RKHS is infinite dimensional. In particular, the K-functional is a common tool to control the approximation quality, and results about the K-functional can be brought to bear to provide bounds on the width of $C$. However, in the finite dimensional setting, these results are of limited use. We develop for this case a simple approach to measure how well constant functions can be approximated: if the constant functions do not lie in the RKHS  $\cH$ then we can construct a new RKHS $\cH^+$ by introducing the kernel function $k^+ = k + \bm 1 \otimes \bm 1$, where $k$ is the kernel of $\cH$. The RKHS $\cH^+$ then contains the constant functions and $\cH \subset \cH^+$. In fact, we have an isometric embedding of $\cH$ into $\cH^+$. Now, in $\cH^+$ it is easy to measure how well constant functions can be approximated by functions in the unit sphere of $\cH$. In detail,
\[
\inf_{h \in \cH, \|h\|=1} \inf_{c\in \mathbb{R}} \|h - c\bm 1\|_{\cH^+} = 1.
\]
There are different ways to move from the norm of $\cH^+$ to $\|\cdot\|_\infty$ which we summarize in Lemma \ref{lem:tighter_lbnd} on p. \pageref{lem:tighter_lbnd}. One of these approaches applies if $k^+$ is a Mercer kernel and $\tilde \lambda_{d+1} > 0$ is the smallest eigenvalue in the series expansion. In this case 
\[
	2 \tilde \lambda_{d+1}^{1/2} \leq \diamh(C),
\]
for all $h\in \cH, \|h\|=1$.

If the constant functions lie already in $\cH$ then  a different approach is necessary. Let us mention that we only need to control the width of $C$ within the affine subspace that is spanned by it. Since $\langle k(x,\cdot), \bm 1 \rangle = 1 $ for all $x \in \X$ we can observe that the space spanned by $\bm 1$ is perpendicular to the affine subspace of $C$. To get a lower bound on $\diamh(C)$ for functions $h$ in the affine subspace we can consider the kernel $k^- = k - \|\bm 1\|^2 \bm 1 \otimes \bm 1$ and the corresponding RKHS $\cH^-$. The constant functions do not lie in $\cH^-$ and $\cH^-$ can be isometrically embedded in $\cH$. Most importantly the functions $h \in \cH^-$ of norm $\|h\|_{\cH^-} = 1$ are exactly the directions in which we need to bound the width of $C$. Now, with an approach analogous to the one involving $\cH$ and $\cH^+$ we get a lower bound of the form
\[
2 \tilde \lambda_d^{1/2} \leq \diamh(C),
\]
for all $h\in \cH^-, \|h\|_{\cH^-} = 1$, and with $\tilde \lambda_d$ being the smallest eigenvalue of the Mercer decomposition of the kernel $k$. Proposition \ref{prop:approx_lower_bnd} on p. \pageref{prop:approx_lower_bnd}
contains these results and results about related approaches to bound the width.

\subsection{Locating $\mean$ within $C$}
Controlling the width of $C$ alone is insufficient since $\mean$ might lie in the boundary of $C$. We need to complement the lower bounds on the width with results that tell us how centered $\mean$ lies. This can be achieved by controlling the ratio $a/b$ and $b/a$ of the segments along any function $h$ from $\mean$ to the boundary as depicted in Figure \ref{fig:overview_ball}.\textit{(iii)}. An observation that is useful in this context is the following:
if we have a probability measure on $\mathbb{R}$ which has a mean value of zero and there exists some measurable set $B$ with $\inf B \geq \epsilon > 0$ and $P(B) > 0$, then there will be probability mass on the negative axis since otherwise 
\[
	0 = \int_{\mathbb{R}} x \,dP  = \int_{[0,\infty)} x \, dP \geq \int_{B} x \,dP  \geq \epsilon P(B) > 0.
\]   
A similar argument applies to $C$ and $\mean$. For instance, if we have a uniform distribution on the boundary of the ellipse shown in 
Figure \ref{fig:overview_ball}.\textit{(iii)}, then $\mean$ cannot lie in the boundary: otherwise, there would exist a function $h \in \cH$, $\|h\|=1$, such that $\langle h, k(x,\cdot) \rangle \geq \langle h, \mean \rangle$ for all $x\in \X$ and an $\epsilon >0 $ such that $A = \{x : \langle h, k(x,\cdot) - \mean \rangle \geq \epsilon\}$ has non-zero measure. Hence,  
\[
0 =  \int \langle h, k(x,\cdot) - \mean \rangle \,dP(x)  \geq \epsilon P(A) > 0. 
\]
Combining this argument with a Lipschitz assumption on the kernel function and a lower bound on the density allows us to show that $\mean$ has to lie away from the boundary. How far away it has to lie is made precise in Proposition \ref{prop:Bauer_inspired} on p. \pageref{prop:Bauer_inspired}.

\subsection{Convergence of $C_n$ to $C$} 
To transfer the results about $C$ and $\mean$ to $C_n$ and $\mean_n$ we use VC and Rademacher arguments to bound the difference between $C_n$ and $C$, and $\mean_n$ and $\mean$. For controlling $\|\mean_n - \mean\|$ a standard argument suffices. However, it is less clear how to best control the difference between $C_n$ and $C$.

The approach that we are taking is the following. We consider indicators $\chi \{\langle k(X,\cdot) - \mean, h \rangle \leq -c  \}$ where $X$ is a random variable with the same distribution as $X_1,\ldots, X_n$ and $c$ is a constant that we vary. Observe that
whenever
\[
P \chi\{ \langle k(X,\cdot) - \mean, h \rangle \leq -c  \} > 0 
\]
then there is a point  $x\in \X$, such that $\langle k(x,\cdot) - \mean, h \rangle \leq -c$, or in other words, there is a point which lies $c$ away from $\mean$ along $h$. A VC argument allows us to control all these indicators simultaneously over all $h$ in the unit ball of $\cH$ and to show that for any such $h$,
\[
| P_n \chi\{ \langle k(X,\cdot) - \mean, h \rangle \leq -c  \} - P  \chi\{ \langle k(X,\cdot) - \mean, h \rangle \leq -c  \} |,  
\]
is small for sufficiently large $n$. This allows us to show that $C_n$ converges along $h$ towards $C$ with a certain rate and since we have guarantees that hold uniformly over the unit ball in $\cH$ we can derive a rate of convergence of $C_n$ to $C$.
A similar approach works for Rademacher complexities with the main difference being that we have to approximate the indicator functions with continuous functions.

Both approaches rely on a lower bound on the probability that $h(X)$ attains values below a threshold. We use two different approaches to get such lower bounds: the first approach uses an assumption on the the density (lower bounded away from zero) and a Lipschitz assumption on the functions in $\cH$. The second approach uses assumptions on the covariance operator. The second approach is more general in the sense that assumptions on the density imply a certain behavior of the covariance operator but our density assumption is certainly not the only way to control the covariance operator. On the other hand, the assumption on the covariance operator is quite abstract while the density assumption is in a sense very concrete.

Combining these different arguments allows us to control the size of the ball around $\mean_n$. In particular, our first theorem  combines the Rademacher approach with a density assumption (Theorem \ref{thm:ball_in_empirical}
on p. \pageref{thm:ball_in_empirical}). This approach brings together some of the  results on the width of $C$, the location of $\mean$ and the convergence results to show that for large enough $n$ there is with high probability a ball of a certain radius around $\mean_n$ in $C_n$. In detail, there exists a ball of size $\delta$ with the dominant term of $\delta$ being  
\[
\frac{2 \tilde c \tilde \lambda_d^{(l+1)/2} \beta_l}{(l+1) L^l},  
\]
where $\X = [0,1]^l$, $L$ is the Lipschitz constant, $\tilde \lambda_d$ the smallest eigenvalue of the Mercer decomposition of $k$, $\tilde c > 0$ is a lower bound on the density of the law of $X_1$ on $\X$ and $\beta_l$ is the Lebesgue measure of the $l$-dimensional unit ball in $\R^l$. 

With probability $q\in (0,1)$ there then exists a ball of radius $\delta/4$ around $\mean_n$ in $C_n$ whenever $n$ is greater than 
\[
	n \geq \left(\frac{\sqrt{2\log(1/q)} + 96 \|k\|_\infty^{1/2}/\delta}{\tilde c\beta_l(\delta/8L)^l}\right)^{2} \vee 
		\left(\frac{4\|k\|_\infty^{1/2} + 3 \sqrt{2\log(1/q)}}{\delta/4}\right)^{2}.
\]
We can observe that $\delta$ is strongly dependent on the dimension $l$ of the space $\X$. This stems from our approach: we identify a point $x_0 \in \X$ which corresponds to an element $k(x_0,\cdot) \in \cH$ that lies far away from  $\mean$. We then identify a second point $x_1$ such that $k(x_1,\cdot)$ lies in the opposite direction of $k(x_0,\cdot)$ with respect to $\mean$. If the space is low dimensional then $k(x_1,\cdot)$ needs to lie far from $\mean$ to counter the mass that is accumulated around $k(x_0,\cdot)$ and, thus, $\mean$ lies reasonably centered between $k(x_0,\cdot)$ and $k(x_1,\cdot)$. However, when the space is high dimensional then no single point $k(x_1,\cdot)$ has to lie far away from $\mean$ because the mass accumulated around $k(x_0,\cdot)$ can be countered by `many points' that lie close to $\mean$ and $\mean$ can lie significantly closer to the boundary. 

To contrast this worst-case bound with the best-case scenario, observe that there is a point in $C$ such that a ball of radius 
$2\tilde \lambda_d^{1/2}$ lies around it within $C$. The factor $\tilde \lambda_d^{1/2}$ itself is in all likelihood tight and reflects the fact that the convex set $C$ is very small in certain directions.

\subsection{Assumptions on the spectrum of the covariance operator}
An alternative approach to controlling the width of $C$ in different directions $h \in \cH$ and then determining how centered $\mean$ lies in each direction is to use assumptions on the covariance operator. In fact, the argument that involves the covariance operator is considerably simpler: when 
\[
E(h^2(X)) - E^2(h(X))  \geq \bar \lambda > 0 
\]
for some positive $\bar \lambda$ then $|h(X)|$ must attain large enough values with some non negligible probability. Furthermore, when $h$ is a bounded function then both $(h(X) - E(h(X)))^+$ and $(h(X)-E(h(X)))^-$ must be large with a non-negligible probability. A simple argument involving the Paley-Zygmund inequality suffices to make these statements precise. To get a lower bound on the largest ball around $\mean$ in $C$ we have to control all $h$ in the unit ball with this approach. In terms of the spectrum of the covariance operator this means that we have to use the smallest non-zero eigenvalue of the covariance operator as $\bar \lambda$. 

Another advantage of the covariance operator approach is that it adapts nicely to settings where the distribution has support $S$ that is not equal to $\X$. Effectively, algorithms like the CGM or kernel herding work implicitly with a subspace of $\cH$ that is isometric to an RKHS $\cH_S$ with kernel $k\!\!\upharpoonright\!\!S \times S$ (the restriction of $k$ to $S \times S$) and for $\cH_S$ we have a covariance operator that has the same non-zero eigenvalues as the covariance operator for $\cH$. Hence, we can use the same $\bar \lambda$ for $\cH_S$ as for $\cH$ and we can control the largest ball around $\mean_S$ in  $\cH_S$  through this argument.  Theorem \ref{thm:covariance} on p. \pageref{thm:covariance} is based on that argument.

\subsection{Adapting the approach to concrete statistical problems} \label{sec:intro_application}
Most methods for inference do not use $\mean$ itself but related quantities. For example, in the least squares problem, where we try to fit observations $Y_i$ through $f(X_i)$ with some function $f$ in an RKHS,
we have 
\begin{align*}
	\frac{1}{n} \sum_{i=1}^n (f(X_i) - Y_i)^2 &= \frac{1}{n} \sum_{i=1}^n \langle f \otimes f, k(X_i,\cdot) \otimes k(X_i,\cdot) \rangle_{\otimes} -  \frac{2}{n} \sum_{i=1}^n \langle f, Y_i k(X_i,\cdot) \rangle + \frac{1}{n} \sum_{i=1}^n Y_i^2 \\
						  &= \langle f \otimes f, \mathfrak{C}_n \rangle_{\cH \odot \cH} + 2 \langle f, \mean_{y,n} \rangle + \frac{1}{n} \sum_{i=1}^n Y_i^2,
\end{align*}
where we denote by $\cH \odot \cH$ the tensor space $\cH \otimes \cH$ when the functions are restricted to the diagonal $\Delta = \{(x,x) : x \in \X\}$,
$\mathfrak{C}_n = (1/n) \sum_{i=1}^n k(X_i,\cdot) \otimes k(X_i,\cdot)\!\!\upharpoonright \!\! \Delta$
and 
$\mean_{y,n} = (1/n) \sum_{i=1}^n Y_i k(X_i,\cdot)$.

There are significant similarities between the problem of compressing $\mean_n$ and that of compressing $\mathfrak{C}_n$ or $\mean_{y,n}$. We discuss  these in Section \ref{sec:related_approx}. Let us highlight a few results. 

The empirical covariance operator $\mathfrak{C}_n$ can be dealt with quite easily by associating it to the element $(1/n) \sum_{i=1}^n \kappa(X_i,\cdot)$, where $\kappa(x,y) = k^2(x,y)$. This way one can apply all the results we developed for $\mean_n$ to $\mathfrak{C}_n$, one only has to substitute $\kappa$ for $k$. 

Dealing with the element $\mean_{y,n}$ is more challenging and there is a certain degree of freedom of how to phrase the compression problem. A natural and simple choice is to consider $Y_i k(X_i,\cdot)$ as the random elements which attain values in $\cH$. A first indicator that things are more complicated is that when $Y_i$ is unbounded then we run into serious problems when trying to define a bounded convex set that contains $(1/n) \sum_{i=1}^n Y_i k(X_i,\cdot)$. Things simplify if we assume boundedness of the $Y_i$ and make some natural assumptions about how the data is generated. In particular, if we assume 
that $X_1,\ldots, X_n$ are i.i.d. and $Y_i = f_0(X_i) + \epsilon_i$, where $f_0$ is some bounded measurable function, the $\epsilon_i$'s are i.i.d., centered, independent of $X_1,\ldots, X_n$ and bounded by $|\epsilon_i | \leq b$ a.s., then $\mean_{y,n}$ converges to 
\[
	\mean_y = \int f_0(X_1) k(X_1,\cdot) \, dP \in \cH
\]
and $\mean_y$ is contained in the convex set 
\[
C_y = \cch \{ (f_0(x) \pm b) k(x,\cdot): x\in \X \}.
\]
In this setting there is also a simple relationship between the width of $C_y$ and $C$: consider some $h\in \cH, \|h\|=1$, then 
\[
\diamh (C_y) \geq b \, \diamh(C) 
\]
and results on $\diamh(C)$ are  applicable. The downside of this approach is that the convergence of the empirical convex set towards $C_y$ can be very slow since the $|\epsilon_i|$ might only have a low  probability of attaining values close to $b$. This problem can be circumvented by using an alternative approach. Instead of considering the convergence of the empirical convex set to a suitable population limit we can directly work with the empirical convex set and analyze how deep the empirical mean element lies within that set. We develop this approach in Section \ref{sec:bnd_Y}. The discussion in that section cumulates in Proposition \ref{prop:bnd_Y}, which provides lower bounds on the radius of a ball that is centered on the empirical mean element $\mean_{y,n}$ and which is contained within the empirical convex set.

We extend this approach to the case of unbounded $Y_i$ by using random variables $\wideparen{Y}_i$ that are capped at a certain, $n$ dependent, threshold. There are a variety of technical challenges that have to be overcome to make this approach work. In particular, one has to verify that the empirical mean element corresponding to the capped random variables is close the empirical mean element of the original variables when the threshold of the cap is selected appropriately. Also, one has now to work with a family of covariance operators corresponding to the different thresholds and the corresponding capped random variables. We show that the lowest eigenvalues of these covariance operators are close the the lowest eigenvalue of the original covariance operator if the threshold for the cap is set in the right way. Proposition \ref{prop:unbounded} contains the details of that result.
 
\paragraph{Simultaneous approximation.} Up to now we considered the approximation problems in isolation but it also makes sense to try to approximate $\mathfrak{C}_n$ simultaneously to $\mean_{y,n}$ by selecting elements $Y_i k(X_i,\cdot)$ that reduce the approximation error for both elements. Quite a different set of techniques are needed to deal with this simultaneous approximation problem. In Section  \ref{sec:simultaneous} we develop an approach based on direct sums of Hilbert spaces to deal with this problem. The analysis is much more intricate and interesting than for the individual approximation problems. In Figure \ref{fig:coreset}.\textit{(ii)} the high level approach is visualized. The space $\widehat{\cH \odot \cH}$ is the space of functions $\cH \odot \cH$ when trivially extended from $\X$ to $\R \times \X$ and the space $\dR \otimes \cH$ is an RKHS with kernel function $((y_1,x_1), (y_2,x_2)) \mapsto \langle y_1,y_2 \rangle_\R k(x_1,x_2)$ which is also defined on $\R \times \X$. The convex sets we introduced above have natural analogues in $\widehat{\cH \odot \cH}$ and in $\dR \times \cH$. By taking the direct sum of these spaces we also get a sort of direct sum of these convex sets and we are trying again to control quantities like the width of that set. The particular problem of approximating $\mathfrak{C}_n$ simultaneously to $\mean_{y,n}$ is benefiting from the fact that 
$\widehat{\cH \odot \cH} \cap (\dR \otimes \cH) = \{0\}$. This allows us to define an RKHS that is isometrically isomorphic to the direct sum. The analysis of the simultaneous approximation problem then breaks down to studying the empirical mean element and the empirical convex set within that RKHS.

The situation that the two Hilbert spaces that we combine through the direct sum are not overlapping is rather special. For instance, if we try to 
approximate $\mean$ and $\mathfrak{C}$ simultaneously then the Hilbert spaces overlap, which adds another layer of difficulties. We are developing for this case a quotient space approach that factors out the intersection between the two Hilbert spaces. An interesting finding in this context is that the direct sum cannot be related directly to an RKHS but, like in the case of approximating $\mean$ and $\mathfrak{C}$ simultaneously, the affine subspace spanned by the convex set can be isometrically isomorphic to an RKHS which then allows us to use results we developed for RKHSs (see Lemma \ref{lem:iso_iso_quotient}, p. \pageref{lem:iso_iso_quotient}). 

When we apply the conditional gradient method to the above RKHSs then we will not end up with a coreset of data points but with elements in $\widehat{\cH \odot \cH}, \dR \otimes \cH$ or $\widehat{\cH \odot \cH} \oplus (\dR \otimes \cH)$. However, that is not a major obstacle and it is for various problems quite easy to adapt the algorithms to deal with these approximations; we highlight that approach for kernel ridge regression in Section \ref{sec:applications}.

\subsection{Implications for algorithms}
The various results that we derived to control the size of the largest ball around $\mean_n$ in $C_n$ can be translated directly to results for algorithms like the CGM. In particular, we can give high probability guarantees on the approximation error when the CGM is being run for $t$ iterations and we can give guarantees on the expected size of a coreset when the kernel herding algorithm is used with a stopping criterion that is an error of below $n^{-1/2}$. The corresponding results are contained in Section \ref{sec:algorithms}.

One problem with these algorithms is that they require an upfront computation of order $O(n^2)$ which is too high for large-scale data. Standard approaches to scale the CGM to large-scale problems do not seem to yield direct computational advantages but there are some interesting directions to explore. In particular, a divide-and-conquer approach has some intriguing features. The performance of the approach depends to a large extent on the bias of the algorithms (CGM or kernel herding). Section \ref{sec:algorithms} contains a detailed discussion of these ideas.

\subsection{Slow rate of convergence in infinite dimensions} 
It was observed in \cite{BACH12} that the proof technique used to derive fast rates of convergence for the kernel herding algorithm and the conditional gradient method cannot be applied to compact sets in infinite dimensional RKHSs since compact sets in such spaces do not contain norm balls. It was later found that there are general limits to how well the representer of the empirical measure can be approximated. In particular,  \cite[Thm.3.1]{PHIL20} states that there exists a set of $n$ points $x_1,\ldots, x_n$ in $\R^d$, for large enough $d$ and $n \geq d$, such that for any set of points $y_1, \ldots, y_l$ with $l < \sqrt{dn}/2$ it holds that 
\[
\| \frac{1}{l} \sum_{i=1}^l k(y_i,\cdot) - \frac{1}{n} \sum_{i=1}^n k(x_i,\cdot) \| > \|k\|_\infty^{-1} n^{-1/2}, 
\]
under mild assumptions on the kernel. This implies that for this particular set of elements $x_1,\ldots, x_n$ there cannot be any significant compression of the element $\mean = (1/n) \sum_{i=1}^n k(x_i,\cdot)$. The argument in \cite{PHIL20} is not stochastic and is not concerned with draws of samples $X_1,\ldots, X_n$ from a distribution, but it seems likely that the argument can be extended to provide restrictions on how well $\mean_n$ can be approximated by a core-set of points in high probability (over the values that $\mean_n$ attains). Nevertheless, there is hope to circumvent the barries erected by this theorem. First of all, the construction uses approximations of the form $(1/l) \sum_{i=1}^l k(y_i,\cdot)$ and not arbitrary convex combinations of the elements $k(y_i,\cdot), i \leq l$. A greedy algorithm to find such a core-set of points $y_1,\ldots, y_l$ requires generally significantly more points than algorithms that approximate $\mean_n$ with convex combinations of elements $k(y_i,\cdot)$. An interesting question is therefore if there exists an inherent limitation for approximating with core-sets that can be avoided by more general convex combinations, or if this difference in performance is simply due to the algorithms (kernel herding vs. CGM). There is a simple argument that hints at the former: consider the set $\X = [0,1]^{d_\X}$ for some positive $d_\X \in \mathbb{N}$ and a continuous kernel function $k:\X \times \X \to \R$ whose corresponding RKHS is infinite dimensional and separable. The set $C$ is then compact and for any orthonormal basis $\{e_i\}_{i \geq 1}$ of $\cH$ it holds that $\sup_{f \in C} \langle e_i,f \rangle - \inf_{g \in C} \langle g, e_i \rangle$ converges to zero as $i \to \infty$. The rate with which this series converges is in all likelihood of crucial importance for determining how well $\mean_n$ can be approximated. Therefore, let us introduce  
\[
d(P_{U_i} C, C) = \sup_{f\in C} \| P_{U_i} f - f \|, 
\]
where $U_i$ is the subspace spanned by $e_1,\ldots, e_i$ and $P_{U_i}$ is the orthogonal projection onto this subspace. Now, Caratheodory's theorem tell us that for $i \geq 1$ there exists a convex combination $\wh \mean_i$ of $i+1$ elements $k(x_1,\cdot),\ldots, k(x_{i+1},\cdot)$ such that
\[
\| \mean_n - \wh \mean_i \| \leq d(P_{U_i} C,C).
\]
In other words, when $d(P_{U_i} C, C)$ is of order $i^{-\alpha}$, for some $\alpha>0$, and if we are aiming for an approximation error of $n^{-1/2}$ then we need approximately $n^{1/2\alpha}$ many points. Furthermore, when $d(P_{U_i}C,C)$ falls exponentially fast, say with order $\exp(-i)$, then $\log(n)$ many points suffice.

Another important aspect of the compression problem that is not captured by the theorem is the dependence of the compression problem on the distribution of the data. For example, in \cite{BACH12} a lower bound on the density of the distribution was crucial for deriving fast rates of convergence in certain settings. This is because such properties of the density translate directly to geometric properties of the approximation problem (the existence of a ball around $\mean_n$ in $C_n$). Similarly, in this paper, we use the size of the largest non-zero eigenvalue of the covariance operator to control the rate of convergence. One might wonder if such properties also influence the compression performance in infinite dimensional RKHSs. To this end, we provide an example that shows that an assumption on the density alone will in all likelihood be insufficient. The example we construct is \textit{not universal} in the sense that we show that the kernel herding algorithm does not achieve its fast rate of $1/t$ of approximation in this example. As in the example from \cite{PHIL20}, we construct a particular target $\mean$ and do not consider the empirical version $\mean_n$.
However, our construction incorporates properties of the underlying probability measure and might serve as a starting point for more refined analyses that use properties of the distribution of the data. 
The counter-example is constructed for the kernel herding algorithm and not the conditional gradient methods since the behavior of the kernel herding algorithm is easier to control but we strongly suspect that similar problems will also occur with the conditional gradient method.

\begin{figure}[t]  
\begin{center} 
\includegraphics[scale=1.5,trim=0cm 0.3cm 0cm 1cm,clip]{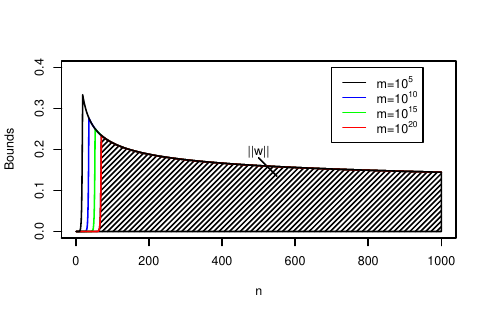}  
\end{center}
\caption{
The figure shows lower bounds on $\abs{\inner{e_n}{w}}$ in dependence of the first element $a_m$ that has not yet been chosen. The shaded area is a lower bound on $\norm{w}$ when $m= 10^{20}$. The norm of $w$ goes to infinity in $m$ which implies that the kernel herding algorithm converges with a rate that is slower than $1/t$.
} \label{fig:CounterEx} 
\end{figure}

In detail, the example we construct shows that there exists a continuous kernel on $[0,1]$, a Borel probability measure on $[0,1]$ which assigns positive measure to open subsets of $[0,1]$, and an initialization for which the kernel herding algorithm converges with a slower rate than $1/t$ when approximating the representer $\mean$ of the probability measure (Theorem \ref{prop:counter-main} on p. \pageref{prop:counter-main}).
The construction of this example is somewhat involved since we need to gain control over the behavior of the kernel herding algorithm. The basic intuition, however, is rather simple. We start with some infinite dimensional Hilbert space $\cH$ and an orthonormal sequence $\{e_n\}_{n\geq 1}$ in it. The construction is best explained when assuming that   $\mean = 0 $  (we cannot set it exactly to $0$ and need later a minor modification). We then construct a compact convex set that contains elements $\{a_n\}_{n\geq 1}$, $\{b_n\}_{n\geq 1}$, where each $a_n$ is a positive multiple of $e_n$ and each $b_n$ is a negative multiple of $e_n$. Furthermore, $b_n$ is of significantly smaller magnitude than $a_n$. Consider now an initialization of the algorithm with an element $c \in \cH$ which is of small magnitude compared to the $a_n$ and has a positive inner product with each $a_n$. Because of this positive inner product the different $a_n$'s will be chosen one by one by the algorithm and because the $b_n$'s are of small magnitude compared to the $a_n$'s hardly any weight will be reduced in the directions $e_n$.  This way the element $w_t$, which measures the approximation error at iteration $t$,  builds up mass in the different directions $e_n$ and its norm grows in $t$. The construction is more involved than this sketch, but, a suitably adapted version of this approach allows us to show that so much mass will be added to $w_t$ that its norm diverges to infinity. This effect is visualized in Figure \ref{fig:CounterEx}. The figure shows four different $w_t$ as inner products with $e_n$ ($n$ being shown on the $x$-axis). The shaded area continues past the right end of the plot (the limit of the shaded area is given in the legend: $10^5$ for the black line etc.). One can observe that the right limit of the shaded area grows significantly from the black line to the red line, i.e. from $10^5$ to $10^{20}$. While the right limit grows exponentially the left limit hardly changes. This is due to the small scale of the $b_n$'s. As a result the overall mass in the shaded area, which corresponds to  $\norm{w_t}$, diverges to infinity. This implies then directly that the algorithm cannot converge with the fast rate of $1/t$ that is achieved under similar assumptions in the finite dimensional setting. All that then remains to complete the example is to show that there exists a continuous kernel that gives rise to this setup. We construct first a continuous function $\phi:[0,1] \rightarrow \cH$ that goes over all $a_n$ and $b_n$ and we then use this Hilbert space and the continuous function to construct an RKHS with a  continuous kernel function.

\vspace{-0.2cm}
\subsection{Literature}
The concept of a coreset is known for at least two decades and there is a wide range of literature on its application to machine learning, Bayesian statistics and geometric approximation problems (e.g. \cite{BAD02,AGAR05,HUG16}). It is natural to apply the conditional gradient method \cite{FW56} in that context (e.g. \cite{HUG16}).  

The kernel herding algorithm and the conditional gradient method are greedy approximation algorithm as they choose at each stage $t$ an element that minimizes the remaining error. Greedy algorithms will generally not return the best possible approximation that can be achieved in $t$ steps but they are easy to compute. This is a big advantage since in the large data context computational efficiency is paramount. Greedy algorithms for approximating functions have been popular at least since the late nineties. An overview of the most popular approaches is provided in \cite{TEMLY11}. The approach is here to make use of a basis of a function space, say of a Sobolev or Besov space, to approximate elements inside these function spaces in a greedy fashion. An important generalization is to use so-called dictionaries which are families of functions that are not necessarily linearly independent, i.e. there are redundancies in the representation of elements in the function space. 
These approaches are very natural if one has access to a basis or related families of functions. In contrast to this approach, we are interested in approximating subsets of the function space that are naturally described by point-evaluators, a kernel function, or, more generally, a set of extremes of a convex set. Instead of working then with linear subspaces of the function space we are working with convex subsets of the function spaces and we apply greedy algorithms to approximate elements inside such convex sets.   

The methods we are studying compress the sample into a potentially small subset of the original sample while retaining optimal, or nearly optimal, rates of convergence. While our approach is inspired by various optimization methods there are links to sample compression schemes as introduced in \cite{LITTLE86,FLOYD95}. Sample compression schemes are concerned with the inference of `concepts', which are indicators $\chi A$, $A$ a Borel subset of some topological space $X$. In this setting, one has given a set of concepts that contains the concept $\chi A$, or are sufficient to approximate $\chi A$ in a suitable way, and one likes to infer $\chi A$ from observations  $(x_1,y_1), \ldots, (x_n,y_n)$, $x_i \in \X$, $y_i \in \{0,1\}$. A sample compression scheme 
compresses these observations into a subset that is sufficient to reconstruct the original labels $y_i$ for all $x_i$, $i\leq n$, if the observations are consistent with some concept $\chi A'$, where $\chi A'$ is contained in the predefined set of concepts. 
Compressibility is directly linked to VC-theory: in \cite{FLOYD95} it is shown that, under some technical conditions, sets with VC-dimension $d$ are $d$-compressible, meaning that one can always reduce the sample to a sub-sample of size $d$ while still being able to reconstruct the sample in the above sense. Furthermore, it is not possible to compress the sample to less than $d$-points without losing the reconstructability property. 
Our aim is quite different in that we do not care about being able to reconstruct the original labels. In that sense our approach is more closely related to sufficient statistics which compress the data to facilitate inference. That being said, there are interesting parallels. For instance, Caratheodory's theorem tells us that, in our setting, there is a compression of the data down to $d+1$-points if we work with a $d$-dimensional RKHS; such an RKHS has VC-dimension $d$. 

Naturally, there are a variety of alternative approaches to deal with large scale data in the RKHS context. In particular, when the RKHS is finite dimensional with dimension $d$ it is straight forward to represent $\mean_n$ using a basis expansion: take points $X_1,\ldots, X_d$ such that $k(X_1,\cdot), \ldots, k(X_d,\cdot)$ are linearly independent and apply the Gram-Schmidt orthogonalization procedure to gain a basis $e_1,\ldots, e_d$ of $\cH$ then 
$\mean$ can be written as a linear combination of $e_1, \ldots, e_d$ which implies that it can be written as a linear combination of $k(X_1,\cdot), \ldots, k(X_d,\cdot)$. 
In more detail, the coefficients $\alpha_1,\ldots, \alpha_d$, such that $\mean_n = \sum_{i=1}^d \alpha_i k(X_i,\cdot)$, can be computed recursively by first computing the basis representation through 
{\allowdisplaybreaks
\begin{gather*}
\langle e_1, \mean \rangle = \langle k(X_1,\cdot), \mean_n \rangle /k(X_1,X_1) \\
\vdots \\
c_i = \langle k(X_i,\cdot), \mean_n \rangle - \sum_{j=1}^{i-1} \langle k(X_i,\cdot), e_{i-j} \rangle \langle e_{i-j}, \mean_n \rangle  \\ 
\langle e_i, \mean_n \rangle = c_i / \| k(X_i,\cdot) - \sum_{j=1}^{i-1} \langle k(X_{i-j}, \cdot), e_j \rangle e_j \|, 
\end{gather*}}
and then to link this back to the coefficients of $k(X_1,\cdot), \ldots, k(X_d,\cdot)$. 
To perform this Gram-Schmidt procedure it is necessary to compute $k(X_i,X_j)$ for all $i,j \leq d$. In other words, we need in the order of $d^2$ many kernel evaluations.
This is a negligible factor when $d \ll n$. Similarly, one can solve concrete statistical problems, like a linear regression problem, by using a $d\times d$ covariance matrix instead of the kernel matrix; one way to gain such a covariance matrix is to use again the  Gram-Schmidt procedure. Our aim in this paper is not to compete with these methods in terms of runtime performance in the context of $d \ll n$, but to gain insights into the behavior of greedy algorithms in the absence of complications that arise in infinite dimensional settings.

The question of how to construct coresets for $\mean$ has garnered significant attention in recent years. In \cite{DWI21} a good overview is given that covers recent approaches most of which focus on the infinite dimensional setting. In the context of finite dimensional RKHSs it is worth mentioning the paper \cite{HARV14} which studies linear kernel functions and shows that under certain conditions they can achieve a compression down to $n^{1/2}$.

\subsection{Preliminaries} \label{sec:prelim}
Throughout this paper we will be working with a set $\X$ in which covariates or features attain values and a kernel function $k:\X \times \X \to \mathbb{R}$ (see \cite[Def.2.12]{PAUL16}). Recall that such a kernel function gives rise to an RKHS $\cH$ \cite[Def2.14]{PAUL16}. While $\X$ does not need a particular structure to define a kernel on, we are interested in integrals involving $k$ and we will assume for most of our results that $\X$ is a measureable space and $k$ is a measurable in the sense that $k(x,\cdot): \X \to \R$ is measurable for all $x\in \X$. This is equivalent to saying that any $h \in \cH$ is a measurable function from $\X$ to $\R$ (see \cite[Lem.4.24]{STEIN08}). We also use the notation $\phi(x) = k(x,\cdot)$ when this is convenient.

We are making use of empirical process theory in various places and to ease the application we will assume that our underlying probability space 
corresponds to a product space and the involved random variables are coordinate projections following essentially \cite[Sec.3.1]{DUDLEY14}. In detail, we will usually have a probability space $(\Omega,\mathcal{A},\mu)$ with independent and identically distributed random variables $X,X_1,X_2,\ldots $ attaining values in $(\X,\mathcal{A}_\X)$, where $\X$ is a topological space and $\mathcal{A}_\X$ is a $\sigma$-algebra on $\X$, which are defined on this probability space. Natural choices for $\mathcal{A}_\X$ are the Borel-algebra or the domain of a Radon measure. We usually do not need assumptions on $\mathcal{A}_\X$ but at various points we need to guarantee that the support of the law $P$ of $X$ is well defined. In these cases we typically assume that $P$ is a $\tau$-additive topological measure and $\mathcal{A}_\X$ is its domain. Alternatively, we could assume that $P$ is a Radon measure which guarantees that $P$ is a $\tau$-additive topological measure (see \cite[411]{FREM}). This is for a wide range of spaces not a strong assumption. In particular, if $\X$ is a Polish space then it is a Radon space \cite[434K(b)]{FREM} and the completion of a Borel measure on $\X$ is a Radon measure \cite[434F(a.iii), 211L]{FREM}.
$\Omega$ will usually be the product $\X^\mathbb{N}$ and for $\omega \in \Omega$, $X_i(\omega) = \omega_i \in \X$ for all $i \geq 1$, and $X(\omega) = \omega_0$. There are multiple natural choices for the $\sigma$-algebra $\mathcal{A}$. In \cite[Sec.3.1]{DUDLEY14} $\mathcal{A}$ is the product $\sigma$-algebra which is the one that is generated by the cylinder sets, that is the smallest $\sigma$-algebra such that all cylinders which are defined by finite many coordinates are measurable. We use in this paper the completion of this $\sigma$-algebra as $\mathcal{A}$.  
If we have pairs $(X_i,Y_i)$, where $X_i$ attains values in $\X$ and $Y_i$ in $\R$  then we use the same setting but let $(X_i,Y_i)(\omega) \in \X \times \R$. We reserve $P$ for the law of the random variables, e.g. the law of $X$, and use $\Pr$ if we want to state probabilities of events in $\mathcal{A}$. In particular, $h(X) \in \mathcal{L}^1(\Omega,\mu)$ if, and only if, $h \in \mathcal{L}^1(\mathcal{X}, P)$ and, in this case, $\int h(X) \, d\mu = \int h \,d P$.   
The empirical measure $P_n$ is $(1/n) \sum_{i=1}^n \delta_{X_i}$, where $\delta_x(A) = 1$ whenever $A \in \mathcal{B}_\X$ and $x \in A$; otherwise $\delta_x(A) = 0$. It is often useful to associate a measure space to $P_n$ to be able to talk about random variables with law $P_n$. For this purpose we will use the measure space  $(\X,\mathcal{B}_\X)$ and equip it with the random measure $P_n$. A random variable will be the measurable function $\tilde X:\X \to \X, \tilde X(x) = x$. If we want to talk about a sequence of independent random variables with law $P_n$ we  use the product space with the product measure assigned to it. 

\paragraph{Separable processes and Rademacher complexities.} There are generally various measurability concerns when working with empirical processes. In this paper these can essentially be avoided by using  separability of $\cH$ to guarantee that suprema are  measurable. In the context of Rademacher complexities we use separability of $\cH$ typically in the following way. Assume we have $x_1,\ldots, x_n \in \X$, let $\mathcal{F}$ be the unit ball of $\cH$ and let $\epsilon_1,\ldots, \epsilon_n$ be i.i.d. Rademacher variables. The map $h \mapsto \sum_{i=1}^n \epsilon_i h(x_i)$ is almost surely continuous on $\cH$. In particular, $\sup_{h \in \mathcal{F}} \sum_{i=1}^n \epsilon_i h(x_i)$ is almost surely equal to a supremum over a countable subset of $\mathcal{F}$ and, due to completeness of the probability space, it follows that $\sup_{h \in \mathcal{F}} \sum_{i=1}^n \epsilon_i h(X_i)$ is measurable. In particular, the Rademacher process is a separable stochastic process \cite[Def.4.1.2]{GINE15} and we have 
\begin{equation} \label{eq:Rade_sep}
E\bigl(\sup_{h \in \mathcal{F}} \sum_{i=1}^n \epsilon_i h(X_i) \bigr) = \sup_{F \subset \mathcal{F}, F \text{ finite}}  E\bigl(\sup_{h \in F} \sum_{i=1}^n \epsilon_i h(X_i)\bigr). 
\end{equation}
When we have i.i.d. variables $X_1,\ldots,X_n$ which are independent of $\epsilon_1,\ldots, \epsilon_n$ 
we will represent this probability space as a product space. It is common to condition wrt. $X_1,\ldots, X_n$ and to study $E_\epsilon  \bigl(\sup_{h \in \mathcal{F}} \sum_{i=1}^n \epsilon_i h(X_i) \bigr)$,   where $E_\epsilon$ denotes Kolmogorov's conditional expectation with respect to $X_1,\ldots, X_n$. Fubini's theorem guarantees us in this setting that we can express $E_\epsilon$ as an integral wrt. the marginal measure corresponding to $\epsilon_1,\ldots, \epsilon_n$.
\paragraph{Bochner integrals and $\mathcal{L}^p(\mu,\cH)$.} We need in various places vector valued integrals. In particular, we make use of Bochner integrals and Hilbert-space valued $\mathcal{L}^p$ spaces. Let $(\Omega,\mathcal{A},\mu)$ be a probability space and $X$ a random variable that attains values in $\X$ then by $\int f(X) \,d\mu$, $f:\Omega \to \cH$ Bochner integrable, we mean the Bochner integral of the function $f(X) : \Omega \to \cH$ with respect to the measure $\mu$. The Hilbert space valued $\mathcal{L}^p$ spaces, where $1\leq p < \infty$, corresponding to this measure space are given by 
\[\mathcal{L}^p(\mu; \cH) = \{f: \Omega \to \R : f \text{ Bochner measurable and } \int \|f(\omega)\|^p \, d\mu(\omega) < \infty \}.\] 
The seminorm on $\mathcal{L}^p(\mu;\cH)$ is $\bn f \bn_p^p = \int \|f\|^p \, d\mu$. We use bold fonts for the  $\mathcal{L}^p(\mu; \cH)$ seminorms throughout this paper. As usual there are corresponding spaces $L^p$ of equivalence classes with norms $\bn \cdot \bn_p$ under which these $L^p$ spaces are complete. The space $L^2(P;\cH)$ is a Hilbert space with the inner product corresponding to the bi-linear function
$\bl \cdot, \cdot \br$ on $\mathcal{L}^2(\mu;\cH)$ given by $\bl f,g \br = \int \langle f(\omega),g(\omega)\rangle \, d\mu(\omega)$  whenever $f,g \in \mathcal{L}^2(\mu;\cH)$. Of particular importance to us is the Bochner integral 
$\int k(X,\cdot) \, d\mu \in \cH$ which is well defined whenever $k(X,\cdot) \in \mathcal{L}^1(\mu;\cH)$ and $\cH$ is separable. We will denote this integral by $\mean$. Finally, we have the following important relation between the inner product in $\cH$ and Bochner integrals: whenever $f \in \mathcal{L}^1(\mu;\cH)$ and $h\in \cH$ then according to \cite[Thm.6,p.47]{DIES77},
\[
\langle h, \int f \, d\mu \rangle = \int \langle f,h\rangle \, d\mu. 
\]
In rare occasions we will make statements about equivalence classes and not functions itself. We use the notation $f^\bullet$ to denote the equivalence class corresponding to $f$, i.e. if $f \in \mathcal{L}^2(\mu)$ then $f^\bullet \in L^2(\mu)$ and, similarly, for Hilbert space valued functions.

\paragraph{Tensor products.} In various parts of this paper we make use of the \textit{tensor product of two Hilbert spaces} $\cH_1$ and $\cH_2$. One way to define this tensor product is to first define an algebraic tensor product of the vector spaces $\cH_1$ and $\cH_2$; given that we are only working  with Hilbert spaces of functions it is natural to define the algebraic tensor product as
\[
\{f:\X \times \Y \to \R : f(x,y) = \sum_{i=1}^n g_i(x) h_i(y), g_i \in \cH_1, h_i \in \cH_2, n\in \mathbb{N}  \},
\]
where we assume that functions in $\cH_1$ map from $\X$ to $\R$ and functions in $\cH_2$ from $\Y$ to $\R$. That this is a tensor product for $\cH_1$ and $\cH_2$ can be verified by applying Criterion 2.3 in \cite{DEF92}. Next, we equip the algebraic tensor product with the inner product $ \langle g_1 \otimes h_1, g_2 \otimes h_2 \rangle_\otimes = \langle g_1, h_1 \rangle_1 \langle g_2, h_2 \rangle_2$, e.g. \cite[Thm.6.3.1]{MUR90}, and  complete the resulting pre-Hilbert space. In the case where $\cH_1$ and $\cH_2$ are RKHSs with kernels $k_1$ and $k_2$ we have bounded point evaluators for elements in the pre-Hilbert space, i.e. $\langle h_1 \otimes h_2, k_1(x,\cdot) \otimes k_2(y,\cdot) \rangle_\otimes = h_1(x) h_2(y)$ for all $x \in \X, y\in \Y, h_1 \in \cH_1$ and $h_2 \in \cH_2$. Due to \cite[second theorem on p.347]{ARON50} there is then a unique functional completion of the algebraic tensor product and we will use this completion when working with RKHSs.  
We do not use the algebraic tensor product itself and, in the following, will reserve the notation $(H_1 \otimes \cH_2, \langle \cdot, \cdot \rangle_\otimes)$ for the above defined tensor product of the two Hilbert spaces, that is $\cH_1 \otimes \cH_2$ is a Hilbert space with inner product $\langle \cdot,\cdot \rangle_\otimes$, and, whenever $\cH_1$ and $\cH_2$ are RKHSs, $\cH_1 \otimes \cH_2$ is a Hilbert space of functions. In fact, in the latter case $\cH_1 \otimes \cH_2$ is an RKHS with kernel $\tilde k((x_1,y_1),(x_2,y_2)) = k_1(x_1,y_1) k_2(x_2,y_2)$. See also \cite[Thm.5.11]{PAUL16}.

When $X,Y$ are independent random variables under the measure $\mu$ attaining values in $\X_1, \X_2$, $k_1,k_2$ are kernel functions on $\X_1$ and $\X_2$ respectively, $g\in \cH_1, h \in \cH_2$, and the Bochner integrals 
$\int k_1(X,\cdot) \, d\mu, \int k_2(Y,\cdot) \, d\mu, \int k_1(X,\cdot) \otimes k_2(Y,\cdot) \, d\mu$  
are well defined then 
\begin{align*}
	&\langle g\otimes h, \int k_1(X,\cdot) \otimes k_2(Y,\cdot) \, d\mu \rangle_\otimes =  \int g(X) h(Y) \, d\mu 
	= \int g(X) \, d\mu \int h(Y) \, d\mu  \\
	&= \langle g, \int k_1(X,\cdot) \, d\mu \rangle_1 \langle h, \int k_2(Y,\cdot) \, d\mu \rangle_2 = \langle g \otimes h, \int k_1(X,\cdot) \, d\mu \otimes \int k_2(Y,\cdot) \, d\mu \rangle_\otimes.
\end{align*}
Since this holds for all $g\otimes h$, $g\in \cH_1, h \in \cH_2$,
\begin{equation} \label{eq:tensor_ind}
	\int k_1(X,\cdot) \otimes k_2(Y,\cdot) \, d\mu =   \int k_1(X,\cdot) \, d\mu \otimes \int k_2(Y,\cdot) \, d\mu.
\end{equation}

There is another natural way to define a tensor product for two RKHSs $\cH_1$ and  $\cH_2$ that is often of use. Here, we identify the tensor product with a rank one operator mapping from $\cH_1$ to $\cH_2$. To distinguish it from the above definition we will use $g \widehat \otimes h$,  $g\in \cH_1, h\in \cH_2$, to denote this tensor product. Whenever $f,g \in \cH_1, h\in \cH_2$, the tensor product is defined by $(g \widehat \otimes h)(f) = \langle g,f \rangle_{\cH_1} h \in \cH_2$. Furthermore, we can define an inner product on this tensor space by letting $\langle f_1 \widehat \otimes f_2, h_1 \widehat \otimes h_2 \rangle_{\widehat \otimes} = \langle f_1, h_1 \rangle_{\cH_1} \langle f_2,h_2 \rangle_{\cH_2}, f_1,h_1 \in \cH_1, f_2,h_2 \in \cH_2$. Using Parseval's identity one can observe that  is just the usual inner product of the space $HS(\cH_1,\cH_2)$ of Hilbert-Schmidt operators and $\spn \{g\widehat \otimes h: g \in \cH_1,h \in \cH_2 \}$ lies dense in $HS(\cH_1,\cH_2)$. It is therefore natural to use $HS(\cH_1,\cH_2)$ as the completion of the algebraic tensor product defined in terms of rank one operators. We will therefore denote the inner product between such tensors by $\langle \cdot,\cdot \rangle_{HS}$.

\paragraph{Covariance operators.} A first application of this tensor product leads us to covariance operators. The covariance operator $\Cov : \cH \to \cH$, given by 
$\langle \Cov g, h \rangle = E(g(X)h(X))$, is linear ($\langle \Cov(\alpha f+g), h \rangle = \alpha E(f \times h) + E(g \times h)=
\langle \alpha \Cov(f) + \Cov(g), h \rangle$ for all $h\in \cH$ and, therefore, $\Cov(\alpha f+g) = \alpha \Cov(f) + \Cov(g)$ whenever $f,g \in \cH, \alpha \in \R$) and is bounded whenever $\cH$ can be continuously embedded in $L^2(\X,P)$, i.e. for some $c>0$, $\|h\|_2 \leq c \|h\|$ for all $h\in\cH$, since  then $ \| \Cov \|_{op} = \sup_{\|f\| = 1} \|\Cov f\| = \sup_{\|f\| = 1} \sup_{\|h\|=1} |\langle \Cov f,h \rangle| = \sup_{\|f\| = 1} \sup_{\|h\|=1} |E(f \times h)| \leq \sup_{\|f\| = 1} \sup_{\|h\|=1} \|f\|_2 \|h\|_2 \leq c^2$. 
In fact it is a Hilbert-Schmidt operator whenever $\cH$ is separable and $k(X,X) \in \mathcal{L}^2(\mu)$ because then for any orthonormal basis $\{e_n\}_{n\in \mathbb{N}}$ of $\cH$, 
$\sum_{n,m \in\mathbb{N}} |\langle \tilde{\mathfrak{C}}e_n,e_m \rangle |^2
\leq E( (\sum_{n\in\mathbb{N}} |\langle e_n, k(X,\cdot) \rangle|^2)^2  ) = E(k^2(X,X))$
due to Beppo Levi's theorem. In this case $\tilde{\mathfrak{C}}$ is also self-adjoint and the spectral theorem applies. Furthermore, we can write the covariance operator as a Bochner-integral of the tensors $k(x,\cdot) \widehat \otimes \, k(x,\cdot)$, i.e. $\tilde{\mathfrak{C}} = \int k(x,\cdot) \widehat \otimes \, k(x,\cdot) \, dP$. This Bochner integral is well defined and attains values in $HS(\cH)$ whenever $\int \|k(x,\cdot) \widehat \otimes \, k(x,\cdot) \|_{HS} \, dP = \int k(x,x) \, dP < \infty$ and $\cH$ is separable. Separability of $\cH$ is important in this context because it implies that $HS(\cH)$ is separable and Bochner measurability, that is necessary for the Bochner integral above to be well defined, is not a restrictive assumption \cite[App.B12]{DEF92}.   

Observe that there is close relationship between the eigen-decomposition of $\tilde{\mathfrak{C}}$ and the expansion of the integral operator $T_k: \mathcal{L}^2(P) \to \mathcal{L}^2(P), (T_k f)(y) =  \int f(x) k(x,y) \, dP(x)$. Whenever $\cH$ is infinite dimensional and Mercer's theorem applies there exists an orthonormal sequence $\{e_i^\bullet\}_{i \geq 1}$ in $L^2(P)$
and corresponding values $\{\tilde \lambda_i\}_{i \geq 1}$ in $\R$ such that $e_i$ are eigenfunctions of $T_k$ with eigenvalues $\tilde \lambda_i$ and $\{\tilde \lambda_i^{1/2} e_i\}_{i\geq 1}$ is an orthonormal basis for $\cH$. Furthermore, $\langle \tilde{\mathfrak{C}} e_i, e_j \rangle = E(e_i(X) e_j(X)) = \langle e_i, e_j \rangle_{\mathcal{L}^2(P)} = \delta_{ij}$ and $\tilde \lambda_1^{1/2} e_i, \tilde \lambda_2^{1/2} e_2, \ldots$ are the eigenvectors of $\tilde{\mathfrak{C}}$ with corresponding eigenvalues $\tilde \lambda_1, \tilde \lambda_2, \ldots$. Also notice that for all $y\in \X$,  $(T_k\bm 1)(y) = \int k(y,x) \, dP = \langle k(y,\cdot), \mean \rangle = \mean(y)$ whenever the Bochner integral $\int k(x,\cdot) dP$ is well defined. Since $T_k \bm 1$ and $\mean$ are real valued functions defined on $\X$ that are equal for all $y \in \X$ it follows that $T_k \bm 1 = \mean$.

The covariance operator as described above is giving us the second moments but not the covariance itself. The centered version $\tilde{\mathfrak{C}}_c = \tilde{\mathfrak{C}} - \mean \widehat{\otimes} \mean$ gives us the covariance itself, i.e. $E((f(X) - E(f(X)))(g(X) - E(g(X)))) = \langle \tilde{\mathfrak{C}}_c f,g \rangle$ for any $f,g \in \cH$. This operator is also self-adjoint under suitable conditions on the kernel and has a spectral decomposition.

\paragraph{Direct sum.} Another construction that we need is the \textit{direct sum of two Hilbert spaces} $\cH_1$ and $\cH_2$.
The direct sum $\cH_1 \oplus \cH_2$ is the Cartesian product $\{(g,h) : g \in \cH_1, h\in \cH_2\}$ equipped with the inner product $\langle (g_1,h_1),  (g_2,h_2) \rangle_{\oplus} = \langle g_1, g_2\rangle_1 + \langle h_1, h_2 \rangle_2$ \cite[p.40, Ex.5]{REED72}. We do not assume here that $\cH_1 \cap \cH_2 = \{0\}$.

\section{Approximating convex sets and locating $\mean$ and $\mean_n$}
We start this section with a discussion of a simple approach for approximating convex sets using $\varepsilon$-nets. We will find that such an approximation is of very limited use only which motivates the remainder of the paper. In this remainder  
we analyze a stochastic approach at length where we consider the random convex set which is induced by the sample. In detail, we control the difference between the empirical convex  set $C_n$ corresponding to the sample and its population limit $C$ using VC-theory and Rademacher complexities in Section \ref{sec:emp_convex_sets}. Such tools are not necessary for the finite dimensional setting but the question of convergence of the empirical convex set to its population limit can easily be developed for the infinite dimensional setting. In particular, the approach based on Rademacher complexities applies directly to infinite dimensional RKHSs. In Section \ref{sec:diameter} we study the width of the convex set $C$. We link here lower bounds on the width of $C$ to how well constant functions can be approximated within the unit ball of the RKHS. Building up on these sections we study how deep $\mean$ lies within $C$ in Section \ref{sec:refinement_location}. We also look in this section at an approach based on covariance operators which adapts automatically to the support of the unknown measure. Finally, in Section \ref{sec:bringing_it_together} we translate these findings to $\mean_n$ and we provide our main theorems in this section which give high probability bounds on the size of balls within the empirical convex set which are centered at $\mean_n$.   

\subsection{Approximation based on $\varepsilon$-nets.} \label{ref:approxEpsNet} Let $\cH$ be an RKHS of real-valued functions acting on $\X= [0,1]^d$ with kernel function $k$ being bounded by $1$. Furthermore, let $\phi:\X \rightarrow \cH$ be the map $\phi(x) = k(x,\cdot)$ and $\mean_n = \frac{1}{n} \sum_{i=1}^n \phi(x_i)$ for certain points $x_1,\ldots x_n \in X$. For $\varepsilon >0$ there exists an $\varepsilon$-net for $[0,1]^d$ that consists of $N_{\varepsilon,d} = \lceil d^{d/2}/\varepsilon^d\rceil$ many closed balls that are centered at points $y_1,\ldots, y_{N_{\varepsilon,d}}$ in $[0,1]^d$.   This $\varepsilon$-cover of $[0,1]^d$ gives rise to a $c \varepsilon^\alpha$-cover of $\phi[\X] = S$  if $\phi$ is $\alpha$-H\"older continuous with Lipschitz constant $c$. Let $s_i = \phi(x_i)$ for all $i\leq n$ and $s_i'$ the closest point to $s_i$ in $\phi[ \{y_1,\ldots,y_{N_{\varepsilon,d}}\}]$. Then the approximation $\mean_n' =   \frac{1}{n} \sum_{i=1}^n s_i'$ of $\mean_n$, which can be written as a sum over at most $N_{\varepsilon,d}$ many terms, achieves an approximation error of 
\[
	\|\mean_n - \mean_n'\| \leq \frac{1}{n} \sum_{i=1}^n \|s_i - s_i'\| \leq c \varepsilon^\alpha.
\]
If we want to achieve an approximation error of at most $n^{-1/2}$ then we need to include $\lceil d^{d/2} (c^2n)^{d/(2\alpha)} \rceil$ many balls in the cover. If $c\geq 1$ then we can only represent $\mean_n$ with less than $n$-points if $d=1$ and $\alpha > 1/2$. The Lipschitz constant $c$ is here only of limited help if we choose our kernel independent of $n$.

We can also observe that a fine cover is necessary for good approximation if we do not impose assumptions on the measure and on  $\mean$. For instance, consider again $\X=[0,1]^d$ and a kernel $k$ such that $k(x,x) = 1$ for all $x\in \X$ and such that $1 - k(x,y) \leq c\|x-y\|$ for some constant $c>0$ and any $x,y \in \X$. Furthermore, assume that we have a cover centered at  $l^d$ points $x_1,\ldots,x_{l^d}$ then there exists a point $x_0$ with $\min_{i\leq l} \|x_0 - x_i\| \geq 1/2l$. If we consider now the measure with unit mass on $x_0$, i.e. $\mean = k(x_0,\cdot)$, then the error, when approximating the expected value of the norm one function $h = k(x_0,\cdot)$, is 
\[
|\langle \mean, h \rangle - \langle \mean_n, h\rangle| = \|\mean\|^2 -
\sum_{i=1}^l \alpha_i k(x_i,x_0) \geq \frac{c}{2l}.
\]
Hence, to attain an approximation error of order $n^{-1/2}$ we need a cover consisting of at least $n^{d/2}$ many points.

\subsection{Empirical convex sets} \label{sec:emp_convex_sets}
In the following, let $\cH$ be a separable RKHS and 
let $C_n = \ch \{\phi(X_i) : i \leq n \}$ be the set valued random variable determined by $X_1,\ldots, X_n$. The variable $C_n$ 
attains values in the closed convex subsets $\mathbf{C}(\cH)$ of $\cH$. There exists various natural topologies on $\mathbf{C}(\cH)$ (see \cite{BEER93}). We equip $\mathbf{C}(\cH)$ with the Vietoris topology and the corresponding Borel-Effron $\sigma$-algebra. The random variable $C_n$ is then well defined as a measurable map from $\Omega$ to $\mathbf{C}(\cH)$. 
The random variable $C_n$ tends to $C = \cch\{ \phi(x) : x \in \X \}$ as $n$ tends to infinity. We aim to quantify how similar $C_n$ is to $C$. We do so by framing the question of convergence in the context of empirical process theory. In the following discussion we assume that $\X$ is \textit{compact}, $\cH$ is \textit{finite dimensional with dimension $d$}, and the corresponding \textit{kernel function  $k$ is continuous}. In particular, $\|k\|_\infty^{1/2} =: b$ is finite. 

Observe that we can reduce the question of convergence of $C_n$ to $C$ to the question of how fast the projection of $C_n$ on some direction $u \in \cH, \|u\| =1$, converges to the projection of $C$ on $u$. More specifically, if we can control the convergence uniformly over all such $u$ then we have control of the convergence of $C_n$ to $C$. Furthermore, since $C_n$ and $C$ are convex we only need to control the end points of the projections; these points correspond to projections of extremes of $C_n$ and $C$ onto $\spn\{u\}$. With this aim in mind, let us introduce the functions  $f_{u,c}(x)=\chi\{ u(x) \leq c \}$, $f_{u,c} : \X \rightarrow \R$, for $u\in \cH, \|u\| =1$, and with $c$ going through the interval
$\{\langle u, h \rangle  : h\in C \} = \ch \{ u(x) : x \in \X \}$ or a superset of this interval.

The importance of the functions $f_{u,c}$ is that $P f_{u,c} > 0$ if, and only if, there is an element $h\in C$ such that $\langle h,u \rangle \leq c$ (given that there is non-zero mass on that element or the mass of all elements whose projection falls below $c$ is strictly greater than zero). For instance, if $C$ contains the origin then we could vary negative $c$'s to explore the extension of the projection of $C$ in direction $u$. Since the extremes of $C$ are a subset of $S := \{ \phi(x) : x \in \X\}$ and the  probability measure is concentrated on $S$ it is sufficient to work with elements in $S$ instead of all of $C$. This setup is depicted in part (i) and (ii) of Figure \ref{figPpsi}. 

The situation is similar for the empirical convex set. The empirical convex set will contain an element which lies $c$ away from the origin in direction $u$ if, and only if, $P_n(u(\tilde X) \leq c) >0$, with $\tilde X$ being a random variable with law $P_n$ (see the preliminaries in Section \ref{sec:prelim}). Notice that the condition $P_n(u(\tilde X) \leq c) >0$ is equivalent to $\min_{i\leq n} u(X_i) \leq c$.

\paragraph{VC-theory.}
To be able make use of this approach to quantify the difference between $C$ and $C_n$ we need to control the convergence of $P_n f_{u,c}$ to $P f_{u,c}$ simultaneously over all these $f_{u,c}$. One simple way to do this is to use VC theory. Since we are working here with  finite dimensional RKHSs this is rather straight forward. In detail, whenever $u \in \cH$, $\|u\| = 1$, then $|u(x)| \leq b$ and we can use $[-b,b]$ as the interval over which we vary $c$. 
Hence, let
\[
	\mathcal{F} = \{f_{u,c} : u \in \cH, \|u\|=1, -b \leq c \leq b  \}.
\]
We want to show that $\mathcal{F}$ is a VC-subgraph class of functions. In fact, it is convenient to work with a countable dense subset of $\mathcal{F}$ to sidestep measure theoretic complications relating to the empirical process. 
To this end, let $\tilde{\cH}$ be a countable  dense subset of $\cH$ such that $\tilde{\cH} \cap \{u : \|u\|=1,u\in \cH\}$ lies dense in $\{u : \|u\|=1,u \in \cH\}$ and define the countable set $\tilde{\mathcal{F}} = \{f_{u,c} : u \in \tilde{\cH}, \|u\|=1, c \in (\mathbb{Q} \cap [-b,b]) \cup \{-b,b\}\} \subset \mathcal{F}$.  

The family $\mathcal{F}$ is a VC-subgraph class and its VC-dimension is upper bounded by $d+1$: consider the family of function $\mathcal{G} = \spn(\cH \cup \{ c\bm 1  : c \in \R \})$. The dimension of $\mathcal{G}$ is at most $d+1$ and $c - u(x) \in \mathcal{G}$ for every $u \in \cH$, $-b \leq c \leq b$. Applying \cite{DUDLEY14}, Theorem 4.6, shows that the VC dimension of $\text{Pos}(\mathcal{G}) = \{ \text{pos}(g) : g \in \mathcal{G}\}$, where $\text{pos}(g) = \{x: g(x) \geq 0 \}$, is at most $d+1$. Furthermore, the family $\mathcal{G}'$ of sets of the form $\{(x,t) : x \in \text{pos}(g), t \leq 1\}$, $g \in \mathcal{G}$, has the same VC-dimension. But $\mathcal{G}'$ is a family of subgraphs that contains all the subgraphs of functions in $\mathcal{F}$ and the claim follows. Since $\tilde{\mathcal{F}} \subset \mathcal{F}$ it also follows that $\tilde{\mathcal{F}}$ is a VC-subgraph class with VC-dimension at most $d+1$.  

The family $\tilde{\mathcal{F}}$ has the measurable envelope $\chi \X$ and, due to \cite[Thm3.6.9]{GINE15}, its covering numbers can be bounded by
\[
	N(\tilde{\mathcal{F}},\mathcal{L}^2(Q),\varepsilon) \leq 4 (8/\varepsilon^2)^{d+2} \vee \tilde c, 
\]
where $\tilde c$ can be chosen as $\max\{m\in \mathbb{N}_+ : \log m \geq m^{1/(d+1)(d+2)} \}$ and whenever $Q$ is a probability measure on $\X$. Be aware that the $\nu$-index as defined in \cite{GINE15}
is equal to one plus the VC-dimension when using the definition of \cite{DUDLEY14} for the VC-dimension.

Now, applying H\"older's inequality, 
\begin{align*}
J(\delta) = \int_0^\delta \sup_Q \sqrt{\log 2 N(\tilde{\mathcal{F}},\mathcal{L}^2(Q),\varepsilon)} \, d\varepsilon 
\leq \delta \left( 
	\log(2\tilde c) 
	\vee
	(1 + 2(d+2))
\right)^{1/2}. 
\end{align*}
In particular, $J(1) \leq  \sqrt{\log 2\tilde c} \vee \sqrt{1+ 2(d+2)}$. 
By Remark 3.5.5 and Theorem 3.5.4 from \cite{GINE15} we can conclude that
\[
	E(\sup_{f \in \tilde{\mathcal{F}}} |P_n f  - P f|) \leq 12 J(1) n^{-1/2}.
\]
We use now Bousquet's version of Talagrand's inequality to move to a high probability bound (e.g. \cite{GINE15}, Theorem 3.3.9). For simplicity, we will denote the supremum over $u,c$, such that $f_{u,c} \in \tilde{\mathcal{F}}$, by $\sup_{u,c}$ in the following. 
Let $S_n =  \sup_{u,c} |\sum_{i=1}^n (f_{u,c}(X_i)- Pf_{u,c} )|= n \sup_{u,c} | P_n f_{u,c} - P f_{u,c}|$. Observe that $\| Pf_{u,c} - f_{u,c}\|_{\infty} \leq 1$ and $E S_n = 
n E(\sup_{u,c} |P_n f_{u,c} - P f_{u,c}|)$. Applying Talagrand's inequality yields
\begin{align*}
	e^{-x} &\geq \Pr\left(\max_{j\leq n} S_j \geq E S_n + \sqrt{2x(2 E S_n + n)} +x/3 \right) 
\end{align*}
for all $x\geq 0$. In particular, with probability at least $1 - \exp(-x)$, 
\begin{equation}
	\sup_{u,c} | P_n f_{u,c} - P f_{u,c}| \leq 12 J(1) n^{-1/2} + n^{-1/2} \sqrt{2x(24 J(1) n^{-1/2}  + 1)} +x/3n .
\end{equation}

\paragraph{Rademacher complexities.} \label{sec:Rademacher_expl} As is usually the case with metric entropy based bounds, the constants are loose and $n$ needs to be large to gain useful results. Tighter bounds can often be attained by using \textit{Rademacher complexities} (see \cite{BART02,GINE15}). While the resulting bounds are generally tighter it is not possible to work directly with the indicator functions $f_{u,x}$ but we need a continuous approximation of these. Also, in the Rademacher approach that we develop it is beneficial to center the functions $h \in C$ by moving to $C_c = \{h - \mean : h\in C\}$. 

In the following, let $F = \{(u,c) : u\in \cH, \|u \|=1, -b \leq c \leq b \}$ and $\tilde F =  \{(u,c) : u\in \tilde{\cH}, \|u \|=1, c \in ([-b,b] \cap \mathbb{Q})\cup \{-b,b\}\}$. Furthermore,  consider the function $\psi_\gamma:\mathbb{R} \to \mathbb{R}$, with $\gamma >0$, defined by
\[
\psi_\gamma(x) = 
\begin{cases}
	1 & x \leq -\gamma, \\
	-x / \gamma & -\gamma < x  < 0, \\
	0 &  0 \leq x.
\end{cases}
\]
Then $f_{u,c}(h) = \chi\{\langle u, h\rangle  \leq c \} \geq \psi_\gamma( \langle u,h \rangle -c)$ for any $u,h\in \cH$ and $-b \leq c \leq b$. The function $\psi_\gamma$ is depicted in part (iii) of Figure \ref{figPpsi}. 
Importantly, $\psi(0) = 0$ and $|\psi_\gamma(x) - \psi_\gamma(y)| \leq |x-y|/\gamma$, that is $\gamma \psi_\gamma(\cdot)$ is a contraction vanishing at zero (see \cite[Thm3.2.1]{GINE15} or \cite[Thm4.12]{TAL13}). 

\begin{figure}[t]
\begin{tikzpicture}
\node at (6.5,0.6){\includegraphics[scale=0.8,trim=4cm 9.7cm 1cm 1cm,clip]{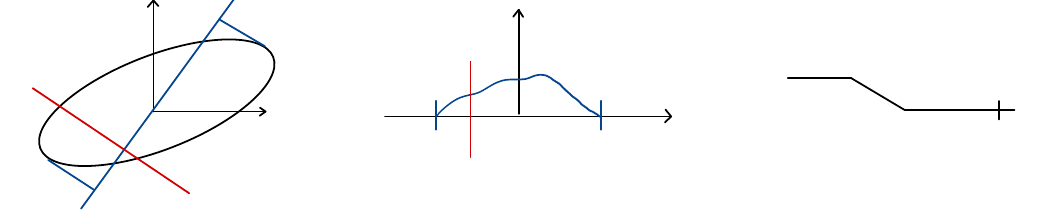}};
\node at (0,2.5) {\textit{(i)}};
\node at (4.5,2.5) {\textit{(ii)}};
\node at (9.5,2.5) {\textit{(iii)}};
\node at (0.5,1.3) {$C$};
\node at (3.2,2.3) {$\spn\{u\}$}; 
\node at (7.1,1.4) {$p(u(x))$}; 
\node at (8.5,0) {$u(x)$};
\node at (5.65,-0.3) {$c$};
\node at (9.8,1) {$1$}; 
\node at (13.2,0.55) {$0$};
\node at (11.5,0.2) {$c$};
\node at (11,1.3) {$c -\gamma$};
\node at (11.5,2) {Graph of $\psi_\gamma(u(x) -c)$};
\end{tikzpicture}
\caption{\label{figPpsi} \textit{(i)} The figure show $C$ as a subset of $\cH$. The diagonal (blue) line is the $\spn\{u\}$ for some function $u\in \cH, \|u\|=1$, The short lines connecting this line to the ellipse indicate the projection of $C$ on $\spn\{h\}$. In particular, the distance between the two short lines is $\diam_u(C)$. The long line which is orthogonal to $\spn\{u\}$ (red) indicates a threshold; the interest is here if $C$ extends past this threshold. \textit{(ii)} The question if $C$ extends past the threshold is rephrased in this figure by focusing on $\spn\{u\}$ and considering the probability that values $u(x)$ are attained that lie beyond the threshold. In this figure, we assume for simplicity that the measure on $C$ induces a density function $p(y)$ through the projection on $\spn\{u\}$, where $y$ goes over the range of $u$. The threshold is in this figure set to $-c$ and $C$ extends past the threshold if the density function is non-zero to the left of $-c$. \textit{(iii)} To link this construction to the empirical measure we use the function $\psi_\gamma$ whose graph is plotted in this figure against $u(x)$. The motivation is here to appromxiate the indicator function corresponding to the event $u(X) \leq - c$ from below by a continuous function. The parameter $\gamma$ controls the approximation and for $\gamma \to 0$  the function $\psi_\gamma$ converges to the indicator function. 
}

\end{figure}

We have that  $P_n f_{u,c + \langle u,\mean \rangle} \geq P_n \psi_\gamma(\langle u, \phi(\cdot) - \mean \rangle -c)$. The proof of \cite[Thm3.4.5]{GINE15}
gives us a high probability lower bound on the latter term (in the notation of the book, combine $S_n < E S_n + \sqrt{2x/n}$ with $ES_n \leq 2 E\tilde S_n$). In detail, with probability $1-p$, simultaneously for all $u \in \tilde \cH, \|u\| = 1$, and $c \in ([-b,b] \cap \mathbb{Q}) \cup \{-b,b\}$, we have have the following lower bound on
$P_n \psi_\gamma(\langle u, \phi(\cdot) - \mean \rangle -c)$,
\begin{align*}
&P \psi_\gamma(\langle u, \phi(\cdot) - \mean \rangle -c) - 2 E\bigl(  \sup_{(u',c') \in \tilde{F}}  | \frac{1}{n} \sum_{i=1}^n \epsilon_i \psi_\gamma(\langle u', \phi(X_i) - \mean \rangle -c') |\bigr)  - \sqrt{\frac{2 \log(2/p)}{n}}, 
\end{align*}
where $\epsilon_i$ are i.i.d. Rademacher variables that are independent of $X_1,\ldots, X_n$.  Because $\gamma \psi_\gamma$ is a contraction vanishing at zero 
\[
	E\bigl(\sup_{(u',c') \in \tilde F}  \bigl| \frac{1}{n} \sum_{i=1}^n \epsilon_i  \psi_\gamma(\langle u', \phi(X_i) - \mean \rangle -c') \bigr|\bigr) \leq \frac{2}{\gamma} E\bigl(  \sup_{(u',c') \in \tilde{F}}  \bigl| \frac{1}{n} \sum_{i=1}^n \epsilon_i (\langle u', \phi(X_i) - \mean \rangle -c') \bigr|\bigr).
\]
Applying \cite[Thm12 (7) and Lem22]{BART02} and using that the Rademacher complexity for the constant functions $\langle u', \mean \rangle + c', (u',c') \in \tilde F$, where $|\langle u', \mean\rangle + c'| \leq b + |c'| \leq 2 b$,   
is upper bounded by $4b n^{-1/2}$, 
\begin{align*}
&E\bigl(  \sup_{(u',c') \in \tilde F}  \bigl| \frac{1}{n} \sum_{i=1}^n \epsilon_i (\langle u', X_i \rangle -c') \bigr|\bigr) \leq 4b n^{-1/2} + E\bigl((2/n)\bigl( \sum_{i=1}^n k(X_i, X_i)\bigr)^{1/2}  \bigr) \leq 6 bn^{-1/2}
\end{align*}
and simultaneously for all $(u,c) \in \tilde F$ with probability $1-p$,
\[
	P_n \psi_\gamma(\langle u, \phi(\cdot) - \mean \rangle -c) \geq P \psi_\gamma(\langle u, \phi(\cdot) - \mean \rangle -c) - (
	\sqrt{2\log(2/p)} + 24b/\gamma) n^{-1/2}.
\]

The VC and Rademacher bounds allow us to control the size of the empirical convex set in terms of $P f_{u,c}$ and  $P \psi_\gamma(\langle u,\phi(\cdot) - \mean \rangle -c)$. In either case we need to get a handle on $P$ to move further. In particular, we need to understand how $P$ concentrates around the extremes of $C$. We are now looking at a few examples to get a better understanding of how $P$ concentrates and what this implies for the convergence of the empirical convex set to $C$. Of major  importance is how smooth $\phi: \X \to \cH$ is and how the distribution of $X_1,\ldots, X_n$ on $\X$ looks like. We start with a couple of simple examples and discuss links to stochastic geometry before addressing typical settings that one faces in practice.


\subsubsection{Example 1: Unit circle}
Consider the unit circle in $\mathbb{R}^2$ with the uniform distribution on it. What can we say about the interior of $C_n$ as a function of $n$? In particular, what can be said about the size of $C_n$ in direction $u \in \mathbb{R}^2$, $\|u\|=1$? Due to the symmetry of the unit sphere and because the uniform distribution is used  it is sufficient to consider the vector $u = (1 , 0)^\top$. The probability that a sample point, when we sample just once, lies to the right of $c u$, with $c \in [-1,1]$, is $(1/2\pi) \int_0^{2\pi} \chi\{ \langle u, (\cos \theta, \sin \theta )^\top \rangle \geq c\} = \arccos(c)/\pi$, using here that $\arccos$ is a monotonically decreasing function. Similarly, the probability that a sample point lies to the left of $c u$ is $1-\arccos(c)/\pi$. Furthermore, if we draw $n$ independent samples then the probability to see at least one sample point to the right of $c u$ is $1 - (1-\arccos(c)/\pi)^n$ and that at least one sample point lies to the left of $c u$ is $1 - ( \arccos(c)/\pi)^n$. Moving on to the distribution of the length of the interval, which corresponds to the projection of $C_n$ onto $u$, that is the distribution of $\diam(u,C_n)$, we can observe that $\diam(u,C_n)$ attains values in $[0,2]$ and that    
$\diam(u,C_n) = \max_i \cos \theta_i - \min_i  \cos \theta_i$, where we denote with $\theta_i$ independent and uniformly distributed random variables on $[0,2\pi)$. We could now try to calculate the distribution of $\diam(u,C_n)$ by controlling the maximum and minimum. Since we are interested in getting a better understanding of the VC and Rademacher approach we  use instead the uniform  guarantees on $P_n f_{u,c}$. Let $\X$ be the unit circle and let the kernel function be $k(x,y) = \langle x, y \rangle_{\R^2}$. This way $\cH$ becomes the dual space $(\R^2)'$ of $\R^2$: 
 recall that a basis of $(\R^2)'$ is given by $\langle e_1, \cdot \rangle_{\R^2}, \langle e_2, \cdot \rangle_{\R^2}$, where $e_1,e_2$ is the standard basis in $\R^2$, and for any $i,j\in\{1,2\}$, $\langle \langle e_i, \cdot \rangle_{\R^2}, \langle e_j, \cdot \rangle_{\R^2} \rangle_{(\R^2)'} = \langle e_i,e_j \rangle_{\R^2}$. Since $e_1,e_2$ lie in $\X$ it holds for any $i,j \in \{1,2\}$, $\langle \langle e_i, \cdot \rangle_{\R^2}, \langle e_j, \cdot \rangle_{\R^2} \rangle_{(\R^2)'} = \langle k(e_i,\cdot), k(e_j,\cdot) \rangle$ and the claim follows. 


Associate to $u \in \mathbb{R}^2$ the function $\tilde u \in \cH$ given by $\tilde u(x) = \langle u, x \rangle_{\R^2}$. Let $\tilde \cH$ be a countably dense subset of $\cH$ such that $\{u : u\in \tilde \cH, \|u\|=1\}$ lies dense in the unit sphere of $\cH$. Define 
$\mathcal{F} = \{f_{\tilde u,c} : u \in \cH, \|u\|=1, -1 \leq c \leq 1 \}$
and $\tilde{\mathcal{F}} = \{f_{\tilde u,c} : u \in \tilde{\cH}, \|u\|=1, -1 \leq c \leq 1 \}$, where $f_{\tilde u,c}(x) = \chi\{ \tilde u(x) \leq c\} = \chi\{\langle u, x\rangle \leq c\}$. The family of functions $\tilde{\mathcal{F}}$  is a VC-subgraph class and 
on an event of probability at least $p$ it holds simultaneously for all $f_{\tilde u,c} \in \tilde{\mathcal{F}}$ that
\begin{align*}
P_n f_{\tilde u,c} \geq P f_{\tilde  u,c} - n^{-1/2} \xi_n = 1- \arccos(c)/\pi - n^{-1/2} \xi_n,  
\end{align*}
where $\xi_n = 12 J(1) + \sqrt{2 \log(1/p) (24 J(1) n^{-1/2} + 1)} + \log(1/p)n^{-1/2}/3$. We use here that
the uniform distribution on the unit circle is invariant under rotations, i.e. for a given $u$ let $A$ be the rotation matrix for which $Au = (1, 0)^\top$. Then, 
\[
	P f_{\tilde  u,c} 
	=  \frac{1}{2\pi} \int_0^{2\pi} \chi\{\langle A u, A (\cos(\theta),\sin(\theta))^\top \rangle \leq c \} 
	=  \frac{1}{\pi} \int_0^{\pi} \chi\{\cos(\theta) \leq c \}.
\]
In other words, 
on an event of probability $p$, whenever $n$ is such that  $n^{-1/2} \xi_n <1/2$,  
and for any $c > \cos((1-n^{-1/2}\xi_n) \pi)$, there will be a sample point which has an inner product with $u$ which is smaller than $c$. 
To be exact, let $c_0 < 0$ be a real number strictly larger than  $\cos((1-n^{-1/2}\xi_n) \pi)$ and let the above event be denoted by $B$. It holds that $P(B) \geq p$ and for any $\omega \in B$,    $\min_{i\leq n} \langle u, X_i(\omega) \rangle \leq c_0$. In fact,  the VC-argument shows that $B$ can be chosen such that $P(B)\geq p$ and for all $\omega \in B$,
\[
\adjustlimits	\sup_{u \in \cH, \|u\|=1} \min_{i\leq n} \langle u, X_i(\omega) \rangle \leq c_0.
\]
From this we can infer that a ball centered at the origin and of radius $c_0$ is contained in the empirical convex set $C_n(\omega)$, whenever $\omega \in B$: consider without loss of generality the vector $v= (c_0, 0)^\top$ and an element $\omega \in B$. There exist elements $X_j$ such that $X_j(\omega)$ lies on the unit circle and $\langle - v, X_j(\omega) \rangle \leq c_0^2$, that is $\langle v, X_j(\omega) \rangle \geq \|v\|^2$.
Let $X_i(\omega)$ be such an element which also attains the maximum of the map $j \mapsto \langle v,X_j(\omega) \rangle$. 

Assume that $X_i(\omega)$ does not lie in $\spn v$ and that $X_i(\omega)$ lies north of $\spn v$, i.e. $\langle X_i(\omega), (0,1)^\top \rangle > 0$.
Consider the lines between $X_i(\omega)$ and the elements $X_j(\omega), j \leq n, j\not = i$. There will be an index $j_0 \leq n, j_0 \not = i,$ such that the line between $X_i(\omega)$ and $X_{j_0}(\omega)$ intersects with $\spn v$. Consider the vector $w=(0,-c_0)^\top$. There will be a sample point $X_{j_1}(\omega)$ such that $\langle w,X_{j_1}(\omega) \rangle \geq \|w\|^2$ and the line between $X_i(\omega)$ and $X_{j_1}(\omega)$ crosses $\spn v$. 
Order the samples according to how large the inner product between the point of intersection of the line between the sample and $X_i(\omega)$ and $v$ is. Let  $X_{j_2}(\omega)$ be the maximum in this ordering. Assume that $\langle X_{j_2}(\omega), v \rangle < \|v\|^2$, that is the intersection lies to the left of $v$. Let $\tilde v$ be the point on the circle with radius $c_0$ for which the line between $X_i(\omega)$ and $X_{j_2}(\omega)$ is tangent and which lies to the right of the line. There is now a point $X_{j_3}(\omega)$ on the sphere such that $\langle X_{j_3}(\omega), \tilde v \rangle \geq \|\tilde v\|^2$. The point $X_{j_3}(\omega)$ cannot lie north of $v$ since this would contradict the maximality of $X_i(\omega)$. However, if $X_{j_3}(\omega)$ lies south of $v$ then the line between $X_{j_3}(\omega)$ and $X_i(\omega)$ crosses $\spn v$ further to the right than the line between $X_{j_2}(\omega)$ and $X_i(\omega)$ which contradicts the maximality of $X_{j_2}(\omega)$. 

Hence, we have either two points to the right of $v$, one on the north side and one  on the south side of the sphere, or the point $(1,0)^\top$ is contained in the sample. By the same argument, either $(-1,0)^\top$ is contained in the sample or there are two points left of $-v$, one on the north side and one on the south side. The convex hull of these points is a subset of $C_n(\omega)$ and contains $v$.         

To provide a concrete example, let $p=0.9$ and observe that the $\tilde c$ which appears in the bound of $J(\delta)$ can be chosen as $10^{21}$. Then 
$J(1) \leq 8 \vee 3 = 8$ and $\xi_n \leq 96 + \sqrt{2\log(10)(192 n^{-1/2} +1)} + n^{-1/2}\log(10)/3$. Hence, a ball of radius $0.2$ exists around the origin inside the empirical convex set with probability $p$ for $n$ being about $52000$ or larger.

As expected $n$ needs to be large to guarantee the existence of the ball or radius $0.2$. Using \textit{Rademacher complexities} we can attain significantly tighter bounds in this setting. Building up on our discussion and using $\mean=0$ we can see that
\[
	P_n \psi_\gamma(\langle u, \cdot \rangle -c) \geq P \psi_\gamma(\langle u, \cdot \rangle -c) - (
	\sqrt{2\log(2/p)} + 12/\gamma) n^{-1/2}.
\]
Finally, by using a rotation of $u$ and with $c_\gamma = (c- \gamma) \vee -1$ it follows that
\begin{align*}
	P \psi_\gamma(\langle u, \cdot \rangle -c) 
	&= \frac{1}{\pi} \int_0^\pi \psi_\gamma(\cos(\theta) - c) = 1- \frac{\arccos(c_\gamma)}{\pi} + \frac{1}{\pi\gamma} \int_{\arccos(c)}^{\arccos(c_\gamma)}  (c-\cos(\theta))   \\
	&= 1 - \frac{\arccos(c_\gamma)}{\pi} \, (1- c/\gamma) -
	\frac{c \arccos(c)}{\pi \gamma}
	+\frac{1}{\pi\gamma} (\sqrt{1-c^2} - \sqrt{1-c_\gamma^2}).
\end{align*}
For instance, with $\gamma=1$ and $c<0$ this leads to
\begin{align*}
	P_n \psi_\gamma(\langle u, \cdot \rangle -c) \geq 
	c (1-  \arccos(c)/\pi) + \frac{\sqrt{1-c^2}}{\pi}
-  (\sqrt{2\log(2/p)} + 12) n^{-1/2}.
\end{align*}
The bound guarantees in this case the existence of a ball of radius $0.2$ around the origin within the empirical convex set with probability at least $0.9$ when $n$ is about $5000$, a 10-fold improvement in the constant over  the VC-bound. While the bound is significantly better it does not come close to capture the right magnitude: even a number as small as $n=10$ suffices in experiments  
for the empirical convex set to contain a ball of radius $0.2$ with high probability.

\subsubsection{Example 2: Polytopes with finitely many extremes}
Let us consider next a simple polytope. Let $\cH = \mathbb{R}^d$ with the usual inner product and $k(x,y) = x^\top y$. Furthermore, consider $C = \ch\{x_i : i \leq m \} $ with $x_1,\ldots, x_m \in \mathbb{R}^d$ and such that the random variable $X$ attains values in $\{x_1, \ldots, x_m\}$ and $\Pr(X = x_i) \geq \alpha > 0$ for all $i\leq m$. Then for all $u\in \cH, \|u\| =1, |c| \leq \|k\|_\infty^{1/2}$ either $P f_{u,c} =0$ or $P f_{u,c} 
\geq \alpha$. Hence, we have that $P_n f_{u,c} > 0$  with probability at least $1-e^{-\sqrt{n}}$ whenever 
\[
	n \geq (12 J(1)  +   \sqrt{2(24 J(1)   + 1)} +1/3)^2/\alpha^2.
\]
In other words, for $n$ that large the empirical convex set equals $C$ on an event of probability at least $1-e^{-\sqrt{n}}$.

If each of the $x_i$ is an extreme then we can compare this probability to the probability that in $n$ independent trials all $m$ extremes are drawn: the probability that  element $i$ is not drawn in $n$ independent trials is $1-\Pr(X=x_i)^n$ and the probability that at least one element $i$ is not drawn is upper bounded by 
\[
	\Pr(\bigcup_{i \leq m} \bigcap_{j\leq n} \{X_j \not = x_i\} ) \leq \sum_{i \leq m} (1-\Pr(X=x_i))^n \leq m (1-\alpha)^n =
	m \exp(- \beta n).
\]
where $\beta = -\log(1-\alpha)$. In other words, instead of $1-e^{-\sqrt{n}}$ we get a probability of 
$1-m \exp(-\beta n)$ that the empirical convex set matches the convex set $C$.

Consider now the special case of the \textit{$d$-dimensional simplex} $\ch S$, with $S = \{0,e_1,\ldots,e_d\}$ and $e_1,\ldots, e_d$ being an orthonormal basis in $\mathbb{R}^d$. Furthermore, assume that each $x \in S$ has probability $1/(d +1)$ to be sampled. The interior of the empirical convex set $C_n$ is empty  unless all points have been sampled. Hence, in this example either $\intr C_n = \emptyset$ or $C_n = C$ and the interior of $n$ does not grow slowly in size as $n$ increases but changes abruptly.    

As a final example consider a \textit{rhombus} given by $C = \ch \{e_1,-e_1, r e_2, -r e_2\}$ where $e_1,e_2$ are orthonormal vectors in $\mathbb{R}^2$ and $r \in (0,1)$. Furthermore, let $X$ be uniformly distributed on the boundary of $C$. We can again consider the functions $f_{\tilde u,c}$ to measure the interior of the empirical convex set. However, in contrast to the unit circle the measure in direction $u$ that lies $c$ apart from the origin is not the same for all $u$ but depends strongly on the direction. For instance, for direction $  -e_1$ and $c \in (0,1)$ it holds that $P f_{-\tilde{e}_1,-c} = 2 p (1-c)\sqrt{1+r^2}$, where $p$ denotes here the density of the uniform distribution on the boundary, while for $-e_2$ we get $Pf_{-\tilde{e}_2,-c'} = 2p(1-c'/r) \sqrt{1+r^2}$, for $c' \in (0,r)$. In particular, for $c = 0.9, c'=0.9 r$ the probabilities  $P f_{-\tilde{e}_1,-c}$ and   $Pf_{-\tilde{e}_2,-c'}$ are equal and the probability   $P f_{-\tilde{e}_1,-c}$, which is spread out over an interval of length $0.1$ in direction $e_1$, is contained in an interval of length $0.1 r$ in direction $e_2$ irrespective of how small $r$ is.

\subsubsection{Example 3: Image of a Lipschitz-continuous kernel function} \label{sec:Lower_bnd_P_Lipschitz}
Let us go back to the setting that we discussed at the beginning of the section. In detail, let $k$ be a \textit{continuous kernel function on compact set $\X$} that is upper bounded by $b$. Furthermore, let us assume that  the corresponding feature map $\phi(x) = k(x,\cdot)$ is  \textit{$L$-Lipschitz continuous} with Lipschitz constant $L>0$ and the law of $X_1,\ldots,X_n$  has \textit{a density on $\X$ which is lower bounded by $b'> 0$}. We are now aiming to quantify the extension of the convex set in a direction $u$ after centering the convex set around $\mean$. 
In detail, for  $u\in \cH, \|u\|=1$, let $x_u \in \X$ be a point at which $c_u^* := \langle u,\phi(x_u) - \mean \rangle  = \min_{x\in \X} \langle u,\phi(x) -\mean \rangle$. As before, for $c\in \mathbb{R}$, let $f_{u,c} = \chi \{\langle u, \phi(\cdot) - \mean \rangle \leq c \}$ and observe that 
$|\langle u, \phi(x_u) - \phi(x)\rangle| \leq L \|x_u - x\|$.   
Therefore, with $r_u = (c - c_u^*)/L$ and
whenever $c > c_u^*$, 
\begin{align*}
	P f_{u,c} &\geq \int_\X b' \times \chi\{\langle u, \phi(x) - \mean \rangle \leq c \} dx \geq \int_{B(x_u,r_u)\cap \X} b' = b'\, \vol(B(x_u,r_u)\cap \X).
\end{align*}
For example, when $\X = [0,1]$, $u$ any element in $\cH$ with $\|u\|=1$ and $c > c_u^*$ such that $x_u - (c-c_u^*)/L \geq 0$, it follows that $\vol(B(x_u,r_u) \cap \X) \geq r_u$ and $P f_{u,c} \geq b'r_u = b' (c- c_u^*)/L$. 

This lower bound on $P f_{u,c}$ can directly be combined with a metric entropy bound. If we want to use instead a \textit{Rademacher complexity bound} then we have to apply $\psi_\gamma$ to $\langle u, \phi(\cdot) - \mean\rangle -c$. 
Under the above Lipschitz assumption for any $u, \|u\|=1$, and whenever $c_u^* \leq c - \gamma$,
\begin{align*}
	&1 = \psi_\gamma(\langle u, \phi(x) - \mean \rangle -c ) 
	\enspace \Leftarrow \enspace c-\gamma \geq  \langle u, \phi(x) - \mean \rangle  \\
	&\Leftarrow \enspace c - \gamma - c_u^*  \geq \langle u, \phi(x) - \phi(x_u) \rangle
	\enspace \Leftarrow \enspace  c - \gamma - c_u^*  \geq L \|x - x_u\|.
\end{align*}
Also, $\psi_\gamma(\langle u, \phi(x) - \mean \rangle -c)$ is strictly positive whenever $L \|x - x_u\| \leq c - c_u^*$.

Let $r_{u,1} = (c - \gamma - c_u^*)^+/L$. For $x \in B(x_u,r_{u,1})$ we have that $\psi_\gamma(\langle u, \phi(x) - \mean \rangle -c) =1$ which gives us right away the following lower bound 
\begin{equation}
	P\psi_\gamma(\langle u, \phi(\cdot) - \mean \rangle - c) \geq b' \vol(\X \cap B(x_u,r_{u,1})).  
\end{equation}
This bound can now be combined with the Rademacher complexity bounds. However, to say anything concrete about the size  of the empirical convex set some knowledge of $c_u^*$ is required. In the Section \ref{sec:diameter} we derive approaches to measure the width of $C$ in any direction, then we derive lower bounds on $c_u^*$. We combine these bounds in Section \ref{sec:bringing_it_together} with the above bound.

\subsubsection{Example 4: Data attaining values in a subset} \label{sec:example_subset}
Working with empirical convex sets has the advantage that the empirical convex set is adapted to the support of the distribution, if $P$ has support $S$ on $\X$ then $C_n$ converges to $\cch \{k(x,\cdot) : x \in S \}$ (see \cite[Def.411N]{FREM} for the defintion of support). Instead of showing this under a density and Lipschitz assumption we are using an assumption on the covariance operator.  

A simple way to deal with $S$ is to consider the space $\cH_S = \{ h\!\upharpoonright\!S : h \in \cH \}$ which is again an RKHS with kernel $k_S = k\!\upharpoonright\! S \times S$. 
We discuss $\cH_S$ and how covariance operators are naturally adapted to $S$ at length in Section \ref{sec:data_subset_covariance_arg}. For the moment it is sufficient to note that for all $h\in \cH$, 
\[
\langle \tilde{\mathfrak{C}}_c h,h \rangle = \langle \tilde{\mathfrak{C}}^S_c h\!\upharpoonright\! S, h\!\upharpoonright\! S\rangle_{\cH_S}, 
\]
where $\tilde{\mathfrak{C}}_c$ is the centered covariance operator and $\tilde{\mathfrak{C}}_c^S$ the corresponding operator for the RKHS $\cH_S$. In Section \ref{sec:data_subset_covariance_arg} we also show that $\cH_S$ is naturally linked to the affine subspace spanned by $k(x,\cdot), x\in S$ and that statements about the behavior of $C_n$ can be derived by analyzing $\cH_S$. In particular, the eigenfunctions of $\tilde{\mathfrak{C}}_c$ which have eigenvalue zero are almost surely constant on $S$ and are all mapped to the same one dimensional subspace of $\cH_S$. Important for the analysis later on are the eigenfunctions which have non-zero eigenvalues and are therefore not constant on $S$. Let $u \in \cH_S$ be an eigenfunction of $\tilde{\mathfrak{C}}_c^S$ with eigenvalue $\bar \lambda$. Let $\phi_S$ be the feature map corresponding to $k_S$ and $\mean_S$ the corresponding mean embedding.
The function $u$ has by definition norm one and for $\gamma > 0, c \in \mathbb{R}$, and $X$ a random variable with law $P$, 
\begin{align*}
P\psi_\gamma(\langle u,\phi_S(\cdot) - \mean_S \rangle_{\cH_S} -c )
&
\geq \Pr( - (u(X) -  E(u(X))) \geq - \gamma - c). 
\end{align*}
In the following, let $Z = - (u(X) - E(u(X)))$ and write 
$Z = Z^+ - Z^-$ where $Z^+ = Z \times \chi\{Z \geq 0\}, Z^-= Z \times \chi\{Z \leq 0\}$. Since $Z$ has  mean zero we have that $E(Z^+) = E(Z^-)$.
Whenever $\|k\|_\infty < \infty$ we also have that $E((Z^+)^2) \leq \|k\|_\infty^{1/2} E(Z^+)$ and $E(Z^+) = E(Z^-) \geq E((Z^-)^2) / \|k\|_\infty^{1/2}$. Furthermore, 
\[
\bar \lambda = E(Z^2)  = E((Z^+)^2) + E((Z^-)^2) \leq E((Z^+)^2) + \|k\|_\infty^{1/2} E(Z^+) \leq 2 \|k\|_\infty^{1/2} E(Z^+).
\]
Consider now $\gamma,c$ such that $ 0 < -\gamma -c \leq  \bar \lambda /2 \|k\|_\infty^{1/2} 
\leq E(Z^+)$ 
then the Paley-Zygmund inequality yields
\begin{align*}
	\Pr(Z \geq -\gamma - c) &= \Pr(Z^+ \geq -\gamma -c) \geq \frac{(E(Z^+) - (-\gamma -c))^2}{E((Z^+)^2)} \\
			       &\geq \frac{(\bar \lambda/2 \|k\|_\infty^{1/2} - (-\gamma -c))^2}{\|k\|_\infty} = (\bar \lambda/2 \|k\|_\infty  - (-\gamma -c)/\|k\|_\infty^{1/2})^2.
\end{align*}
In particular, when $-\gamma -c =  \bar \lambda /8\|k\|_\infty^{1/2}$,
\begin{equation} \label{eq:lbnd_P_cov}
P\psi_\gamma(\langle u,\phi_S(\cdot) - \mean_S \rangle_{\cH_S} -c )
\geq \bar \lambda^2/8\|k\|_\infty.
\end{equation}

\subsection{Width of the convex set $C$} \label{sec:diameter}
The width of a convex set plays an important role when trying to control the convergence behavior of various convex approximation algorithms. By the width of the convex set $C = \cch\{\phi(x) : x \in \X\}$, where $\X$ is as usual a measurable space and $\phi$ is a feature map,  we mean the size of the projection of $C$ on a function of norm one within the RKHS corresponding to $\phi$,
\[
\diamh(C) := \sup_{x\in \X} \langle h, \phi(x) \rangle  - \infd_{x\in \X} \langle h, \phi(x) \rangle =
\sup_{x\in \X} h(x) - \infd_{x\in \X} h(x),
\]
where $h \in \cH$, $\|h\| =1 $.

There is a simple relationship between the width of the convex $C$ in direction $h$ and how close $h$ is to a constant function. 
In the following, let $\bm 1$ denote the function that is equal to one for all $x\in \X$ and let $\| f\|_\infty =
\sup_{x\in \X} |f(x)|$ for any function $f: \X \to \R$, allowing for $\|f\|_\infty = \infty$. For any $h\in \cH$, $\|h \| =1$,   
\begin{equation} \label{eq:constant_vs_diameter}
	\diamh(C) = 2 \inf_{c\in \mathbb{R}} \|h - c \bm 1\|_\infty,
\end{equation}
In particular, $h$ is a constant function  if, and only if,  $\diamh(C) = 0$. 

Small widths of $C$ in any direction $h$ are a concern when trying to approximate $\mean$ because various performance bounds of algorithms discussed in later sections depend on a  lower bound on the width; the higher this lower bound the faster the convergence. To be precise, the set $C$ can lie in an affine subspace that is not all of $\cH$ and the algorithms we study depend only on the affine subspace. Denote the closure of the affine span of $C$ by $\caff C$. In other words,  $\caff C$ is  the closure of $\{\alpha_1 h_1 + \ldots + \alpha_n h_n  : n\in \mathbb{N}, h_i \in C, \alpha_i \in \R \text{ for all } i \leq n \}$ 
which is a closed affine subspace. Furthermore, let $U_C = \caff C - f$, where $f$ is any element of $C$, then $U_C$ is a closed subspace of $\cH$. Observe that the dimension of $U_C^\perp$ is at most one since for $h \in U_C^\perp$ it holds that $h(x) = \langle h,\phi(x) \rangle = \langle h, \phi(y) \rangle = h(y)$ for all $x,y \in \X$, and, hence, only constant functions can lie in $U_C^\perp$.

The key quantity which influences the behavior of the algorithms is now 
\[
\inf_{h \in U_C, \|h\| =1} \diamh(C).
\]
If $\bm 1$ lies in the RKHS then $\bm 1$, and all constant functions, lie in $U_C^\perp$ and we do not have to worry about them. The important question is now, how closely can an $h \in U_C, \|h\| =1$, approximate a constant function.


Before leveraging Equation \eqref{eq:constant_vs_diameter} for controlling the width of $C$  we recall some topological properties. If $\X$ is compact and $\phi$ is continuous then
$\phi[\X]$ is compact \cite[Thm. 3.1.10]{ENG89}. Due to Mazur's Theorem  $C = \cch \phi[\X]$ is then also compact \cite[Thm. 12, p.51]{DIES77}. This implies, in particular, that there exists no norm ball inside $\phi[\X], \ch \phi[\X]$ or $\cch \phi[\X]$ whenever $\cH$ is infinite dimensional because a closed norm ball inside the compact set $\cch \phi[\X]$ would be compact \cite[Thm. 3.1.2]{ENG89}. However, closed norm balls in infinite dimensional Hilbert spaces are not compact \cite[S. I.2.7]{Wer02}. Similarly, there exist no norm ball $B$ such that $B \cap \caff C$ lies inside $C$.

Furthermore, whenever $\cH$ is infinite dimensional, $C$ is compact, $(e_n)_{n\geq 1}$ is an orthonormal sequence in $\cH$ and $\epsilon >0$, it holds that for only finite many of the $e_n$ the width $\diame_{e_n}(C)$ can be greater than  $\epsilon$. Assume otherwise and let $I:\mathbb{N} \rightarrow \mathbb{N}$ be an enumeration of all the elements $e_n$ for which the width is greater than $\epsilon$. Furthermore, assume w.l.o.g. that $C$ is centered in the sense that 
for all $n \in \mathbb{N}$, $\sup_{u \in C} \langle u,e_{I(n)}\rangle + \infd_{u\in C} \langle u, e_{I(n)} \rangle = 0$. Since $C$ is compact $\sup_{u \in C} \langle u,e_{I(1)}\rangle$ is attained at some point $u_1 \in C$. Inductively, we can select a countably infinite sequence of points $(u_n)_{n\geq 1}$ in $C$ such that $\|u_n - u_m\| \geq \epsilon/4 > 0$ whenever $n\not = m$: given points $u_1,\ldots, u_n$ there exists $m\in \mathbb{N}$ such that $\max_{i\leq n} |\langle u_i, e_{I(m')} \rangle | \leq \epsilon/4$ for all $m' \geq m$. Let $u_{n+1}$ be a point in $C$ such that  $\epsilon/2 \leq \sup_{u \in C} \langle u,e_{I(m')}\rangle = \langle u_{n+1}, e_{I(m')} \rangle$. Then $\|u_{n+1} - u_i\| \geq \epsilon/4$ for all $i\leq n$. Hence, we have countably infinitely many points with distance at least $\epsilon/4$ between them. These points give rise to an open cover of $C$ that does not contain a finite sub-cover, contradicting the compactness of $C$.  

This last statement implies that whenever $\cH$ is infinite dimensional, $C$ is compact and $\|k\|_\infty < \infty$ then for any $\epsilon>0$ there are infinitely many orthonormal elements $h_1,h_2, \ldots$ in $\cH$ such that for each $i$, $\sup_{x\in \X} h_i(x) - \inf_{x\in \X} h_i(x) \leq \epsilon$. Furthermore, at most one of the $h_i$'s can be constant, because if $h_i$ and $h_j$, $i\not=j$, were both constants then they clearly would not be orthogonal. 

\subsubsection{Interpolation spaces} \label{sec:interpol}
Interpolation spaces are useful when trying to quantify the width of $C$ because we can use them to measure how well the constant functions can be approximated. Consider $\bm 1$ as an element of $C(\X)$ and let $\cH$ be an RKHS that is continuously embedded in $C(\X)$; for simplicity we will treat $\cH$ as a subset of $C(\X)$.
Furthermore, define for $\theta \in (0,1)$ the interpolation space $\cH_\theta:= (C(\X),\cH)_\theta = \{f : \|f\|_\theta < \infty \}$, where 
$\|f\|_{\theta} = \sup_{t > 0} K(f,t)/t^\theta$ and $K:C(\X) \times (0,\infty)\to \mathbb{R}$ is the K-functional defined by
$K(f,t) = \inf_{h\in \cH}( \|f- h\|_\infty + t \|h\|)$. If $\bm 1 \in \cH_\theta$ then for any $r>0$ there exists an element $h\in \cH$, $\|h\| \leq r$, such that $\|\bm 1 - h\|_{\infty} \leq \|\bm 1\|_\theta^{1/(1-\theta)} r^{-2\theta/(1-\theta)}$.
In particular, for any $\epsilon >0$ there exists an $r$ and $h\in\cH, \|h\|\leq r$, such that $\|\bm 1 - h\|_\infty < \epsilon$. Therefore, with $c = 1/\|h\|$ and $h^* = h/\|h\|$, i.e. $\|h^*\|=1$, it holds that $\|c \bm 1 - h^*\|_\infty < \epsilon$ and $\diame_{h^*}(C) \leq 2 \epsilon$. If $\bm 1$ itself does not lie in $\cH$ then $h^*$ lies in the affine span of $C$ and is a problematic direction. 

In the finite dimensional case the situation is simpler. If $\cH$ is 
finite dimensional and if the constant function is not in $\cH$ then it is also not in any of the interpolation spaces since $\cH_\theta$ is a subset of the closure of  $\cH$ in $C(\X)$. But because $\cH$ is finite dimensional the closure of $\cH$ is equal to $\cH$, i.e. $\cH_\theta = \cH$ for all $\theta \in (0,1)$. The K-functional can be used in this case to quantify how well $\bm 1$ can be approximated. 

The K-functional has a few useful properties with regard to the constant function. Observe that
$K(\bm 1,1) \leq \|\bm 1 - 0\|_\infty = 1$, which does not need any conditions on the kernel function.    
When $\|k\|_\infty \leq 1$ then we also have for any $h\in \cH$ that 
\[
	\|h - \bm 1\|_\infty + \|h\| \geq (1- \|h\|) + \|h\| = 1
\]
since $\|h\|_\infty \leq \|k\|^{1/2}_\infty \|h\| \leq \|h\|$. Hence, $K(\bm 1,1) = 1$ whenever $\|k\|_\infty \leq 1$. 
It is straight forward to generalize this to any $c \in \mathbb{R}$ whenever $\|k\|_\infty < \infty$, i.e.
\begin{equation}
	K(c\bm 1,\|k\|^{1/2}_\infty) = c.
\end{equation}
Also, for any $c \in \mathbb{R}$, $t>0$ we have the trivial bound $K(c\bm 1,t) \leq c$. For $t< \|k\|^{1/2}$ the value $K(c \bm 1,t)$ can be smaller than $c$. If  $K(c \bm 1,t) < c$ then for any $\epsilon >0$ there exists a function $h\in \cH$, $h\not = 0$, such that
\[
	K(c\bm 1,t) + \epsilon \geq \|c \bm 1 - h\|_\infty  + t \|h\|,  
\]
and the norm of such an element $h$ is bounded by 
\[
\frac{c - K(c \bm 1,t)  - \epsilon}{\|k\|_\infty^{1/2}}   \leq  \|h\| \leq \frac{K(c\bm 1,t) + \epsilon}{t}.
\]  
Furthermore, 
\begin{equation*} 
	K(c \bm 1, t) = |c| \inf_{h\in \cH} (\|\bm 1 - h/c\|_\infty + t \|h/c\|) = |c| K(\bm 1,t)
\end{equation*} 
and a minimizer exists for $K(\bm 1,t)$ if, and only if, there exists a minimizer for $K(c \bm 1,t)$. The relation between these minimizers is straight forward: $h^*$ is a minimizer of $K(\bm 1,t)$ if, and only if, $c h^*$ is a minimizer of $K(c \bm 1,t)$. 

When $\cH$ is finite dimensional and $\|k\|_\infty < \infty$ then there exists a minimizer of the K-functional. For any $c\in \mathbb{R}, t>0$, 
\[
	K(c \bm 1,t) = \inf_{h\in \cH} \| c \bm 1-h\|_\infty + t\|h\| = \min_{h \in A}  \|c \bm 1-h\|_\infty + t\|h\|,  
\]
where $A = \{ h : h \in \cH, \|h\| \leq (c/t) \wedge (1+ K(c \bm 1,t)/t)\}$. This holds because $A$ is compact and 
$h\mapsto \|c \bm 1-h\|_\infty + t \|h\|$ is continuous whenever $\|k\|_\infty < \infty$. The norm of such a minimizer $h_{t,c}^*$
is bounded by $(c \wedge K(c \bm 1,t))/t$ (the additional one in the definition of $A$ is, in fact, unnecessary as the above argument shows that the infimum is attained). Hence, we have that
\[
	\frac{1 - K(\bm 1,t)}{\|k\|_\infty^{1/2}}   \leq  \frac{\|h^*_{t,c}\|}{c} \leq \frac{1 \wedge K(\bm 1,t)}{t}.
\]
In fact, we can say more about the norm of $h^*_{t,c}$ in the finite dimensional case. Notice first that $\|h^*_{t,c}\|_\infty \leq 2c$ since otherwise $0$ would be a better approximation of $c \bm 1$. Since the RKHS is finite dimensional this implies an upper bound on the RKHS-norm of $h^*_{t,c}$ as the next lemma shows. The lemma is actually of major importance in this paper and we develop it further than what is needed for the current discussion.  
In particular,  the second part of the Lemma is concerned with the relation between $\|h\|_\infty$ and $\|h\|$ when Mercer's theorem (e.g. \cite[Thm.4.49]{STEIN08}) applies. Recall that Mercer's theorem provides us under certain conditions with 
orthonormal elements $e_1^\bullet, \ldots, e_d^\bullet$ in $L^2(\X,\mu)$, $\mu$ being a Borel measure on $\X$, where $e_1,\ldots, e_d$ are continuous functions and such that $\tilde e_i = \tilde \lambda_i e_i$ for all $i\leq d$, where $\tilde \lambda_1 \geq \ldots \geq \tilde \lambda_d > 0$, lie in the RKHS $\cH$ and are an orthonormal basis of $\cH$. The kernel function has to be continuous for Mercer's theorem to hold. 
There are various forms of Mercer's theorem together with a variety of assumptions for the theorems to hold. Instead of making such assumption the following lemma 
assumes directly in its second part that the $e_1,\ldots, e_d$ exist and have the above properties.

\begin{lemma} \label{lem:tighter_lbnd}
	Let $\X$ be a set, $k$ a kernel on $\X$ such that the corresponding RKHS $\cH$ is $d$-dimensional. For any $c\in  \mathbb{R}$, $\{h : \|h\|_\infty \leq c \}$ is a compact subset of $\cH$. Furthermore, for $h \in \cH$ and any points $x_1,\ldots,x_d$ for which $k(x_1,\cdot),\ldots, k(x_d,\cdot)$ are linearly independent,  
\[
(\lambda_d/d)^{1/2}	\, \|h\| \leq   \|h\|_\infty,
\]	
where $\lambda_d$ is the smallest eigenvalue of the kernel matrix for the points $x_1,\ldots,x_d$.    

Whenever $\X$ is a topological space, $k$ is a continuous kernel function on $\X$ and there exist continuous functions $e_i: \X \to \mathbb{R}, i \leq d,$ and a Borel probability measure $\mu$ on $\X$ such that $e_1^\bullet,\ldots, e_d^\bullet$ are orthonormal in $L^2(\X,\mu)$, and $\{\tilde e_i\}_{i\leq d}$ is an orthonormal basis of $\cH$ where $\tilde e_i = \tilde \lambda_i^{1/2} e_i$, for all $i \leq d$, and $\tilde \lambda_1 \geq \tilde \lambda_2 \ldots \geq \tilde \lambda_d>0$, then
\[
 \tilde \lambda_d^{1/2} \, \| h\| \leq \|h\|_\infty.	
\]
\end{lemma}
\begin{proof}
	\textbf{(a)} For the first statement let $x_1,\ldots, x_d$ be such that $k(x_1,\cdot), \ldots, k(x_d,\cdot)$ are linearly independent. Observe that such points always exist: assume that $d' < d$ points $x_1,\ldots, x_{d'}$ exist such that any $k(x,\cdot)$ lies in the span of $k(x_1,\cdot),\ldots, k(x_{d'}, \cdot)$. Now any $h \in \cH$ of the form $\sum_{j=1}^N \alpha_j k(z_j,\cdot)$  with coefficients $\alpha_j$ and $z_j \in \X$ can be written as a sum $\sum_{i=1}^{d'} \beta_i k(x_i,\cdot)$ with suitable coefficients $\beta_i$. The family of functions $h$ that can be written this way lies dense in $\cH$, that is, $\spn \{k(x_i,\cdot) : i \leq d' \}$ is a dense subspace of $\cH$. But this subspace is closed and therefore equal to $\cH$. Hence, $\cH$ is $d'$-dimensional contradicting our assumption about $\cH$.     

Consider the linear operator $A:\cH \rightarrow \mathbb{R}^d$, defined for any $f\in \cH$ by
\[
A f = (f(x_1), \ldots, f(x_d))^\top = (\langle f,k(x_1,\cdot) \rangle, \ldots, \langle f,k(x_d,\cdot) \rangle)^\top.
\]
The operator is bounded since $\|A f\|^2_{\mathbb{R}^d} \leq \|f\|^2 \sum_{i=1}^d k(x_i,x_i)$ and $\|A\|^2_{op} \leq \sum_{i=1}^d k(x_i,x_i)$.  

$A$ is also injective. One way to see this is by means of Gram-Schmidt orthogonalization through which we gain an orthonormal basis $e_1,\ldots, e_d$ of $\cH$ from $k(x_1,\cdot), \ldots, k(x_d,\cdot)$ and for any $f,g \in \cH$ it holds that  $f=g$ if, and only if, $\langle e_i,f\rangle = \langle e_i, g \rangle $ for all $i\leq d$ if, and only if, $\langle k(x_i,\cdot), f\rangle = \langle k(x_i, \cdot), g \rangle $ for all $i\leq d$. 

Since $A$ is injective and the dimension of $\cH$ is $d$ it follows that $A$ is surjective and invertible.
By the open mapping theorem $A^{-1}$ is continuous and $A^{-1}[\{v : v\in \mathbb{R}^d, \|v\|_\infty \leq c \}  ]$ is a compact  subset of $\cH$.

\textbf{(b)} 
Let $K$ be the kernel matrix corresponding to the points $x_1, \ldots, x_d$. The rows of the kernel matrix are linearly independent since they are the images of the linearly independent elements $k(x_1,\cdot), \ldots, k(x_d,\cdot)$ under the isomorphism $A$. Hence, $K$ is invertible and for any $y\in \mathbb{R}^d$, with $\alpha = K^{-1}y$, 
\[
A(\sum_{i=1}^d \alpha_i k(x_i,\cdot)) = \sum_{i=1}^d \alpha_i (k(x_i,x_1), \ldots, k(x_i,x_d))^\top =  K \alpha = y.
\]
In particular, for $f= \sum_{i=1}^d \alpha_i k(x_i,\cdot)$, with $\alpha_i \in \mathbb{R}$, it follows that $\alpha = K^{-1} A(f)$. We have a useful inner product on $\mathbb{R}^d$ given by $\langle x, y\rangle_{K^{-1}} = x^\top K^{-1} y$. For arbitrary $f,g \in \cH$ with $f = \sum_{i=1}^d \beta_i k(x_i,\cdot)$ and $g = \sum_{i=1}^d \alpha_i k(x_i,\cdot)$, 
\[
\langle f,g \rangle = \beta^\top K \alpha = (K^{-1}Af)^\top K (K^{-1} Ag) = \langle Af, Ag\rangle_{K^{-1}}.  
\]
Applying this to $h$,
\begin{align*}
\|h\|^2 &= (Ah)^\top K^{-1} (Ah) = \tr(K^{-1} (Ah)(Ah)^\top) \\
&\leq \|K^{-1}\|_{op} (Ah)^\top(Ah) \leq d \|K^{-1}\|_{op} \|h\|^2_\infty.
\end{align*}

\textbf{(c)} Now assuming that $k$ is continuous and the $e_1,\ldots, e_d$ have the assumed properties, we can write  any $h \in \cH$ as  $h = \sum_{i=1}^d \alpha_i \tilde e_i$, $\sum_{i=1}^d \alpha_i^2 = \|h\|^2$, and
\[
\|h \|^2_2 = \sum_{i=1}^d \alpha_i^2 \tilde \lambda_i \geq \tilde \lambda_d \|h\|^2.
\]
Since $\|h\|_2 \geq \tilde \lambda_d \|h\|$ and $\mu$ is a probability measure, there has to be some point $x \in \X$ at which $|h(x)| \geq \tilde \lambda_d \|h\|$. 
\end{proof}
\begin{example}
 	Consider the space $\X = \{1,\ldots, d\}$ with kernel function $k(x,y) = 1$ if $x=y$ and zero otherwise. Then $\|h\|^2 = \sum_{i=1}^d |h(i)|^2$ and if $h(i) = c>0$ for all $i\leq d$ then $\|h\| = \sqrt{d} \|h\|_\infty$ which matches the bound if we use $x_1 =1, \ldots, x_d = d$.  
\end{example}

Coming back to the case of $\cH$ being $d$-dimensional, $\|k\|_\infty < \infty$ and $x_1,\ldots,x_d \in \X$ be any points such that
$k(x_1,\cdot),\ldots, k(x_d,\cdot)$ are linearly independent and the kernel matrix is full rank. Furthermore, let $\lambda_d$ be the smallest eigenvalue of the kernel matrix. Consider the map $\psi(h) = \|\bm 1 - h\|_\infty$. By a similar argument as above we can infer that \textit{there exists a minimizer of $\psi$}.
However, \textit{the minimizer is usually not unique}. Consider, for example, $\X = [-1,1]$ and the RKHS consisting of linear and quadratic functions such that $x\mapsto x$ and $x\mapsto x^2$ both have norm $1$. Then both of these functions minimize the distance to $\bm 1$ as does $0$.
Any minimizer $h$ of $\psi$ has norm $\|h\|_\infty \leq 2$ and, therefore, according to Lemma  \ref{lem:tighter_lbnd},  it has an RKHS norm $\|h\| \leq (4d/\lambda_d)^{1/2} =: r$. In particular, all minimizers of $\psi$ are included in the compact ball $B = \{ h : h \in \cH, \|h \| \leq r\}$. Let $A$ be the set of all minimizers of $\psi$ then $A$ is a compact set: if $A$ is finite then this follows right away. Otherwise, take a convergent sequence $\{h_n\}_{n\in \mathbb{N}}$ in $A$ and denote the limit by $h$. Since for all $x\in \X$,  $|h(x) - 1| = \lim_{n\to \infty} |h_n(x) -1|  \leq \min_{h\in \cH} \|\bm 1 -h \|_\infty$ and $h \in A$. Finally, consider the norm as a function on $A$. The norm is continuous and the image of the compact set $A$ under the norm is a compact subset in $\mathbb{R}$. Hence, there exists an element $h^*$ in $A$ of maximal norm. Let us assume first that $h^*\not = 0$. For such an element $h^*$ let $b = 1/\|h^*\|$ and note that$\|b \bm 1 - b h^*\|_\infty = \inf_{c \in \mathbb{R}} \min_{\|h\|=1} \|c \bm 1 - h\|_\infty$. Otherwise, there is an element $\tilde h, \|\tilde h\|=1,$ and a $c$ such that $\|c \bm 1 - \tilde h\|_\infty < \|  b \bm 1 - b h^*\|$.  The constant $c$ cannot be equal to $b$ since then $\|\bm 1 - \tilde h/b \|_\infty < \| \bm 1 - h^*\|$ in contradiction to our assumption on $h^*$. It also cannot be larger than $b$ because then
$\|\bm 1 - \tilde h /c \|_\infty < (b/c) \|\bm 1- h^*\|_\infty < \|\bm 1 -h^*\|_\infty$ which is again in contradiction to $h^*$ being a best approximation of $\bm 1$. But $c$ can also not be smaller than $b$; whenever $\|c \bm 1 - \tilde h\|_\infty$ is minimal it follows that $\|\bm 1 - \tilde h/c\|_\infty$ is minimal and equal to $\|\bm 1 - h^*\|_\infty$. However, $\|\tilde h / c \| = 1/c > 1/b = \|h^*\|$ in contradiction to the assumption that $\|h^*\|$ has maximal norm within $A$. 
Therefore,
\[
	\lim_{t \rightarrow 0} K(b \bm 1,t) = \|b\bm 1 - b h^*\|_\infty = \inf_{c\in\mathbb{R}} \min_{\|h\|=1} \|c\bm 1 - h\|_\infty = (1/2) \inf_{\|h\| =1} \diam_h(C) 
\]
and, since $K(b\bm 1,t) = b K(\bm 1,t) \geq (1/r) K(\bm 1,t)$,  it follows that 
\begin{equation} \label{eq:K_functional_LB}	 
	\left(\frac{\lambda_d}{d}\right)^{1/2} \lim_{t\rightarrow 0} K(\bm 1,t) \leq \inf_{\|h\|=1} \diam_h(C).  
\end{equation}
If $h^* = 0$ then $\lim_{t\to 0} K(\bm 1,t) = \|\bm 1\|_\infty = 1$ but also for any $c\in \mathbb{R}$, $h \in \cH$,  $\|c \bm 1 - h\|_\infty >
\|c \bm 1\|_\infty$ since otherwise $h/c$ would  be a minimizer of norm greater than zero, contradicting the assumption that $h^* = 0 $ is the minimizer with the largest norm. 
For $h \in \cH$, there is a sequence of points $x_1,x_2,\ldots$ such that $\lim_{n\to \infty} h(x_n)$ converges and $|h(x_n)|\to \|h\|_\infty$. Fix one such sequence and let $\sigma(h)$ be the sign of all but finitely many elements of this sequence $h(x_1),h(x_2),\ldots$, e.g. if $\sigma(h)$ is positive and $h$ attains maxima then there is a point $x$ such that $h(x) = \|h\|_\infty$. By another application of Lemma  \ref{lem:tighter_lbnd}  
it follows for any $h\in \cH$, $\|h\|=1$, that
\[
\diam_h(C)  = \| \sigma(h) \|h\|_\infty  \bm 1 - h \|_\infty > \|h\|_\infty \geq  \left(\frac{\lambda_d}{d}\right)^{1/2}  \lim_{t\to 0} K(\bm 1,t) 
\]
and 
\[
\inf_{\|h\|=1} \diam_h(C) \geq  \left(\frac{\lambda_d}{d}\right)^{1/2}  \lim_{t\to 0} K(\bm 1,t).
\]
We can set in the above derivation $r$ to  $(4/\tilde \lambda_d)^{1/2}$ when \textit{Mercer's theorem} applies, where $\tilde \lambda_d$ is the $d$-th eigenvalue of $T_k$. The bound then becomes 
\[
	\tilde \lambda_d^{1/2} \lim_{t\rightarrow 0} K(\bm 1,t) \leq \inf_{\|h\|=1} \diam_h(C). 
\]
These results are only meaningful if $\bm 1$ is not in the RKHS. In the next section we discuss an approach to remove constants from an RKHS which allows us, among other things, to extend these results to RKHSs that contain constants.

\subsubsection{Adding and removing constants}
It is sometimes useful to be able to remove constant functions from an RKHS or to add constant functions to an RKHS. There is an efficient way to do this by manipulating the kernel function.

In the following let $\mathcal{X}$ be some topological space and consider the p.s.d. functions $k:\mathcal{X} \times \mathcal{X} \to \mathbb{R}$ that lie in $\mathcal{L}^2(\mathcal{X} \times \mathcal{X})$ and denote these by 
$\mathcal{K}$. Furthermore, consider the partial order on $\mathcal{K}$ given by $k \succeq l$ if, and only if,
$k - l$ is p.s.d. where $k,l \in \mathcal{K}$. Also note that $\mathcal{K}$ is not a lattice, i.e. for $k,l \in \mathcal{K}$
the infimum $k\wedge l$ and the supremum $k \vee l$ will generally not be defined. 

For a function $f: \X \to \R$ we let $f \otimes f$ be the function that maps $(x,y)$ to $f(x)f(y)$ for any $x, y \in \X$. There is a simple criterion which tells us if $f \in \cH_k$ for a kernel function $k \in \mathcal{K}$. Assume that $f\otimes f \in \mathcal{L}^2(\mathcal{X} \times \mathcal{X})$, then $f\in \cH_k$ 
if, and only if, there exists a $c>0$ with $c^2 k \succeq f \otimes f$. In case that $f\in \cH_k$ it holds that $\|f\|_k = \inf \{c:  c^2 k \succeq f \otimes f\}$.

This observation motivates the following definitions. For an RKHS $\cH$ with kernel $k$ that does not contain $\bm 1$ let
\begin{equation} \label{def:k_plus}
	k^+ := k + \bm 1\otimes \bm 1 \text{\quad and \quad} \cH^+ := \cH_{k^+}. 
\end{equation}
The function $k^+$ is a kernel function being the sum of the kernel functions $k$ and $\bm 1 \otimes \bm 1$ and $\mathcal{H}^+$ is well defined. We denote the norm of $\mathcal{H}^+$ by $\|\cdot \|_+$ and we can observe that 
\[
	\| \bm 1\|_+ = \inf\{c : c^2 (k + \bm 1 \otimes \bm 1) \succeq \bm 1 \otimes \bm 1   \} \leq 1. 
\]
In fact, $\|\bm 1\|_+ = 1$ because otherwise there exists a $c < 1$ such that
\[
	c^2 k \succeq (1-c^2) \bm 1 \otimes \bm 1 \quad \Rightarrow \quad \Bigl(\frac{c}{\sqrt{1-c^2}}\Bigr)^2 k \succeq \bm 1 \otimes \bm 1 \quad	\Rightarrow  \quad \bm 1 \in \cH.
\]
We also have that $\mathcal{H} \subset \mathcal{H}^+$, since $k \preceq k^+$, and for $h\in \cH$,
\begin{equation} \label{eq:order_plus_leq}
\|h\|_+ \leq \|h\|.
\end{equation}
When the RKHS $\cH$ is finite dimensional then $\|h\|_+$ is actually equal to $\|h\|$. To show this we make use of the following lemma which is a simple extension of  \cite[Sec5.3]{PAUL16}. 
\begin{lemma} \label{lem:order_orthogonal}
Let $h_1,\ldots, h_d$ be linearly independent functions mapping from some topological space $\X$ to $\mathbb{R}$ and let $a_1,\ldots, a_d >0$ then $\kappa = \sum_{i=1}^d a_i h_i \otimes h_i$ is a kernel function, the functions $h_i$ lie in $\cH_\kappa$ and are orthogonal in $\cH_\kappa$. Furthermore, the dimension of $\cH_\kappa$ is $d$ and $\|h_i\|_\kappa =1/ \sqrt{a_i}$. 
\end{lemma}




Now, let $d< \infty $ be the dimension of $\cH$, choose orthogonal 
functions $h_1,\ldots, h_d \in \cH$, $h_1,\ldots,h_d \not = 0$, and define the kernel $\kappa = \sum_{i=1}^d (1/\|h_i\|^2) h_i \otimes h_i$.
Then $k = \kappa$. This follows because, according to the above lemma, both spaces consist of $\spn\{h_1,\ldots, h_d\}$,  $\|h_i\|_\kappa = \|h_i\|_k$, for all $i\leq d$, and the $h_i$'s are orthogonal in both spaces, i.e. $H_k = H_\kappa$ which implies that $k=\kappa$. The importance of this statement is that it shows that we can write the kernel as a finite sum of weighted tensor products.

From this description of $\kappa$ we also gain that  $k^+ = \sum_{i=1}^d a_i h_i \otimes h_i + \bm 1 \otimes \bm 1$ and, because $\bm 1$ is not in the original RKHS $\cH$, it follows that $\bm 1$ is linearly independent of $h_1,\ldots, h_d$ which implies that $\bm 1$ is orthogonal to $h_1,\ldots, h_d$ in $\cH^+$.

Consider now one of the $h_i$'s. We like to show that $\|h_i\|_+ \geq \|h_i\|$ which then implies, together with  \eqref{eq:order_plus_leq}, that $\|h_i\|_+ = \|h_i\|$ and $\|h\|_+ = \|h\|$ for all $h\in \cH$; the $h_i$'s are orthogonal in both $\cH$ and $\cH^+$. Let $l = k^+ - (1/\|h_i\|^2) h_i \otimes h_i$ so that $h_i \not \in \cH_l$. Furthermore, consider any $c$ such that  $0 < c < \|h_i\|$. If $\|h_i\|_+ = c$ then 
\[
	c^2 k^+ \succeq h_i \otimes h_i \, \Rightarrow \, c^2 l \succeq  (1-c^2/\|h_i\|^2) h_i \otimes h_i \, \Rightarrow \, \frac{c^2\|h_i\|^2}{\|h_i\|^2-c^2} l \succeq h_i \otimes h_i \,	\Rightarrow  \, h_i \in \cH_l,
\]
which is impossible and, therefore, $\|h_i\|_+ \geq \|h_i\|$.

Similarly, for an RKHS $\cH$ that does contain $\bm 1$ and is not of dimension $1$ let 
\begin{equation} \label{def:k_minus}
 	k^- = k - c^2 \bm 1 \otimes \bm 1,\text{ where } c = 
\inf \{\tilde c:  \tilde c^2 k \succeq 1 \otimes 1\}, \text{\quad  and \quad} \mathcal{H}^- := \mathcal{H}_{k^-}.
\end{equation}
It follows right away that $\bm 1 \not \in \cH_-$ and because, $k^- \preceq k$ we know that  $\cH^- \subset \cH$
and $\|h\|_- \leq \|h\|$ for all $h\in \cH^-$. Next, notice that we can write $k = \sum_{i=1}^{d-1} a_i h_i \otimes h_i + c^2 \bm 1 \otimes \bm 1$ where  $h_1,\ldots, h_{d-1}, \bm 1$ are orthogonal in $\cH$ and $a_1, \ldots, a_{d-1} >0$. Due to the orthogonality it follows that the $h_1,\ldots, h_{d-1}$ are linearly independent elements in $\cH^-$ and $\cH^-$ is $d-1$ dimensional. Lemma \ref{lem:order_orthogonal} tells us furthermore that $h_1,\ldots, h_{d-1}$ are orthogonal in $\cH^-$.
Finally, for all $i\leq d-1$ we have that  $\|h_i\|_- = \|h_i\|$; assume $c=1$ and observe that in this case $(\cH^-)^+ = \cH$ and due to the above results for $\cH^+$ we can conclude that $\|h_i\|_- = {\|h_i\|_-}_+ = \|h_i\|$. The above argument for $\cH^+$ does not rely on $\|\bm 1\| = c= 1$ and we can generalize this result right away to any $c >0$. Because the norm of the $h_i$ does not change and since the $h_i$ are orthogonal we can conclude that $\|h\|_- = \|h\|$ for all $h\in \cH_-$.

We summarize these results for the case when $\cH$ is finite dimensional in the following lemma.
\begin{lemma} \label{lem:approx_error_sup_finite}
	If $\cH$ is a finite dimensional RKHS with dimension $d$, kernel $k \in \mathcal{K}$, and which does not contain $\bm 1$ then
	$\cH^+$, as defined in \eqref{def:k_plus}, is $d+1$ dimensional, $\cH \subset \cH^+$, $\bm 1 \in \cH^+$ with $\|\bm 1\|_+ = 1$, $\langle g,h \rangle_+ = \langle g,h \rangle$ for all $g,h \in \cH$,  and $\bm  1$  is orthogonal in $\cH^+$ to all $h\in \cH$.   
	Similarly, if $\cH$ is a finite dimensional RKHS with dimension $d > 1$, kernel $k\in \mathcal{K}$, and which does contain $\bm 1$ then
	$\cH^-$, as defined in \eqref{def:k_minus}, is $d-1$ dimensional, $\cH^- \subset \cH$, $\bm 1 \not \in \cH^-$,$\langle g,h \rangle_- = \langle g,h \rangle$ for all $g,h \in \cH$ which are orthogonal to $\bm 1$.    
\end{lemma}





\subsubsection{Lower bounds on the approximation error in finite dimensions} \label{sec:finite_dim_appox_LB}
In finite dimensions we can now provide lower bounds on the approximation error of any function $f:\X \to \mathbb{R}$. 
Before specializing to constant functions we take a short detour and discuss the general technique.    
The approach to get lower bounds is the following: let $k = \sum_{i=1}^d a_i h_i \otimes h_i$ for linearly independent $h_1,\ldots, h_d$ and $a_i>0$. If $f$ is linearly dependent on the $h_i$'s then $f \in \cH$. Otherwise, we can move to the kernel function
$k' = \sum_{i=1}^d a_i h_i\otimes h_i + f\otimes f$ and the corresponding RKHS $\cH'$. The function $f$ is orthogonal to $h_1,\ldots, h_d$ in $\cH'$. That means that the lowest approximation error, when approximating $f$ by functions in the subspace corresponding to $\cH$, is given by the projection onto this subspace. Due to the orthogonality the projection of $f$ onto this subspace is just the origin and the approximation error is 
$\|f\|_{\cH'} = 1$ when measured in the RKHS norm of $\cH'$. If we consider the constraint that the approximation has 
to lie in $\cH$ and has to have norm $\|h\| =1 $ then  the best approximation error of $f$ is $\sqrt{2}$, i.e.
\[
	\inf_{h \in \cH, \|h\|=1} \|f-h\|_{\cH'} = \sqrt{2}.
\]
To gain a lower bound on the approximation error in $\|\cdot\|_\infty$ we use Lemma \ref{lem:tighter_lbnd} 
which shows that 
\[
	\inf_{h\in \cH, \|h\|=1} \|f-h\|_\infty \geq \sqrt{2} \left( \frac{\lambda_{d+1}}{d+1}\right)^{1/2}, 
\]
where we get $d+1$ since we use the RKHS $\cH'$ which has dimension $d+1$. The constant $\lambda_{d+1}$ is the smallest eigenvalue of a kernel matrix corresponding to points $x_1,\ldots, x_{d+1}$ such that $k'(x_1,\cdot),\ldots, k'(x_{d+1}, \cdot)$ are linearly independent. Notice, that this approximation error depends implicitly on the particular function $f$ through the kernel matrix and the smallest eigenvalue. The bound can become loose when $\|f\|_\infty$ is significantly larger than $\|h_i\|_\infty$, but observe that we can always replace $f$ by $cf$ for some constant $c<1$ to rescale the infinity norm. In the following, let $k'' =  \sum_{i=1}^d a_i h_i\otimes h_i +  (cf)\otimes (cf)$ and treat $\cH$ as a subset of $\cH'' := \cH_{k''}$. Such a rescaling leads to a problem in the constraint $\|h\|=1$ because 
\[
	\inf_{h\in \cH, \|h\|=1} \|cf-h\|_\infty  = c \inf_{h\in \cH, \|h\|=1/c} \|f-h\|_\infty. 
\]
We can compensate for this by using the constraint $\|h\| = c$. Since $\|cf\|_{\cH''} = 1$, 
\[
	\inf_{h\in \cH, \|h\|=1} \|f-h\|_\infty = \frac{1}{c} \inf_{h\in \cH, \|h\|=c} \|cf- h\|_\infty \geq \frac{\sqrt{1 + c^2}}{c} \left( \frac{\lambda_{d+1}}{d+1}\right)^{1/2}, 
\]
where $\lambda_{d+1}$ is again the smallest eigenvalue of a kernel matrix but now for the kernel $k''$.

\begin{example}
	Let $\X =\{0,1\}$ and $h:\X \to \mathbb{R}$ be given by $h(0)= 1, h(1) =0$, and let $f: \X \to \mathbb{R}$ be defined by $f(0) = 0, f(1) = r$ for $r > 0$. Let $\cH$ be the RKHS with kernel $h\otimes h$ which consists of $\spn \{h\}$.
	The smallest approximation error of $f$ by elements in $\cH$ which have norm $1$ is attained by $-h$ and $h$ and is equal to  $\|h - f\|_\infty = r\vee 1$. Considering now the bound:  let the kernel of the RKHS $\cH'$ be $k = h \otimes h + f \otimes f$. Consider $x_1= 0, x_2 =1$ and the corresponding kernel matrix
	\[
		K = \begin{pmatrix}
			h(0)^2 & 0 \\
			0 & f(1)^2 
		\end{pmatrix} =  \begin{pmatrix}
			1 & 0 \\
			0 & r^2 
		\end{pmatrix} 
	\]
	which has minimal eigenvalue $1 \wedge r^2$. The corresponding lower bound is
	\[
		\inf_{g\in \cH, \|g\|=1} \|g - f\|_\infty \geq \sqrt{2} \left(\frac{1 \wedge r^2}{2} \right)^{1/2} = 1 \wedge r
	\]
	which is exact when $\|f\|_\infty = 1$ but degrades for $r$ away from $1$. 

	Scaling $f$ by $c = 1/\|f\|_\infty = 1/r$ gives us the kernel $k' = h \otimes h + (1/r)^2 f\otimes f$ and a kernel matrix
	\[
	K' =  \begin{pmatrix}
			h(0)^2 & 0 \\
			0 & (1/r)^2 f(1)^2 
		\end{pmatrix} =  \begin{pmatrix}
			1 & 0 \\
			0 & 1 
		\end{pmatrix} 
	\]
which has minimal eigenvalue $1$. The bound becomes 
\[
	\inf_{g\in \cH, \|g\|=1} \|g - f\|_\infty \geq \sqrt{\frac{1 + r^2}{2} } \geq 1 \wedge r.
\]
\end{example}
Coming back to the approximation of constant functions. When $\cH$ does not contain the constant functions then an  approach to calculate lower bounds  is to use the kernel $k_+$ and the corresponding RKHS $\cH^+$. The norm of $c \bm 1$ in this RKHS,  where $c \in \mathbb{R}$, is $|c|$ and for any such $c$,
\[
	\inf_{h \in \cH, \|h\|=1} \|h-c \bm 1\|_\infty \geq \frac{\sqrt{1 + |c|^2}}{|c|} \left(\frac{\lambda_{d+1}}{d+1}\right)^{1/2} \geq  \left(\frac{\lambda_{d+1}}{d+1}\right)^{1/2} 
\]
with $d$ being the dimension of $\cH$ and $\lambda_{d+1}$ the lowest eigenvalue of a kernel matrix corresponding to points $x_1,\ldots,x_{d+1}$ for the kernel $k^+$. 
Using the right hand side as the lower bound 
has the advantage that we only deal with one RKHS, i.e. with $\cH^+$, and we only need $\lambda_{d+1}$ for that kernel.
Scaling of the function $\bm 1$ in dependence of which constant $c \bm 1$ we want to approximate might improve the lower bounds but then $\lambda_{d+1}$ has to be calculated for the individual scalings.  

When Mercer's theorem applies we gain the bound  
\[
\inf_{c \in \R} \inf_{h \in \cH, \|h\|=1} \|h - c\bm 1\|_\infty \geq \tilde \lambda_{d+1}^{1/2}, 
\]
where $\tilde \lambda_{d+1}$ is the $(d+1)$-th eigenvalue of $T_{k^+}$. For Mercer's theorem to apply it is important that $k^+$ is continuous.
But when $k$ is continuous then so is $k^+$. 

If $\cH$ already contains the constant functions then we are interested in determining the width of the convex set in the affine subspace spanned by $C$. In particular, because $\langle k(x,\cdot), \bm 1 \rangle  = 1$ for all $x \in \X$, there exists a subspace $S$ of $\cH$ that is orthogonal to $\bm 1$ and a $c \not = 0$ such that $\aff C = \aff \{k(x,\cdot) : x \in \X\} = c \bm 1 + S$. In fact, $c = \argmin_{c' \in \mathbb{R}} \| k(x,\cdot) - c' \bm 1\|$, where we can use an arbitrary $x \in \X$ and $S = \cH^-$.  
This is exactly the same situation that we faced above with $\cH^+$ and a lower bound on the width of the convex set in the affine space spanned by it can be gained through
\[
	\inf_{c\in \mathbb{R}} \inf_{\|h\|_- = 1} \| h - c \bm 1  \|_\infty \geq \left(\frac{\lambda_d}{d} \right)^{1/2},
\]
where $d$ is the dimension of $\cH$ and $\lambda_d$ the smallest eigenvalue of any kernel matrix for kernel $k$. If we can use Mercer's theorem then we also gain the lower bound 
\[
\inf_{c \in \R} \inf_{\|h\|_- =1} \|h - c \bm 1 \|_\infty \geq \tilde \lambda_{d}^{1/2}, 
\]
where $\tilde \lambda_{d}$ is the $d$-th eigenvalue of $T_{k}$.

We can also extend the results from Section \ref{sec:interpol} on the application of K-functionals. We summarize in the following proposition these results together with a variety of results on the width of $C$ that we derived up to now. We use the notation $K_-(\bm 1,t)$ for the $K$-functional corresponding to $\cH_-$. We hope that the use of the letter $K$ for both the $K$-functional and the kernel matrix does not lead to confusion. To streamline the statement of the following proposition let us say that $k$ has a \textit{Mercer decomposition with lowest eigenvalue $\tilde \lambda_d$} if $k$ is a continuous kernel function on $\X$ and there exist continuous functions $e_i: \X \to \mathbb{R}, i \leq d,$ and a Borel probability measure $\mu$ on $\X$ such that $e_1^\bullet,\ldots, e_d^\bullet$ are orthonormal in $L^2(\X,\mu)$, $\{\tilde e_i\}_{i\leq d}$ is an orthonormal basis of $\cH$, where $\tilde e_i = (\tilde \lambda_i)^{1/2} e_i$, for all $i \leq d$, and $\tilde \lambda_1 \geq \tilde \lambda_2 \ldots \geq \tilde \lambda_d>0$. Notice that the Mercer decomposition based results in the following proposition do not seem to have a dependence on $d$ beyond the eigenvalue $\lambda_d$ but this is somewhat misleading as the discussion in Section \ref{Sec:diam_emp_convex} demonstrates.
\begin{proposition} \label{prop:approx_lower_bnd}
Let $\X$ be a measurable set and $k \in \mathcal{K}$ a kernel function defined on $\X$. 
The following holds. 
\begin{enumerate}
\item If $\cH$ is infinite dimensional, $\X$ is compact and  $k$ is continuous, then for every $\epsilon>0$ there exist infinitely many orthonormal elements $(e_n)_{n \geq 1}$  in $\cH$ such that $\sup_{n \geq 1} \diam_{e_n}(C) < \epsilon$. 
\end{enumerate}
If $\cH$ is finite dimensional with dimension $1 \leq d$ then the following hold.
\begin{enumerate}
\setcounter{enumi}{1}
\item If $\bm 1 \in \cH_\theta$ for some $\theta\in (0,1)$ then there exists $h\in\cH, \|h\|=1,$ such that $\diam_h(C) = 0$. 
\item If $\bm 1 \not \in \cH$ then for any $x_1, \ldots, x_{d+1} \in \X$  and corresponding kernel matrix $K^+ = (k^+(x_i,x_j))_{i,j\leq d+1}$ with smallest eigenvalue $\lambda_{d+1}$,  
\[
\inf_{\|h\| = 1} \diam_h(C)  \geq  2 \left(\frac{\lambda_{d+1}}{d+1}\right)^{1/2}.	
\]
\item If $\bm 1 \not \in \cH$, $\|k\|_\infty<\infty$, then for any $x_1, \ldots, x_{d} \in \X$  and corresponding kernel matrix $K = (k(x_i,x_j))_{i,j\leq d}$ with smallest eigenvalue $\lambda_{d}$,  
\[
	\inf_{\|h\| = 1} \diam_h(C)  \geq \left(\frac{\lambda_{d}}{d}\right)^{1/2} \lim_{t\to 0} K(\bm 1,t).	
\]
\item	If $\bm 1 \in \cH$ and $2 \leq d$, then for any $x_1, \ldots, x_{d} \in \X$  with corresponding kernel matrix $K = (k(x_i,x_j))_{i,j\leq d}$ and with the smallest eigenvalue of $K$ being $\lambda_{d}$,
\[
	\inf_{\|h\|_- = 1} \diam_h(C)  \geq 2 \left(\frac{\lambda_{d}}{d} \right)^{1/2}.
\]
\item If $\bm 1 \in \cH$, $2 \leq d$, then for any $x_1, \ldots, x_{d-1} \in \X$  with corresponding kernel matrix $K^- = (k(x_i,x_j))_{i,j\leq d-1}$ and with the smallest eigenvalue of $K^-$ being $\lambda_{d-1}$,
\[
	\inf_{\|h\|_- = 1} \diam_h(C)  \geq  \left(\frac{\lambda_{d-1}}{d-1} \right)^{1/2} \lim_{t \to 0} K_-(\bm 1,t).
\]
\end{enumerate}
In the following, let $\X$ be a compact space and $k$ a continuous kernel function on $\X$. The following hold.  
\begin{enumerate}
\setcounter{enumi}{6}
\item If $k^+$ has a Mercer decomposition with smallest eigenvalue $\tilde \lambda_{d+1}$ and $\bm 1 \not \in \cH$ then
\[
\inf_{\|h\| = 1} \diam_h(C)  \geq  2 \tilde \lambda_{d+1}^{1/2}. 	
\]
\item If $k$ has a Mercer decomposition with smallest eigenvalue $\tilde \lambda_d$ and $\bm 1 \not \in \cH$ then  
\[	
\inf_{\|h\| = 1} \diam_h(C) \geq \tilde \lambda_d^{1/2} \lim_{t\to 0} K(\bm 1,t). 
\]
\item If $k$ has a Mercer decomposition with smallest eigenvalue $\tilde \lambda_{d}$ and $\bm 1 \in \cH$ then
\[
\inf_{\|h\|_- = 1} \diam_h(C)  \geq  2 \tilde \lambda_{d}^{1/2}. 	
\]
\item If $\cH$ is $d\geq 2$ dimensional, $k^-$ has a Mercer decomposition with smallest eigenvalue $\tilde \lambda_{d-1}$ and $\bm 1 \not \in \cH$ then 
\[	
\inf_{\|h\|_- = 1} \diam_h(C) \geq \tilde \lambda_{d-1}^{1/2} \lim_{t\to 0} K_-(\bm 1,t). 
\]
\end{enumerate}
\end{proposition}

\begin{example}
	Consider the kernels $k_d(x,y) = \sum_{u=1}^d x^u y^u$, with $x,y \in [-1,1]$, which corresponds to polynomials of order $1$ to $4$ but without the constant functions. To test the kernel matrix based lower bound  in a simple experiment we are calculating upper bounds on $\inf_{c\in \mathbb{R}} \inf_{h \in \cH, \|h\|=1} \|h - c \bm 1\|_\infty$ in the following way: the functions $x^u$ and $x^v$ are orthogonal in the corresponding RKHSs whenever $u\not = v$ and have norm $1$. Therefore, functions of the form $(1/\sqrt{d})\sum_{u=1}^d x^u$ have norm $1$. To get a good approximation of constant functions we use such functions for $d=3,4$, with signs adjusted so that the different terms cancel each other as well as possible. In detail, for $d=1$ we use the function $h_1(x) = x$ which has approximation error $1$ when approximating the (constant) function $0$; for $d=2$ we use $h_2(x) = x^2$; for $d=3$ we use $h_3(x) = (1/\sqrt{3}) (x + x^2 - x^3)$; and for $d=4$, $h_4(x) =  (1/\sqrt{4}) (-x + x^2 + x^3 - x^4)$. The functions for $d=2,3$ and $4$ are shown in Figure  \ref{fig:approx_error} in the left three plots in blue. The constant that are best approximated by these functions are shown in orange. In the right plot the corresponding approximation error in $\|\cdot \|_\infty$ norm  is plotted against $d$ (top curve; orange). The blue curve in the right plot corresponds to the lower bound where we use $-1 = x_1 < \ldots < x_d =1$ with equidistant spacing to get full rank kernel matrices.        	
\end{example}

\begin{figure}[t]  
\begin{center} 
\includegraphics[scale=1.]{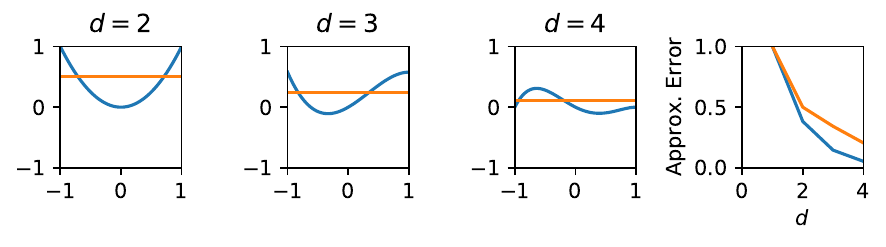}  
\end{center}
\caption{The three plots on the left show in blue polynomials of degree $2,3$ and $4$ respectively. The orange lines correspond to the constant functions that are best approximated by these polynomials. The right most plot shows the corresponding approximation error in $\|\cdot\|_\infty$ (orange curve) and our lower bound on the approximation error (blue curve). Note that the approximation error is calculated for the three curves in the left plots and is only an upper bound for the best approximation error that can be attained.        
} \label{fig:approx_error} 
\end{figure}

\subsubsection{Quantifying the width of the empirical convex set $C_n$} \label{Sec:diam_emp_convex}
The above techniques can also be applied to the empirical convex set $C_n$. An easy way to do so is to identify the subspace spanned by $C_n$ with a new RKHS. In particular, the subspace spanned by the empirical convex set $C_n$ can be identified with an RKHS in a similar way to how we dealt with measures that attain values in a subspace in Section \ref{sec:example_subset}; see also Section \ref{sec:data_subset_covariance_arg} for a more detailed discussion.
For an experiment $\omega \in \Omega$  
let  $S_\omega = \{X_1(\omega), \ldots, X_n(\omega) \}$ be the support of the empirical measure for the realization $\omega$. If $k$ is our original kernel function then let $k_\omega$ be $k\!\!\upharpoonright\!\! S_\omega \times S_\omega$ and let $\cH_\omega$ be the corresponding RKHS. The empirical convex set $C_n$, as an element of $\cH$, has then a corresponding convex set $C_\omega = \ch \{k_\omega(X_i(\omega),\cdot) : i \leq n \}$ within $\cH_\omega$. For ease of notation fix an $\omega \in \Omega$ and let $x_1,\ldots, x_n \in \X$ be  $x_1 = X_1(\omega), \ldots, x_n = X_n(\omega)$ for the rest of this section. 

Importantly, there is a linear map $\psi:\cH \to \cH_\omega$ defined in the following way: if $h \in \cH$ is of the form $\sum_{i=1}^n \alpha_i k(x_i,\cdot)$ for some $\alpha_i \in \R$ then let $\psi(h) = \sum_{i=1}^n \alpha_i k_S(x_i,\cdot)$. Also, let $U = \spn \{k(x_i,\cdot) : i \leq n\}$ be the subspace of $\cH$ corresponding to these functions $h$. For functions $g \in U^\perp$, define $\psi(g) = 0$, and extend $\psi$ to all of $\cH$ by linearity. The function $\psi: \cH \to \cH_\omega$ defined in this way has the following properties: for all $g,h \in \spn \{k(X_1(\omega),\cdot), \ldots, k(X_n(\omega),\cdot)\} \subset \cH$ we have $\langle g,h \rangle = \langle \psi(g), \psi(h) \rangle_{\cH_\omega}$ (this follows right away from the kernel expansion of $g,h$ because $k$ and $k_S$ are equivalent on $x_1,\ldots,x_n$) and $\psi(f) = 0$ if $f$ is orthogonal to the subspace spanned by the data. In other words, $\psi$ is a partial isometry between $\cH$ and $\cH_\omega$ and an isometry between $U$, with the inherited inner product, and $\cH_\omega$.

Beside this natural link between $\cH$  and $\cH_\omega$ there is also the linear map $A$ that we considered in Lemma \ref{lem:tighter_lbnd}. We have to adapt the approach from Lemma \ref{lem:tighter_lbnd} slightly to make use of it in this new context. First, observe that if $\cH_\omega$ is $d_\omega$-dimensional then we have the operators $A_\omega: \cH_\omega \to \mathbb{R}^{d_\omega}$ defined by $A_\omega f = (f(x_{\iota(1)}), \ldots , f(x_{\iota(d_\omega)}) )^\top$ for a given injective function $\iota:\{1,\ldots, d_\omega\} \to \{1, \ldots, n\}$ such that the matrix $(k(x_{\iota(i)},x_{\iota(j)}))_{i,j\leq d_\omega}$ has full rank. This dimension can obviously depend on $\omega$ and will always be upper bounded by the dimension $d_\cH$ of $\cH$. Consider now the kernel matrix $K_\omega = (k(x_{\iota(i)},x_{\iota(j)}))_{i,j \leq d_\omega}$  and equip $\mathbb{R}^{d_\omega}$ with the inner product $\langle a, b \rangle_{K_\omega^{-1}} =a^\top K_\omega^{-1} b$, $a,b \in \mathbb{R}^{d_\omega}$. As in the proof of Lemma \ref{lem:tighter_lbnd} it follows that
$\langle g,h \rangle_{\cH_\omega} = \langle A g, Ah \rangle_{K^{-1}_\omega} $ for all $g,h \in \cH_\omega$ and $A_\omega : \cH_\omega \to \R^{d_\omega}$ is an isometry.

We have the following commutative diagram summarizing the relationship between the three spaces.
\begin{center}
\begin{tikzcd}
	\cH \arrow{rr}{\psi} \arrow[swap]{dr}{A_\omega \circ \psi} & & \cH_\omega \arrow{dl}{A_\omega} \\
     & \mathbb{R}^{d_\omega}
\end{tikzcd}
\end{center}
Furthermore, when $U = \spn\{k(x_1,\cdot), \ldots, k(x_n,\cdot) \}$ is the subspace of $\cH$ induced by the data it follows that the following three spaces are isometric isomorphic 
\[ (U,\langle \cdot,\cdot \rangle) \cong (\cH_\omega, \langle \cdot, \cdot \rangle_{\cH_\omega}) \cong (\mathbb{R}^{d_\omega},\langle \cdot, \cdot \rangle_{K_\omega^{-1}}).
\]
In particular, $A_\omega \circ \psi$ is an isometry between $U$ and $\mathbb{R}^{d_\omega}$. This isometry has takes a  simple form: let $h = \sum_{i=1}^n \alpha_i k(x_i,\cdot)$ then 
\[
	(A_\omega \circ \psi)(h) = \sum_{i=1}^n \alpha_i (k(x_i,x_{\iota(1)}),\ldots, k(x_i,x_{\iota(d_\omega)}))^\top = (h( x_{\iota(1)}),\ldots, h(x_{\iota(d_\omega)}))^\top.
\]
This relation allows us to apply the techniques we developed for measuring the size of $C$ to the empirical convex set $C_n$. For example, if $\cH_\omega$ does not contain constant functions then using the kernel $k_S^+$ and denoting the corresponding RKHS by $\cH_\omega^+$, we can lower bound the width of $C_\omega$. The RKHS $\cH_\omega^+$ has dimension $d_\omega +1$ and there exists an injection  $\iota: \{1,\ldots, d_\omega+1 \} \to \{1,\ldots,n\}$ such that $k_S^+(x_{\iota(1)},\cdot), \ldots, k_S^+(x_{\iota(d_\omega+1)},\cdot)$ are linearly independent. Then, as above, $A_\omega^+: \cH_\omega^+ \to \R^{d_\omega +1}$ defined by $A_\omega^+(h) = (h( x_{\iota(1)}),\ldots h(x_{\iota(d_\omega+1)}))^\top$ is an isometry between $\cH_\omega^+$ and $\R^{d_\omega+1}$ when the latter is equipped with the inner product $\langle a,b \rangle_{(K_\omega^+)^{-1}}$, for all $a,b \in \R^{d_\omega+1}$ and $K_\omega^+$  is the kernel matrix corresponding to the points $x_{\iota(1)},\ldots, x_{\iota(d_\omega+1)}$. From this we can infer a lower bound on the width of $C_\omega$ within $\cH_\omega$. Alternatively, we can apply directly  Proposition \ref{prop:approx_lower_bnd} to $\cH_\omega$ to get this lower bound. Since we have an isometry between $U$ and $\cH_\omega$ these lower bounds translate directly to lower bounds on the width of $C_n$ within $U$.

There is another point worth noting. The lower bound on the width of $C_n$ depends on the choice of $\iota$. Finding the subset of points $x_1,\ldots, x_n$ that maximizes this lower bound seems like a hard problem. Therefore one might wonder if there is a simpler way to \textit{optimize the lower bound}. In particular, there seems hope to get the largest smallest eigenvalue $\lambda_{d_\omega}$ when using the full kernel matrix. To that end, let $K^\star_\omega = (k(x_i,x_j))_{i,j \leq n}$ be the kernel matrix corresponding to all the data. 
Since the subspace spanned by the data has dimension $d_\omega$ it follows that there are exactly $d_\omega$ non-zero eigenvalues $\lambda_1^\star,\ldots, \lambda^\star_{d_\omega}$.
There is a useful interplay between $K_\omega^\star$ and the following linear operator $A_\omega^\star :\cH_\omega \to \R^n$ given by $A_\omega^\star(h) = (h(x_1),\ldots, h(x_n))^\top$. First note that for $h= \sum_{i=1}^n \alpha_i k(x_i,\cdot)$, with suitable $\alpha_i \in \R$, we have that $A^\star_\omega(h) = K_\omega^\star \alpha$. Also observe that $A_\omega^\star$ in injective because if $A_\omega^\star(f) = (f(x_1), \ldots, f(x_n))^\top = (g(x_1),\ldots, g(x_n))^\top = A_\omega^\star(g)$ for two functions $f,g \in \cH_\omega$, $f: S\to \R, g:S \to \R$, then $f$ and $g$ are equal on $S$ and are therefore the same function. 
While $A_\omega^\star$ is injective there are generally for a given $h\in \cH_\omega$ many $\alpha \in \R^n$ such that $A_\omega^\star(h) = K_\omega^\star \alpha$ and $K_\omega^\star$ is not invertible. Therefore, consider the Moore-Penrose pseudo-inverse $(K_\omega^\star)^\dagger$, and observe that  
with  $\alpha_h^\star = (K_\omega^\star)^\dagger A_\omega^\star(h)$ we get $K_\omega^\star \alpha_h^\star = K_\omega^\star (K_\omega^\star)^\dagger A_\omega^\star(h) = A_\omega^\star(h)$ since $A_\omega^\star(h)$ lies in the range of $K_\omega^\star$ \cite[Def.1.1.2(a)]{CAMP09}.  In particular, for $f,g \in \cH_\omega$,
\[
\langle f,g \rangle = (\alpha^\star_f)^\top K_\omega^\star \alpha_g^\star = 
(A_\omega^\star(f))^\top (K_\omega^\star)^\dagger K_\omega^\star (K_\omega^\star)^\dagger A_\omega^\star(g) =  (A_\omega^\star(f))^\top (K_\omega^\star)^\dagger A_\omega^\star(g).
\]
From this relation we get a lower bound on the supremums norm of a function $h \in \cH_\omega$,
\begin{align*}
	\|h\|_{\cH_\omega}^2 &= (A_\omega^\star(f))^\top (K_\omega^\star)^\dagger A_\omega^\star(f) = \tr(
(K_\omega^\star)^\dagger A_\omega^\star(f) (A_\omega^\star(f))^\top) \\
&\leq \| (K_\omega^\star)^\dagger\|_{op}  (A_\omega^\star(f))^\top A_\omega^\star(f) 
\leq n \| (K_\omega^\star)^\dagger\|_{op}  \|h\|_\infty^2. 
\end{align*}
The term  $\|(K_\omega^\star)^\dagger\|_{op}$ is equal to $1/\lambda_{d_\omega}^\star$ but, unfortunately,  
instead of the constant $d_\omega$ we have now the constant $n$.

In Proposition \ref{prop:approx_lower_bnd} seemingly no price had to be paid for the dimension of $\cH$ when using a \textit{Mercer decomposition}. Since intuitively $K^\star_\omega$ is closely related to the integral operator that appears in Mercer's theorem when the underlying measure is the empirical measure $P_n$ one might wonder if the constant $d_\omega$, or $n$, can be removed by following that route. Unfortunately, this approach does not, in fact, remove the constant: consider the integral operator
\[
	(T_\omega f)(y) = \int f(x) k(x,y) \,dP_n(x),
\]
for $f \in \mathcal{L}^2(S, P_n)$ where $S = \{x_1,\ldots,x_n\}$. Observe that $\mathcal{L}^2(S,P_n)$ is the same set of functions as $\cH_\omega$ but the $\mathcal{L}^2$-inner product does not have to be equal to the inner product of $\cH_\omega$. Then for $f\in \mathcal{L}^2(S,P_n) = \cH_\omega$  and  $j\leq n$,
\[
	(T_\omega f)(x_j) = \frac{1}{n} \sum_{i=1}^n f(x_i) k(x_i,x_j) = n^{-1} ((A_\omega^\star(f))^\top K_\omega^\star)_j 
\]
and for $f,g \in \mathcal{L}^2(S,P_n)$,
\[
\langle T_\omega f,g \rangle_2 = n^{-2} (A_\omega^\star(f))^\top K_\omega^\star A_\omega^\star(g).
\]
The eigenfunctions of $T_\omega$ are closely related to the eigenvectors of $K_\omega^\star$. Let $q_1,\ldots, q_n \in \mathbb{R}^n$ be the eigenvectors of $K_\omega^\star$ and 
let $\lambda^\star_1,\ldots, \lambda^\star_n$ be the corresponding eigenvalues. Observe that $q_1,\ldots, q_{d_\omega}$ lie in the range of $A^\star_\omega$ since  for any $i\leq d_\omega$, $q_i= \lambda^\star_i K_\omega^\star q_i = A_\omega^\star(e_i)$ where $e_i = \sum_{j=1}^n \lambda^\star_i (q_i)_j k(x_j,\cdot)$.  Also, it follows directly that $n^{1/2} e_1,\ldots,n^{1/2} e_{d_\omega}$ are an orthonormal basis in $\mathcal{L}^2(S,P_n)$ as $\langle n^{1/2} e_i, n^{1/2} e_j\rangle_2 = A_\omega^\star(e_i)^\top A_\omega^\star(e_j) = q_i^\top q_j$ and $\mathcal{L}^2(S,P_n)$ is $d_\omega$-dimensional. Furthermore, $n^{1/2} e_1,\ldots, n^{1/2} e_{d_\omega}$ are the eigenfunctions of $T_\omega$,
\[
n \langle T_\omega e_i, e_j \rangle_2 = n^{-1} q_i^\top K_\omega^\star q_j = (\lambda_j^\star/n) q_i^\top q_j 
\]
and the corresponding eigenvalues of $T_\omega$ are $\lambda^\star_1/n, \ldots, \lambda_{d_\omega}^\star/n$.

To summarize, we discussed two approaches in this section to get a lower bound on the width of the empirical convex set. The first approach uses a selection of $d_\omega$ sample points and the eigenvalues of the kernel matrix corresponding to these points. It is unclear if there is an efficient way to optimize over this subset selection. The second approach uses instead the full kernel matrix, which sidesteps the problem of selecting sample points, and  leads to a larger eigenvalue but then the constant degrades significantly if $n \gg d_\omega$. There is a third way which `interpolates' between the two approaches. For instance, it might be reasonable to use $2 d_\omega$ many points to help with the subset selection problem while keeping the constant small.

There are multiple hurdles to using these approaches in practice. First off, it is not just the width that needs to be controlled but also how centered $\mean_n$ lies within $C_n$.  Furthermore, the current approach is only applicable in the small sample regime since we need the smallest eigenvalue of the kernel matrix to control the width. This eigenvalue can be computed by applying the power iteration method. The power iteration returns the largest absolute value of a matrix. A standard way to find $\lambda_d$ is the following: apply the power iteration to $K$ to find $\lambda_1$; then move to matrix $B = K - \lambda_1 I$, which is negative definite, and apply the power iteration to get $\lambda_d - \lambda_1$. Each iteration of the power iteration relies on a multiplication of an $n \times n$ matrix with a vector. This makes this method prohibitively costly to apply in the large sample regime. We come back to these issues in Section \ref{sec:algorithms} where we study, among other things, algorithms which split the data into small batches. In such settings it becomes possible to control the width of the empirical convex sets that correspond to the small batches of data.

\subsection{Locating $\mean$}\label{sec:refinement_location}
%
For various convex approximation methods the distance from $\mean$ to the boundary of the convex set characterizes the rate of convergence: the larger the distance the faster the rate of convergence. A crude way to measure the distance is to consider the largest ball that fits within the convex set around $\mean$. Having a closed ball of size $\delta>0$ around $\mean$ in $C$ is equivalent to 
\begin{equation} \label{cri:alt_ball}
	\infd_{\|h\| = 1} \sup_{x \in \X} \langle h, k(x,\cdot) - \mean \rangle \geq \delta, 
\end{equation}
and similarly for affine subspaces.
This can be seen in the following way: clearly when there exists a closed ball around $\mean$ with the stated properties then for any $h, \|h\|=1$, some extreme of the convex set must fulfill \eqref{cri:alt_ball}. On the other hand, $C-\mean$ is equal to the intersection of the closed half-spaces tangent to it \cite[Thm18.8]{ROCK72}. To each of these half-spaces there exists a normal $h \in \cH, h\not =0$, and an $\alpha_h \in \mathbb{R}$ such that $\langle g,h \rangle \leq \alpha_h$ whenever $g$ lies in the half-space. In particular, for any such normal   $\langle g,h \rangle \leq \alpha_h$ whenever $g \in C-\mean$. Without loss of generality we can assume that the normals have norm one and by assuming that \eqref{cri:alt_ball} holds we know that for any such normal $h$, $\alpha_h \geq \delta$. If there would not exist a ball of size $\delta$ around $\mean$ in $C$ then there would be an element $g \not \in C-\mean$, $\|g\| \leq \delta$. But then $\langle h,g \rangle \leq \delta$ for all $h\in \cH, \|h\|=1$, and $g$ would lie in the intersection of the half-spaces and then also in $C - \mean$ due to \cite[Thm18.8]{ROCK72} which cannot be.

In the previous section we quantified the width of the set $C$ in direction $h$, i.e. $\diam_h(C) = \sup_{x \in \X} h(x) - \inf_{x\in\X} h(x)$. The width tells us how large a ball around $\mean$ can be in the ideal case where $\mean$ lies centered within $C$, however, we do not know how centered $\mean$ lies within $C$. Obviously, $\mean$ can lie in the boundary for instance when $\mean = k(x,\cdot)$ and $k(x,\cdot)$ is an extreme of $C$, and assumptions on the distribution of the data are needed to guarantee the existence of a ball around $\mean$. Our aim in this section is to show how natural assumptions on the probability distribution translate to statements of how centered $\mean$ lies. In the following, we are studying two such conditions: (1) a lower bound on the density of the law of $X_1$ together with a Lipschitz condition on $\phi:\X \to \cH$; (2) an assumption on the covariance operator $\tilde{\mathfrak{C}}_c$. We finish this section with a look at the case where the law of $X_1$ does not have full support in $\X$. 

Before looking at these conditions let us add a short comment about the relation between the extremes of $C$ and $\mean$. For the convex set $C = \cch\{ \phi(x) : x \in \X\}$ the extremes of $C$ which are close to $\mean$ are images under $\phi$ of points $x$ which lie close to each other. In detail, consider a kernel function with  $k(x,x) = 1$ for all $x\in \X$. Whenever $\|\phi(x) - \mean\| < \epsilon$ and $\|\phi(y) - \mean \| < \epsilon$ for some $\epsilon > 0$ then $4\epsilon^2 \geq \|\phi(x) - \phi(y)\|^2 =  2(1-k(x,y))$. In other words, 
if there exists an extreme $\phi(x_0)$ of $C$ that lies $\epsilon$ close to $\mean$ then all the extremes of $C$ that are $\epsilon$ close to $\mean$ are contained in 
\[
	\phi[\{y : k(x_0,y) \geq 1- 2 \epsilon^2 \}]. 
\]
Obviously, the case that we have an extreme $\phi(x_0)$ close to $\mean$ is rare since this means that for all functions $h \in \cH, \|h\| \leq 1,$  the expected value $E(h(X))  \approx h(x_0)$ and $\cH$ cannot distinguish between $P$ and a probability measure that puts mass one on $x_0$.

\subsubsection{Assumptions on the density} \label{sec:assumption_density}
In  \cite{BACH12} it was observed that when the probability measure corresponding to $\mean$ has a density on $X$ which is bounded away from $0$ and $\cH$ is finite dimensional then it is at least guaranteed that some open ball exists around $\mean$ in $C$. This result can be strengthened and turned into a quantitative statement by using a simple observation.

Consider first the  Lebesgue integral on $\mathbb{R}$. If we have a (non-atomic) probability measure on $\mathbb{R}$ which has a mean value of $0$ and there exists some measurable set $B$ with $\inf B \geq \epsilon$ and $P(B) > 0$, then there will be probability mass on the negative axis to counter the ``pull'' from $B$ since otherwise 
\[
	0 = \int_{\mathbb{R}} x \,dP  = \int_{[0,\infty)} x \, dP \geq \int_{B} x \,dP  \geq \epsilon P(B) > 0.
 \]   
This argument can also be applied to $\mathfrak{m}$. Consider the set $X=[0,1]$, an RKHS $\cH$ with continuous kernel function $k(x,y)$ and assume that $k(x,\cdot) \in \mathcal{L}^1(P;\cH)$ with Bochner-integral $\mean$ and  the probability measure $P$ has a density function that is bounded away from $0$. For every $y \in X$ with $k(y,\cdot) - \mean \not = 0$ there exists an $x \in X$  such that $\inner{k(y,\cdot) - \mean}{k(x,\cdot) - \mean} < 0$. Otherwise, let $\epsilon = \norm{k(y,\cdot) - \mean}^2/2$ then $B = \{x: \inner{k(y,\cdot) - \mean}{k(x,\cdot) - \mean} > \epsilon \}$ is non-empty as $y \in B$
and contains an open interval $I$ of $X$, with $P(I)>0$. Hence, $P(B) > 0$ and 
because $\|\mean\|^2 = \int_X \inner{\mean}{k(x,\cdot)} \,dP(x)$, 
 \begin{align*}
	 0 &= 
	 \int_{X} \inner{k(y,\cdot) - \mean}{k(x,\cdot) - \mean} dP(x) \geq \int_B \inner{k(y,\cdot) - \mean}{k(x,\cdot) - \mean} dP(x)\\ &\geq \epsilon P(B) >0.
 \end{align*}
This implies that we have on both sides of $\mean$ (with respect to the direction $k(y,\cdot) - \mean$) elements of $\cch \{k(x,\cdot) : x \in X\} = C$. 

To provide lower bounds on the radius of a ball around $\mean$ in $C$ we need more. Ideally, we like to have assumptions on the kernel function and the measure which guarantee the existence of some strictly positive function $\psi:(0,\infty) \rightarrow (0,\infty)$ such that for any $h \in \cH, \norm{h} \leq 1$, $x \in \X$, if  $\inner{h}{k(x,\cdot) - \mean} > 0 $ then  
\[
\inf_{y \in \X} \inner{h}{k(y,\cdot)-\mean} \leq - \psi(\inner{h}{k(x,\cdot) - \mean}).
\]
Under a Lipschitz assumption on the functions in $\cH$ we can provide such a function $\psi$. The Lipschitz assumption we are using is that any $h \in \cH$ fulfills
\begin{equation} \label{def:Lip_for_RKHS}
	\sup_{x \not = x'} \frac{|h(x) - h(x')|}{\|x-x'\|} \leq \|h\| L,
\end{equation}
where $L >0 $ is the Lipschitz-constant.
When the space $\X$ is a compact subset of $\mathbb{R}$ this Lipschitz assumption is often fulfilled. For instance, when a polynomial or Gaussian kernel is used. In fact, whenever we have a well behaved domain $\X$ like $[0,1]^d$, 
$h \! \upharpoonright \! \intr \X   \in C^1(\intr \X)$, $h \in C(\X)$ and $\|D_x h\|_{op} \leq \|h\| L$ (compare to \cite[Cor.4.36]{STEIN08}) then the condition is fulfilled. 

In the following, $\beta_d$ denotes $d$-dimensional Lebesgue measure of the unit ball in $\mathbb{R}^d$ and $\mu_d$ denotes Lebesgue measure.
\begin{proposition} \label{prop:Bauer_inspired}
	Let $\X = [0,1]^d$ and let $\cH$ be an RKHS such that for all $h\in \cH$ and $x,x' \in \X$, 
	$ | h(x) - h(x')| \leq L \|h\| \|x - x'\|$. Furthermore, let $P$ be a probability measure on $\X$ and assume that $P$ has a density $p$ with $\inf_{x\in \X} p(x) \geq c > 0$. Then for any $h \in \cH$, $\norm{h} \leq 1$,  	
	\[
		\max_{y \in X} \inner{-h}{k(y,\cdot) - \mean} \geq \frac{c\gamma^{d+1}}{(d+1) (2L)^d}  \beta_d.
	\]
whenever there exists an $x\in \X$ such that $\inner{h}{k(x,\cdot)-\mean} \geq \gamma > 0 $ and $\gamma/L \leq 1$. 
\end{proposition}
\begin{proof}
	Fix any $h$ in the unit ball of $\cH$ and let $f(x) = \inner{h}{k(x,\cdot) - \mean}$. Let $x^*\in \X$ be a point at which $ f(x^*) = \inner{h}{k(x^*,\cdot)-\mean} \geq \gamma$. The function $f$ is also Lipschitz continuous with Lipschitz-constant $\|h\| L \leq L$  and $f$ is therefore non-negative on the set $A = \{y : \|y-x^*\| \leq \gamma/L, y \in \X\}$. 
	Let $B= \{y: \|y\| \leq \gamma/L, y\in \X\}$  then $P(A) \geq c \mu_d(A) \geq c \mu_d(B)$ because $B$ minimizes the volume of the intersection of $\X$ with a ball of radius $\gamma/L$. Furthermore, $\mu_d(B) = (\gamma/2L)^d \beta_d$; this is the volume of a $d$-dimensional ball of radius $\gamma/L$ scaled by $2^{-d}$. Now, integrating over $A$ and using \cite[265G, 265H]{FREM}  again
\begin{align*}
	\int_A f(x) \,dP(x) &\geq  \int_{A-x^*} c (f(x^*) - L\|x\|) \,dx \geq c  
	\int_{B} f(x^*) - L\|x\| \,dx \\
			    &\geq \frac{c \gamma^{d+1}}{(2L)^d}  \beta_d - cL2^{-d} \frac{d}{d+1} \frac{\gamma^{d+1}}{L^{d+1}} \beta_d 
			    = \frac{c\gamma^{d+1}}{(2L)^d}  \beta_d \left(1 - \frac{d}{d+1}\right).  
\end{align*}
Since $\int_X f(x) \, dP(x)  = 0$ there must be a point $y\in \X$ such that 
\[
	f(y) \leq - 
\frac{c\gamma^{d+1}}{(2L)^d}  \beta_d \left(\frac{1}{d+1}\right).
\]
\end{proof}
Under the conditions of the proposition we can state a lower bound on the size of a ball included in $C$ around $\mean$.
Let $h\in \cH$, $\|h\|=1$, and assume w.l.o.g. that $s := \sup_{x\in \X} \langle h, k(x,\cdot) - \mean \rangle
\geq - \inf_{x\in \X} \langle h,k(x,\cdot) - \mean \rangle =: i$. Then $s \geq (1/2) \diam_h(C)$ and 
\begin{align} \label{eq:width_lower_bnd}
	i \geq \frac{c s^{d+1}}{(d+1) (2L)^d} \beta_d \geq \frac{c ((1/2)\diam_h(C))^{d+1}}{(d+1) (2L)^d} \beta_d.  
\end{align}
If we have a lower bound $b$ on $\diam_h(C)$ for all such $h$ then we can conclude that there exists a ball of radius 
\[
	\delta =  \min \{(b/2) ,  \frac{c (b/2)^{d+1}\beta_d}{(d+1) (2L)^d} \} 
\]
around $\mean$ in $C$.

\subsubsection{Assumptions on the covariance operator} \label{sec:cov_centered}
Let us start with a useful relationship between $\bm 1$ and $\mean$ whenever $\bm 1$ lies in the RKHS. For any measure $P$, with corresponding element $\mean$ we have  that $\langle \bm 1, \mean \rangle  = 1$. Also, for any $x\in \X$, $\langle k(x,\cdot), \bm 1 \rangle = 1$ and $C$ lies within the closed affine subspace $\{h \in \cH : \langle h,\bm 1 \rangle =1 \}$, where closure follows from $h \mapsto \langle h, \bm 1\rangle$ being a continuous function and $\{1\}$ being closed. Also, since $1 = \langle \mean, \bm 1 \rangle \leq \|\bm 1\| \|\mean\|$ it has to hold that $\|\mean\| \geq 1/\|\bm 1\|$. An upper bound on $\| \bm  1 \|$ is therefore giving us a lower bound on $\| \mean\|$. For instance, we can get a lower bound on $\| \mean \|$ by inspecting the kernel function $k$ of the RKHS in the sense that
\begin{equation}
\|\mean\| \geq 1/ \inf\{ c: c^2 k \succeq \bm 1 \otimes \bm 1\}.
\end{equation}
One might be tempted to move to $\cH^+$ whenever $\bm 1$ does not lie in the RKHS $\cH$; recall that $\cH^+$ has the kernel $k + \bm 1 \otimes \bm 1$. Since $\cH$ can be regarded as a subspace of $\cH^+$ we have that $\|\mean\| = \|\mean\|_+$; however, only for $h\in \cH$ do we have that $\int h \, dP = \langle \mean, h \rangle_+$ since $\bm 1 \in \cH^\perp$ by construction. But there is then an element $\mean^+ \in \cH_+$ for which $P_\cH  \mean^+ = \mean$ and $(I-P_\cH) \mean^+ = \langle \bm 1, \mean^+ \rangle = 1$, where $P_\cH$ is the orthogonal projection onto the subspace $\cH$ and $I$ the identity operator. Now, $1 =  \langle \bm 1,\mean^+ \rangle_+ \leq \|\mean^+\|_+ = 1 + \|\mean\|$ and we only learn from this the trivial fact that $0 \leq \|\mean\|$.

An alternative approach gives us more insight. Whenever there is a function $\tilde{\bm 1} \in \cH$ such that $\|\bm 1 - \tilde{\bm 1}\|_\infty \leq \beta < 1$ then $ 1 - \beta \leq \langle \mean, \tilde{\bm 1} \rangle \leq 1+ \beta$ and 
\begin{equation}
\| \mean \| \geq (1-\beta)/\|\tilde{\bm 1}\| = (1-\beta)/ \inf\{ c: c^2 k \succeq \tilde{\bm 1} \otimes \tilde{\bm 1}\}.
\end{equation}
	 
In the following we will make use of the covariance operator $\tilde{\mathfrak{C}}$ (see Section \ref{sec:prelim}) to determine the location of $\mean$. Before exploring the relation between the covariance operator and the location of $\mean$ we note the following adaptation of the above discussion: If $\bm 1 \in \cH$ then  
$\langle \Cov \bm 1, \bm 1 \rangle = 1$ and whenever $\tilde{\bm 1} \in \cH$ fulfills $\|\bm 1 - \tilde{\bm 1}\|_\infty \leq \beta < 1$ then 
\begin{equation} \label{lbnd:cov_op}
\langle \Cov \bm 1, \bm 1 \rangle \geq (1-\beta)^2 \|\tilde{\bm 1}\| =  (1-\beta)^2/ \inf\{ c: c^2 k \succeq \tilde{\bm 1} \otimes \tilde{\bm 1}\}.
\end{equation}

Coming now to the problem of locating $\mean$ within $C$ we can take note of the following fundamental relationship.
Whenever $\cH$ is $d$-dimensional, $\|k\|_\infty <\infty$ and $\tilde{\mathfrak{C}}$ has an eigen-decomposition with smallest eigenvalue $\bar \lambda_d$ then for any $h \in \cH$, $\|h\|=1$, with $\langle h, \mean \rangle  = 0$, it follows that $\int h(x) \, dP = 0$ and 
\[
\int (\langle h, k(x,\cdot) - \mean \rangle)^2 \, dP = \int h^2(x) \, dP \geq \bar \lambda_d. 
\]
Since $\sup_{x\in\X} |h(x)| \leq \|k\|_\infty^{1/2}$ and $\int h(x) \, dP = 0$, a short calculation shows that 
\[
r(\mean,h) \geq \frac{\bar{\lambda}_d}{\|k\|_\infty^{1/2}}, 
\]
where with $h \in \cH$,
\[
	r(\mean,h) := 	(\sup_{x\in \X} \langle h, k(x,\cdot) - \mean \rangle) \wedge (- \inf_{x\in \X} \langle h, k(x,\cdot) - \mean \rangle). 
\]

In detail, it is sufficient to consider the case of a discrete measure where with probability $p$ the function $\langle h, k(x,\cdot) - \mean \rangle$ attains value $a$  and with probability $1-p$ attains value $b$, with $a < 0 < b$, and $a^2 p + b^2 (1-p) = \bar \lambda_d$. The condition $\langle h,\mean \rangle = 0$ then implies that $a p + b(1-p)  = 0$. For a particular value of $a$ we get that $p=b/(b-a)$ and $b = \bar \lambda_d/(-a)$. The value $b$ is minimized by maximizing $-a$ but $-a = \langle h,k(x,\cdot) \rangle$ for some $x \in \X$ and
$-a = |a| \leq \|k\|_\infty^{1/2}$. By symmetry we get that $(- a) \wedge b \geq \bar \lambda_d/\|k\|_\infty^{1/2}$.

The remaining direction we have to take care of is $h^* = \mean/\|\mean\|$
The distance of $\mean$ to the boundary in direction $h^*$ can be lower bounded away from zero when the smallest eigenvalue of the covariance operator is sufficiently large since 
\[ 
\int (\langle \mean ,k(x,\cdot) - \mean \rangle)^2 \,dP  =  \int (\langle \mean, k(x,\cdot) 
\rangle)^2 \,dP - \|\mean\|^4  = \langle \tilde{\mathfrak{C}} \mean,\mean \rangle - \|\mean\|^4 
\]
and 
\[
\int (\langle h^* ,k(x,\cdot) - \mean \rangle)^2 \,dP  = \langle \tilde{\mathfrak{C}} h^*,h^* \rangle  - \|\mean\|^2  \geq \bar \lambda_d - \|\mean\|^2. 
\]
Also, $|\langle h^*,k(x,\cdot) - \mean \rangle| 
\leq \|k(x,\cdot)\| + \|\mean\| \leq 2 \|k\|_\infty^{1/2}$ and 
\[
r(\mean,h^*) \geq \frac{\bar{\lambda}_d- \|\mean\|^2}{2\|k\|_\infty^{1/2}}. 
\]
For this approach to yield a useful bound $\bar \lambda_d$ has to be strictly greater than $\|\mean\|^2$.
However, $\bar \lambda_d$ can even be smaller than $\|\mean\|^2$. A better bound can  be gained by using an eigen-decomposition of $\tilde{\mathfrak{C}}_c = \tilde{\mathfrak{C}} - \mean \widehat{\otimes} \mean$. In the following let $\bar \lambda_1, \bar \lambda_2, \ldots$ be the eigenvalues of $ \tilde{\mathfrak{C}}_c$ then by the same argument as for $\mathfrak{C}$, whenever $h \in \cH$, $\|h\|=1$, is such that $\langle h, \mean \rangle = 0$, it follows that $\int (\langle h, k(x,\cdot) - \mean \rangle)^2 dP = \langle \tilde{\mathfrak{C}}_c h, h \rangle \geq \bar \lambda_d$, where $\bar \lambda_d$ is the smallest eigenvalue of $ \tilde{\mathfrak{C}}_c$, and $r(\mean,h) \geq \bar \lambda_d/\|k\|_\infty^{1/2}$. 
For $h^* = \mean/\|\mean\|$ we get now that 
\[
\int (\langle h^* ,k(x,\cdot) - \mean \rangle)^2 \,dP  = 
\langle \tilde{\mathfrak{C}}_c h^*,h^* \rangle   \geq \bar \lambda_d 
\]
and
$r(\mean,h^*) \geq \bar \lambda_d/2\|k\|_\infty^{1/2}$. Also notice that when $\bm 1 \in \cH$ then $\langle \tilde{\mathfrak{C}}_c \bm 1,\bm 1 \rangle = 0$ and $\bar \lambda_d = 0$.
We summarize these finding in the following proposition.
\begin{prop} \label{prop:cov_bound_radius} 
Let $(\X,\mathcal{A},P)$ be some probability space with measurable kernel function $k$ defined on it and such that the corresponding RKHS  $\cH$ has dimension $d<\infty$. Furthermore, assume that $\|k\|_\infty < \infty$ and that the centered covariance operator $\tilde{\mathfrak{C}}_c$ has an eigen-decomposition with smallest eigenvalue $\bar \lambda_d > 0$.  Then $\bm 1 \not \in \cH$ and for any $h\in \cH, \|h\| =1$,
\[
r(\mean,h) \geq \frac{\bar \lambda_d}{2\|k\|_\infty^{1/2}}.
\]
\end{prop}
The \textit{dimension dependence} is in this result not as obvious as in Proposition \ref{prop:Bauer_inspired}, but observe that when $\mathcal{X}$ is the $d$-dimensional unit sphere, $d =2l +1$ for some $l \in \mathbb{N}$, $P$ is the Lebesgue measure restricted to $\X$ and normalized, $k(x,y) = \langle x,y \rangle_{\R^d}$ for any $x,y \in \X$ and $\cH$ is the corresponding RKHS which is of dimension $d$, then any $h \in \cH, \|h\|=1,$
is of the form $\langle x,\cdot \rangle_{\R^d}$ for some $x\in \X,\|x\|_{\R^d}=1$ and for such an $h$
\begin{align*}
\langle \tilde{\mathfrak{C}} h,h \rangle &= \beta_d^{-1} \int h^2(y) \,dP(y)	= \beta_d^{-1} \int \langle x,y\rangle^2 \, dP(y) \\
					 &= \beta_d^{-1} \int_{-1}^1 \langle x,\tilde yx\rangle^2 \mu_{d-1}(B_{d-1}(\sqrt{1-\tilde y^2})) \,d \mu(\tilde y) \\ 
&= \beta_{d-1}\beta_d^{-1}  \int_{-1}^1 \tilde y^2 (1-\tilde y^2)^{l} \,d\mu(\tilde y)
=\frac{2^ll! \beta_{d-1}\beta_d^{-1}} {\prod_{i=1}^l (2i +1)}\int_{-1}^1 \tilde y^{2(l+1)} \,d\mu(\tilde y) \\
&= \frac{2^ll! \beta_{d-1}\beta_d^{-1}}{(\prod_{i=1}^l (2i +1))( l+ 3/2)} = \frac{(2l+1)!}{l! 2^{l+1}(\prod_{i=1}^l (2i +1))( l+ 3/2)}
= \frac{1}{d+ 2}, 
\end{align*}
where $\mu_{d-1}$ denotes here the $d-1$-dimensional Lebesgue measure and $\beta_{d-1}$ the Lebesgue measure of the $d-1$-dimensional unit sphere. Hence, the eigenvalues of $\tilde{\mathfrak{C}}$ shrink to zero as the dimension $d$ increases.

If $\bm 1 \in \cH$ then we get a similar result with $\bar \lambda_d$ being replaced by $\bar \lambda_{d-1}$.
\begin{corollary} \label{cor:cov_constant_inH}
Let $(\X,\mathcal{A},P)$ be some probability space with measurable kernel function $k$ defined on it and such that the corresponding RKHS  $\cH$ has dimension $d<\infty$ and $\bm 1 \in \cH$. Furthermore, assume that $\|k\|_\infty < \infty$ and that the centered covariance operator $\tilde{\mathfrak{C}}_c$ has an eigen-decomposition with eigenvalue $\bar \lambda_{d-1} > 0$.  Then for any $h\in \cH, \|h\| =1, \langle h, \bm 1 \rangle = 0$,
\[
r(\mean,h) \geq \frac{\bar \lambda_{d-1}}{2\|k\|_\infty^{1/2}}.
\]
\end{corollary}
\begin{proof}
First note that $\langle \tilde{\mathfrak{C}}_c \bm 1, \bm 1 \rangle = 0$ and for any $h \in \cH$ s.t.  $\langle h, \bm 1 \rangle =0 $, $\langle \tilde{\mathfrak{C}}_c \bm 1, h \rangle =\langle \tilde{\mathfrak{C}}_c h, \bm 1 \rangle = E(h) - E(h) = 0$. In particular, $\tilde{\mathfrak{C}}_c \bm 1 = 0$ and for any $h \in \cH, \|h\| =1, \langle h, \bm 1 \rangle = 0$, 
$\langle \tilde{\mathfrak{C}}_c h, h \rangle \geq \bar \lambda_{d-1}$. Now, for $h\in \cH, \|h\|=1, \langle h, \bm 1\rangle =0 $,
\begin{align*}
&\int (\langle h, k(x,\cdot) - \mean \rangle)^2 \, dP = \langle \tilde{\mathfrak{C}}_c h, h \rangle \geq \bar \lambda_{d-1}
\end{align*}
and the lower bound follows by the same argument as in Proposition \ref{prop:cov_bound_radius}. 

\end{proof}

\subsubsection{Data attaining values in a subset} \label{sec:data_subset_covariance_arg}
Unless $h \in \cH$ is constant on the support of $P$  it holds that $E( (h - E(h))^2)  > 0 $ and by the above arguments it follows that $r(\mean,h) > 0$.   
But $\mean$ can certainly lie in the boundary, for example, when $\mean = k(x,\cdot)$ and $k(x,\cdot)$ is an extreme of $C$. Therefore, there must be some direction $h$ in which $r(\mean,h) = 0$. The point is that when $\mean$ is an extreme or lies in a face of the convex set $C$ and this face has extremes $\{ k(x, \cdot) : x \in S \subsetneq \X \}$ then $P( X \in \X \setminus S)= 0$ has to hold. Furthermore, there then exists a normal $h^*$ to this face and $h^*$ is constant on the support of $P$ which implies that $\langle \tilde{\mathfrak{C}}_c h^*, h^* \rangle = 0$. 

This observation suggests that $\mean$ will either be an extreme or there will be a ball in an affine subset of $\cH$ around $\mean$ within a face of the convex set. Furthermore, we can hope that this affine subspace is directly related to the non-zero eigenvalues of $\tilde{\mathfrak{C}}_c$ and that these eigenvalues characterize a lower bound on the width of this ball. Alternatively, it is natural to consider the space $\cH_S = \{ h\!\!\upharpoonright\!\!S : h \in \cH \}$ where $S$ is the support of $P$ \cite[Def.411N]{FREM}. The space $\cH_S$ is again an RKHS with kernel $k_S = k\!\!\upharpoonright\!\! S \times S$ and norm $\|h\|_{k_S} = \inf\{\|u\| : u \!\upharpoonright\!S = h, u \in \cH\}$ \cite[Cor.5.8]{PAUL16}. In the proposition below we show that the covariance operator $\tilde{\mathfrak{C}}_c^S$ corresponding to $P$ and $\cH_S$ characterizes the ball around $\mean$ within the affine subspace spanned by $C$.

\begin{proposition}\label{prop:data_dep_ball}
Let $(\X,\mathcal{A},P)$ be a probability space with $P$ being a topological $\tau$-additive probability measure which has support $S$ and let $k$ be
a continuous kernel function $k$ defined on $\X$ such that the corresponding RKHS  $\cH$ has dimension $d<\infty$. Furthermore, assume that $\|k\|_\infty < \infty$, and that the centered covariance operator $\tilde{\mathfrak{C}}_c$ has an eigen-decomposition with eigenvalue such that (i)
$\bar \lambda_{l} > 0 = \bar \lambda_{l+1}$ for some $l< d$ or (ii) $ \bar \lambda_d > 0$. It follows that $\cH_S$ has dimension $l+1$ and $\tilde{\mathfrak{C}}^S_c$ has eigenvalues
$\bar \lambda^S_i = \bar \lambda_i$ for all $i\leq l$ and $\bar\lambda_{l+1}^S = 0$ under (i), and $\cH_S$ has dimension $d$ and the eigenvalues of $\tilde{\mathfrak{C}}^S_c$ are the same eigenvalues as of $\tilde{\mathfrak{C}}_c$    under (ii). Furthermore, under condition (ii) there exists a closed ball $B$ centered at $\mean$ with radius $\bar \lambda_d/2 \|k\|_\infty^{1/2}$  inside $C$. If $\lambda_1 = 0$ and there are no two points $x,y \in \X$ such that $k(x,x) = k(y,y) =k(x,y)$ then $S$ consists of a single element and $P = k(x,\cdot)$ for some $x\in \X$. If $\lambda_1 > 0$ then $\mean$ lies in the relative interior of  $F = \cch\{k(x,\cdot) : x \in S \}$. In particular, under condition (i) there exists a closed ball $B$ centered at $\mean$ such that $B\cap \aff F \subset F$ and $B$ has radius  
$\bar \lambda_l / 2 \|k\|_\infty^{1/2}$.
\end{proposition}
\begin{proof}
\textbf{(a)} When $\bar\lambda_d  >0$ it follows that $E(h^2(X))  > E(h(X))^2$ for all $h \in \cH, h \not = 0$. In particular, let $e_1,\ldots, e_d \in \cH$ be linearly independent  and fix  any $a_1,\ldots, a_d \in \mathbb{R}$ such that  $a_1, \ldots, a_d$ are not simultaneously equal to zero. Consider $e_1\!\!\upharpoonright\!\!S, \ldots, e_d\!\!\upharpoonright\!\!S$ then
\[
E\bigl( \bigl( \sum_{i=1}^d a_i e_i\!\!\upharpoonright\!\!S\bigr)^2 \bigr) =  E\bigl( \bigl( \sum_{i=1}^d a_i e_i \bigr)^2 \bigr) > 0.
\]
Since this holds for all such $a_1,\ldots, a_d$ it follows that $e_1\!\!\upharpoonright\!\!S, \ldots, e_d\!\!\upharpoonright\!\!S$ are linearly independent and $\cH_S$ is $d$-dimensional. Similarly, when $\lambda_l > 0$ and $\lambda_{l+1} = 0$ it follows that there is an $h_{\bm 1} \in \cH_S$ which is almost surely equal to $\bm 1 \!\!\upharpoonright\!\!S$ and $l$ linearly independent functions $e_1\!\!\upharpoonright\!\!S,\ldots, e_l\!\!\upharpoonright\!\!S \in \cH$ which are also linearly independent of $h_{\bm 1}$ (for any non-trivial linear combination of $e_1\!\!\upharpoonright\!\!S,\ldots, e_l\!\!\upharpoonright\!\!S$ it follows  that the second moment is strictly larger than the squared expected value and therefore these linear combinations are not equal to a constant function). In fact, $h_{\bm 1}$ is equal to $\bm 1\!\!\upharpoonright\!\!S$ since by assumption the kernel function is continuous: Assume that $h_{\bm 1}$ is not equal to $\bm 1$, take a $x\in S$ such that $h_{\bm 1}(x) \not = 1$ and let $\epsilon = | h_{\bm 1}(x) - 1|$. Take a function $h\in \cH$ such that $h\!\!\upharpoonright\!\!S = h_{\bm 1}$ and $\|h\| \leq \|h_{\bm 1}\|_{\cH_S} + \epsilon/4$. The set $A = h^{-1}[\{y: |y - h_{\bm 1}(x)| < \epsilon/4  \}]$ is open and has non-empty intersection with $S$. Also, $h_{\bm 1}$ is different from $\bm 1$ on all of $A$. Due to \cite[411N]{FREM} it follows that $P(A)>0$ and $\bm 1$ is not almost surely equal to $h_{\bm 1}$ which is impossible. By the same argument it follows that the other  eigenfunctions of $\tilde{\mathfrak{C}}_c$ with zero eigenvalues are constant on $S$ and therefore lie in the span of $h_{\bm 1}$ and the dimension of $\cH_S$ is $l+1$. 

\textbf{(b)} In case (ii), let $e_1,\ldots, e_d$ be orthonormal in $\cH$ then, due to (a), the functions $e_1\!\!\upharpoonright\!\!S,\ldots, e_d\!\!\upharpoonright\!\!S$ are linearly independent. Also $\|e_i\!\!\upharpoonright\!\!S\|_{\cH_S} = \| e_i \| $ for all $i\leq d$, because for any given $i\leq d$,  $e_i\!\!\upharpoonright\!\!S$ does not lie in the linear subspace spanned by $\{e_j\!\!\upharpoonright\!\!S\}_{j\not=i}$. Similarly, for any $i,j \leq d$,  $\|e_i\!\!\upharpoonright\!\!S + e_j\!\!\upharpoonright\!\!S\|_{\cH_S} = \|e_i + e_j\|$ and  $\|e_i\!\!\upharpoonright\!\!S + e_j\!\!\upharpoonright\!\!S\|^2_{\cH_S} = \|e_i\!\!\upharpoonright\!\!S\|_{\cH_S}^2 + \|e_j\!\!\upharpoonright\!\!S\|^2_{\cH_S}$. Hence, 
$\{e_i\!\!\upharpoonright\!\!S\}_{i\leq d}$ is an orthonormal basis of $\cH_S$. In particular, when $e_1,\ldots, e_d$ are the eigenfunctions of $\tilde{\mathfrak{C}}_c$ then 
$e_1\!\!\upharpoonright\!\!S,\ldots, e_d\!\!\upharpoonright\!\!S$ are the eigenfunctions of $\tilde{\mathfrak{C}}^S_c$ since for $i\not = j$,
\[
\langle \tilde{\mathfrak{C}}^S_c  e_i\!\!\upharpoonright\!\!S, e_j\!\!\upharpoonright\!\!S \rangle_{\cH_S} = \langle \tilde{\mathfrak{C}}_c  e_i, e_j \rangle = 0 
\]
and for any $i\leq d$, $\langle \tilde{\mathfrak{C}}^S_c  e_i, e_i \rangle= \bar \lambda_i$. 

By the same argument it follows that in case (i) that $e_1\!\!\upharpoonright\!\!S,\ldots,e_l\!\!\upharpoonright\!\!S, \bm 1\!\!\upharpoonright\!\!S / \|\bm 1\!\!\upharpoonright\!\!S\|_{\cH_S}$ is an orthonormal basis of $\cH_S$. Hence, if $e_1,\ldots, e_l$ are the first $l$ eigenfunctions of $\tilde{\mathfrak{C}}_c$ then  $e_1\!\!\upharpoonright\!\!S,\ldots, e_l\!\!\upharpoonright\!\!S$ are the eigenfunctions of $\tilde{\mathfrak{C}}^S_c$ and the eigenvalues match.

\textbf{(c)} When $\bar\lambda_1 = 0$ it follows that $\cH_S= \spn\{\bm 1\!\!\upharpoonright\!\!S\}$ and $\langle \mean_S,c \bm 1 \rangle_{\cH_S} = c = \langle \mean_{S,n}, c \bm 1\rangle_{\cH_S}$ for all $c\in \R$ where $\mean_S$ and $\mean_{S,n}$ are just $\mean$ and $\mean_n$ when the kernel is restricted to $S$. Also,  every function $h\in \cH$ is constant on $S$ and $\langle \mean, h \rangle = \langle \mean_n, h \rangle = \langle h, k(x,\cdot) \rangle$ for all $x\in S$. In particular, if $S$ does not consist of a single element it follows that $k(x,y) = k(x,x) = k(y,y)$ for all $x,y \in S$. Reversing this statement leads to the claim made in the proposition.
 
\textbf{(d)} Let $U = \aff F$ be the affine subspace spanned by $k(x,\cdot), x \in S$. The element $\mean$ lies within $U$ since for any $h \in \cH$ for which $\langle k(x,\cdot), h \rangle = c$ for all $x\in S$ and for some $c\in \R$ it follows that $\langle \mean,h \rangle = \int_S \langle k(x,\cdot), h \rangle \, dP = c$. In other words, if $\mean$ would not lie in  $U$ then there would be an $h, \|h\|> 0,$ that stands perpendicular on $U$ and such that $\mean = h + \argmin_{g \in U} \|g- \mean\|$, and $ \langle \mean ,h \rangle \not = \langle k(x,\cdot), h \rangle$ for all $x\in S$.

\textbf{(e)} Under condition (ii) the constant functions are not in $\cH$ and the conditions of Proposition \ref{prop:cov_bound_radius} are fulfilled and the existence of the ball with the specified radius follows directly.

Under condition (i) the constant functions are contained in $\cH_S$ and there exists a function $h_0 \in \cH$ such that $h_0(x)= 1$ for all $x \in S$. Consider $\inf \{\|h \| : h(x) =1  \text{ for all } x \in S \} \leq \|h_0\|$. Since $h \mapsto \|h\|$ is continuous and $\{ h : \|h \| \leq \|h_0\|\}$ is compact it follows that the infimum is attained at some $h^* \in \cH$. Also, $\|\bm 1\!\!\upharpoonright\!\!S\|_{\cH_S} = \inf \{\|h\| : h\!\!\upharpoonright\!\!S = \bm 1\!\!\upharpoonright\!\!S \} = \|h^*\|$ and $\bm 1\!\!\upharpoonright\!\!S = h^*\!\!\upharpoonright\!\!S$.

Let $U=  \aff \{ k(x,\cdot) : x\in S\} \subset \cH$  and  $V = U-U$ be the subspace parallel to $U$. For any $h \in U$ there exists $m \in \mathbb{N}$, $\lambda_1,\ldots, \lambda_m \in \R$,  $\lambda_1 + \ldots + \lambda_m = 1$ and $x_1,\ldots, x_m \in S$ such that  $h = \sum_{i=1}^m \lambda_i k(x_i,\cdot)$. In particular, if $g \in \cH$ is constant on $S$ and attains value $c$, then $\langle g,h \rangle = c \sum_{i=1}^m \lambda_i = c$. This implies that $g$ is orthogonal to $V$ since $\langle g, h_1 - h_2 \rangle = 0 $ for all $h_1,h_2 \in U$. 
Due to assumption (i) there are $d-l$ eigenfunctions $e_{l+1},\ldots, e_d$ of $\tilde{\mathfrak{C}}_c$ which are constant on $S$, that is $e_{l+1},\ldots, e_d$ are orthonormal and each of them is orthogonal to $V$. Also, any function $h$ that is orthogonal to $V$ has to be constant on $S$ since $h(x)  = \langle k(x,\cdot), h \rangle = \langle k(y,\cdot), h \rangle =h(y)$ for all $x,y \in S$. Since the eigenfunctions $e_1, \ldots, e_l$ have corresponding eigenvalues which are strictly greater than zero it follows that $e_1,\ldots, e_l$ cannot be constant and $V$ has dimension $l$. Also, note that $U$ cannot be equal to $V$, or better, $U$ cannot be a subspace but only an affine subspace: assume otherwise then $0 \in U$ and 
if $h$ is constant on $S$, attaining value $c \not = 0$, then $0 = \langle h, k(x,\cdot) \rangle = h(x) = c $ for all $x \in S$. In other words there cannot be functions that are constants on $S$ in $\cH$, but we know already that there are functions which are constant on $S$ in $\cH$.   

\textbf{(f)} We claim that $W = \spn( V \cup \{h^*\})$, when  equipped with the inner product of $\cH$, is isometric isomorphic to $\cH_S$, $U \subset W$ and $(\spn \{h^*\}) \cap U \not = \emptyset$. We start with the latter claim. Since $U$ is not a subspace it follows that the orthogonal projection of $0$ onto $U$ is not $0$ itself. In detail, take any $x\in S$ and let $P_V$ the projection onto the subspace $V$ then 
the orthogonal projection onto $U$ is the operator defined by $P_U h = k(x,\cdot) + P_V(h - k(x,\cdot))$. Now $P_U 0 = k(x,\cdot) - P_V k(x,\cdot) \not = 0$ and,  $P_U 0$ is orthogonal to $V$. But that means the $P_U 0$ is constant on $S$ and lies in the span of $e_{l+1},\ldots, e_d$. In particular, there is function that is constant on $S$ which lies in $U$.
Notice that for any function $h$ which is constant on $S$ there exists an element in $U$ that lies in the span of $h$ if $\|h\| = \sup_{g\in U} |\langle h, g \rangle|/\|g\|$.  
Recall that a function $g \in U$ can be written as $g = \sum_{i=1}^m \lambda_i k(x_i,\cdot)$ for some $m\in \mathbb{N}$, $x_1,\ldots, x_m \in S$ and such that $\lambda_1 + \ldots + \lambda_d = 1$. Hence, $\langle h, g \rangle = c$ if $h$ attains value $c$ and $h \in U$ if $\|h\| = |c| \sup_{g \in U} 1/\|g\|$. Consider now a function $h$ that is constant on $S$ and attains value $1$ then $\|h\| = \sup_{g \in U} 1/\|g\|$. For $0 < \epsilon  <  \|h^*\| $ take $g \in U$ such that $\|h\| \leq \epsilon + 1/\|g\|$, then
\[
 \frac{|\langle h^*, g \rangle|}{\|h^*\|} \leq \|g\| \leq \frac{|\langle h, g \rangle|}{\|h\| - \epsilon} = \frac{|\langle h^*, g \rangle|}{\|h\| - \epsilon} \leq  \frac{|\langle h^*, g \rangle|}{\|h^*\| - \epsilon}. 
\]
In other words,
\[
 \frac{|\langle h^*, g \rangle|}{\|g\|} \leq \|h^*\| \leq   \frac{|\langle h^*, g \rangle|}{\|g\|} + \epsilon 
\]
and $\sup_{g \in U} |\langle g, h^* \rangle|/\|h^*\| = \|h^*\|$ which implies that $(\spn \{h^*\}) \cap U \not = \emptyset$ and $U\subset W$

Let $\psi:W \to \cH_S$ be the function that associates with $h\in W$ the function $h\!\!\upharpoonright\!\!S \in \cH_S$. The function $\psi$ is linear since $(a f+g)\!\!\upharpoonright\!\!S = a (f\!\!\upharpoonright\!\!S) +g\!\!\upharpoonright\!\!S$ for all $f,g \in W$, $a\in \R$. We have seen already that $\psi(h^*) = \bm 1\!\!\upharpoonright\!\!S$ and $\|\psi(h^*)\|_{\cH_S} = \|h^*\|$. 
Also any $h \in U$ lies in the span of $\{ k(x,\cdot) : x\in S\}$ and, since $\cH$ is finite dimensional, $h$ can be written as $\sum_{i=1}^d a_i k(x_i,\cdot)$ for some $a_1,\ldots, a_d \in \R$ and $x_1,\ldots, x_d \in S$. Hence, any $g,h \in U \subset W$ can be written as $h =  \sum_{i=1}^d a_i k(x_i,\cdot)$, $g =  \sum_{i=1}^d b_i k(y_i,\cdot)$, with some $a_i,b_i \in \R, x_i, y_i \in S$ for all $i\leq d$, and 
\[
\langle g,h \rangle = \sum_{i,j}^d a_i b_j k(x_i,y_j) = \sum_{i,j}^d a_i b_j k_S(x_i,y_j) = \langle \psi(g), \psi(h) \rangle_{\cH_S}. 
\]
But this implies that for any $g,h \in V$, that is in the subspace parallel to $U$, there exists $\tilde g, \tilde h \in U$ and $x\in S$ such that 
$g = \tilde g - k(x,\cdot), h = \tilde h - k(x,\cdot)$ and 
\begin{align*}
	\langle g,h \rangle &= \langle \tilde g - k(x,\cdot), \tilde h - k(x,\cdot) \rangle  \\
			    &= \langle \psi(\tilde g), \psi(\tilde h) \rangle_{\cH_S} - \langle  \psi(\tilde g), k_S(x,\cdot) \rangle_{\cH_S} - \langle  \psi(\tilde h), k_S(x,\cdot) \rangle_{\cH_S} + k_S(x,x) \\
			    &= \langle \psi(\tilde g - k(x,\cdot)), \psi(\tilde h - k(x,\cdot)) \rangle_{\cH_S}\\
			    &= \langle \psi(g ), \psi( h ) \rangle_{\cH_S}, 
\end{align*}
since $\psi(k(x,\cdot)) = k_S(x,\cdot)$ for any $x\in S$. Finally, for any $g \in V$, write $g = \tilde g - k(x,\cdot)$ for some $x\in S$, $\tilde g \in U$ then 
\[
\langle g, h^* \rangle =  \langle \tilde g, h^* \rangle - \langle k(x,\cdot), h^* \rangle = 0  =  \langle \psi(\tilde g) - k_S(x,\cdot), \bm 1\!\!\upharpoonright\!\!S\rangle_{\cH_S} = \langle \psi(g),\psi(h^*) \rangle_{\cH_S},
\]
since $h^*$ is constant on $S$, $\psi(\tilde g) = \sum_{i=1}^m a_i  \psi(k(x_i,\cdot))$ for some $m\in \mathbb{N}$, $x_1,\ldots, x_m \in S$ and $a_1,\ldots, a_m$ such that $\sum_{i=1}^m a_i = 1$, and $\langle \psi(\tilde g), \bm 1\!\!\upharpoonright\!\!S \rangle_{\cH_S} = 1$. Hence, $\psi$ is an isometry and since $\cH_S$ and $\spn(V \cup \{h^*\})$ have the same it follows that they are isometric isomorphic. The existence of the ball around $\mean$ of the specified radius follows now directly from Corollary   \ref{cor:cov_constant_inH}.
\end{proof}

Notice that $\cH_S$ does not have to be equal to $\cH$ when $\bar \lambda_d > 0$. For instance, when $\cH$ consists of the quadratic functions on $[-1,1]$ and has therefore dimension $1$. If the measure $P$ is discrete with $P(\{-1\}) = P(\{1\}) = 1/2$ then $\bar \lambda_1 > 0$ but $S = \{ -1,1\}$ and $\cH_S \not = \cH$. 

\subsection{Locating $\mean_n$ within the empirical convex set} \label{sec:bringing_it_together}
We are now combining the various results we have derived. Section \ref{sec:refinement_location} allows us to refer the size of a ball within $C$ around $\mean$ back to the question of the width of $C$. In Section \ref{sec:diameter} we derived various ways to lower bound the width of $C$. We also know that $C_n$ converges to $C$. Section \ref{sec:emp_convex_sets} contains various results on that. These results combine Rademacher or VC bounds with lower bounds on$P \psi_\gamma(\langle u,\phi(\cdot) - \mean \rangle -c$ and  $P f_{u,c}$. These lower bounds are closely related to the bounds in Section  \ref{sec:refinement_location} since in both settings we need to measure how much probability mass lies in various directions behind some threshold. To get now high probability bounds for the existence of ball of a certain size around $\mean_n$ within $C_n$ we also need to control the convergence of $\mean_n$ to $\mean$. But that is easy to do with another VC or Rademacher argument. The following two theorems combine these results under some natural conditions. The first result  applies when $\X =[0,1]^l$, $\cH$ is finite dimensional, that functions in $\cH$ are Lipschitz continuous and that we have lower bound on the density of the law of $X_1,\ldots,X_n$ on $\X$. We also assume that $\bm 1 \in \cH$ but the result can easily be adapted to the case of $\bm 1 \not \in \cH$. 
\begin{theorem} \label{thm:ball_in_empirical}
Let $\X = [0,1]^l, l \geq 1,$ and $k$ a continuous kernel function on $\X$ such that the corresponding RKHS $\cH$ is $d$-dimensional, $1 \leq d < \infty$, functions $h \in \cH$ are Lipschitz continuous in the sense of \eqref{def:Lip_for_RKHS} with Lipschitz constant $L>0$, and $\bm 1 \in \cH$. Furthermore, let $X_1,\ldots, X_n$ be i.i.d. random variables defined on some probability space and such that the law $P$ of $X_1$ has a density $p$ on $\X$ and $\inf_{x\in\X} p(x) \geq c >0$ for some constant $c$. Mercer's theorem applies to $k$. Let $\tilde \lambda_d$ be the smallest eigenvalue of the Mercer decomposition. There exists a ball of radius
\[
\delta = 2 \tilde \lambda_d^{1/2} \wedge  \frac{2 c \tilde \lambda_d^{(l+1)/2}\beta_l}{(l+1) L^l}  
\]
around $\mean$ in $C$ in the affine subspace spanned by $C$. Furthermore, for any $q \in (0,1)$ and whenever 
\[
	n > \left(\frac{\sqrt{2\log(6/q)} + 96 \|k\|_\infty^{1/2}/\delta}{c\beta_l(\delta/8L)^l}\right)^{2} \vee 
		\left(\frac{4\|k\|_\infty^{1/2} + 3 \sqrt{2\log(3/q)}}{\delta/4}\right)^{2}
\]
then with probability $1-q$ there exists a ball of radius $\delta/4$ around $\mean_n$ in $C_n$ within the affine subspace spanned by $C$.
\end{theorem}
\begin{proof}
The existence of the ball around $\mean$ in $C$ has already been derived at the end of Section \ref{sec:assumption_density}
and the bound on the width in terms of the lowest eigenvalue of the Mercer decomposition has been stated in Proposition \ref{prop:approx_lower_bnd}.
	
\textbf{(a)} We start with high probability bounds for $\|\mean_n - \mean\|$ being small using Rademacher complexities. Let $\tilde{\mathcal{F}}$ be a countable dense subset of the unit ball of $\cH_-$ then for any $\alpha \in \R$,
\[
\Pr\left(\|\mean_n -\mean\| \geq \alpha \right) = \Pr\left(\sup_{h\in \tilde{\mathcal{F}}} | (1/n)\sum_{i=1}^n f(X_i) - P f| \geq \alpha \right)
\]
since, $\|\mean_n - \mean\| = \sup_{h\in\tilde{\mathcal{F}}} \langle h,\mean_n -\mean \rangle$. In particular, for any $q>0$ and  $\alpha = 4 n^{-1/2} \|k\|_\infty + 3\sqrt{2\log(1/q)} n^{-1/2}$
we can infer from \cite[Lem.22]{BART02} and \cite[Thm.3.4.5]{GINE15} that,
\begin{equation} \label{eq:mean_rademacher}
\Pr\left(\|\mean_n -\mean\| \geq (4 \|k\|^{1/2}_\infty + 3 \sqrt{2\log(1/q)})  n^{-1/2} \right) \leq q.	
\end{equation}

\textbf{(b)} Next, we expand the argument from Section \ref{sec:Lower_bnd_P_Lipschitz} to control the difference between $C_n$ and $C$. Let $c' = -\delta/2$ then   
for any $\gamma >0$ and with probability $1 - q$ simultaneously for all $u \in \tilde{\mathcal{F}}$, 
\begin{align}
&P_n \psi_\gamma(\langle u, \phi(\cdot) - \mean \rangle -c') \\
&\geq P \psi_\gamma(\langle u, \phi(\cdot) - \mean \rangle -c') - (
	\sqrt{2\log(2/q)} + 24\|k\|_\infty^{1/2} /\gamma) n^{-1/2}. \notag
\end{align} 
Chose $\gamma = \delta/4$ and let $x_0 \in \X$ be a point such that $\langle u, \phi(x_0) - \mean \rangle \leq -\delta$.  
Then,
\begin{align*}
P \psi_\gamma(\langle u, \phi(\cdot) - \mean \rangle -c') \geq  \Pr( \langle u,\phi(X_1) -\mean \rangle \leq -\gamma + c').
\end{align*}
As in the proof of Proposition \ref{prop:Bauer_inspired} let $A = \{ y: \|y-x_0\| \leq \delta/4L, y \in \X \}$ and $B$ the translation of $A$ to the origin, $B = \{y: \|y\| \leq \delta/4L\}$. Then the Lebesgue measure of $B$ is  $\mu_l(B) = (\delta/8L)^l \beta_l$ and $\Pr( X_1 \in A) \geq c \mu_l(B)$. Hence,
\[
 \Pr( \langle u,\phi(X_1) -\mean \rangle \leq -\gamma +  c') \geq \Pr(X_1 \in B) \geq c \beta_l (\delta/8L)^l.
\]
For a given $q$ let 
\[
	N_q =  \left(\frac{\sqrt{2\log(2/q)} + 96 \|k\|_\infty^{1/2}/\delta}{c\beta_l(\delta/8L)^l}\right)^{2}
\]
then whenever $n> N_q$ with probability $1-q$ there is a ball of radius $\delta/2$ around $\mean$ in $C_n$. 

\textbf{(c)} Finally, we transfer the lower bound that we have for a ball within $C_n$ around $\mean$ to $\mean_n$. For $q \in (0,1)$ let 
\[
\tilde N_q = \left(\frac{4\|k\|_\infty^{1/2} + 3 \sqrt{2\log(1/q)}}{\delta/4}\right)^{2}.
\]
Then for any $n > \tilde N_q$ with probability at least $1-q$, $\| \mean_n - \mean\| \leq \delta/4$.

Bringing this together, with probability $1-q$ there is ball of size $\delta/4$ in $C_n$ around $\mean_n$ whenever $n > N_{q/3} \vee \tilde N_{q/3}$. 

\end{proof}

The second result uses an assumption on the centered covariance operator instead of an assumption on the density. For this result we actually do not need to use the results on the minimal width of $C$. We forumlate this result directly for the case where $P$ is allowed to attain values in a strict subset of $\X$. A fortunate circumstance in that setting is that the empirical convex set converges to the intersection of $C$ with the minimal face that contains $\mean$ and algorithms that work with the empirical convex set adapt automatically to the structure of the covariance operator. The result relies on the existence of the support of the measure $P$. A weak assumption to guarantee this existence is that $P$ is a $\tau$-additive topological measure \cite[411N]{FREM}.    
\begin{theorem} \label{thm:covariance}
Let $(\X,\mathcal{A},P)$ be some probability space with $P$ being a topological measure that is $\tau$-additive, and 
with measurable kernel function $k$ defined on $\X$ such that the corresponding RKHS  $\cH$ is finite dimensional.  Furthermore, let $X_1,\ldots, X_n$ be i.i.d. random variables attaining values in $\X$ and with law $P$ and assume that $\|k\|_\infty < \infty$, and that the centered covariance operator $\tilde{\mathfrak{C}}_c$ has an eigen-decomposition with smallest non-zero eigenvalue being $\bar \lambda_d$. There exists a ball of radius $\delta = \bar \lambda_d/2 \|k\|_\infty^{1/2}$ around $\mean$ in $C$ within the affine subspace spanned by $C$. 
Furthermore, for any $q \in (0,1)$ and whenever $n$ is (strictly) greater than 
\[
	\left(
\frac{8 \|k\|_\infty (\sqrt{2\log(6/q)} + 192\|k\|_\infty /\bar \lambda_d)}{\bar \lambda_d^2} \right)^{2} \vee 
		\left(\frac{16\|k\|_\infty^{1/2} +  \sqrt{288\log(2/q)}}{\delta}\right)^{2}
\]
then with probability $1-q$ there exists a ball of radius $\delta/4$ around $\mean_n$ in $C_n$ within the affine subspace spanned by $C$.
\end{theorem}
\begin{proof}
\textbf{(a)} The existence of the ball with radius $\delta$ follows directly from Proposition \ref{prop:data_dep_ball}. Furthermore, the same high probability bound for $\|\mean_n -\mean\|$ as in the proof of Theorem \ref{thm:ball_in_empirical} applies. The bound for $C_n$ also runs along the same line as in Theorem \ref{thm:ball_in_empirical}. Consider in the following the RKHS $\cH_S$. Let $\gamma = \delta/4$ and $c = - \delta/2$ then Equation \eqref{eq:lbnd_P_cov} tells us that $ P\psi_\gamma(\langle u,\phi_S(\cdot) - \mean_S \rangle_{\cH_S} -c ) \geq 
\bar \lambda^2/8\|k\|_\infty$ whenever $u\in \cH_S$ has unit norm. Hence, with probability $1-q$ and simultaneously for all $u \in \tilde{\mathcal{F}}$, where $\tilde{\mathcal{F}}$ is countable dense subset of the unit ball of $\cH_S$, 
\[
P_n \psi_\gamma(\langle u, \phi(\cdot) - \mean \rangle -c') \\
\geq  \bar \lambda_d^2/8\|k\|_\infty - (\sqrt{2\log(2/q)} + 192\|k\|_\infty /\bar \lambda_d) n^{-1/2}. 
\]

\textbf{(b)} For a given $q \in (0,1)$ let 
$\tilde N_q = ((16\|k\|_\infty^{1/2} + 12 \sqrt{2\log(1/q)})/ \delta)^2$ then for any $n > \tilde N_q$, $\| \mean_n - \mean\| \leq \delta/4$ with probability $1-q$. Similarly, with probability $1-q$ for $n >N_q$ there is a ball of radius $\delta/2$  around $\mean$ in $C_n$ (as a subset of the affine subspace spanned by $C$), where
\[
N_q = \frac{64 \|k\|^2_\infty (\sqrt{2\log(2/q)} + 192\|k\|_\infty /\bar \lambda_d)^2}{\bar \lambda_d^4 }. 
\]
\end{proof}

The convergence of the  empirical mean embedding and empirical convex set are both unproblematic in the large sample case in both theorems. The bottleneck of the approach is rather the size of the convex set $C$ itself.

\section{Related approximation problems} \label{sec:related_approx}
When confronted with a concrete statistical problem it is typically insufficient to only approximate $\mean$. For example, the least-squares error, when a regressor $f$ from an RKHS $\cH$ with kernel $k$ is used, is
\[
	\frac{1}{n} \sum_{i=1}^n (f(X_i) - Y_i)^2 = \frac{1}{n} \sum_{i=1}^n \langle f \otimes f, k(X_i,\cdot) \otimes k(X_i,\cdot) \rangle_{\otimes} -  \frac{2}{n} \sum_{i=1}^n \langle f, Y_i k(X_i,\cdot) \rangle + \frac{1}{n} \sum_{i=1}^n Y_i^2.  
\]
Looking at the right hand side we can note that we have to deal with \textit{multiple approximation problems}. There are two high level approaches to addressing multiple approximation problems. We can either solve each approximation problem individually or we can solve them  simultaneously. In terms of finding core-sets this means that we will get three different core-sets when solving the approximation problems individually and a single core-set when we solve the approximation problems simultaneously. Before getting back to this discussion, let us have a look at the individual terms in the above least-squares problem.

The \textit{third term} on the right hand side is rather unproblematic since it does not depend on $f$ and can be summarized by a single real number. In particular, if we approximate each term individually then we can compress this term down to a single real number. The 
\textit{first term} on the right hand side 
corresponds to an empirical covariance  and can be treated in a similar way to the empirical measure, i.e. 
\[
	\frac{1}{n} \sum_{i=1}^n  k(X_i,\cdot) \otimes k(X_i,\cdot) 
\]
attains values in  $\cH \otimes \cH$. It fact, since we are only interested in the terms $f^2(X_i)$,
there is an RKHS that is better suited for our purposes than $\cH \otimes \cH$. Due to \cite[Thm.5.16]{PAUL16} there exists a function $g\in \cH \odot \cH$ such that
$f^2(X_i) = g(X_i)$, where $\cH\odot \cH$ is the RKHS that corresponds to the kernel function $\kappa(x,y) = k^2(x,y)$. The empirical covariance, when restricted to $\{ (h,h) : h \in \cH \}$, can be identified with 
\[
	\mathfrak{C}_n = \frac{1}{n} \sum_{i=1}^n \kappa(X_i,\cdot) \in \cH \odot \cH,
\]
i.e. for any $h \in \cH$, 
\[
	\langle \mathfrak{C}_n, h^2 \rangle_{\cH \odot \cH} = \frac{1}{n} \sum_{i=1}^n h^2(X_i) = \frac{1}{n} \sum_{i=1}^n \langle k(X_i,\cdot) \otimes k(X_i, \cdot), h\otimes h \rangle_{\otimes}.
\]
The random element $\mathfrak{C}_n$ attains values within the empirical convex set
\[
	C_{\odot, n } := \cch \{\kappa(X_i,\cdot) : i \leq n\}.
\]
The corresponding population covariance element is given by
\begin{equation} \label{eq:cov_as_mean_emb}
	\mathfrak{C} = \int_{\X} \kappa(x,\cdot)  \, dP(x)
\end{equation}
which is contained in the convex set 
\[
	C_\odot = \cch \{ \kappa(x,\cdot) : x\in \X \}.
\]
In Section \ref {sec:cov_centered}
we used the covariance operator $\Cov:\cH \to \cH$. Notice that $\Cov$ and $\mathfrak{C}$ are closely related since for any $h\in \cH$,
$\langle \Cov h,h \rangle = E(h^2(X)) = \langle \mathfrak{C}, h^2 \rangle_{\cH \odot \cH}$.

The \textit{second term} in the above sum is more difficult to deal with than the other two due to the elements $Y_i$.
We are looking at two approaches in Section \ref{sec:DealWithY}: In the first approach we consider $\mean_{y,n} = (1/n) \sum_{i=1}^n Y_i k(X_i,\cdot)$ as a subset of $\cch \{Y_i k(X_i,\cdot) : i \leq n\}$. That approach works well when we consider the approximation problem in isolation, but it does lead to complications when trying to approximate $\mean_{y,n}$ simultaneously to $\mathfrak{C}_n$. In the second approach we incorporate the $Y_i$'s into the kernel by using  $\langle Y_i, \cdot \rangle_\mathbb{R} \otimes k(X_i,\cdot)$ as a kernel function and by mapping $f\in \cH$ to $\langle 1,\cdot \rangle_\mathbb{R} \otimes f(\cdot) \in \dR \otimes \cH$, i.e. for $i\leq n$, 
\[
	\langle \langle 1,\cdot \rangle_\mathbb{R} \otimes f(\cdot), \langle Y_i,\cdot \rangle_\mathbb{R} \otimes k(X_i,\cdot) \rangle_{\dR \otimes \cH} = Y_i f(X_i) = \langle f,Y_i k(X_i,\cdot) \rangle.  
\]
In this approach we are aiming to approximate 
\[
\mean^\otimes_{y,n} = \frac{1}{n} \sum_{i=1}^n \langle Y_i,\cdot \rangle_\R \otimes k(X_i,\cdot) = \langle 1, \cdot \rangle_\R \otimes \mean_{y,n}. 
\]
In the least-squares problem we might like to use the same points $X_i$ with the same weights $w_i$ to 
\textit{approximate $\mathfrak{C}_n$ and $\mean_{y,n}$ simultaneously}. As we mentioned above this approach is facilitated by incorporating the $Y_i$'s and by moving to $\mean_{y,n}^\otimes$. Similarly, it is useful to extend the functions in $\cH \odot \cH$ to $\R \times \X$ by setting $\hat h(y,x) = h(x)$ for $h\in \cH \odot \cH$. We denote the resulting space by $\widehat{\cH \odot \cH}$ which is again a Hilbert space when using the inner product $\langle \hat h, \hat g \rangle_{\widehat{\cH \odot \cH}} = \langle h, g \rangle_{\cH \odot \cH}$, for any $g,h \in \cH \odot \cH$.
In fact, it is an RKHS with kernel function $\kappa_y((y_1,x_1),(y_2,x_2)) = \kappa(x_1,x_2)$, since 
\[
	 \langle \hat h, \kappa_y((y,x),\cdot) \rangle_{\widehat{\cH \odot \cH} } = \langle h, \kappa(x, \cdot \rangle) \rangle_{\cH \odot \cH} = h(x) = \hat h(y,x).
\]
The empirical covariance operator now becomes $\mathfrak{C}_{y,n} = (1/n) \sum_{i=1}^n \kappa_y((Y_i, X_i), \cdot)$.

One way to achieve a simultaneous approximation of $\mathfrak{C}_{y,n}$ and $\mean_{y,n}^\otimes$ is to use a direct sum $\mathcal{G} = (\widehat{\cH \odot \cH}) \oplus (\dR \otimes \cH)$ and consider the convex set 
\[
	C_{\oplus,n} = \cch \{ (\kappa((Y_i,X_i),\cdot), \langle Y_i,\cdot \rangle_\R \otimes k(X_i, \cdot)) : i \leq n  \}	\subset \mathcal{G}.
\]
The element that we like to approximate is in this context $(\mathfrak{C}_{y,n}, \mean^\otimes_{y,n})$ which lies in $C_{\oplus,n}$. Let $(\bar{\mathfrak{C}}_{y,n}, \bar{\mean}_{y,n}^\otimes)$ be some element in $\mathcal{G}$, then 
\begin{align*}
	&\| (\bar{\mathfrak{C}}_{y,n}, \bar{\mean}_{y,n}^\otimes)  -  (\mathfrak{C}_{y,n}, \mean_{y,n}^\otimes)  \|_{\mathcal{G}}^2 =
	\| \bar{\mathfrak{C}}_{y,n}  -  \mathfrak{C}_{y,n}  \|_{\widehat{\cH \odot \cH}}^2 + \|  \bar{\mean}^\otimes_{y,n}  -  \mean_{y,n}^\otimes  \|_{\dR \otimes \cH}^2 
\end{align*}
and a good approximation in $\cG$ guarantees good approximations of $\mathfrak{C}_{y,n}$ and $\mean_{y,n}^\otimes$ simultaneously.

\subsection{Assumptions}
There are some \textit{minimal assumptions} that we need to impose on $\mathfrak{C}, \mean_y$ and variations thereof to be well defined.
Generally, we assume that we have independent pairs of random variables $(X,Y),(X_1,Y_1), \ldots$ defined on some probability space 
$(\Omega,\mathcal{A},\mu)$. For $\mathfrak{C}$ to be well-defined it suffices to assume that $\kappa(Y,\cdot) \in \mathcal{L}^1(\mu;\cH \odot \cH)$ and, similarly, for $\mean_y$ it suffices to assume that $Y \in \mathcal{L}^2(\mu)$ and  $k(X,\cdot) \in \mathcal{L}^2(\mu;\cH)$ since 
then $\int \| Y k(X,\cdot) \| \, d\mu \leq \|Y\|_2 \bn k(X,\cdot)\bn_2 < \infty$ and 
$\mean_y = \int Y  k(X,\cdot) \, d\mu \in \cH$. 

Some \textit{further assumptions} are useful to facilitate the following analyses. In particular, in the least-squares setting it is natural to assume that 
$Y = f_0(X) + \epsilon$, where $f_0$ is a suitable function, $\epsilon$ is a zero mean real-valued random variable representing measurement noise, and $X$ and $\epsilon$ are independent. When making this assumption we are assuming that $\epsilon$ is a random variable that is defined on the probability space $(\Omega,\mathcal{A},\mu)$. To guarantee that $Y \in \mathcal{L}^2(\mu)$ it is enough to assume that $f_0(X) \in \mathcal{L}^2(\mu)$ and $\epsilon \in \mathcal{L}^2(\mu)$. 


\subsection{Covariance operators} \label{sec:cov_op_approx}
Since Equation \eqref{eq:cov_as_mean_emb} tells us that the covariance operator $\mathfrak{C}$ can be treated like a mean element after changing the kernel, it follows immediately that the techniques we developed for approximating $\mean_n$ and $\mean$ can be applied to the covariance operator $\mathfrak{C}$ and its empirical version $\mathfrak{C}_n$. 
One might also wonder if the approximation problems for $\mean_n$ and $\mathfrak{C}_n$ are  related and if any information that we might deduce about $C_n$ and $\mean_n$ can give us insights into the approximation problem for $\mathfrak{C}_n$. For instance, can we say anyhting about the width of the convex set in $\cH_\kappa$ based on the width of the convex set in $\cH$? Or, does an assumption on the variance of functions in $\cH$ translate to statements about the variance of certain functions in $\cH_\kappa$?  Unsurprisingly, this seems to be impossible to do in general. However, for certain functions in $\cH_\kappa$ we can infer statements about the width and the variance. Similarly, under very stringent assumptions on certain eigenvalues we can say something about the width of the convex set with respect to any function in $\cH_\kappa$. Before looking into these questions we start by taking a closer look at the RKHS $\cH \odot \cH$.

Note that the space $\cH \odot \cH$ is not just the RKHS corresponding to the kernel $\kappa(x,y) = k^2(x,y)$ but it is also closely related to the tensor product $\cH \otimes \cH$. In particular, $ h \in \cH \odot \cH$ if, and only if, there exists $u\in \cH \otimes \cH$  such that $h(x) = u(x,x)$ for all $x \in \X$, and then $\|h\|_{\cH \odot \cH} = \inf\{\|u\|_{\cH \otimes \cH} : h(x) = u(x,x) \text{ for all } x\in \X \}$ \cite[Thm.5.16]{PAUL16}. Observe that for any 
$h \in \cH \odot \cH$ the set $A_h:= \{u \in \cH \otimes \cH : h(x) = u(x,x) \text{ for all } x\in \X  \}$ is a convex subset of $\cH \otimes \cH$. Define the linear operator $T: \cH \otimes \cH \to \cH \odot \cH$ by $T u = u \circ \psi$ where $\psi: \X \to \X \times \X, \psi(x) = (x,x)$, and observe that $T$ is bounded since $\|Tu\|_{\cH \odot \cH} \leq \|u\|_{\cH \otimes \cH}$. Hence, $A_h = T^{-1}[\{h\}]$ is  closed. Also, observe that when 
$h = (f\otimes g) \circ \psi$ for some $f,g \in \cH$ then there exists  functions $f_1,f_2,f_3 \in \cH$ such that $h = (1/2) (f_1 \otimes f_1 - f_2 \otimes f_2 - f_3 \otimes f_3) \circ \psi$ and $\|f \otimes g\|_\otimes \geq \|f_1 \otimes f_1 + f_2 \otimes f_2 + f_3 \otimes f_3 \|_\otimes$: choose $f_1 = (f- g), f_2 = f, f_3 = g$ then  
\[
	f \otimes g + g \otimes f =  (f+g) \otimes (f+g) - (f \otimes f) - (g \otimes g) = f_1 \otimes f_1 - f_2 \otimes f_2 - f_3 \otimes f_3
\]
and, since $(f \otimes g) \circ \psi = (g \otimes f) \circ \psi$, the first statement follows. In terms of the norm observe that $\| f \otimes g + g\otimes f\|_\otimes = \| f_1 \otimes f_1 - f_2 \otimes f_2 - f_3 \otimes f_3\|_\otimes \leq 2\|f\|\|g\| = 2 \|f\otimes g\|_\otimes$. Also, note that 
\[
	(1/2)	\| f \otimes g + g \otimes f\|_\otimes^2 =  \|f\|^2 \|g\|^2 + |\langle f , g \rangle|^2 = \|f \otimes g\|_\otimes^2 + |\langle f , g \rangle|^2 (\leq 2 \|f \otimes g\|^2),   
\]
where the expression in the bracket follows from the Cauchy-Schwarz inequality. Hence, when
$f$ and $g$ are orthogonal, we have that 
\[
\| f_1 \otimes f_1 - f_2 \otimes f_2 - f_3 \otimes f_3\|_\otimes = \sqrt{2} \|f \otimes g\|_\otimes. 
\]
The point is that for any tensor $f \otimes g$ we can express $(f \otimes g) \circ \psi$  as a linear combination of `symmetric' tensor elements applied to $\psi$ without increasing the tensor norm.

Let us next consider the closed subspace $U = \cspn \{ k(x,\cdot) \otimes k(x,\cdot) : x \in \X\} \subset \cH \otimes \cH$ and the orthogonal projection $P_U$ onto it. For any $u \in \cH \otimes \cH$  we have that $P_U u $ lies in the subspace spanned by $k(x,\cdot) \otimes k(x,\cdot)$ and $\|P_U u\|_\otimes \leq \|u\|_\otimes$. In fact, for any $x \in \X$, 
\begin{align*}
u \circ \psi(x) &= \langle u, k(x,\cdot) \otimes  k(x,\cdot) \rangle_\otimes  
 = \langle u, P_U(k(x,\cdot) \otimes  k(x,\cdot)) \rangle_\otimes \\ 
&= \langle P_U u, k(x,\cdot) \otimes  k(x,\cdot) \rangle_\otimes 
= (P_U u) \circ \psi(x).
\end{align*}
In other words, for $h \in \cH \odot \cH$, if we can show that the infimum will be attained over $A_h$, then there exists an element $u$ in $U$ such that $h =u \circ \psi$, $\|u\|_\otimes = \|h\|_\odot$, and for any $v \not \in U$ that fulfills 
$v \circ \psi = h$ it follows that $\|v\|_\otimes > \|u\|_\otimes$.

When $\cH$ is finite dimensional and the kernel function is bounded it follows that $A_h \cap U$ is compact: the tensor space $\cH \otimes \cH$ is finite dimensional since $\cH$ is finite dimensional. Also, $h$ is bounded since the kernel function is bounded. For $\epsilon>0$ consider the centered closed ball $B$ of radius $\|h\|_\odot + \epsilon$ within $\cH \otimes \cH$. The intersection $B\cap A_h \cap U$ is non-empty and compact and the infimum is attained within this compact set.

\subsubsection{Lower bound on the width of $C_\odot$}
How can this \textit{tensor product characterization} of $\cH \odot \cH$ be used to characterize the width of $C_\odot$? Let us first consider tensors of the form $f\otimes f$, $f\in \cH$, $\|f \otimes f\| = 1 =  \|f\|$. The width of $C_\odot$ in direction $h$, where $h  = (f \otimes f) \circ \psi$, is
\begin{align*}
	\diam_{h} C_\odot &= \sup_{x\in \X} h(x) - \infd_{x\in \X} h(x) \\
			  &= \sup_{x\in \X}  \langle f \otimes f, k(x,\cdot) \otimes k(x,\cdot) \rangle_\otimes - \infd_{x\in \X} \langle f\otimes f, k(x,\cdot) \otimes k(x,\cdot) \rangle_\otimes \\
				     &= \sup_{x\in \X} f^2(x) - \infd_{x\in \X} f^2(x).  
\end{align*}
If $\X$ is \textit{path connected},  $k$ is \textit{continuous} and, hence, $f$ is continuous, then we can relate $\diam_h C_\odot$ to 
$\diam_f(C)$: whenever there exist $x,x' \in \X$ such that $f(x) > 0 > f(x')$, then there also exists an $\tilde x \in \X$ with $f(\tilde x) = 0$ due to the mean value theorem. In this case,
\[
	\diam_{h} C_\odot \geq (\sup_{x \in \X} f(x) \vee - \infd_{x\in X} f(x))^2 \geq ((1/2) \diam_f(C))^2.  
\]
If there is no $x$ such that $f(x) =0$, that is $f$ attains only positive or only negative values, then we can argue in the following way. W.l.o.g. assume that $f$ attains only positive values. Since for $a,b \geq 0$, $a \geq b$, $a^2 - b^2 \geq (a-b)^2$,
it also follows in this case that 
\begin{equation*}
	\diam_{h} C_\odot \geq ((1/2) \diam_f(C))^2.
\end{equation*}
\textit{This lower bound can fail to hold when the assumptions about $k$ and $\X$ are not fulfilled}. Consider $\X = \{-1,1\}$ with $k(x,y) = \delta_{x,y}$ and the function $f = k(1,\cdot) - k(-1,\cdot)$ which lies in the RKHS. For this function $f$,
\[
	\sup_{x\in \X} f^2(x) - \infd_{x \in \X} f^2(x) = 0 < 1 = ((1/2) (\sup_{x\in \X} f(x) - \infd_{x\in \X} f(x)))^2.
\]
The factor $1/2$ in the lower bound is redundant  when $f$ attains only positive or negative values since 
$a^2 - b^2 \geq (a-b)^2$ whenever $a$ and $b$ have the same sign. More importantly, the bound becomes loose when $\infd_{x\in \X} f^2(x)$ is large since $a^2 - b^2 - (a-b)^2 = 2b (a-b)$ whenever $a \geq b > 0$.

Moving on to other directions $h \in \cH \otimes \cH$, $\|h\|_\otimes = 1$, we can first observe that the arguments are not as straightforward as elements in $U$ are of the form 
$\sum_{i=1}^d \alpha_i k(x_i,\cdot) \otimes k(x_i,\cdot)$ for some $d\in \mathbb{N}$, $\alpha_i \in \R$ and $x_i \in \X$ for $i\leq d$, and, when considering the supremum over $x \in \X$, the different terms $\alpha_i k^2(x_i,x)$ can potentially cancel each other; in particular, there  is no reason why the $\alpha_i$ should all be positive.

Alternatively, we can consider the set $S= \{k(x,\cdot) \otimes k(x,\cdot) - \int k(z,\cdot) \otimes k(z,\cdot) dP(z) : x \in \X\}$ and the convex hull of $S$. This convex hull is closely related to $C_\odot$. Note that $\ch S$ is a subset of $\cH \otimes \cH$ and contains the origin. Fix now any $u\in \spn S$, $\|u\|_\otimes = 1$, then $\spn \{u\}$ intersects $\ch S$. In fact, $\spn \{u\} \cap \ch S$ consists of more than a single point since otherwise $\langle k(x,\cdot) \otimes k(x,\cdot), u \rangle_\otimes = \langle k(y,\cdot) \otimes k(y,\cdot), u \rangle_\otimes $ for all $x,y \in \X$ and $\langle u, s \rangle_\otimes =0$ for any $s \in S$ which would imply that $u = 0$. Consider the two points at which $\spn \{u\}$ intersects with the boundary of $\ch S$ and let $d$ be the dimension of $\spn S$. Each of these points can be expressed as a convex combination of $ l\leq d+1$ points in $S$ due to Carath\'eodory's theorem. Hence, $u = \tilde u/\|\tilde u\|$ where $\tilde u = \sum_{i=1}^{l} \alpha_i k(x_i,\cdot) \otimes k(x_i,\cdot)$ for some strictly positive $\alpha_i$'s that sum to one and suitable points $x_1,\ldots, x_{l}$. We can note right away that 
\begin{equation} \label{eq:bnds_on_u}
\inf_{x\in \X} k(x,x) /(d+1) \leq \Bigl(\sum_{i,j=1}^{l} \alpha_i \alpha_j k^2(x_i,x_j) \Bigr)^{1/2} =
\|\tilde u \|_\otimes \leq \sup_{x\in \X} k(x,x).
\end{equation}
Also, when \textit{$k(x,y)$ is a non-negative function} an application of Jensen's inequality yields further results. In detail,
\begin{align*}
&\sup_{x,y \in \X}  \tilde u(x,x) - \tilde u(y,y) = \sup_{x,y \in \X}  \sum_{i=1}^{l} \alpha_i (k^2(x_i,x) - k^2(x_i,y)) \\
&\geq  \sup_{x \in \X}  \Bigl(\, \sum_{i=1}^{l} \alpha_i k(x_i,x) \Bigr)^2 - \infd_{y \in \X} \sum_{i=1}^{l} \alpha_i k^2(x_i,y). 
\end{align*}
In the following, let $c = \inf_g \diam_g(C)$, where the infimum is taken over $\{g : g\in \cH, \|g\|=1\}$. When \textit{$c$ is large, $k$ is non-negative and $\|k\|_\infty \leq 1$}, then this simple argument might be of use: since $k(x_i,y) \leq 1$ it follows that $k(x_i,y) \geq k^2(x_i,y)$ and
\[
\sup_{x \in \X}  \Bigl(\, \sum_{i=1}^{l} \alpha_i k(x_i,x) \Bigr)^2 - \infd_{y \in \X} \sum_{i=1}^{l} \alpha_i k^2(x_i,y) \geq \sup_{x \in \X}  \Bigl(\, \sum_{i=1}^{l} \alpha_i k(x_i,x) \Bigr)^2 - \infd_{y \in \X} \sum_{i=1}^{l} \alpha_i k(x_i,y).  
\]
Furthermore, $(\sum_{i=1}^{l} \alpha_i k(x_i,y))^2 \geq   \sum_{i=1}^{l} \alpha_i k(x_i,y) - 1/4$ and 
\[
\sup_{x,y \in \X}  \tilde u(x,x) - \tilde u(y,y) \geq c - 1/4.
\]
This is only useful for large $c$. If, in fact, $c > 1/4$, we can proceed and
\[
\sup_{x,y \in \X}  u(x,x) - u(y,y) \geq \frac{c - 1/4}{\sup_{x \in \X} k(x,x)} \geq c - 1/4.  
\]

\subsubsection{Lower bounds on the fourth moments}
Instead of controlling the width of $C_\odot$ 
 we can also aim to control the covariance operator corresponding to the kernel $\kappa$. Effectively, this corresponds to bounds on the fourth moments. A bound on the non-centered fourth moment is, in fact, easy to derive:  let $\tilde{\mathfrak{C}}$ be the covariance operator (now interpreted as a linear operator) corresponding to the kernel $k$. We need to control $E((h(X) - E(h(X))^2)$ for a function $h \in \cH_\kappa$. Choosing again a suitable $u \in \cH \otimes \cH$ such that $h = u \circ \psi$, we find that
\begin{align*}
	E(h^2(X)) = E(( u \circ \psi(X))^2).
\end{align*}
If $U$ is $d$-dimensional then, as above, we can write $u = \tilde u/\|\tilde u\|$ where $\tilde u = \sum_{i=1}^{l} \alpha_i k(x_i,\cdot) \otimes k(x_i,\cdot)$ for some positive $\alpha_i$, $\sum_{i=1}^{d+1} \alpha_i =1$, and suitable points $x_i \in \X$ with $l\leq d+1$. Given this representation of $u$ and \textit{assuming that the smallest eigenvalue of $\tilde{\mathfrak{C}}$ is $\bar \lambda$}, 
\begin{align*}
	\|\tilde u\|^2 E(h^2(X)) &= \sum_{i,j=1}^{l}\alpha_i \alpha_j E((k(x_i,X)k(x_j,X))^2) 
		\geq \sum_{i,j=1}^{l}\alpha_i \alpha_j |\langle \tilde{\mathfrak{C}} k(x_i,X),k(x_j,X) \rangle|^2 \\
		  &\geq \sum_{i=1}^{l} \alpha_i^2 \frac{\bar \lambda}{k(x_i,x_i)}.
\end{align*}
In particular, \textit{when $k(x,x) = 1$ for all $x \in \X$} then 
\[
E(h^2(X)) \geq \frac{\bar \lambda}{(d+1) \|\tilde u\|^2}  \geq \frac{\bar \lambda}{d+1} 
\]
follows from Eq. \eqref{eq:bnds_on_u}.

However, to say something about the largest ball that lies around $\mathfrak{C}$ within $C_\odot$ we need a lower bound on the variance of $h \in \cH_\kappa$. This is not straight-forward and will need, in all likelihood, some stringent assumptions: consider the variance of an arbitrary functions $h \in \cH_\kappa$, when $\cH_\kappa$ has dimension $d < \infty$. By the above argument, there exists a $u = \tilde u / \|\tilde u\|$, where $\tilde u = \sum_{i=1}^{d+1} \alpha_i k(x_i,\cdot) \otimes k(x_i,\cdot)$ for some non-negative $\alpha_i$ that sum to one and points $x_i \in \X$. Hence,  
\begin{align*}
	&E((h(X) - E(h(X)))^2) = E(( u \circ \psi(X) - E(u \circ \psi(X)))^2) \\
			      &= \|\tilde u\|^{-2} \sum_{i,j=1}^{d+1} \alpha_i \alpha_j E((k^2(x_i,X) - E(k^2(x_i,X)))(k^2(x_j,X) - E(k^2(x_j,X)))). 
\end{align*}
There is no reason why the sum over the off-diagonal elements  
\[
	\sum_{i\not =j} \alpha_i \alpha_j E((k^2(x_i,X) - E(k^2(x_i,X)))(k^2(x_j,X) - E(k^2(x_j,X))))
\]
should be positive or should be of considerably smaller magnitude than the sum over the diagonal elements.

\subsection{Weighted mean embedding} \label{sec:DealWithY}
There are different ways to address the term $Y_i f(X_i)$, $i\leq n$, that occurs in the least squares problem and there are a variety of natural assumptions under which one can study the corresponding compression problem. Let us first have a look at how $Y_i f(X_i)$ can be lifted into the RKHS so that we can apply the compression techniques. A first approach to do so was introduced at the beginning of Section  \ref{sec:related_approx}, where we wrote $Y_i f(X_i) = \langle f, Y_i k(X_i,\cdot)\rangle$. Using this representation we can try to approximate
$(1/n) \sum_{i=1}^n Y_i k(X_i,\cdot) \in \cH$. A second approach is to map the $Y_i$'s to linear functionals, that is to elements in the dual space $\R'$, and to consider the tensor products 
\[
\langle Y_i, \cdot \rangle_\R \otimes k(X_i,\cdot).
\]
We still would like to work with an inner product similarly to $\langle f, k(X_i,\cdot) \rangle$ and we can do so if we work with $\langle 1,\cdot \rangle_\R \otimes \langle f,\cdot \rangle$ instead of $f$. In particular,
\begin{equation} \label{eq:second_modelling_appr}
Y_i f(X_i) = \langle \langle Y_i, \cdot \rangle_\R \otimes k(X_i,\cdot), \langle 1,\cdot \rangle_\R \otimes \langle f,\cdot \rangle \rangle_\otimes.
\end{equation}
Both approaches are natural when the $Y_i$'s are bounded but various issues arise when they are not. In particular, for the latter approach the elements $ \langle Y_i, \cdot \rangle_\R \otimes k(X_i,\cdot)$ are not contained with probability one in a ball in the corresponding tensor product space. We therefore discuss the case where the $Y_i$'s are bounded first before moving on to  the unbounded case. Finally, it is often natural to impose an assumption on the relation between $X_i$ and $Y_i$, like
\begin{equation} \label{eq:std_reg_assump}
Y_i  = f_0(X_i) + \epsilon_i
\end{equation}
with $f_0 \in \mathcal{L}^2(P), \epsilon_i$ independent of $X_i$, $E(\epsilon_i) = 0$, and the $X_i$ and $\epsilon_i$ are i.i.d.

This leaves us with a total of eigth different settings. But not all of these settings are useful for deepening our understanding. In particular, little can be said without the assumption Eq. \eqref{eq:std_reg_assump} and we assume in the following, up to short discussions, that  Eq. \eqref{eq:std_reg_assump} holds. Beyond that we focus on three settings: the first setting uses the assumption that $Y_i$ is bounded and we use the form $Y_i k(X_i,\cdot)$. We then move on to translate the results to the approach where we model $Y_i f(X_i)$ through the tensor product as in Eq.   \eqref{eq:second_modelling_appr}. Finally, we lift the assumption that $Y_i$ is bounded and study the problem in the context of the assumptions Eq. \eqref{eq:second_modelling_appr}
and Eq. \eqref{eq:std_reg_assump}. We assume throughout that the $X_i$'s are i.i.d. and that $k(X_i,\cdot) \in \mathcal{L}^2(P;\cH)$.

\subsubsection{First setting: Bounded $Y_i$'s \& $Y_i f(X_i) = \langle f,Y_i k(X_i,\cdot)\rangle$} \label{sec:bnd_Y}
There are few natural question when working with the empirical estimate $\mean_{y,n} = (1/n) \sum_{i=1}^n Y_i k(X_i,\cdot)$: what is a natural convex set which contains $\mean_{y,n}$ and over which we can optimize efficiently?  Do we have suitable population limits of the empirical quantities?  What can be said about the diameter of the empirical convex set, about how centered $\mean_{y,n}$ lies within the set and are assumptions on the covariance operator of use? In terms of an empirical convex set which contains $\mean_{y,n}$ it is natural to consider the set 
\[
	C_{y,n} = \cch \{Y_i k(X_i,\cdot) : i \leq n\}
\]
and optimization over this set is possible since we have control over the extremes of it.

Under the assumption  that the $Y_i$'s are of the form $f_0(X_i) + \epsilon_i$, there are natural expressions for the population limits $\mean_y$ and the convex set $C_y$. For concreteness, we assume in the following that $\X$ is a \textit{Borel space}, $\cH$ is \textit{separable}, 
$f_0$ and the feature map $\phi:\X \to \cH$ are measurable, $f_0$, the kernel function $k$ and the $\epsilon_i$'s  are \textit{bounded}, $\epsilon_i$ is independent of $X_i$, and $E(\epsilon_i) = 0$.  We can define the population limit of $\mean_{y,n}$ through
\[
	\mean_y = \int f_0(x) k(x,\cdot) \, dP(x),  
\]
where $P$ is the law of $X_1$. The element $\mean_y$ lies in $\cH$:  The function $f_0 \times \phi : \X \to \cH$ is weakly measurable since when $h \in \cH$, then $\langle f_0(x) \phi(x), h \rangle = f_0(x) h(x)$ is the product of two Borel measurable functions and is therefore Borel measurable. Because $\cH$ is separable it follows that $f_0 \times \phi$ is Bochner measurable. Furthermore, $ \| f_0(x) \phi(x) \| \leq |f_0(x)| k^{1/2}(x,x)$ is a bounded function of $x$ and  $\mean_y = \int f_0(x) \phi(x) \,dP$ is well defined and lies in $\cH$. 

\paragraph{Controlling $\|\mean_{y,n} - \mean_y\|$.} We can quantify the deviation of $\mean_{y,n}$ from $\mean_y$ in the following way. Let $\tilde{\mathcal{F}}$ be a countable dense subset of the unit ball of $\cH$ and using that dual elements can be moved through the Bochner integral, we get for $\alpha>0$, 
\begin{align} \label{eq:mean_y_n_vs_y}
	&\Pr(\| \mean_y - \mean_{y,n}\| \geq \alpha ) \notag \\ 
	&= \Pr\Bigl(\sup_{h \in \tilde{\mathcal{F}}} \langle h, \int f_0(X_1) k(X_1,\cdot) \, d\mu - \frac{1}{n} \sum_{i=1}^n Y_i k(X_i,\cdot) \rangle \geq \alpha \Bigr) \notag \\
&\leq \Pr\Bigl(\sup_{h \in \tilde{\mathcal{F}}} \int f_0(X_1) h(X_1) \, d\mu - \frac{1}{n} \sum_{i=1}^n f_0(X_i) h(X_i) \geq \alpha/2\Bigr) \notag  \\
&\quad\hspace{7cm}+ \Pr\Bigl(\sup_{h \in \tilde{\mathcal{F}}}  - \frac{1}{n} \sum_{i=1}^n \epsilon_i h(X_i) \geq \alpha/2 \Bigr). 
\end{align} 
The \textit{latter term} can be bounded by means of Theorem 3.3 in \cite{PIN94} which is a Bernsteintype  theorem for Banach spaced valued random variables. Note that the theorem statement in \cite{PIN94} contains and error which is corrected in \cite{PIN99}. The bound is the following,
\[\sup_{h \in \tilde{\mathcal{F}}}  - \frac{1}{n} \sum_{i=1}^n \epsilon_i h(X_i) = 
\|\frac{1}{n} \sum_{i=1}^n \epsilon_i \phi(X_i)\|\]
and $v_n =  \sum_{i=1}^n \epsilon_i \phi(X_i)$ attains values in $\cH$. The sequence $v_1,\ldots, v_n$ is a martingale sequence in $\cH$ with regard to the filtration $F_t = \sigma(X_1,\ldots,X_t,\epsilon_1,\ldots, \epsilon_t)$, $t \leq n$, since $E(v_t |F_{t-1} ) = 
E(\epsilon_t \phi(X_t)) + v_{t-1} = v_{t-1}$ (a.s.) for $2\leq t\leq n$ due to the independence between $\epsilon_t$ and $X_t$. Furthermore, let $v_0 = 0$ and $F_0 = \{\emptyset,\Omega\}$ so that $E(v_1 | F_0) = E(v_1) = v_0$ a.s. 
We can continue the sequence by letting $v_t = v_n$ and  $F_t = F_n$ for all $t> n$, which preserves the martingale property. To apply  \cite[Thm.3.3]{PIN94,PIN99} we need the following moment bounds; for all $m\geq 2$, 
\[
\sum_{t=1}^n E(\|\epsilon_t \phi(X_i) \|^m |F_{t-1}) \leq n (c \|k\|_\infty)^m \text{\quad\quad(a.s.)},   
\]
where $c$ is an upper bound on $|\epsilon_t|$. This implies that we can set $\Gamma =c \|k\|_\infty, B= n^{1/2} c \|k\|_\infty$ in  \cite[Thm.3.3]{PIN94,PIN99} and 
\begin{equation} \label{eq:bernstein_phi}
\Pr\big( \| \frac{1}{n} \sum_{i=1}^n \epsilon_i \phi(X_i) \| 
\geq \frac{\alpha}{2} \big) \leq 2 \exp\Bigl(
-\frac{n \alpha^2}{4 c \|k\|_\infty ( c \|k\|_\infty +\alpha +  \sqrt{c^2 \|k\|_\infty^2 + \alpha c \|k\|_\infty /n})}
\Bigr). 
\end{equation}
The \textit{first term} can be controlled with a standard \textit{Rademacher argument} after changing the kernel. Define the kernel $l = f_0 \otimes f_0$ and consider the product kernel $l \times k : \X \times \X \to \R$ with RKHS $\cH_{l \times k}$. For $h\in \cH$ it follows that $f_0 \times h \in \cH_{l \times k}$; for an $h$ of the form $\sum_{i=1}^m \alpha_i k(x_i,\cdot)$ one can write down the representation explicitly as 
$f_0 \times h = \sum_{i=1}^m \alpha_i f_0(x) k(x_i,x) = \sum_{i=1}^m (\alpha_i/f_0(x_i)) (l \times k)(x_i,x)$,   
whenever $f_0(x_i) \not = 0$ for all $i\leq m$. 
We can thus write
\begin{align} \label{eq:mean_y_n_vs_y_second}
&\Pr\Bigl(\sup_{h \in \tilde{\mathcal{F}}} \int f_0(X_1) h(X_1) \, d\mu - \frac{1}{n} \sum_{i=1}^n f_0(X_i) h(X_i) \geq \alpha/2\Bigr) \notag  \\
&= \Pr\Bigl(\sup_{g \in \tilde{\mathcal{F}}_{l\times k}} \int g(X_1) \, d\mu - \frac{1}{n} \sum_{i=1}^n g(X_i) \geq \alpha/2\Bigr), 
\end{align}
where 
$\tilde{\mathcal{F}}_{l\times k}$ is dense subset of the unit ball of $\cH_{l \times k}$.
The Rademacher argument that we are using in \eqref{eq:mean_rademacher} can now be applied.

\paragraph{Population limit of $C_{y,n}$.} The next question to address is how to \textit{define the population limit of $C_{y,n}$}. Under the boundedness assumption of $Y_1$ we can characterize the limit of $C_{y,n}$ in the following way. Let
$\bar b = \inf\{b : b \in \mathbb{R}, \epsilon_1 \leq b \text{ a.s.} \}$, $\ubar{b} = 
- \sup\{b : b \in \mathbb{R}, \epsilon_1 \geq b \text{ a.s.} \}$ and let 
\[
	C_y = \cch(\{ (f_0(x) + \bar b) k(x,\cdot) : x \in \X \} \cup  \{ (f_0(x) - \ubar b) k(x,\cdot) : x \in \X \}). 
\]
Note that for any $x\in \X$,  $(f_0(x) + \bar b) k(x,\cdot)$ and $(f_0(x) - \ubar{b}) k(x,\cdot)$ lie in $\cH$ and $C_y \subset \cH$. Furthermore, $\epsilon_i \in [-\ubar{b},\bar b]$ (a.s.) and for $i\leq n$,  $(f_0(X_i) + \epsilon_i) k(X_i,\cdot)$ is  almost surely a convex combination of $(f_0(X_i) + \bar b) k(X_i,\cdot)$ and $(f_0(X_i) - \ubar{b}) k(X_i,\cdot)$. Therefore, $C_{y,n}$ is almost surely contained within $C_y$.

While $C_y$ is a natural limit of $C_{y,n}$, we face the problem that the convergence towards $C_y$ can be arbitrarily slow since $\epsilon_1$ can have a low probability of attaining values close to $\bar b$ or $\ubar{b}$. Assumptions on the distribution of $\epsilon_1$ are one way to address this problem. Alternatively, we can work directly with $C_{y,n}$ and study how deep $\mean_{y,n}$ lies in $C_{y,n}$ by controlling events of the form $ \langle h, \tilde Y k(\tilde X,\cdot) - \mean_y \rangle \leq c $, where $(\tilde X,\tilde Y)$ has law $P_n$, and by comparing the random variables $\langle h, \tilde Y k(\tilde X,\cdot) - \mean_y \rangle$ and $\langle h, \tilde Y k(\tilde X,\cdot) - \mean_{y,n} \rangle$. 
We follow this latter approach and we use Rademacher complexities to control these events uniformly over the unit ball of $\cH$. In the Rademacher approach, we control such events by  lower bounding terms of the form
\begin{equation*} P_n \psi_\gamma( \langle h, \tilde y k(\tilde x,\cdot) - \mean_{y,n} \rangle - c ) \quad  \bigl( =  \int \psi_\gamma( \langle h, \tilde y k(\tilde x,\cdot) - \mean_{y,n} \rangle - c ) \,dP_n(\tilde x,\tilde y) \bigr)
\end{equation*}
for suitable $c,\gamma \in \R$ and all unit norm elements $h\in \cH$. The element $\mean_{y,n}$ converges to $\mean_y$ and, because $\psi_\gamma$ is $1/\gamma$-Lipschitz continuous, it follows that 
\begin{align} 
&|P_n \psi_\gamma( \langle h, \tilde y k(\tilde x,\cdot) - \mean_{y,n} \rangle - c ) 
 - P_n \psi_\gamma( \langle h, \tilde y k(\tilde x,\cdot) - \mean_{y} \rangle - c ) | \notag \\
&\leq P_n (1/\gamma) | \langle h, \mean_{y,n} - \mean_y \rangle| \leq (1/\gamma) \|\mean_{y,n} - \mean_y\|, \text{\quad (a.s.)} \label{eq:lip_to_mean} 
\end{align}
where the $P_n$ term becomes redundant since no variables $\tilde y$ and $\tilde x$ are present in the last line.

Next, we consider the convergence of  
$P_n \psi_\gamma( \langle h, \tilde y k(\tilde x,\cdot) - \mean_{y} \rangle - c )$ to its population limit $P(\psi_\gamma( \langle h, y k(x,\cdot) - \mean_{y} \rangle - c ))$ uniformly over the unit ball of $\cH$. The convergence can be controlled by using Rademacher complexities. Because $\epsilon$ is used in this section to denote the noise terms we will use $\zeta$  to denote Rademacher variables.
Since $\psi_\gamma$ is continuous and is applied to a subset of $\R$ we can note that
\[
\sum_{i=1}^n \zeta_i \psi_\gamma(Y_i h(X_i) - E(Y h(X)) - c)
\]
is well defined.  Also, $\gamma \psi_\gamma$ is a contraction vanishing at zero \cite[Sec.5.2.1]{GINE15} and  for any finite subset $F$ of the unit ball of $\cH$ it follow that
\begin{align*}
&E_\zeta \bigl(
	\sup_{h \in F}  \sum_{i=1}^n \zeta_i (\gamma/2) \psi_\gamma(Y_i h(X_i) - E(Y h(X)) - c)  \bigr) \\
&\leq  E_\zeta \bigl(
	\sup_{h \in F} \sum_{i=1}^n \zeta_i  Y_i h(X_i)   \bigr) \leq E_\zeta \bigl( \sup_{h \in F} \bigl| \sum_{i=1}^n \zeta_i  Y_i h(X_i) \bigr|  \bigr) \text{\quad (a.s.)},
\end{align*}
from \cite[Thm.5.2.1, Eq.5.50]{GINE15}. Note that, conditional on $Y_i$, the probability laws of $\zeta_i Y_i$ and $\zeta_i |Y_i|$ are the same. In particular, 
\[
E_\zeta \bigl( \sup_{h \in F} \bigl| \sum_{i=1}^n \zeta_i  Y_i h(X_i) \bigr|  \bigr)  = 
E_\zeta \bigl( \sup_{h \in F} \bigl| \sum_{i=1}^n \zeta_i  |Y_i| h(X_i) \bigr|  \bigr)  \text{\quad (a.s.)}. 
\]
Recall that Rademacher complexities are stable under taking absolute convex hulls. Let $C_F = \{ (|Y_1| h(X_1), \ldots, |Y_n| h(X_n)) : h \in F\}$ and $\hat C_F = \{(\|Y\|_\infty h(X_1),\ldots, \|Y\|_\infty h(X_n)) : h\in F\}$ then   $\text{abs conv } C_F \subset \text{ abs conv } \hat C_F$. Furthermore, 
\begin{align*}
&E_\zeta \bigl( \sup_{h \in F} \bigl| \sum_{i=1}^n \zeta_i  |Y_i| h(X_i) \bigr|  \bigr) 
= E_\zeta \bigl( \sup_{t \in C_F} \bigl| \sum_{i=1}^n \zeta_i  t_i \bigr|  \bigr) 
= E_\zeta \bigl( \sup_{t \in \text{abs conv } C_F} \bigl| \sum_{i=1}^n \zeta_i  t_i \bigr|  \bigr)\\
&\leq 
E_\zeta \bigl( \sup_{t \in \text{abs conv } \hat C_F} \bigl| \sum_{i=1}^n \zeta_i  t_i \bigr|  \bigr) =  \|Y\|_\infty E_\zeta \bigl( \sup_{h \in F} \bigl| \sum_{i=1}^n \zeta_i h(X_i) \bigr|  \bigr). \text{\quad (a.s.)} 
\end{align*}
In summary, we have shown that
\[
E_\zeta \bigl(
	\sup_{h \in F}  \sum_{i=1}^n \zeta_i  \psi_\gamma(Y_i h(X_i) - E(Y h(X)) - c)  \bigr)  \leq (2/\gamma) \|Y\|_\infty E_\zeta\bigl( \sup_{h \in F} \bigl| \sum_{i=1}^n \zeta_i h(X_i) \bigr|  \bigr). \text{\quad (a.s.)} 
\]
A simple variation of the above argument gives us a bound on the absolute value. In detail, 
\begin{align} \label{eq:Rademacher_for_function}
&E_\zeta \bigl(
	\sup_{h \in F} \bigl|  \sum_{i=1}^n \zeta_i (\gamma/2) \psi_\gamma(Y_i h(X_i) - E(Y h(X)) - c) \bigr| \bigr) \notag \\
&\leq E_\zeta \bigl( \sup_{h \in F} \bigl| \sum_{i=1}^n \zeta_i  Y_i h(X_i) \bigr|  \bigr) 
+ \sup_{h\in F} |E(Yh(X)) - c| E\bigl(\bigl|\sum_{i=1}^n \zeta_i\bigr|\bigr) \notag \\
&\leq \|Y\|_\infty E_\zeta \bigl( \sup_{h \in F} \bigl| \sum_{i=1}^n \zeta_i h(X_i) \bigr|  \bigr)
+ (\|Y\|_\infty \|k\|_\infty^{1/2}  + |c|) \sqrt{2\pi n},
\text{\quad (a.s.)}
\end{align}
where the last inequality follows from integrating a Hoeffding bound on $\Pr(|\sum_{i=1}^n \zeta_i| \geq t)$. Since this holds for all finite $F$ we can take the supremum over finite sets $F$ on both sides and move to $\tilde{\mathcal{F}}$ (see \eqref{eq:Rade_sep}).

\paragraph{Lower bounds.} With the Rademacher argument we control the difference between the empirical and population value. To make use of this bound we need a lower bound on the population value. This can be attained in the following way. 
Let $p = \Pr(\epsilon_1 \geq 0) \wedge \Pr(\epsilon_1 \leq 0)$. Because $E(\epsilon_1) = 0$ it holds that $p >0$. Using the towering rule for conditional expectations and that $\psi_\gamma$ is monotonically decreasing, we can now argue in the following way for the population limit and any $h\in \cH$,
\begin{align} \label{eq:lower_bnd_with_p}
&E(\psi_\gamma( \langle h, Y k(X,\cdot) - \mean_y \rangle - c )) \notag \\
&\geq 
E(\psi_\gamma( \langle h, Y k(X,\cdot) - \mean_y \rangle - c ) \times \chi\{\epsilon h(X) \leq 0 \}) \notag \\
&\geq 
E(\psi_\gamma( \langle h, f_0(X) k(X,\cdot) - \mean_y \rangle - c ) \times \notag \\
&\quad \quad  (E(\chi\{\epsilon \leq 0 \}|X) \times \chi\{h(X) \geq 0\} + E(\chi\{ \epsilon \geq 0 \}|X) \times \chi\{h(X) \leq 0\})) \notag \\
&\geq p E(\psi_\gamma( \langle h, f_0(X) k(X,\cdot) - \mean_y \rangle - c )). 
\end{align}
The final expectation term can be dealt with in the usual way after moving to the kernel function $l \times k$,  where $l = f_0 \otimes f_0$. We demonstrate this for the case that we work with a covariance operator assumption and we derive high probability lower bounds on the radius of a ball centered on $\mean_{y,n}$ which lies within $C_{y,n}$.  

\paragraph{Assumptions on the covariance operator.}
We denote the centered covariance operator for the kernel $l \times k$ by $\tilde{\mathfrak{C}}_{c,l\times k}$. As usual, we need an assumption on the smallest non-zero eigenvalue of this operator. Recall that $\|f_0\|_{l}  = 1$ and 
\[
	\|f_0 \times h\|_{l \times k} = \min \{ \|u\|_{\cH_l \otimes \cH_k} : f_0(x) h(x) =  u(x,x) \text{ for all } x\in \X\}.
\]
In particular, $\|f_0 \times h\|_{l \times k} \leq \|f_0 \otimes h \|_{\cH_l \otimes \cH_k} = \|h\|_k$.  In fact, this can be tightened by using  \cite[Prop5.20]{PAUL16}: the RKHS 
$\cH_{l \times k}$ is the set $\{ f_0 \times h: h \in \cH\}$ and for any $g,h \in \cH$ we have that $\langle f_0 \times g, f_0 \times h \rangle_{l \times k} = \langle g, h \rangle$.
In particular, all eigenfunctions of $\tilde{\mathfrak{C}}_{c,l\times k}$ are of the form $f_0 \times h$ for some $h\in \cH, \|h\|=1$. Therefore, our bounds will depend on 
\begin{equation}
\bar \lambda_\star = \inf \{ \langle \tilde{\mathfrak{C}}_{c,l\times k} f_0 \times h, f_0 \times h \rangle_{l \times k} : \|h \|= 1, f_0 \times h \in (\text{ker }  \tilde{\mathfrak{C}}_{c,l\times k})^\perp \}.
\end{equation}
As before, it is beneficial to move to the RKHS $\cH_{l\times k,S}$ corresponding to the kernel function $\kappa = (l \times k)\!\!\upharpoonright\! S \times S$, where $S$ is the support of $X$. Observe that 
\[
(l \times k) \!\!\upharpoonright\! S \times S = (l \!\!\upharpoonright\! S \times S) \times (k \!\!\upharpoonright\! S \times S) = 
((f_0  \!\!\upharpoonright\! S ) \otimes  (f_0  \!\!\upharpoonright\! S)) \times(k \!\!\upharpoonright\! S \times S). 
\]
Using \cite[Prop5.20]{PAUL16} again shows that  $\cH_{l\times k,S} = \{ (f_0  \!\!\upharpoonright\! S) \times h :  h \in \cH_S \}$, where $\cH_S$ is defined as before. Furthermore,
for $g,h \in \cH_S$, 
\[
	\langle (f_0 \restr S) \times g, (f_0 \restr S) \times h \rangle_{l\times k, S} = 
	\langle g, h \rangle_S.
\]
Let us also introduce $\mean_{l\times k, S} = \int \kappa(x, \cdot) \,dP(x)$, where $P$ is the law of $X$, and which is well defined whenever $(l \times k)(X,\cdot)$ is Bochner integrable. For $h\in \cH$,
\begin{align*}
\langle h, \mean_y \rangle &= \int_\X \langle h, f_0(x) k(x,\cdot) \rangle \,dP(x) = \int_S (f_0 \restr S)(x) (h\restr S) (x)  \, dP(x) \\
&= \int_S \langle h\restr S, (f_0\restr S)(x) (k\restr S\times S)(x,\cdot) \rangle_S \,dP(x) \\ 
&=\int_S \langle (f_0 \restr S)\times (h\restr S), (f_0\restr S) \times (f_0\restr S)(x) (k\restr S\times S)(x,\cdot) \rangle_{l\times k,S} \,dP(x)  \\
&=\int_S \langle (f_0 \restr S)\times (h\restr S), ((l \times k)\restr S \times S)(x,\cdot) \rangle_{l\times k,S} \,dP(x) \\
&= \langle (f_0 \restr S)\times (h\restr S), \mean_{l\times k,S} \rangle_{l\times k, S}.
\end{align*}
The key observation is now the following, for any $h \in \cH$ and almost surely 
\begin{align*}
&\langle h, f_0(X) k(X,\cdot) - \mean_y \rangle =  (h \restr S)(X) (f_0 \restr S)(X)  - \langle (f_0 \restr S)\times (h\restr S), \mean_{l\times k,S} \rangle_{l\times k, S} \\
&= \langle (f_0 \restr S) \times (h \restr S), \kappa(X,\cdot) - \mean_{l\times k,S} \rangle_{l\times k, S}.
\end{align*}
This leads directly to a first result. Under suitable assumptions and with $\delta = \bar \lambda_\star/2 \|(f_0 \otimes f_0) \times k\|_\infty^{1/2}$, $\gamma = \delta/4$ and $c = - \delta/2$, Equation \eqref{eq:lbnd_P_cov} shows that for any $h\in \cH$, such that $\|h\restr S\|_S=1$, 
\begin{align} \label{eq:paley_zygmund_f_0}
&E(\psi_\gamma(\langle h,f_0(X)k(X,\cdot)  - \mean_y \rangle -c )) \notag \\
&= E(\psi_\gamma(\langle (f_0 \restr S) \times (h \restr S), \kappa(X,\cdot) - \mean_{l\times k,S} \rangle_{l\times k, S} - c)) \notag \\
&\geq \bar \lambda_\star^2/8\|(f_0 \otimes f_0) \times k\|_\infty.
\end{align}
Combining the various steps above leads to the following proposition, which is an adaptation of Theorem \ref{thm:covariance}.
\begin{prop} \label{prop:bnd_Y}
Let $(\X \times \R,\mathcal{A},P)$ be some probability space, let $P$ be a topological measure that is $\tau$-additive, and 
let $k$ be a measurable kernel function defined on $\X$ s.t. the corresponding RKHS  $\cH$ is finite dimensional.  Furthermore, let $(X_1,Y_1),\ldots, (X_n,Y_n)$ be i.i.d. random variables attaining values in $\X \times \R$, with law $P$, and of the form $Y_i = f_0(X_i) + \epsilon_i$ where $\epsilon_1,\ldots, \epsilon_n$ are centered i.i.d. random variables which are independent of $X_1,\ldots, X_n$ and such that $|\epsilon_1| \leq c_\epsilon$ (a.s.), and $f_0 \in \mathcal{L}^2(P)$. Assume that $\|(f_0 \otimes f_0)  \times k\|_\infty < \infty$, and that the centered covariance operator $\tilde{\mathfrak{C}}_{c,l \times k}$ has an eigen-decomposition with smallest non-zero eigenvalue being $\bar \lambda_\star$. Let $\delta = \bar \lambda_\star/2 \|(f_0 \otimes f_0) \times k\|_\infty^{1/2}$ and $p = \Pr(\epsilon_1 \geq 0) \wedge \Pr(\epsilon_1 \leq 0)$. For any $q \in (0,1)$ and whenever $n$ is (strictly) greater than 
\begin{align*}
	&\frac{1024 \|f_0 \otimes f_0 \times k\|^2_\infty}{p^2 \bar \lambda_\star^4 \delta^2} 
\Bigl(
16 (8  \|f_0 \otimes f_0 \times k\|_\infty + 6\sqrt{2\log(3/q)})^2 \\ 
&\quad \enspace  \vee 16 (4 c_\epsilon \|k\|_\infty \log(6/q) +  ( (4c^2_\epsilon \|k\|^2_\infty + \sqrt{1+c^2_\epsilon \|k\|_\infty^2}  )\log(6/q))^{1/2})^2 \\
&\quad \enspace  \vee  (\delta \sqrt{2\log(6/q)} + 16 \sqrt{2\pi} ((\|f_0\|_\infty + c_\epsilon)  \|k\|_\infty^{1/2}  +\delta/2) 
+  32 \|k\|_\infty^{1/2} (\|f_0\|_\infty + c_\epsilon))^2 
\Bigr) 
\end{align*}
then, with probability at least $1-q$, there exists a ball of radius $\delta/4$ around $\mean_{y,n}$ in $C_{y,n}$ within the affine subspace spanned by $C_{y,n}$.
\end{prop}
\begin{proof}
\textbf{(a)}  Since $ f_0 \otimes f_0 \times k$ is a bounded kernel function we can apply  \eqref{eq:paley_zygmund_f_0}
 and conclude for $\gamma = \delta/4$ and $c = - \delta/2$ that   $E(\psi_\gamma(\langle h,f_0(X)k(X,\cdot)  - \mean_y \rangle -c )) \geq  \bar \lambda_\star^2/8\|(f_0 \otimes f_0) \times k\|_\infty$ and from  \eqref{eq:lower_bnd_with_p} it follows that
\[
E(\psi_\gamma( \langle h, Y k(X,\cdot) - \mean_y \rangle - c )) \geq p \bar \lambda_\star^2/8\|(f_0 \otimes f_0) \times k\|_\infty.
\]
\textbf{(b)} Next, we have to incorporate a few triangle inequalities. First, we fill in the details in \eqref{eq:mean_y_n_vs_y}. From \eqref{eq:mean_y_n_vs_y_second}
it follows that  for $p_1 \in (0,1)$ and $\alpha_1 = 8 n^{-1/2} \|f_0 \otimes f_0 \times k\|_\infty + 6\sqrt{2\log(1/p_1)} n^{-1/2}$,
\begin{align*}
&\Pr\Bigl(\sup_{h \in \tilde{\mathcal{F}}} P f_0 \times h  - P_n f_0 \times h \geq \alpha_1/2\Bigr) =  \Pr\Bigl(\sup_{g \in \tilde{\mathcal{F}}_{l\times k}}  |P g  - P_n g| \geq \alpha_1/2\Bigr) \leq p_1.
\end{align*}
Also note that we can simplify \eqref{eq:bernstein_phi} to 
\[
\Pr\big( \| \frac{1}{n} \sum_{i=1}^n \epsilon_i \phi(X_i) \| 
\geq \frac{\alpha_2}{2} \big) \leq 2 \exp\Bigl(
-\frac{n \alpha_2^2}{4 c_\epsilon \|k\|_\infty ( c_\epsilon \|k\|_\infty +\alpha_2 +  \sqrt{1 + c_\epsilon^2 \|k\|_\infty^2})}
\Bigr),
\]
whenever $n \geq \alpha_2 c_\epsilon \|k\|_\infty$. Hence, for such $n$, for $p_2 \in (0,1)$ and 
\[
	\alpha_2 = \frac{4 c_\epsilon \|k\|_\infty \log(2/p_2)}{n} + \left(\frac{(4c^2_\epsilon \|k\|^2_\infty + \sqrt{1+c^2_\epsilon \|k\|_\infty^2}  )\log(2/p_2)}{n} \right)^{1/2} 
\]
it follows that
\[
	\Pr\big( \| \frac{1}{n} \sum_{i=1}^n \epsilon_i \phi(X_i) \| 
\geq \frac{\alpha_2}{2} \big) \leq p_2.
\]
In particular, for $\alpha_{12} = \alpha_1 \vee \alpha_2$ and whenever $n \geq \alpha_2 c_\epsilon \|k\|_\infty$,
\[
\Pr( \| \mean_y - \mean_{y,n} \| \geq \alpha_{12}) \leq p_1 + p_2.
\]
\textbf{(c)} From  \eqref{eq:lip_to_mean}  we can infer that almost surely
\begin{align*}
&P \psi_\gamma( \langle h, y k(x,\cdot) - \mean_{y} \rangle - c )-  
P_n \psi_\gamma( \langle h, \tilde y k(\tilde x,\cdot) - \mean_{y,n} \rangle - c )  \\
 &\leq \gamma^{-1} \|\mean_{y,n} - \mean_y\| + 
P \psi_\gamma( \langle h, y k(x,\cdot) - \mean_{y} \rangle - c ) - P_n \psi_\gamma( \langle h, \tilde y k(\tilde x,\cdot) - \mean_{y} \rangle - c ). 
\end{align*}
The same inequality holds almost surely if we consider the supremum over $\tilde{\mathcal{F}}$. Therefore,
with $\gamma = \delta/4$ and $\alpha > 0$,
\begin{align*}
&\Pr(\sup_{h \in \tilde{\mathcal{F}}} P \psi_\gamma( \langle h, y k(x,\cdot) - \mean_{y} \rangle - c )-  
P_n \psi_\gamma( \langle h, \tilde y k(\tilde x,\cdot) - \mean_{y,n} \rangle - c )  \geq 2 \alpha ) \\
&\leq \Pr( \|\mean_{y,n} - \mean_y\| \geq  \delta \alpha /4) \\ 
& \quad \quad + \Pr( \sup_{h \in \tilde{\mathcal{F}}}  P \psi_\gamma( \langle h, y k(x,\cdot) - \mean_{y} \rangle - c )- 
P_n \psi_\gamma( \langle h, \tilde y k(\tilde x,\cdot) - \mean_{y} \rangle - c )   \geq \alpha). 
\end{align*}
The latter term can be dealt with by a Rademacher argument when using \eqref{eq:Rademacher_for_function}. In detail, for any $p_3 \in (0,1)$, with probability 
$1-p_3$ simultaneously for all $h \in \tilde{\mathcal{F}}$, 
\begin{align*}
P_n \psi_\gamma( \langle h, \tilde y k(\tilde x,\cdot) - \mean_{y} \rangle - c ) \geq&  P \psi_\gamma( \langle h, y k(x,\cdot) - \mean_{y} \rangle - c ) -  \sqrt{2\log(2/p_3)/n}  \\
& \quad  - 2 E_\zeta \bigl(\sup_{h \in \tilde{\mathcal{F}}} \bigl|\frac{1}{n}  \sum_{i=1}^n \zeta_i \psi_\gamma(Y_i h(X_i) - E(Y h(X)) - c) \bigr| \bigr)  
\end{align*}
follows from \cite[Thm3.4.5]{GINE15}; also see p. \pageref{sec:Rademacher_expl}. Substituting \eqref{eq:Rademacher_for_function} leads to the following lower bound on the $P_n$ term,
\begin{align*}
&P \psi_\gamma( \langle h, y k(x,\cdot) - \mean_{y} \rangle - c )  -  \sqrt{2\log(2/p_3)/n} \\
&- (4/\gamma) ( \|Y\|_\infty E_\zeta \bigl( \sup_{h \in \tilde{\mathcal{F}}} \bigl|\frac{1}{n} \sum_{i=1}^n \zeta_i h(X_i) \bigr|  \bigr)
+ (\|Y\|_\infty \|k\|_\infty^{1/2}  + |c|) \sqrt{2\pi/ n} ).
\end{align*}
Filling in $\gamma$, $c$, the upper bound on $|Y|$ and the Rademacher complexity of $\tilde{\mathcal{F}}$, reduces the lower bound to
\[
	\alpha_3 = \frac{\sqrt{2\log(2/p_3)}}{\sqrt{n}} + \frac{16 \sqrt{2\pi} ((\|f_0\|_\infty + c_\epsilon)  \|k\|_\infty^{1/2}  +\delta/2) }{\delta \sqrt{n}}
+  \frac{32 \|k\|_\infty^{1/2} (\|f_0\|_\infty + c_\epsilon)}{\delta \sqrt{n}}  
\]
and 
\begin{align*}
	\Pr\bigl( \sup_{h \in \tilde{\mathcal{F}}} P \psi_\gamma( \langle h, y k(x,\cdot) - \mean_{y} \rangle - c ) -  P_n \psi_\gamma( \langle h, \tilde y k(\tilde x,\cdot) - \mean_{y} \rangle - c ) \geq \alpha_3 \bigr) \leq p_3.   
\end{align*}
\textbf{(d)} Combining these bounds we can derive a lower bound on  $P_n \psi_\gamma( \langle h, \tilde y k(\tilde x,\cdot) - \mean_{y,n} \rangle - c )$.
In detail, let $p_1 = p_2 = p_3 = q/3$  and set
\[
\alpha^\star = \frac{4 \alpha_{12}}{\delta} \vee \alpha_3  
\]
then with probability $1-q$ simultaneously for all $h \in \tilde{\mathcal{F}}$,
\[
P_n \psi_\gamma( \langle h, \tilde y k(\tilde x,\cdot) - \mean_{y,n} \rangle - c ) \geq p \bar \lambda_\star^2/8\|(f_0 \otimes f_0) \times k\|_\infty -  
2 \alpha^\star.
\]
To guarantee that the right hand side is strictly positive we can choose the $n$ which is provided in the statement of the proposition. 
\end{proof}

\subsubsection{Second setting: Bounded $Y_i$'s \& Eq. \eqref{eq:second_modelling_appr}}
In this section we map   $h \in \cH$ to $\check{h} =  \langle 1,\cdot\rangle_\mathbb{R} \otimes h(\cdot)$ and we work with the kernel function
\[
	\rho ((y_1,x_1),(y_2,x_2)) = \langle y_1, y_2 \rangle_\R k(x_1,x_2).
\]
In the introduction to this section we denoted the RKHS $\cH_\rho$ by $\R' \otimes \cH$. We will use in the following the more compact notation $\cH_\rho$, $\|\cdot \|_\rho$ etc. 

Compared to the approach in the previous section, using $\cH_\rho$  offers a dramatic simplification of the analysis and leads to improved bounds.
As usual, we are interested in approximating a mean element. In the current context, this is $\mean^\otimes_y = \int \rho((y,x), \cdot) \, dP(x,y)$ and 
there is a straight forward relation to the element $\mean_y$ that was used in the previous section, 
\[
	\mean^\otimes_{y} = \int \langle y, \cdot \rangle_\mathbb{R} \otimes k(x,\cdot) \, dP(x,y) = \langle 1,\cdot \rangle_\R \otimes 
	\int y k(x,\cdot) \, dP(x,y) = \langle 1,\cdot \rangle_\R \otimes \mean_y,
\]
using \eqref{eq:tensor_ind}. The element $\mean^\otimes_y$  lies in $\dR \otimes \cH$ whenever $Y^2 k(X,X) \in \mathcal{L}^1(\mu)$.  
Under our assumption that $Y = f_0(X) + \epsilon$, $X$ and $\epsilon$ independent random variables, the representation of $\mean_y^\otimes$ simplifies to 
\begin{align*}
	\mean_y^\otimes &= E( \langle f_0(X) + \epsilon ,\cdot \rangle_\R \otimes k(X,\cdot)) \\
		&=  E(f_0(X) \langle 1, \cdot \rangle_\R \otimes k(X,\cdot)) + \langle E(\epsilon), \cdot \rangle   \otimes E(k(X,\cdot)) \\  
		 &= \langle 1,\cdot \rangle_\R \otimes E(f_0(X) k(X,\cdot)),
\end{align*}
where we used \eqref{eq:tensor_ind} in the second and in the last equality. This is again just $\langle 1, \cdot \rangle_\R \otimes \mean_y$.

If $P$ is $\tau$-additive as a measure on $\X \times \R$ then the support $S$ of $P$ is well defined and we have a natural population limit 
\[
C^\otimes_y = \cch \{
\langle y, \cdot \rangle_\R \otimes k(x,\cdot) :  (x,y) \in S\}
\]
of the empirical convex set 
\[
	C^\otimes_{y,n} = \cch \{ \langle Y_i, \cdot \rangle_\R \otimes k(X_i,\cdot) :  i \leq n\}.
\]
These are just the convex sets associated with the kernel function $\rho$ acting on $\R \times \X$ and $\mean_y^\otimes, \mean_{y,n}^\otimes$ are the corresponding mean and empirical mean elements. In fact, we can apply right away Theorem \ref{thm:covariance}. Our sample space is then $[- \|f_0\|_\infty - c_\epsilon, \|f_0\|_\infty + c_\epsilon] \times \X$,
where $c_\epsilon$ is a constant such that $|Y| \leq c_\epsilon$ a.s. The kernel function is $\rho$ restricted to the sample space and 
\[
	\|\rho \!\!\upharpoonright \!\! S \times S \|_\infty  \leq (\|f_0\|_\infty + c_\epsilon) \|k\|_\infty.
\]
In this formulation it might not be directly obvious how assumptions on the distribution of $Y$ enter. Using a Rademacher argument we can control the difference between $P$ and $P_n$ when acting on indicator functions. To do so we do do not need any assumption on the distribution of $Y$ beside boundedness. But, if you recall our earlier arguments, you will notice that we used lower bounds on $P$ when applied to indicator functions to control the size of $C^\otimes_{y,n}$. This lower bound on $P$ depends on the distribution of $Y$. In particular, with the covariance operator approach, it depends on the variance of $Y$.   

Let $\tilde{\mathfrak{C}}^\otimes_c$ be the centered covariance operator corresponding to kernel $\rho$ then for $g,h \in \cH$,
\begin{align*}
	\langle \tilde{\mathfrak{C}}^\otimes_c \check g,\check h \rangle_\rho 
	&= E( Y^2 g(X) h(X)) - E(Y g(X)) E(Y h(X))\\
	&= E( f_0^2(X) g(X) h(X)) + \sigma^2 E(g(X) h(X)) - E(f_0(X) g(X)) E(f_0(X) h(X)),
\end{align*}
where $\sigma^2$ is the variance of $\epsilon$. We can also relate this expression back to the covariance operator discussed in the earlier approach, for $h\in \cH$,
\[
\langle \tilde{\mathfrak{C}}^\otimes_c \check g,\check h \rangle_\rho 
= \langle \tilde{\mathfrak{C}}_{c,l\times k} f_0 \otimes g, f_0 \otimes h \rangle_{l\times k} + \sigma^2 E(g(X)) E(h(X)). 
\]
Recall that for $g,h \in \cH$,  $\langle g, h \rangle = \langle f_0 \times g, f_0 \times h \rangle_{l\times k}$ and note that $\langle  \check g, \check h \rangle_\rho = \langle g,h \rangle$. In particular,
\begin{equation} \label{eq:norms_equiv}
\|h \| = \|\check h\|_\rho = \|f_0 \times h\|_{l \times k}
\end{equation}
for all $h\in \cH$. If $h\in \cH, \|h\|=1$,  is such that $f_0 \times h$ is an eigenfunction of $\tilde{\mathfrak{C}}_{c,l \times k}$ with eigenvalue $\lambda$ and $\sigma^2$ is the variance of $\epsilon$ then 
\[
\langle \tilde{\mathfrak{C}}^\otimes_c \check h,\check h \rangle_\rho = \lambda + \sigma^2 E(h^2(X)).
\]
Furthermore, if $f_0 \times h$ is an eigenfunction of $\tilde{\mathfrak{C}}_{c,l \times k}$ and $g$ is such that $\langle g, h \rangle = 0$ then 
\[
\langle \tilde{\mathfrak{C}}^\otimes_c \check h,\check g \rangle_\rho =  \langle \tilde{\mathfrak{C}}_{c,l \times k} f_0 \times h, f_0 \times g \rangle_{l\times k} + \sigma^2 E(g(X)) E(h(X)) = \sigma^2 E(g(X)) E(h(X)). 
\]
There is no reason why the latter term should be zero and the two operators will generally not have the same eigenfunctions (in the sense that $\check h$ is an eigenfunction of $\tilde{\mathfrak{C}}^\otimes_c$ iff  $f_0 \otimes h$ is an eigenfunction of $\tilde{\mathfrak{C}}_{c,l \times k}$). 

\begin{remark} \label{remark:lowest_EV_product}
	If $v =  \inf \{ E(h^2(X)) : h \in \cH, \|h \|=1\} > 0$ then $\tilde{\mathfrak{C}}^\otimes_c$ has no eigenvalue below $\sigma^2 v$. This can help with the compression, but notice that larger values of $\sigma^2$ are related to larger values $\|\rho\|_\infty$ which hinders the compression. 
\end{remark}

\paragraph{Lower bounds on the width.} 
We could also look at the width of the convex set $C_y^\otimes$ by means of the kernel function $\rho$. While this is a useful exercise we only  want to highlight here a simple relation between the width of $C_y^\otimes$ and the width of usual convex set $C_S$ (as a subset of $\cH_S$). For $h \in \cH$, $\|h\|=1$,   
\begin{align*}
	\diam_{\check h} (C_y^\otimes) &= \sup_{(x,y) \in S}  (y - f_0(x))  \langle h,  k(x,\cdot) \rangle  + h(x) f_0(x) \\ 
&\quad \quad \quad -  \infd_{(x',y') \in S} 
((y' - f_0(x'))  \langle h,  k(x',\cdot) \rangle  + h(x') f_0(x')) \\
&\geq \sup_{x \in \X_S} \sup_{y,y' \in S_x} (y-y') h(x) 
= \sup_{x \in \X_S} \sup_{y,y' \in S_x} (y-y') |h(x)|,
\end{align*}
where $\X_S = \{x: (x,y) \in S\}$ and $S_x = \{ y : (x,y) \in S\}$. Also, note that 
\[
\sup_{x\in \X_S} |h(x)| \geq (1/2) \diam_h(C_S)
\]
and we have a lower bound on  $\diam_{\check h} (C_y^\otimes)$ which is a product of the width of $C_S$ and the spread of $\epsilon$.


\subsubsection{Third setting: Unbounded $Y_i$'s \& Eq. \eqref{eq:second_modelling_appr}} \label{sec:unbounded_Y}
If $Y$ is unbounded 
then one way to approach the approximation problem is to cap the observations $Y_i$ and to control the cap as a function of $n$. We demonstrate this for the case that we have the model stated in \eqref{eq:std_reg_assump},
$f_0$ is \textit{bounded} and \textit{measurable}, and $\epsilon$ is \textit{sub-Gaussian} with variance factor $\nu$ (see \cite[Sec.2.3]{GAB13}) but is not necessarily bounded. The natural sample space is now $\R \times \X$. The kernel $\rho$ is well defined on $\R \times \X$ but is unbounded. When $f_0$ is bounded, $k$ is measurable and bounded, and $\epsilon$ is sub-Gaussian, we have that  $\langle Y,\cdot \rangle_\R k(X,\cdot) \in \mathcal{L}^2(\mu,\cH_\rho)$ since 
\[
	E( \| \langle Y,\cdot \rangle_\R k(X,\cdot) \|_\rho^2) = E( Y^2 k(X,X)) \leq \|k\|_\infty E(Y^2)
\]
and the latter term is finite since $Y$ is sub-Gaussian.
Let $\{r_n\}_{n \geq 1}$ be a non-negative and non-decreasing sequence, and let $\wideparen Y^{(n)} = (Y \wedge (r_n + \|f_0\|_\infty)) \vee - (r_n + \|f_0\|_\infty)$. The $\|f_0\|_\infty$ can obviously be replaced by an upper bound on the norm, but as the argument is developed such a bound is needed to control the error introduced by capping the observations $Y_i$. Define $\wideparen{\mean}_{y,n}^\otimes = (1/n) \sum_{i=1}^n \langle \wideparen Y^{(n)}_i, \cdot \rangle_\R k(X_i, \cdot)$. In this section our aim is to derive a suitable adapted version of Theorem \ref{thm:covariance} for this setting where $\epsilon$ is sub-Gaussian. We start by investigating the effect of the capping of $Y$.

\paragraph{Bounding $\| \mean_{y,n}^{\otimes} - \wideparen{\mean}_{y,n}^\otimes \|$.}
A simple expansion yields
\begin{align*}
	\| \mean_{y,n}^{\otimes} - \wideparen{\mean}_{y,n}^\otimes \|^2_{\rho} &= \frac{1}{n^2} \|  \sum_{i=1}^n  \langle Y_i - \wideparen Y^{(n)}_i, \cdot \rangle_\R k(X_i,\cdot) \|_\rho^2\\
&= \frac{1}{n^2} \sum_{i,j =1}^n (Y_i - \wideparen Y^{(n)}_i) ( Y_j - \wideparen Y_j^{(n)}) k(X_i, X_j). 
\end{align*}
Due to the independence of the observations and  by using the Cauchy-Schwarz inequality, 
\begin{align*}
	E( \|\mean_{y,n}^{\otimes} - \wideparen{\mean}_{y,n}^\otimes \|^2_{\rho}) \leq  &\frac{1}{n^2} \sum_{i =1}^n 
| E (Y_i - \wideparen Y^{(n)}_i)^2 k(X_i, X_i)| \\
&+ \frac{1}{n^2} \sum_{i \not = j} |\langle E((Y_i - \wideparen Y^{(n)}_i) k(X_i,\cdot)), E(( Y_j - \wideparen Y_j^{(n)})  k(X_j, \cdot) ) \rangle | \\
\leq & \frac{\|k\|_\infty}{n^2} \sum_{i =1}^n E (Y_i - \wideparen Y^{(n)}_i)^2 +  \frac{n-1}{n} E^2( |Y_1 - \wideparen Y^{(n)}_1| \|k(X_1,\cdot)\|) \\
\leq &\frac{\|k\|_\infty}{n^2} \sum_{i =1}^n E (Y_i - \wideparen Y^{(n)}_i)^2 +  \frac{(n-1)\|k\|_\infty}{n} E^2( |Y_1 - \wideparen Y^{(n)}_1|). 
\end{align*}
Let $\lambda$ denote the law of $\epsilon$. Noting that $\epsilon$ is sub-Gaussian and using \cite[Sec.2.3]{GAB13} together with \cite[252O]{FREM} and \cite[Prop2.5(a)]{DUDLEY14}, we obtain
\begin{align}
	E(|Y_i - \wideparen Y^{(n)}_i|) &\leq \int_0^\infty (t - r_n) \times \chi\{t \geq r_n \} \, d\lambda(t) + \int_{-\infty}^0  (-t - r_n) \times \chi\{ - t \geq r_n\} \, d\lambda(t)  \notag \\
					&\leq \int_0^\infty \lambda\{ t : t \geq s + r_n\} \, ds + \int_0^\infty \lambda\{t : -t \geq s + r_n \} \, ds \notag \\
					&\leq 2 \int_{r_n}^\infty e^{-s^2/2 \nu} \, ds  
					\leq \sqrt{2\pi \nu} e^{-r_n^2/ 2 \nu}.
					\label{eq:cappedYabs}
\end{align}
Similarly, 
\begin{align}
E(Y_i - \wideparen Y^{(n)}_i)^2 &\leq \int_0^\infty (t-r_n)^2 \times \chi\{t \geq r_n \} \, d\lambda(t) + \int_{-\infty}^0  (-t-r_n)^2 \times \chi\{ - t \geq r_n \} \, d\lambda(t)  \notag \\
&\leq \int_0^\infty \lambda\{ t : t \geq \sqrt{s} + r_n\} \, ds + \int_0^\infty \lambda\{t : t \leq - \sqrt{s} - r_n \} \, ds \notag \\
&\leq  2 e^{-r_n^2/2\nu} \int_{0}^\infty e^{-s/2\nu} \, ds =    4 \nu e^{-r_n^2/2 \nu}. \label{eq:cappedYVar}
\end{align}
Combining these yields
\begin{equation*} 
E( \|\mean_{y,n}^{\otimes} - \wideparen{\mean}_{y,n}^\otimes \|^2_{\rho}) \leq 2 \nu \|k\|_\infty  e^{-r_n^2/2 \nu} \Bigl(2/n + \pi  e^{-r_n^2/2 \nu}  \Bigr)\end{equation*}
and
\begin{align*}
	\Pr(\|\mean_{y,n}^{\otimes} - \wideparen{\mean}_{y,n}^\otimes \|_{\rho} \geq t ) & \leq \frac{2 \nu \|k\|_\infty  e^{-r_n^2/2 \nu} \Bigl(2/n + \pi  e^{-r_n^2/2 \nu}  \Bigr) }{t^2}.
\end{align*}
In other words, if \textit{we have an upper bound on $\|f_0\|_\infty$} and cap the observations as described above then
with probability $1- \delta$ for any $\delta \in (0,1)$, 
\begin{equation} \label{eq:mean_capped}
\|\mean_{y,n}^{\otimes} - \wideparen{\mean}_{y,n}^\otimes \|_{\rho}  \leq  \sqrt{2 \nu} \|k\|^{1/2}_\infty  e^{-r_n^2/4 \nu} \Bigl(2/n + \pi  e^{-r_n^2/2 \nu}  \Bigr)^{1/2} 
\delta^{-1/2}.
\end{equation}
\paragraph{Spectrum of the covariance operator.} 
In Remark \ref{remark:lowest_EV_product} we observed that no eigenvalue of $\tilde{\mathfrak{C}}^{\otimes}_c$ can be lower than $\sigma^2v$, where $\sigma^2$ is the variance of $\epsilon$ and   $v =  \inf \{ E(h^2(X)) : h \in \cH, \|h \|=1\}$. There is a peculiar detail that we have to be careful about: if $h_1, h_2 \in \cH$ are linearly independent but $h_1\!\!\upharpoonright \!\! \mathcal{X}_S  = h_2\!\!\upharpoonright \!\! \mathcal{X}_S$, where $\mathcal{X}_S = \{x : (x,y) \in S\}$, then there exists an $h \in \cH$ of norm one for which $E(h^2(X)) = 0$. Furthermore, the corresponding functions $\check h_1, \check h_2$ in $\cH_\rho$ are not constant on $S$ but some linear combination of $\check h_1$ and $\check h_2$ is zero on $S$ (at least when the involved functions are continuous). To make use of the lower bound $\sigma^2 v$  it makes therefore sense to move right away to functions restricted to $S_f = \{(y,x) : (x,y) \in S\}$ or $\mathcal{X}_S$.

Before coming back to the lower bound we want to take a paragraph to understand better $\cH_{\rho,S_f}$, which is the RKHS corresponding to the kernel function $\rho\!\!\upharpoonright\!\! S_f \times S_f$. First, notice that there are no (non-zero) constant functions in $\cH_{\rho,S_f}$ if $\epsilon$ is not almost surely zero. In particular, the covariance operator $\Cov_c^{\otimes,S}$ has then only eigenvalues that are strictly positive. If $\Cov_c^{\otimes}$ has zero eigenvalues then there must be elements $\check h$ which are constant on $S$ but this means that these elements have to be equal to zero on $S$ and correspond to the origin in $\cH_{\rho,S_f}$.      
Also notice that $\cH_{\rho,S_f}$ is not the same RKHS as the RKHS with kernel function $(l \!\!\upharpoonright \!\! \R_S \times \R_S) \times (k \!\!\upharpoonright \!\! \X_S \times \X_S)$, where $\R_S = \{y: (x,y) \in S\}$ and $l(y,y') = \langle y,y' \rangle_\R$ for all $y,y' \in \R$. This follows directly since they have different domains. The latter kernel is defined for pairs $(x,y)$ in $\mathcal{X}_S \times \R_S$ while the former is defined for pairs $(x,y) \in S$. This is inconvenient since we like to use $E(h^2(X))$ for  $h$ in some RKHS of functions acting on the support of some measure and it is not directly 
obvious what this support should be like. In the following, let $\mathfrak{T}_\X$ be a topology on $\X$, let $\mathfrak{T}$ be the corresponding product topology on $\X \times \R$, and assume that the law $P$ of $(X,Y)$ is a \textit{Radon measure} with $\sigma$-algebra $\mathcal{A}$; in particular, it is a  topological $\tau$-additive measure and  $\mathfrak{T} \subset \mathcal{A}$. Then $S$ is well defined as a subset of $\X \times \R$. Let us also introduce our probability space $(\Omega,\Sigma,\mu)$, assume that $\mu$ is \textit{complete} and $(X,Y)$ is a well defined random variable in the sense that $(X,Y)^{-1}[A] \in \Sigma$ for all $A \in \mathcal{A}$.  
Furthermore, consider the $\sigma$-algebra $\mathcal{A}_\X = \{A : A \times \R \in \mathcal{A}\}$ and let $P_X = P \circ \pi_\X^{-1}$ where the function $\pi_\X: \X \times \R \to \X$ projects onto the first coordinate.  The $\sigma$-algebra $\mathcal{A}_\X$ contains $\mathfrak{T}_\X$ since for $O \in \mathfrak{T}_\X$  it holds that $O \times \R$ is in the product topology $\mathfrak{T}$. 
Hence, $P_X$ is a topological measure. If $\X$ is a \textit{Hausdorff space} then $P_X$ is, in fact, a Radon measure (apply \cite[418I]{FREM} to $\pi_\X$ and note that $\pi_\X$ is continuous). This implies that  the support $S' \subset \X$ of $P_X$ is well defined. Observe that $\overline{\X_S} = S'$: the projection $\pi_\X$ is a continuous inverse-measure preserving function from $\X \times \R $ to $\X$ and, due to \cite[411N.e]{FREM}, the support of $P_\X$ is
$\overline{\pi_\X[S]} = \overline{ \{ x : (x,y) \in S\} } =   \overline{S_\X}$. We also have to check that $X$ is actually a well defined random variable in the sense that 
$X^{-1}[A] \in \Sigma$ for all $A \in \mathcal{A}_\X$, and that $E(h^2(X)) = \int h^2 \, dP_X = \int (h\!\!\!\upharpoonright\!\! \overline{\X_S})^2 \, dP_X$. The former can be seen in the following way. Since $\pi_\X$ is a measurable function from $(\X \times \R, \mathcal{A})$ to $(\X, \mathcal{A}_\X)$ it follows that $X = \pi_\X \circ (X,Y)$ is measurable as a mapping from $(\Omega,\Sigma)$ to $(\X,\mathcal{A}_\X)$ and is a well defined random variable. For the latter, if $h$ is in $\mathcal{L}^2(\X, P_X)$ then $\int h^2 \,dP_X$ is well defined and obviously equal to $\int (h\!\!\!\upharpoonright\!\!\! \overline{\X_S})^2 \, dP_X$. It remains to show the $E(h^2(X)) = \int h^2 \, dP_X$. One way to show this is to use \cite[235E]{FREM}. This can be applied since $\mu X^{-1}[A] = \Pr(X \in A, Y\in \R) = P \pi_\X^{-1}[A] = P_\X(A)   $  for any $A\in \mathcal{A}_\X$. 

We can now define an RKHS of functions that act on the support $\overline{\X_S}$. For compactness of notation let $\tilde S = \overline{\X_S}$, 
let $k_{\tilde S} = k\!\!\upharpoonright\!\!\tilde S \! \times \! \tilde S$ and denote the corresponding RKHS by $\cH_{\tilde S}$. This RKHS allows us to carry over Remark \ref{remark:lowest_EV_product} to the case where we work with $\cH_{\rho,S_f}$. Let $\Cov^{\otimes,S}_c$ be the centered covariance operator that corresponds to $\cH_{\rho,S_f}$. Since,
$\cH_{\rho,S_f} = \{ \check h \!\!\upharpoonright\!\!S_f : h \in \cH\}$, it follows that  
\begin{align} 
\langle \Cov^{\otimes,S}_c \check h \!\!\upharpoonright \!\! S_f, \check h \!\!\upharpoonright \!\! S_f \rangle_{\rho,S_f} &= E(f^2_0(X) h^2(X))  + \sigma^2 E(h^2(X)) - E^2(f_0(X) h(X))\notag \\
&= E(f_0^2(X) h^2(X))  + \sigma^2 \langle \Cov^{\tilde S} h \!\!\upharpoonright \!\! \tilde S, h \!\!\upharpoonright \!\! \tilde S   \rangle_{\tilde S} - E^2(f_0(X) h(X)). \label{eq:cov_lbd_hcheck} 
\end{align}
When the kernel function $k$ is \textit{continuous}, then the only function $h \in \cH_{\tilde S}$  for which $E(h^2(X)) = 0$ is $h = 0$ which has norm zero (otherwise there is an open set on which $h^2(X)$ is bounded away from zero and the intersection of this open set with the support has measure strictly larger than zero \cite[411N]{FREM}). In this case,  
all the eigenvalues of $\Cov^{\tilde S}$ are strictly positive and also all eigenvalues of  $\Cov^{\otimes,S}_c$ are strictly positive, implying that there is no constant function in $\cH_{\rho,S_f}$. Also, note that $\|h \!\! \upharpoonright \!\! \tilde S \|_{\tilde S} = \| \check h  \!\! \upharpoonright \!\! S_f \|_{\rho,S_f}$ for all  $h \in \cH$: first observe that $\check g \!\! \upharpoonright \!\! S_f = \check h \!\! \upharpoonright \!\! S_f$ if, and only if, $g  \!\! \upharpoonright \!\! \tilde S = h \!\! \upharpoonright \!\! \tilde S$. If  
$g  \!\! \upharpoonright \!\! \tilde S = h \!\! \upharpoonright \!\! \tilde S$ then for $(x,y) \in S$, $\check g (y,x) = y g(x) = y h(x) = \check h(y,x)$ because $x\in \X_S \subset \tilde S$. On the other hand, if $\check g \!\! \upharpoonright \!\! S_f = \check h \!\! \upharpoonright \!\! S_f$ then for any $x \in \X_S$ there exist a point $y$ such that $(x,y) \in S$ and $y g(x)  = y h(x)$ which implies $g(x) = h(x)$ if $y \not = 0$. In fact, if $\sigma^2 >0$, there exist at least two such points and, in particular, there exists a $y \not = 0$ such that $y g(x) = y h(x)$. Since $g=h$ on the dense subset $\X_S$ of $\tilde S$ and both $g,h$ are continuous (assuming $k$ is continuous) it follows that $g = h$ on $\tilde S$ (e.g. \cite[Thm1.5.4]{ENG89} and using that $\R$ is a Hausdorff space). Using  \eqref{eq:norms_equiv},
\begin{align}
	\|  \check h  \!\! \upharpoonright \!\! S_f \|_{\rho,S_f} 
	&= \inf\{ \|\check g \|_{\rho} : \check g  \!\! \upharpoonright \!\! S_f =  \check h \!\! \upharpoonright \!\! S_f, g \in \cH \} 
	=  \inf\{ \|g \| : \check g  \!\! \upharpoonright \!\! S_f =  \check h \!\! \upharpoonright \!\! S_f, g \in \cH \} \notag \\
	&=  \inf\{ \|g \| : g  \!\! \upharpoonright \!\! \tilde S =  h \!\! \upharpoonright \!\! \tilde S, g \in \cH \}
	= \|  h \!\! \upharpoonright \!\! \tilde S\|_{\tilde S}. \label{eq:check_h_S}
\end{align}
This implies that a strictly positive lower bound on the  
$\langle \Cov^{\tilde S} h \!\!\upharpoonright \!\! \tilde S, h \!\!\upharpoonright \!\! \tilde S   \rangle_{\tilde S}$ is given by the smallest eigenvalue of $\Cov^{\tilde S}$ 
(when $\cH$ is finite dimensional). We can also express this bound in terms of $\cH$ since for $h \in \cH$, $\langle \Cov h, h \rangle = \langle \Cov^{\tilde S}  h \!\!\upharpoonright \!\! \tilde S, h \!\!\upharpoonright \!\! \tilde S \rangle_{\tilde S}$ and $\|h\| \geq \|h \!\!\upharpoonright \!\! \tilde S \|_{\tilde S}$ it follows that the smallest eigenvalue of $\Cov$ provides a lower bound on the smallest eigenvalue of $\Cov^{\tilde S}$ and this lower bound is strictly positive. We might be tempted to improve this lower bound by recalling 
that the eigenvalues of $\Cov_c$ tell us the dimension of $\cH_{\tilde S}$, but notice that there is no reason why the eigenfunctions of $\Cov_c$ and $\Cov$ should be related since one corresponds to  the variance and the other corresponds to the second moment, and it is not directly obvious of how to benefit from the additional information that $\Cov_c$ provides. 

\paragraph{A family of covariance operators.}
The move from $Y$  to $\wideparen{Y}^{(n)}$ affects the covariance, and the covariance operator corresponding to $\langle Y, \cdot \rangle_\R k(X,\cdot)$ is not the same as the covariance operator corresponding to $\langle \wideparen{Y}^{(n)}, \cdot \rangle_\R k(X,\cdot)$, $n\geq 1$. Let us denote the covariance operators corresponding to the  $\wideparen{Y}$'s by the somewhat unwieldy $\CovCap$ and the covariance operator corresponding to the support  $S_n$ of the law of $(X,\wideparen{Y}^{(n)})$ by $\CovCapn$; we assume that the laws $P^{(n)}$ of $(X,\wideparen{Y}^{(n)})$ are  \textit{Radon measures}, which guarantees that \textit{the support of $P^{(n)}$ is well defined}. It is easy to verify that \textit{ $P^{(n)}$ is a Radon measure if $P$ itself is a Radon measure} and the topology corresponding to $P$ is a Hausdorff topology. 
 Consider the set $A  =   \mathcal{X} \times [- (r_n + \|f_0\|_\infty), r_n + \|f_0\|_\infty]$ equipped with the subspace topology which is also a  Hausdorff topology  \cite[Thm.2.1.6]{ENG89}, and the continuous function $f: \mathcal{X} \times \R \to A$ given by $f(x,y) = (x, \wideparen{y}^{(n)})$, where we mean the same transformation as for the random variable $Y$. The push-forward $P^{(n)} = f_\# P$ is a Radon measure according to \cite[418I]{FREM}.


We need lower bounds on the smallest non-zero eigenvalues of the different $\CovCap$ operators to use our compression approach. 
It seems natural to 
work with an assumption on the smallest eigenvalue of the covariance operator $\tilde{\mathfrak{C}}^{\otimes,S}_c$, which corresponds to the original $Y$, and to relate the eigenvalues of $\CovCapn$ back to the eigenvalues of $\tilde{\mathfrak{C}}^{\otimes,S}_c$. As discussed on the previous page, \textit{the covariance operator $\tilde{\mathfrak{C}}^{\otimes,S}_c$ does not have an eigenvalue that is zero if $k$ is continuous}. In this case, the smallest eigenvalue $\wideparen{\lambda}^{(n)}_\star$  of $\CovCapn$  is at least of size  $\bar \lambda_\star/2$, where $\bar \lambda_\star$ is  the smallest eigenvalue of  $\tilde{\mathfrak{C}}^{\otimes,S}_c$, whenever
\[
	\| \tilde{\mathfrak{C}}^{\otimes,S}_c - \CovCapn \|_{op} \leq \frac{\sigma^2 v}{2},
\]
where \eqref{eq:cov_lbd_hcheck} tells us that we can choose $v$ either as the smallest eigenvalue of $\Cov^{\tilde S}$ or the smallest non-zero eigenvalue of $\Cov$, and where $0< \sigma^2 \leq \nu$ is the variance of $\epsilon$.  Alternatively, we can obviously also directly impose assumptions on the eigenvalues of $\tilde{\mathfrak{C}}^{\otimes,S}_c$. We can bound the operator norm in the following way,
\begin{align*}
&\| \tilde{\mathfrak{C}}^{\otimes,S}_c - \CovCapn \|_{op} = 
	\sup_{\| \check h \! \upharpoonright \! S_f\|_{\rho,S_f} =1} \sup_{\| \check g \! \upharpoonright \! S_f\|_{\rho,S_f} =1\vphantom{ \| \check h \! \upharpoonright \! S_f\|_{\rho,S_f} =1}}  
	\langle \Cov^{\otimes,S}_c \check h \!\! \upharpoonright \!\! S_f- \CovCapn \check h \!\! \upharpoonright \!\! S_f, \check g \!\! \upharpoonright \!\! S_f \rangle_{\rho,S_f} \\  
&= \sup_{\|  h \! \upharpoonright \! \tilde S\|_{\tilde S} = 1} \sup_{\|  g \! \upharpoonright \! \tilde S\|_{\tilde S} = 1}  E( (Y^2 - (\wideparen{Y}^{(n)})^2) h(X)g(X))- E( Y g(X)) E( Y h(X)) \\
&\quad \quad \quad \quad \quad \quad \quad  \quad \quad \quad + E( \wideparen{Y}^{(n)} g(X)) E( \wideparen{Y}^{(n)} h(X)). 
\end{align*}
Let us first address the second moment term. For $h,g$ such that $\|  h \!\! \upharpoonright \!\! \tilde S\|_{\tilde S} = 1 = \|  g \!\! \upharpoonright \!\! \tilde S\|_{\tilde S} $,
\begin{align*}
&|E( (Y^2 - (\wideparen{Y}^{(n)})^2) h(X)g(X))| \leq \|k\|_\infty E(Y^2 - (\wideparen{Y}^{(n)})^2) \\
&= \|k\|_\infty \int_0^\infty (t^2 - r_n^2) \times \chi\{t \geq r_n \} \,d \lambda(t) +  \|k\|_\infty \int_{-\infty}^0 (t^2 - r_n^2) \times \chi \{ -t \geq r_n\} \, d\lambda(t) \\
&\leq \|k\|_\infty \int_{0}^\infty \lambda\{t: t \geq \sqrt{s + r_n^2}\} \,ds + \|k\|_\infty \int_0^\infty \lambda\{t: t\leq - \sqrt{s + r_n^2}  \} \, ds \\
&\leq 2 \|k\|_\infty \int_{r_n^2}^\infty e^{-s /2 \nu} \, ds = 4\nu \|k\|_\infty e^{-r^2_n/2 \nu}.   
\end{align*}
The other term can be controlled in the following way,
\begin{align*}
&|  E( Y g(X)) E( Y h(X)) - E( \wideparen{Y}^{(n)} g(X)) E( \wideparen{Y}^{(n)} h(X))| \\ 
&\leq | E( (Y -  \wideparen{Y}^{(n)}) g(X)) E( f_0(X) h(X))| +   |E(\wideparen{Y}^{(n)} g(X)) E((Y - \wideparen{Y}^{(n)}) h(X))| \\
&\leq \|k\|_\infty (2 \|f_0\|_\infty + r_n)  E( |Y -  \wideparen{Y}^{(n)}|) \leq   \sqrt{8 \pi \nu} \|k\|_\infty (2\|f_0\|_\infty + r_n) e^{-r_n^2/2\nu}. 
\end{align*}
Combining these yields 
\begin{equation} \label{eq:EV_lowerbnd}
\| \tilde{\mathfrak{C}}^{\otimes,S}_c - \CovCapn \|_{op} \leq  (\sqrt{8 \pi \nu} (2\|f_0\|_\infty + r_n) + 4 \nu ) \|k\|_\infty e^{-r_n^2/2\nu}.
\end{equation}
In particular, if we use $\nu = \sigma^2$,  
\begin{equation}\label{eq:def_r1}
r_1 = (1 \vee 2 \sigma^2) \vee  2 \sigma \log^{1/2} \Bigl( 
	\frac{ \|k\|_\infty (12 \sqrt{\pi} (\|f_0\|_\infty + 1) + 8 \sigma ) }{\sigma \bar \lambda_{\star,\tilde S}}
\Bigr),
\end{equation}
where $\bar \lambda_{\star,\tilde S}$ is the smallest non-zero eigenvalue of $\Cov^{\tilde S}$, and let $\{r_n\}_{n\geq 1}$ be a non-decreasing sequence then 
for all $n \geq 1$,
\[
\wideparen{\lambda}_\star^{(n)} \geq \bar \lambda_{\star,\tilde S} /2.
\]
This follows from the argument on the last page and because this choice guarantees that the right hand side of Equation \eqref{eq:EV_lowerbnd}
 is upper bounded by $\sigma^2  \bar \lambda_{\star,\tilde S}$: first notice that $r_1^2/4 \sigma^2 \geq \log r_1$ for any $r_1 \geq 1 \vee 2 \sigma^2$. Hence,
\begin{align*}
&(\sqrt{32 \pi} \sigma \|f_0\|_\infty  + 4 \sigma^2 ) \|k\|_\infty e^{-r_1^2/2\sigma^2} +  \sqrt{8 \pi} \sigma \|k\|_\infty e^{-r_1^2 / 2\sigma^2 + \log r_1} \\
&\leq (\sqrt{32 \pi} \sigma \|f_0\|_\infty  + 4 \sigma^2 ) \|k\|_\infty e^{-r_1^2/2\sigma^2} +  \sqrt{8 \pi} \sigma \|k\|_\infty e^{-r_1^2 / 4\sigma^2} \\
&\leq (\sqrt{32 \pi} \sigma (\|f_0\|_\infty +1) + 4 \sigma^2 ) \|k\|_\infty e^{-r_1^2/4\sigma^2}. 
\end{align*}
The same arguments applies to any $r_n > r_1$ and the final display is non-increasing in the $r_1$ argument. Setting this final display equal to $\sigma^2  \bar \lambda_{\star,\tilde S} /2$ yields the expression in \eqref{eq:def_r1}.

Also, notice that $r_1$ depends logarithmically on the unknown terms $\|f_0\|_\infty$ and  $\bar \lambda_{\star,\tilde S}$.

\paragraph{Compression in the case of  sub-Gaussian noise}
We have now all the ingredients to state a proposition for the \textit{sub-Gaussian noise case under the assumption that we have an upper bound on $\|f_0\|_\infty$, a lower bound on $\bar \lambda_{\star,\tilde S}$ and some control over the variance term $\sigma^2$}. In particular, we know that when 
$r_1$ is chosen as in \eqref{eq:def_r1}
 that $\CovCapn$ has eigenvalues that are closely related to the eigenvalues of $\tilde{\mathfrak{C}}^{\otimes,S}_c$ and that we can apply our results to compress
 $\wideparen{\mean}_{y,n}^\otimes$. Furthermore, we know how to control the difference between $\wideparen{\mean}_{y,n}^\otimes$ and $\mean_{y,n}^\otimes$. The following proposition ties these results together.
\begin{prop} \label{prop:unbounded}
Let $(\X, \mathfrak{T}_\X)$ be a Hausdorff space, $(\X \times \R,\mathfrak{T}, \mathcal{A},P)$ be a topological measure space such that $P$ is a Radon probability measure which has support $S$, $\mathfrak{T}$ is the product topology corresponding to $\mathfrak{T}_\X$ and the standard topology on $\R$, and let $k$ be a continuous and bounded kernel function defined on $\X$ such that the corresponding RKHS  $\cH$ is finite dimensional.  
Furthermore, let $(X_1,Y_1),\ldots, (X_n,Y_n)$ be i.i.d. random variables attaining values in $\X \times \R$, with law $P$, and of the form $Y_i = f_0(X_i) + \epsilon_i$, where $\epsilon_1,\ldots, \epsilon_n$ are centered i.i.d. random variables which are independent of $X_1,\ldots, X_n$ and such that $\epsilon_1$ is sub-Gaussian with variance $0 < \sigma^2$, and $f_0$ is a measurable and bounded function. Let $\bar \lambda_{\star,\tilde S}$ be the smallest eigenvalue of the covariance operator $\Cov^{\tilde S}$ that corresponds to the kernel function 
$k\!\!\upharpoonright\!\! \tilde S \times \tilde S$, where $\tilde S$ is the closure of $\{x : (x,y) \in S\}$ in $\X$. Furthermore, let $\bar \lambda_\star$ be the smallest eigenvalue of of the covariance operator $\Cov^{\otimes,S}_c$ corresponding to the kernel function $\rho \!\!\upharpoonright\!\! S \times S$, $\rho((y_1,x_1),(y_2,x_2)) = y_1 y_2 k(x_1,x_2)$ for all $x_1,x_2 \in \X, y_1,y_2 \in \R$. Given $q \in (0,1)$ define the sequence $\{r_n\}_{n\geq 1}$ in the following way. Define $r_1$ as in \eqref{eq:def_r1}
and for $n \geq 2$ through 
\[
r_n = r_1 \vee \sqrt{2} \sigma \log(n/q). 
\]
Under these conditions, for any $n\geq 1$, the smallest eigenvalue $\wideparen{\lambda}_\star^{(n)}$ of $\CovCapn$ fulfills  $\wideparen{\lambda}_\star^{(n)} \geq  \bar \lambda_\star/2 \geq \sigma^2 \bar \lambda_{\star,\tilde S}/2 > 0$ and there exists a ball of radius $\wideparen{\delta}^{(n)} = \wideparen{\lambda}_\star^{(n)} / (2 (\|f_0\|_\infty + r_n)^{1/2} \|k\|^{1/2}_\infty)$ around $\wideparen{\mean}^\otimes_{y}$ within the affine space spanned by $\wideparen{C}_\rho = \{ \rho((y,x),\cdot) : (x,y) \in S_n\}$ as a subset of  $\cH_\rho$, where $S_n$ is the support of the law of $P^{(n)}$ corresponding to $(X,\wideparen{Y}^{(n)})$.  Whenever $n$ is (strictly) greater than 
\begin{align*}
	&\left(
\frac{8 (\|f_0\|_\infty + r_n) \|k\|_\infty (\sqrt{2\log(12/q)} + 192 (\|f_0\|_\infty + r_n)  \|k\|_\infty /\wideparen{\lambda}^{(n)}_\star)}{(\wideparen{\lambda}^{(n)}_\star)^2} \right)^{2} \\
	&\vee  \left(\frac{16 (\|f_0\|_\infty + r_n)^{1/2} \|k\|_\infty^{1/2} +  \sqrt{288\log(4/q)}}{\wideparen{\delta}^{(n)}}\right)^{2}
\end{align*}
with probability $1-q$ there exists a ball of radius $\wideparen{\delta}^{(n)}/4$ around $\wideparen{\mean}^\otimes_{y,n}$ in $\wideparen{C}^{\otimes}_{\rho,n}$ within the affine subspace spanned by $\wideparen{C}_\rho$ and 
\[
\|\wideparen{\mean}^\otimes_{y,n}  - \mean_{y,n}^\otimes \|_\rho \leq 6 \sigma \|k\|_\infty^{1/2}  n^{-3/4}.
\]
\end{prop}
\begin{proof}
We derived the inequalities concerning the eigenvalues earlier in this section. Furthermore, the bound on $\wideparen{\delta}^{(n)}$ follows directly 
when applying Theorem  \ref{thm:covariance} to the random variables $(X,\wideparen{Y}^{(n)})$ and the kernel function $\rho\!\!\upharpoonright \!\! S_n \times S_n$, noting that $\|\rho\!\!\upharpoonright \!\! S_n \times S_n\|_\infty \leq (\|f_0\|_\infty + r_n) \|k\|_\infty $. The bound on $n$ is also taken from Theorem \ref{thm:covariance} with the only modification being that a union bound is used to guarantee simultaneously the existence of the ball around  $\wideparen{\mean}^\otimes_{y,n}$ within  $\wideparen{C}^{\otimes}_{\rho,n}$ and that 
$\|\wideparen{\mean}^\otimes_{y,n}  - \mean_{y,n}^\otimes \|_\rho$ is upper bounded by $n^{-1/2}$. In detail, for the stated $q$ with probability $1- q/2$ there exists a ball around $ \wideparen{\mean}^\otimes_{y,n}$ and with the given choice of $r_n$, with probability $1- q/2$
\begin{align*}
	\|\wideparen{\mean}^\otimes_{y,n}  - \mean_{y,n}^\otimes \|_\rho &\leq \sqrt{2 \sigma^2} \|k\|^{1/2}_\infty  e^{-r_n^2/4 \sigma^2} \Bigl(2/n + \pi  e^{-r_n^2/2 \sigma^2}  \Bigr)^{1/2} 
(q/2)^{-1/2} \\ 
	&\leq 2 \sigma \|k\|^{1/2}_\infty n^{-1/2} n^{-1/4} (2 n^{-1/2} + \pi)^{1/2} \leq 3 \sigma \|k\|_\infty^{1/2} \pi^{1/2} n^{-3/4}. 
\end{align*}
follows from \eqref{eq:mean_capped}.

\end{proof}

\subsection{Simultaneous compression} \label{sec:simultaneous}
In this section we are interested in compressing different quantities like the covariance operator and the mean element simultaneously, meaning that we want to find \textit{a single} convex combination of a subset of the data that allows us to approximate both quantities well. As mentioned in the introduction, we are utilizing a direct sum approach to approach the simultaneous compression problem.  In this section,  we start our exploration with  $\mathfrak{C}_n$ and $\mean_{y,n}$ for bounded $Y$, which is in some sense easy to deal with since the RKHSs corresponding to them have intersection $\{0\}$ (after some minor adjustments of the kernel functions) which makes the direct sum approach easy to apply. We then explore how we can deal with RKHSs $\cH_1,\cH_2$ for which the intersection is a non-trivial subspace. This problem is more challenging and we combine the direct sum approach with a quotient space approach to deal with it. We conclude this section by applying this approach to  approximate simultaneously  $\mathfrak{C}_n$, $\mean_{y,n}$ and $\sum_{i=1}^n Y_i$, which allows us to calculate the least squares error for  RKHS functions using only a core set of the data. 

\subsubsection{Compressing the covariance and weighted mean embedding simultaneously}
One of the main challenges when trying to control the approximation error of $\mathfrak{C}_n$ and $\mean_{y,n}$ simultaneously is to determine the size of the convex set that contains $(\mathfrak{C}_n,\mean_{y,n})$ within the direct sum of two RKHSs and to locate $(\mathfrak{C}_n,\mean_{y,n})$ within this convex set, or, alternatively, to analyze the covariance operator corresponding to this new space. These problems would be easier to handle if we could identify the direct sum space with an RKHS and apply the techniques that we have developed for RKHSs. When using the first approach for $\mean_{y,n}$, we face directly a problem in that we will gain some weighted sum of $(\kappa(X_i,\cdot),k(X_i,\cdot))$ as an approximation, but we need a weighted sum of $(\kappa(X_i,\cdot),Y_i k(X_i,\cdot))$. This problem can be circumvented by incorporating the $Y_i$'s into the kernel as we have done in the second approach, i.e. for a given kernel function $k$ on $\X$ let,   
\begin{equation} \label{eq:def_rho}
	\rho((y_1,x_1),(y_2,x_2)) = y_1 y_2 k(x_1,x_2) = \langle \langle y_1, \cdot \rangle_\R \otimes k(x_1,\cdot),\langle y_2, \cdot \rangle_\R \otimes k(x_2,\cdot)\rangle_\otimes
\end{equation}
then $\rho$ is a kernel function on $\mathbb{R}\times \X$ and we move from $\mean_{y,n}$ to $\mean_{y,n}^\otimes$. It helps to also extend $\kappa$ to $\mathbb{R}\times \X$ by setting $\kappa_y((y_1,x_1),(y_2,x_2)) = \kappa(x_1,x_2)$. Let $\hat h$ be the extension of $h \in \cH_\kappa$ to $\R \times \X$, i.e. $\hat h(y,x) = h(x)$ for all $x\in \X,y\in\R$, then 
$\|\hat h\|_{\kappa_y} = \|h\|_{\kappa}$.  For finite linear combinations this follows from  $\|\sum_{i=1}^n \alpha_i \kappa_y( (y_i,x_i),\cdot)\|^2_{\kappa_y} = \sum_{i,j=1}^n \alpha_i \alpha_j \kappa_y((y_i,x_i), (y_j,x_j)) = \|\sum_{i=1}^n \alpha_i \kappa(x_i,\cdot)\|^2_\kappa$, where $n\in \mathbb{N}, \alpha_i \in \R, x_i\in \X, y_i \in \R$ for all $i\leq n$ and extends to all of $\cH_\kappa$ by a denseness argument. By a similar argument we can see that the extension map is surjective.

Observe that   $\widehat{\cH \odot \cH}$, that is the RKHS corresponding to $\kappa_y$, and $\cH_\rho = \dR \otimes \cH$ are linearly independent, i.e. $(\widehat{\cH\odot \cH}) \cap (\dR \otimes \cH) = \{0\}$, because $\kappa_y((y_1,x_1),(y_2,x_2))$ does not depend on the values $y_1,y_2$ while $\rho$ does.
Due to this linear independence  
we have that $\mathcal{K} = (\widehat{\cH \odot \cH}) \oplus (\dR \otimes \cH)$ is isometrically isomorphic to $\cH_{\kappa_y + \rho}$: 
Let $\mathcal{G} = \{g + h : (g,h) \in \mathcal{K} \}$ with 
norm $\|f \|_\mathcal{G} = \inf \{\|(g,h)\|_{\mathcal{K}} : g+ h = f, (g,h) \in \mathcal{K} \}$. There exists a surjective isometry between $\mathcal{K}$ and $\mathcal{G}$. Because $\widehat{\cH \odot \cH}$ and $\dR \otimes \cH$ are linearly independent there is for every $ f\in \mathcal{G}$ exactly one pair $(g,h) \in \mathcal{K}$ such that $g + h = f$ and $\|f\|_\mathcal{G} = \|(g,h)\|_\mathcal{K}$. Furthermore, we have an inner product on $\mathcal{G}$ which is given by
$\langle f_1, f_2 \rangle_\mathcal{G} = \langle (g_1,h_1), (g_2,h_2) \rangle_\mathcal{K}$ whenever $f_1 = g_1 + h_1$ and $f_2 = g_2 + h_2$.   
For $(g,h) \in \mathcal{K}$ we have that $g\in \cH_{\kappa_y}$ and $h \in \cH_{\rho}$. By \cite[Thm.,p.353]{ARON50} the kernel $\kappa_y + \rho$ is the kernel of $\mathcal{G}$ and, therefore, $\cH_{\kappa_y + \rho}$ is isometrically isomorphic to $\mathcal{K}$.

When $P$ is a Radon measure with support $S \subset \X \times \R$ then we can look at $\mathcal{K}_S = (\widehat{\cH \odot \cH})_S \oplus (\R' \otimes \cH)_S$,  where $(\widehat{\cH \odot \cH})_S = \{u \!\! \upharpoonright \!\! S_f : u\in  \widehat{\cH \odot \cH}\}= \cH_{\kappa_y\! \upharpoonright  S_f \times S_f} $ with norm 
$\|u\|_{\kappa_y\! \upharpoonright  S_f \times S_f} = \inf \{ \|v\| : u = v \!\! \upharpoonright \!\! S_f,  v \in \widehat{\cH \odot \cH} \}$, with $S_f = \{(y,x) : (x,y) \in S\}$, and similarly we define  $(\R' \otimes \cH)_S$. If $\R_S = \{y : (x,y) \in S \}$ contains at least two elements then  $(\widehat{\cH \odot \cH})_S \cap (\R' \otimes \cH)_S = \{0\}$ and the above argument shows that $\mathcal{K}_S$ is isometrically isomorphic to $\cH_{(\kappa_y + \rho)\!\upharpoonright S_f \times S_f}$. We summarize this in the following lemma.
%
\begin{lemma} \label{lem:iso_iso_simult}
Let $\X$ be a measurable space and $k$ a measurable kernel function on $\X$ with corresponding RKHS $\cH$ then 
\[
	\widehat{\cH \odot \cH}\oplus (\dR \otimes \cH) \cong \cH_{\kappa_y + \rho}.
\]
Furthermore, if $P$ is a Radon measure on $\X \times \R$ with support $S$ and $\R_S$ contains at least two elements then 
\[
	(\widehat{\cH \odot \cH})_S \oplus (\R' \otimes \cH)_S \cong \cH_{(\kappa_y + \rho)\!\upharpoonright S_f \times S_f}.
\]
\end{lemma}
In the following, we focus on the case where $P$ is a Radon measure and study the RKHS $\cH_{(\kappa_y + \rho)\!\upharpoonright S_f \times S_f}$. For ease of notation let $\kappa_\oplus = (\kappa_y + \rho)\!\upharpoonright \! S_f \times S_f$. Similar to before, there is a natural definition for the convex set that contains our mean element. This convex set is 
\[
C_{\kappa_\oplus} = \cch \{\kappa_{\oplus}((y,x),\cdot) :  (x,y) \in S\}
\]
and the empirical analogue is 
\[
	C_{\kappa_\oplus,n} = \cch \{ \kappa_{\oplus}((Y_i,X_i),\cdot)  :  i \leq n\}.
\]
The mean element that we want to approximate is then $\mean_{\kappa_\oplus} = \int \kappa_{\oplus}((y,x),\cdot) \,dP(x,y)$ when this is well defined, and the empirical analogue is   $\mean_{\kappa_\oplus,n} = (1/n) \sum_{i=1}^n \kappa_{\oplus}((Y_i,X_i),\cdot)$.

In the following, we will assume that $Y = f_0(X) + \epsilon$ with both $f_0$ and $\epsilon$ being bounded and $\epsilon$ independent of $X$. 

\paragraph{Covariance operator} We denote the covariance operator corresponding to $\kappa_\oplus$ by $\Cov_{\kappa_\oplus}$. Because we are dealing with a direct sum one might suppose that it follows directly that the covariance operator factors into the individual covariance operators corresponding to $\kappa_y \!\upharpoonright \! S_f \times S_f$ and $\rho\!\upharpoonright \! S_f \times S_f$. Unfortunately that is not the case: for $h_1,h_2 \in \cH_{\kappa_\oplus}$ there exists $f_1,f_2 \in \cH_{\kappa_y \upharpoonright  S_f \times S_f}$ and $g_1,g_2\in \cH_{\rho \upharpoonright  S_f \times S_f }$ such that $h_i = f_i + g_i$ for $i \in \{1,2\}$, and $\|h_i\|^2_{\kappa_\oplus} = \|f_i\|^2_{\kappa_y \upharpoonright  S_f \times S_f}
+ \|g_i\|^2_{\rho \upharpoonright  S_f \times S_f }$. Hence,
\begin{align} 
	\langle \Cov_{\kappa_\oplus} h_1, h_2 \rangle_{\kappa_\oplus}  =& 
\langle \Cov_{\kappa_y \upharpoonright S_f \times S_f}  f_1, f_2 \rangle_{\kappa_y\upharpoonright S_f \times S_f} + \langle \Cov_{\rho\upharpoonright S_f \times S_f} g_1, g_2 \rangle_{\rho\upharpoonright S_f \times S_f} \notag \\ 
&+ E(f_1 \times g_2) + E(f_2 \times g_1).  \label{eq:cov_for_directsum}
\end{align}
The cross-terms do not vanish even if we use the centered covariance operator.

\paragraph{Width of $C_{\kappa_\oplus}$} We can apply our standard approach directly to the kernel function $\kappa_\oplus$ (recall that in our definition of this kernel the reduction to the support of $P$ is already incorporated) to gain insights into the convex set $C_{\kappa_\oplus}$. Alternatively, we can aim to link the width of $C_{\kappa_\oplus}$ back to the width of the corresponding convex sets corresponding to the kernel $k$ and $\kappa$.  
Due to Lemma \ref{lem:iso_iso_simult} we have that 
\[
\diam_{u} C_{\oplus} = \diam_{h} C_{\kappa_\oplus},
\]
where  $u \in (\widehat{\cH \odot \cH})_S \oplus (\R' \otimes \cH)_S , \|u\| =1$,  $h$, with $\|h\|=1$, is the corresponding element in $\cH_{\kappa_\oplus}$, and 
\[
	C_\oplus = \cch\{ 
((\kappa_y \!\!\upharpoonright\!\! S_f \times S_f)((y,x),\cdot), (\rho \!\!\upharpoonright\!\! S_f \times S_f)((y,x),\cdot)\}\subset (\widehat{\cH \odot \cH})_S \oplus (\R' \otimes \cH)_S.
\]
Hence, we can bound the width of $C_\oplus$ instead of bounding directly the width of $C_{\kappa_\oplus}$. We can write any $u \in (\widehat{\cH \odot \cH})_S \oplus (\R' \otimes \cH)_S$  as $(\hat g\!\!\upharpoonright \!\! S_f ,v\!\!\upharpoonright \!\! S_f)$, where $g \in \cH \odot \cH$, $\hat g$ is the extension of $g$ to $\R \times \X$, and  $v \in \dR \otimes \cH$. Observe that if $v$ is given by a finite linear combination of elements $\langle y_i,\cdot \rangle_\R \otimes h_i$, $y_i \in \R, h_i \in \cH$, then 
\begin{equation} \label{eq:dir_tens_helper}
	v = \sum_{i=1}^n \alpha_i (\langle y_i, \cdot \rangle_\R \otimes h_i) = \langle 1, \cdot \rangle_\R \otimes \bigl( 
		\sum_{i=1}^n y_i \alpha_i h_i  \bigr). 
\end{equation}
For such finite linear combinations let $\psi: \dR \otimes \cH \to \cH$ be $\psi(v) =\sum_{i=1}^n y_i \alpha_i h_i$. The map $\psi$ is independent of the particular representation of $v$ because if
\[
\sum_{i=1}^n \alpha_i (\langle y_i, \cdot \rangle_\R \otimes h_i) = v = \sum_{i=1}^m \beta_i (\langle z_i, \cdot \rangle_\R \otimes g_i)
\]
for a suitable $m\in\mathbb{N}$ and corresponding $\beta_i,z_i \in \R$, $g_i \in \cH$ for all $i \leq m$, then 
\[
0 = \| \langle 1, \cdot \rangle_\R \otimes (\sum_{i=1}^n y_i \alpha_i h_i - \sum_{i=1}^m z_i \beta_i g_i  )  \|^2_\rho = 
\|\sum_{i=1}^n y_i \alpha_i h_i - \sum_{i=1}^m z_i \beta_i g_i\|^2 
\]
We can also observe that $\|v\|^2_\rho  = \| \sum_{i=1}^n y_i \alpha_i h_i\|^2 = \|\psi(v)\|^2$. Furthermore, $\psi$ is linear and therefore an isometry. Since the finite linear combinations lie dense in $\dR \otimes \cH$ and $\cH$, and both $\dR \otimes \cH$ and $\cH$ are complete, we can extend $\psi$ to a surjective isometry between $\dR \otimes \cH$ and $\cH$ \cite[Cor.4.3.18]{ENG89}. In particular, any
$v \in \dR \otimes \cH$ can be represented as $\psi^{-1}(h)$ with a unique $h\in \cH$.


The width of $C_{\oplus}$  can now be lower bounded in the following way: choose $\alpha > 0$, let $b_\alpha = \sup\{b : \Pr(\epsilon \geq b) > \alpha \text{ and } \Pr(\epsilon \leq - b) > \alpha \}$ and $I = [-b_\alpha, b_\alpha]$. Then 
\begin{align*}
\diam_{(\hat g\!\upharpoonright \! S_f,\psi^{-1}(h)\!\upharpoonright \! S_f)} C_{\oplus} =\!\!\! \sup_{x\in\X_S, z \in I } \!\!\!\!(g(x) + (f_0(x) + z) h(x)) -\!\!\! \infd_{x\in \X_S, z \in I} \!\!\!\!\! (g(x) +  (f_0(x) + z)h(x))  
\end{align*}
whenever $(\hat g \!\upharpoonright \! S_f,\psi^{-1}(h)\!\upharpoonright \! S_f)$ has norm one, $g\in \cH\odot \cH$ and $h\in \cH$. In particular, when choosing the same point $x$ and using $z$ to move to absolute values, we gain
\begin{equation} \label{eq:sim_diam_lb_h}
\diam_{(\hat g\!\upharpoonright \! S_f,\psi^{-1}(h)\!\upharpoonright \! S_f)} C_\oplus \geq 
2 \|h\| b_\alpha  \sup_{x \in \X_S} \frac{|h(x)|}{\|h\|} \geq b_\alpha \|h\| \diam_{h/\|h\|}(C_{\X_S}),   
\end{equation}
where $C_{\X_S}$ is the usual convex set for the kernel $k\!\!\upharpoonright\!\! \X_S \times \X_S$. 

We need to complement this bound with a bound that is based on $g$ to deal with cases where $\|h\|$  is small. 
When $\|f_0\|_\infty$ is smaller than $b_\alpha$ then there is a  simple way to get a lower bound that involves $g$. For two points $x_1,x_2 \in \X_S$ and  any $h\in \cH$, we can chose $z_1,z_2 \in \R_S$ such that $(f_0(x_1)  + z_1) h(x_1) \geq 0 \geq (f_0(x_2) + z_2) h(x_2)$ and, hence, 
\begin{equation} \label{eq:sim_diam_lb_g}
\diam_{(\hat g \!\upharpoonright \! S_f,\psi^{-1}(h)\!\upharpoonright \! S_f)} C_\oplus \geq \sup_{x\in \X_S} g(x) - \infd_{x \in \X_S} g(x) \geq \|g\|_{\cH \odot \cH} \, \diam_{g/\|g\|_{\cH\odot \cH}} C_\odot.
\end{equation}
We can combine \eqref{eq:sim_diam_lb_h} and \eqref{eq:sim_diam_lb_g} to gain
a lower bound on the width of $C_\oplus$ in terms of  the widths of $C$ and $C_\odot$. 

\subparagraph{The low noise setting.} The situation gets more complicated when $|f_0|$ attains values that are significantly larger than $b_\alpha$. For instance, when there is no noise, i.e. $\epsilon = 0$ (a.s.), and there exists some $h\in \cH, g\in \cH \odot \cH$ such that $h(x) \not = 0$ and $f_0(x) = - g(x)/h(x)$ on $\R_S$, and  $(\hat g\!\upharpoonright \! S_f,\psi^{-1}(h)\!\upharpoonright \! S_f)$ has norm one, it holds that $\diam_{(\hat g\!\upharpoonright \! S_f,\psi^{-1}(h)\!\upharpoonright \! S_f)} C_\oplus = 0$. For $f_0$ to be equal or close to $-g /h $ it is necessary that $f_0$ attains large values when $\|h\|$ is small. For example, when $\cH \odot \cH$ is finite dimensional with dimension $d$, $\lambda_d > 0$ is the smallest eigenvalue of a suitable kernel matrix based on the kernel $\kappa$, and $k$ is a bounded kernel, then  $f_0(x) = - g(x)/h(x)$ can only happen if 
\[
\sup_{x\in \X} |f_0(x)| \geq \frac{\|g\|_\infty}{\|h\|_\infty} \geq \frac{\|g\| \lambda_d^{1/2}}{d^{1/2} \|h\| \|k\|_\infty^{1/2}}.
\]
For a small value of $\|h\|$ this implies that $\|g\|$ will be close to $1$ and $|f_0|$ has to attain a large value at some locations $x \in \X$.

\subparagraph{Interpolation and another look at the low noise setting} \label{sec:simul_interpol}
We look now at the case where there is no noise at all, that is $Y = f_0(X)$, 
\[
	C_\oplus = \cch\{(\kappa_y((f_0(x),x),\cdot), \langle f_0(x),\cdot \rangle_{\dR} \otimes k(x,\cdot)) : x\in \X  \} 
\]
and we are interested in interpolating $f_0$. 
In particular,  we are controlling the width of $C_\oplus$ depending on how $f_0$ is related to $\cH$ and $\cH \odot \cH$. The direct sum approach is useful to gain a deeper understanding of how well $(\mathfrak{C}_y, \mean_y^\otimes)$ can be approximated. The width of $C_\oplus$ in this interpolation setting has a simple form.
Assume that the support of the marginal measure is all of $\X$ and since there is no noise it then follows that the support of the measure $P$ is $S = \{ (x, f_0(x)) : x \in \X\}$.   For $g \in \cH \odot \cH, h \in \cH$, 
\[
\diam_{(\hat g,\psi^{-1}(h))} C_\oplus  = \sup_{x\in\X} (g(x) + f_0(x) h(x)) - \infd_{x\in \X} (g(x) +  f_0(x)h(x)) 
\]
The functions $f_0 \times h$  lie in the RKHS $\cH_{f_0}$ which has the kernel function $k_0(x,y) = f_0(x) k(x,y) f_0(y)$. According to \cite[Prop.5.20]{PAUL16} the RKHS $\cH_{f_0}$ is equal to $\{f_0 \times h : h \in \cH\}$ and the inner product on $\cH_{f_0}$ is given by $ \langle f_0 \times h_1, f_0 \times h_2\rangle_{f_0} = \langle h_1,h_2 \rangle$ whenever $h_1,h_2 \in \cH$. If $\cH_{f_0} \cap (\cH \odot \cH) = \{0\}$ then we can embed both $\cH_{f_0}$ and $\cH \odot \cH$ in
the direct sum $\mathcal{G} = (\cH \odot \cH) \oplus \cH_{f_0}$  such that for any $f \in \cH_{f_0}, h \in \cH \odot \cH$ it holds that $\|f\|_{f_0} = \|(0,f)\|_{\mathcal{G}}$ and $\|h \|_{\cH \odot \cH} = \|(h,0)\|_{\mathcal{G}}$. As in Lemma \ref{lem:iso_iso_simult} it holds that $\mathcal{G} \cong \cH_{k_0 + \kappa}$  and, therefore, it also holds that $\|f\|_{k_0} = \|f\|_{k_0 + \kappa}$ and $\|h\|_{\cH \odot \cH} = \|h\|_{k_0 + \kappa}$. In this case, 
\[\diam_{(\hat g,\psi^{-1}(h))} = \diam_{(g, f_0 \times h)} C_\mathcal{G}
\]
whenever $g\in \cH\odot \cH, h \in \cH$, where $C_\mathcal{G} = \cch \{ (\kappa(x,\cdot), k_0(x,\cdot)) : x \in \X \} \subset \mathcal{G}$. This follows directly from
\begin{align*}
g(x) + f_0(x) h(x) = \langle (g,f_0 \times h), (\kappa(x,\cdot),k_0(x,\cdot)) \rangle_\mathcal{G}.  
\end{align*}
We can now follow the approach from Section \ref{sec:finite_dim_appox_LB} and, in particular, apply Proposition \ref{prop:approx_lower_bnd} to the RKHS with kernel $k_{f_0} + \kappa$. Assumptions on $f_0$ imply then lower bounds on the width.


\subsubsection{Linearly dependent spaces} \label{sec:quotient_lin_dep}
The setting above where we approximate $\mathfrak{C}_n$ and $\mean^\otimes_{y,n}$ simultaneously is easy to deal with because the corresponding RKHSs are linearly independent.
On the other hand, when the spaces over which we want to optimize are linearly dependent then the RKHS is not isometrically isomorphic to the direct product space and the approach needs to be modified. This can happen, for example, when we try to approximate $\mathfrak{C}$ simultaneously to $\mean$. In this context the corresponding spaces $\cH \odot \cH$ and $\cH$ can overlap. For instance, when $k$ is a polynomial kernel of order two then $\cH$ and $\cH \odot \cH$ are not linearly independent.

Whenever $\cH \odot \cH$ and $\cH$ are not linearly independent it is natural to identify elements like $(h,0)$ and $(0,h)$, $h\in (\cH \odot \cH) \cap \cH$. One way to do so is to consider the subspace $U = \{(-h,h) : h \in  (\cH \odot \cH) \cap \cH \}$ of $\mathcal{K}:= (\cH \odot \cH) \oplus \cH$. The subspace is closed: let $\{(-h_n,h_n)\}_{n \in \mathbb{N}}$ be a convergent sequence in $U$. This sequence is also a Cauchy sequence and for any $\epsilon >0$ there exists an $N \in \mathbb{N}$ such that for all $n,m \geq \mathbb{N}$, 
\[
	\epsilon > \| (-h_n,h_n) - (-h_m,h_m)\|_\oplus^2 = \|h_m - h_n\|^2 + \|h_n -h_m\|^2_{\cH \odot \cH} 
\]
and $\{h_n\}_{n\in \mathbb{N}}$ is a Cauchy sequence both in $\cH$ and $\cH \odot \cH$. Hence, it converges in both spaces. Let $f$ be its limit in 
$\cH \odot \cH$ and $g$ its limit in $\cH$ then for any $x \in \X$ there exists an $n\in\mathbb{N}$ such that
$|f(x) - g(x)| \leq \epsilon + |h_n(x) - h_n(x)| = \epsilon$ and $f = g$. It also follows right away that 
$\lim_{n\rightarrow \infty} \|(-h_n,h_n) - (-f,g)\|_\oplus  = 0 $ and the sequence has its limit in $U$.

Consider the quotient space $\mathcal{K} / U$ with co-sets $f^\bullet = f + U$, $f\in \mathcal{K}$, and the quotient norm $\|f^\bullet \|_{\mathcal{K}/U} = \inf\{\|f+h\|_\mathcal{K} : h \in U \}$. The space $\mathcal{K}/U$ is again a Hilbert space since $U$ is closed (e.g. \cite[Sec.III.4]{REED72}), and it is isometrically isomorphic to the Hilbert space
$\cH \odot \cH + \cH$ when the latter is equipped with the norm $\|f\|_+^2 = \inf \{\|g\|^2_{\cH \odot \cH} + \|h\|^2 : f=g+h,  g \in \cH \odot \cH, h \in  \cH \} $;
in particular, a co-set $(g,h) + U \in \mathcal{K}/U$ is mapped to the function $f = g + h$. This map is well defined since if $(g_1,h_1) \in (g,h)^\bullet$ then there is some $h_2$ such that $g_1 + h_1 = g - h_2 + h + h_2 = f$. Furthermore, by the choice of $U$, there are no two elements  $u^\bullet,v^\bullet \in \mathcal{K}/U$, $u^\bullet \not= v^\bullet$, that are mapped to the same function $f$. Assume otherwise, then there is some $f$ such that $f = g_1 +h_1 = g_2 + h_2$ and, therefore, $(g_2 + g_1 -  g_2, h_2 - g_1 + g_2) = (g_1, h_1)$. Since $g_1 - g_2 \in \cH \odot \cH$ and $g_1 - g_2 = h_1 - h_2 \in \cH$ it
follows that $(g_2,h_2)^\bullet = (g_1, h_1)^\bullet$ which contradicts the assumption. Finally, any element in $\cH \odot \cH +  \cH$ can be represented this way since if $f =  g + h$, $g \in \cH \odot \cH, h \in \cH$ then $(g,h)^\bullet$ is mapped to $f$. 
Using again \cite[Thm.,p.353]{ARON50} we can conclude that $\mathcal{K}/U$ and $\cH_{\kappa + k}$ are isometrically isomorphic.

While $\mathcal{K}/U$ and $\cH_{\kappa + k}$ are isometrically isomorphic it does not hold in general that $\mathcal{K}$ and $\mathcal{K}/U$  are isometrically isomorphic to $\cH_{\kappa + k}$. Hence, when mapping an element $u\in \mathcal{K}$ to $u^\bullet \in \mathcal{K}/U$, then finding an approximation $v^\bullet$ of $u^\bullet$ in $\mathcal{K}/U$, we generally cannot invert the $\bullet$ operation to gain an approximation of $u$. Selecting an arbitrary element in $v^\bullet$ does not work either since a small value of $\|u^\bullet - v^\bullet\|_{\mathcal{K}/U}$ does not imply that all elements in the corresponding co-sets have small distances, i.e. there is no reason why $\sup_{w\in v^\bullet} \|u - w\|_\mathcal{K}$ should be small. However, we are no trying to approximate arbitrary elements in $\mathcal{K}$ but only elements
\[
	(\mathfrak{C},\mean) = \frac{1}{n} \sum_{i=1}^n (\kappa(X_i,\cdot),k(X_i,\cdot))
\]
and we are optimizing the approximation over $\tilde C = \cch \{(\kappa(x,\cdot), k(x,\cdot))  : x\in \X \} \subset \mathcal{K}$. The important observation is that for any non-zero element $(-h,h) \in U$, that is $h\in (\cH \odot \cH) \cap \cH$, we have
\[
	\langle (-h,h), (\kappa(x,\cdot), k(x,\cdot)) \rangle_\mathcal{K} = -h(x) + h(x) = 0,
\]
and $\tilde C$ is a subset of $U^\perp$. 

The subspace $U^\perp$ together with the inner product inherited from $\mathcal{K}$ is isometrically isomorphic to $\cH_{\kappa + k}$. This follows since $\mathcal{K}/U$ and $\cH_{\kappa+k}$ are isometrically isomorphic and $U^\perp$ and $\mathcal{K}/U$ are isometrically isomorphic. The latter holds since every co-set corresponds to exactly one element in $U^\perp$, and for $u \in U^\perp$, $\|u^\bullet\|_{\mathcal{K}/U} = \inf\{\|u + v\|_\mathcal{K} :v \in U\}= \|u\|_\mathcal{K}$. 

Also, $\cspn \tilde C = U^\perp$. We know already that $\cspn \tilde C \subset U^\perp$. To show that they are equal let 
$ \mathcal{K} = \spn ( (\cspn \tilde C) \cup U)$. Observe that this space is closed since $\cspn \tilde C$ and $U$ are, and because they are orthogonal. It is sufficient to show that $(f,0) \in \mathcal{K}, (0,g) \in \mathcal{K}$ for all $f \in \cH \odot\cH$ and $g \in \cH$ since the smallest closed subspace that contains all these elements is $(\cH \odot \cH) \oplus \cH$.    

For $f = \sum_{i=1}^n \beta_i (\kappa(x_i,\cdot) + k(x_i,\cdot))\in \cH_{\kappa +k}$ define $\psi(f) = \sum_{i=1}^n \beta_i (\kappa(x_i,\cdot), k(x_i\cdot)) \in \cspn \tilde C \subset (\cH \odot \cH) \oplus \cH$. The operator $\psi: \cH_{\kappa +k} \to \cspn \tilde C$ is linear and defined on a dense subset of $\cH_{\kappa + k}$. It is furthermore norm preserving since 
\[
\| \psi(f) \|^2_\oplus  = \sum_{i,j=1}^n \beta_i \beta_j \kappa(x_i,x_j) + \sum_{i,j=1}^n \beta_i \beta_j k(x_i,x_j) = \|f\|_{\kappa +k}^2. 
\] 
Hence, it can be extended to a linear isometry, which we will also denote by $\psi$, between $\cH_{\kappa + k}$ and $\cspn \tilde C$ with the norm inherited from $(\cH \odot \cH) \oplus \cH$. 

For any $h\in (\cH \odot \cH) \cap \cH$ we can infer that it lies  in the RKHS with kernel $\kappa + k$ due to 
 \cite[Thm.,p.353]{ARON50} and $\psi(h)$ lies in $\cspn \tilde C$.  Write $\psi(h)$ as   $(h_1, h_2)$, $h_1 \in \cH \odot \cH, h_2 \in \cH$, then for all $x\in \X$, 
\begin{align*}
	h_1(x) + h_2(x) &= \langle (h_1,h_2), (\kappa(x,\cdot), k(x,\cdot)) \rangle_\oplus = \langle \psi(h), \psi(\kappa(x,\cdot) + k(x,\cdot)) \rangle_\oplus \\
&= \langle h,\kappa(x,\cdot) + k(x,\cdot) \rangle_{\kappa +k} = h(x). 
\end{align*}
In other words, for any $h\in (\cH \odot \cH)\cap \cH$ we have  $h_1 \in \cH \odot \cH, h_2 \in \cH$ such that $h= h_1 + h_2$ and $(h_1,h_2) \in \cspn \tilde C$. Since $h,h_1 \in \cH \odot \cH$ it follows that $h_2 = h - h_1 \in (\cH \odot \cH) \cap \cH$ and $(h_2, -h_2) \in U$. Thus, 
$(h,0) = (h_1,h_2) + (h_2,-h_2) \in \mathcal{K}$. Similarly, we can observe that $(0,h) \in \mathcal{K}$.

For $f\in \cH \odot \cH$ let $\psi(f) = (f_1,f_2)$ with $f_1 \in \cH\odot \cH$ and $f_2 \in \cH$. In other words, $f = f_1 + f_2$ and 
since $\cH \odot \cH$ is a linear space, we know that $f_2 = f - f_1 \in (\cH \odot \cH) \cap \cH$. And, as above, we can conclude that $(f, 0)$ also lies in $\mathcal{K}$. The same argument also shows that for any $g \in \cH$ we have $(0,g) \in \mathcal{K}$. Hence, $\mathcal{K} = (\cH \odot \cH) \oplus \cH$ and 
$\cspn \tilde C = U^\perp$.





\begin{lemma} \label{lem:iso_iso_quotient}
Let $\X$ be a measurable space and $k$ a measurable kernel function on $\X$ with corresponding RKHS $\cH$ then $\cspn\{(\kappa(x,\cdot), k(x,\cdot)) : x \in \X  \} \subset (\cH \odot \cH) \oplus \cH$ equipped with the inner product of $(\cH \odot \cH) \oplus \cH$
is isometrically isomorphic to $\cH_{\kappa + k}$.
\end{lemma}

\subsubsection{Simultaneous least-squares risk approximation for unbounded $Y$} \label{sec:simult_risk} 
Often it is unnecessary to include the  $(1/n) \sum_{i=1}^n Y_i^2$ term in the simultaneous approximation since many methods only rely on the terms that include $f$ (e.g. the ridge regressor) and also $(1/n) \sum_{i=1}^n Y_i^2 \in \R$ itself can be represented by a single real number and does not need to be compressed. However, when selecting points $(X_{\iota(1)},Y_{\iota(m)})$, $m \ll n$, for a coreset then
\[
\frac{1}{m} \sum_{i=1}^m (f^2(X_{\iota(i)}) -2 Y_{\iota(i)} f(X_{\iota(i)})) + \frac{1}{n} \sum_{i=1}^n Y_i^2  
\]
is not the mean squared error of $f$ given the sample $X_{\iota(1)},Y_{\iota(1)}, \ldots,X_{\iota(m)},Y_{\iota(m)}$ and might even be negative.
An easy way to remedy this problem is to move to $(1/m) \sum_{i=1}^m Y_{\iota(i)}^2$ but then we do not have any guarantee that this is close 
to $(1/n) \sum_{i=1}^n Y_i^2$. An alternative is to include the $Y_i$'s in the simultaneous approximation problem. This can be done by, for instance, 
defining a kernel on $\R \times \X$ through $r((y_1,x_1),(y_2,x_2)) = \langle y_1, y_2 \rangle_\R$ and by considering the direct sum
\[
\widehat{\cH \odot \cH} \oplus (\dR \otimes \cH) \oplus \cH_r. 
\]
Alternatively, we can restrict the functions to the support $S$ of the underlying measure and consider  
\[
	(\widehat{\cH \odot \cH})_S \oplus (\dR \otimes \cH)_S \oplus \cH_{r\upharpoonright S_f \times S_f}. 
\]
If there is \textit{no constant function in the RKHS $\cH_{\X_S}$} then  $(\dR \otimes \cH)_S \cap  \cH_{r\upharpoonright S_f \times S_f} = \{0\}$ and
\[
		(\widehat{\cH \odot \cH})_S \cap ((\dR \otimes \cH)_S \cap  \cH_{r\upharpoonright S_f \times S_f}) =\{0\}
\]
if, furthermore, \textit{$\R_S$ contains at least two different values}: any function in $\widehat{\cH \odot \cH}$ is of the form $g^2(x), x \in \X_S, g\in \cH$ and functions in 
$(\dR \otimes \cH)_S \cap  \cH_{r\upharpoonright S_f \times S_f})$ are of the form $(y,x) \mapsto y h(x) + c y$ for some constant $c$.  For any functions $g,h\in \cH$, choose $y_1, y_2 \in \R_S$, $y_1 \not = y_2$, and $x_1, x_2 \in \X_S$ are such that $h(x_1) \not = h(x_2)$. For $g^2$ to be equal to $yh(x) +cy$ it has to hold that   $g^2(x_1)$ is equal to $y_1 h(x_1) + c y_1$  and it also has to be equal to $y_2 h(x_1) + c y_2$ In other words,
$ (y_1 - y_2) h(x_1) = c (y_2 - y_1)$ and $h(x_1) = -c$ and similarly for $h(x_2)$, that is $h(x_1) = h(x_2)$ with a contraction to the choice of $x_1$ and $x_2$. Hence, under these conditions we can identify the direct sum with an RKHS corresponding to a sum of kernels,
\[
	(\widehat{\cH \odot \cH})_S \oplus (\dR \otimes \cH)_S \oplus \cH_{r\upharpoonright S_f \times S_f}  \cong
	\cH_{(\kappa_y + \rho + r)\!\upharpoonright S_f \times S_f}.
\]
When the RKHS $\cH_{\X_S}$ contains the constant function then the intersection $(\R' \otimes \cH)_S \cap \cH_{r \upharpoonright S_f \times S_f}$ is not empty since the function $(y,x) \mapsto y$, with domain $S$,  lies in $(\R' \otimes \cH)_S$ and in $\cH_{r\upharpoonright S_f \times S_f}$. We can follow the same approach as in Section \ref{sec:quotient_lin_dep} and consider the one dimensional subspace $U = \{(-h,h) : h \in (\R' \otimes \cH)_S \cap 
\cH_{r \upharpoonright S_f \times S_f}\}$ of $(\dR \otimes \cH)_S \oplus \cH_{r\upharpoonright S_f \times S_f}$ and consider the quotient space 
$\mathcal{Q} =  ((\dR \otimes \cH)_S \oplus \cH_{r\upharpoonright S_f \times S_f}) / U$ with the usual quotient norm. The space $\mathcal{Q}$ is a Hilbert space  \cite[Sec.III.4]{REED72}. By the same argument as in Section  \ref{sec:quotient_lin_dep} we can infer that 
\[
	\mathcal{Q} \cong \cH_{(\rho + r)\upharpoonright S_f \times S_f} 
\]
and for $(x,y) \in S$, $(\rho((y,x),\cdot), r((y,x),\cdot))$ lies in $U^\perp \subseteq  (\dR \otimes \cH)_S \oplus \cH_{r\upharpoonright S_f \times S_f}$. Also, the space $U^\perp$, with the inherited inner product, is isometric isomorphic to  $\cH_{(\rho + r)\upharpoonright S_f \times S_f}$. 
 \textit{When $\R_S$ contains at least two elements} then $(\widehat{\cH \odot \cH})_S \cap ((\dR \otimes \cH)_S \cap  \cH_{r\upharpoonright S_f \times S_f}) = \{0\}$ and 
\[
		(\widehat{\cH \odot \cH})_S \oplus \mathcal{Q} \cong \cH_{(\kappa_y + \rho + r)\upharpoonright S_f \times S_f}. 
\]
We will apply these results to the problem of ridge regression and it is convenient to have a proposition which provides guarantees on the approximation in the ridge regression context. Since we do not need to approximate the sum of the $Y_i^1$ terms to compute the ridge regression estimator we will consider the space $(\widehat{\cH \odot \cH})_S \cap (\dR \otimes \cH)_S$. We  make the assumption that  \textit{$\cH$ does not contain the constant functions}, which removes the need to consider quotient spaces. Furthermore, we will assume \textit{sub-Gaussian noise} and that \textit{$f_0$ is bounded} but we will allow for \textit{unbounded $Y_i$ random variables}. Recall the definitions of the kernel functions $\rho$ in \eqref{eq:def_rho}, $\kappa:\X \times \X \to \R$, $\kappa = k^2$, and its extension $\kappa_y$ (see just below \eqref{eq:def_rho}). Before stating a result on the compression, we need to modify the arguments that we used to control the difference between $\mean_{y,n}^\otimes$ and $\wideparen \mean_{y,n}^\otimes$, and the difference between $\tilde{\mathfrak{C}}^{\otimes,S}_c$ and  $\CovCapn$. This is necessary since the kernel function, which we will denote by $\tau$ below, is $(\kappa_y + \rho)((y,x), (y',x')) = \kappa(x,x') + yy' k(x,x')$  in the current context, and this kernel function is not of the form $yy' \tilde k(x,x')$, where $\tilde k$ is some kernel on $\X$. Since we assumed that latter form in Section \ref{sec:unbounded_Y} when we derived the bounds on the differences, we cannot simply reuse the earlier results. Fortunately, the necessary modifications are minor: Let $\mean_{\tau,n}$ be the empirical mean element corresponding to the kernel function $\tau$ on $\R \times \X$
and 
\[
	\wideparen \mean_{\tau,n} = \frac{1}{n} \sum_{i\leq n} \tau((\wideparen{Y}_i^{(n)}, X_i), \cdot )  =  \frac{1}{n} \sum_{i\leq n} (\kappa(X_i,\cdot) +   \langle \wideparen{Y}_i^{(n)},\cdot \rangle_\R k(X_i,\cdot)).
\]
The distance between the empirical mean element and its capped version is
\begin{align*}
\| \mean_{\tau,n} - \wideparen \mean_{\tau,n} \|^2_\tau 
&= \frac{1}{n^2} \| \sum_{i\leq n} \langle Y_i - \wideparen Y_i^{(n)},\cdot \rangle_\R k(X_i,\cdot)  \|^2_\tau \\
&\leq\frac{1}{n^2} \| \sum_{i\leq n} \langle Y_i - \wideparen Y_i^{(n)},\cdot \rangle_\R k(X_i,\cdot)  \|^2_\rho  
\end{align*}
since the $\kappa$ terms cancel and because the function inside the norm lies within $\cH_\rho$ (apply \cite[Thm.5.4]{PAUL16} to get the inequality). 
This implies that we can reuse the bound in \eqref{eq:mean_capped}, and with probability $1-\delta$, $\delta \in (0,1)$, we have that
\begin{equation}
\| \mean_{\tau,n} - \wideparen \mean_{\tau,n} \|^2_\tau \leq  \sqrt{2 \nu} \|k\|^{1/2}_\infty  e^{-r_n^2/4 \nu} \Bigl(2/n + \pi  e^{-r_n^2/2 \nu}  \Bigr)^{1/2} 
\delta^{-1/2}, \label{eq:direct_sum_mean}
\end{equation}
where $\nu>0$ is the variance factor corresponding to the sub-Gaussian noise terms.
We also need control over the covariance operators corresponding to the capped $Y_i$'s. We proceed as in Section \ref{sec:unbounded_Y}. Assuming that the law $P$ of $(Y,X)$  is a Radon measure let $S$ be its support and let $S_n$ be the support of $(\wideparen Y^{(n)}, X)$ (which is well defined as the law of this random variable is again a Radon measure). Let $\Cov_{c,\tau}^S : \cH_\tau \to \cH_\tau$ be the covariance operator corresponding to the original random variable and  $\CovCapntau: \cH_\tau \to \cH_\tau$ the covariance operator corresponding to the capped random variable. We start by bounding the difference between these covariance operators in the operator norm,
\begin{align*}
\|\Cov_{c,\tau}^S - \CovCapntau\|_{op} = \sup_{\|h_1 \!\upharpoonright\! S_f\|_{\tau,S_f} =1} \sup_{\|h_2\!\upharpoonright\! S_f\|_{\tau,S_f} =1}
\langle (\Cov_{c,\tau}^S - \CovCapntau) h_1 \!\upharpoonright\! S_f, h_2 \!\upharpoonright\! S_f\rangle_{\tau,S_f}. 
\end{align*}
Due to Lemma \ref{lem:iso_iso_simult} the space $\cH_{\tau,S_f}$ is isometrically isomorphic to a direct sum space and, as above \eqref{eq:cov_for_directsum}, we can write $h_i = f_i + g_i$ for $i \in \{1,2\}$, $f_1,f_2 \in \cH_{\kappa_y\upharpoonright S_f \times S_f}, g_1,g_2 \in \cH_{\rho  \upharpoonright S_f \times S_f}$ and such that the squared norms of the $h_i$ equals the sum of the squared norms of the $f_i$ and $g_i$. We can proceed by expanding the $h_i$'s and observing that $f_i(\wideparen y,x) = f_i(y,x)$ for all $y \in \R$ and $x\in \X$ since $f_i$ is a function of the second coordinate only, 
\begin{align*}
&\langle (\Cov_{c,\tau}^S - \CovCapntau) h_1 \!\upharpoonright\! S_f, h_2 \!\upharpoonright\! S_f\rangle_{\tau,S_f} \\
&= E( h_1(Y,X) h_2(Y,X))  - E(h_1(Y,X)) E(h_2(Y,X))  \\ &\quad - E( h_1(\wideparen Y^{(n)},X) h_2(\wideparen Y^{(n)},X))  + E(h_1(\wideparen Y^{(n)},X)) E(h_2(\wideparen Y^{(n)},X)). 
\end{align*}
The difference of the bias terms becomes 
\begin{align*}
&E(h_1(\wideparen Y^{(n)},X)) E(h_2(\wideparen Y^{(n)},X)) - E(h_1(Y,X)) E(h_2(Y,X))  \\
&= E(g_1 (\wideparen Y^{(n)},X)) E(g_2(\wideparen Y^{(n)},X)) -  E(g_1 (Y,X)) E(g_2(Y,X))\\ 
&\quad + E(f_1(Y,X)) (E(g_2(\wideparen Y^{(n)},X)) - E(g_2(Y,X)))\\
&\quad + E(f_2(Y,X)) (E(g_1(\wideparen Y^{(n)},X)) - E(g_1(Y,X))).
\end{align*}
Since $f_1,f_2$ have norm one it follows that $E(f_1(Y,X))$ and $E(f_2(Y,X))$ are upper bounded by $\|\kappa\|_\infty^{1/2} = \|k\|_\infty$. Also recall that $g_i$, $i\in \{1,2\}$, can be written as $\check{u}_i$ for some $u_i \in \cH_k$ and $u_i$ has norm one (see \eqref{eq:norms_equiv}). Hence 
\[
|E(g_i(\wideparen Y^{(n)},X)) - E(g_i(Y,X)))| \leq \|k\|_\infty^{1/2} E(|\wideparen Y^{(n)} - Y|)
\]
and the bound \eqref{eq:cappedYabs} can be used. Similarly, 
\begin{align*}
	&|E(g_1 (\wideparen Y^{(n)},X)) E(g_2(\wideparen Y^{(n)},X)) -  E(g_1 (Y,X)) E(g_2(Y,X))|\\
	&\leq |E(g_1 (\wideparen Y^{(n)},X))| \, |E(g_2(\wideparen Y^{(n)},X)) - E(g_2(Y,X))| \\
	&\quad + |E(g_2 (Y,X))| \, |E(g_1(\wideparen Y^{(n)},X)) - E(g_1(Y,X))| \\
	&\leq  2 \|k\|_\infty E(|\wideparen Y^{(n)} - Y|) (E(|Y- f_0(X)|) + E(|f_0(X)|)) \\ 
	&\leq 2 \|k\|_\infty (\sigma + \|f_0\|_\infty ) E(|\wideparen Y^{(n)} - Y|), 
\end{align*}
where we assume that the noise term has variance $\sigma^2 > 0$ and $f_0$ is bounded and measurable. Hence, 
\begin{align*}
&|E(h_1(\wideparen Y^{(n)},X)) E(h_2(\wideparen Y^{(n)},X)) - E(h_1(Y,X)) E(h_2(Y,X))|  \\
&\leq 2 \|k\|_\infty (\sigma + \|f_0\|_\infty +\|k\|_\infty^{1/2}) E(|\wideparen Y^{(n)} - Y|). 
\end{align*}
We can deal with the covariance terms in the same way,
\begin{align*}
&E( h_1(Y,X) h_2(Y,X)) -   E( h_1(\wideparen Y^{(n)},X) h_2(\wideparen Y^{(n)},X)) \\ 
&= E( g_1(Y,X) g_2(Y,X)) - E(g_1(\wideparen Y^{(n)},X) g_2(\wideparen Y^{(n)},X)) \\
&\quad + E( f_1(Y,X) (g_2(Y,X) - g_2(\wideparen Y^{(n)},X))) \\
&\quad + E( f_2(Y,X) (g_1(Y,X) - g_1(\wideparen Y^{(n)},X))) \\
& \leq \|k\|_\infty^{1/2} (E(|Y| |g_1(Y,X) - g_1( \wideparen Y^{(n)},X)|
) + E(|Y| |g_2(Y,X) - g_2( \wideparen Y^{(n)},X)|))
\\
&\quad + 2 \|k\|^{3/2}_\infty E(|Y - \wideparen Y^{(n)}|) \\
& \leq 2 \|k\|_\infty E(|Y| |Y -    \wideparen Y^{(n)}|)
)  + 2 \|k\|^{3/2}_\infty E(|Y - \wideparen Y^{(n)}|) \\
&\leq  2 \|k\|_\infty (E(|Y - f_0|^2)^{1/2}  E((Y -    \wideparen Y^{(n)})^2)^{1/2} + \|f_0\|_\infty E(|Y -    \wideparen Y^{(n)}|)
) \\ 
&\quad + 2 \|k\|^{3/2}_\infty E(|Y - \wideparen Y^{(n)}|) \\
&=2 \sigma \|k\|_\infty  E((Y -    \wideparen Y^{(n)})^2)^{1/2} +  2 \|k\|_\infty (\|f_0\|_\infty + \|k\|_\infty^{1/2}) E(|Y -    \wideparen Y^{(n)}|)
\end{align*}
and we can apply \eqref{eq:cappedYabs} and  \eqref{eq:cappedYVar}. Combining the above bounds and substituting \eqref{eq:cappedYabs} and  \eqref{eq:cappedYVar} yields the following bound,
\begin{align}
\|\Cov_{c,\tau}^S - \CovCapntau\|_{op} 
\leq&  2 \sigma \|k\|_\infty  E((Y -  \wideparen Y^{(n)})^2)^{1/2} \!\! + 4 \|k\|_\infty (\sigma/2 + \|f_0\|_\infty + \|k\|_\infty^{1/2}) E(|Y -    \wideparen Y^{(n)}|) \notag \\
\leq& 4 \sigma \nu^{1/2} \|k\|_\infty   e^{-r_n^2/4\nu}  + \sqrt{32\pi\nu}  \|k\|_\infty (\sigma/2 + \|f_0\|_\infty + \|k\|_\infty^{1/2}) e^{-r_n^2/2\nu} \notag \\
=& 4 \sigma^2 \|k\|_\infty   e^{-r_n^2/4\sigma^2}  + \sqrt{32\pi} \sigma  \|k\|_\infty (\sigma/2 + \|f_0\|_\infty + \|k\|_\infty^{1/2}) e^{-r_n^2/2\sigma^2}, \label{eq:direct_sum_cov}
\end{align}
where we used $\nu = \sigma^2$ in the last line. To make use of this bound we need a lower bound on the smallest eigenvalue of $\Cov_{c,\tau}^S$. We proceed as in \eqref{eq:cov_lbd_hcheck}. Instead of imposing such an assumption directly we can also use an assumption on the covariance operator corresponding to the kernel $k$ and the marginal distribution on $\X$, which seems more natural. To see this, fix $h \in \cH_{\tau,S_f}$, $\|h\|_{\tau,S_f} = 1$, and let $f \in \cH_{\kappa_y\upharpoonright S_f \times S_f}, g \in \cH_{k  \upharpoonright \tilde S \times \tilde S}$ be such that $h = f  + \check{g}$  and the squared norm of  $h$  equals the sum of the squared norms of the $f$ and $g$ (see the discussion around \eqref{eq:cov_lbd_hcheck}). 
\begin{align*}
	\langle \Cov_{c,\tau}^S  h, h \rangle_{\tau,S_f} 
	&= E( (f(X) + f_0(X) g(X) + \epsilon g(X))^2) - E^2( f(X) + f_0(X) g(X)) \\
	&=  E( (f(X) + f_0(X) g(X))^2) -  E^2( f(X) + f_0(X) g(X)) + \sigma^2 E(g^2(X)) \\
	&\geq \sigma^2 \langle \Cov^{\tilde S} g, g \rangle_{k,\tilde S}.
\end{align*}
If \textit{$k$ is continuous}, we are guaranteed the smallest eigenvalue $\bar \lambda_{\star,\tilde S}$  of $\Cov^{\tilde S}_k$ is bounded away from zero (see below \eqref{eq:cov_lbd_hcheck}). In particular, if we choose $r_1$ such that the last display in \eqref{eq:direct_sum_cov} is upper bounded by 
 $ \sigma^2 \bar \lambda_{\star,\tilde S}/ 2$ then the smallest eigenvalue $\wideparen \lambda_{\star}^{(n)}$ of the capped covariance operator is at least half the smallest eigenvalue $\bar \lambda_\star$  of  $\Cov_{c,\tau}^S$ and is lower bounded by   $\sigma^2 \bar \lambda_{\star,\tilde S}/2$. To guarantee this, we can define the sequence $\{r_n\}_{n\geq 1}$ similarly to before, starting with 
\begin{equation}
r_1 = 1 \vee 2 \sigma \log^{1/2}\Bigl(\frac{22 \|k\|_\infty (\sigma + \|f_0\|_\infty + \|k\|_\infty^{1/2}) }{\sigma^2 \bar \lambda_{\star,\tilde S}}
\Bigr) \label{def:r1_regression}
\end{equation}
and by assuring that the sequence is non-decreasing.

We are now in a position to state a compression result in the regression context 
along the lines of Proposition \ref{prop:unbounded}, but in the case where we compress the data simultaneously for the kernel $\kappa_y$ and $\rho$. For conciseness we will use the notation $\|\tau\|_{S_{f,n},\infty}$ for $\|\tau \!\!\upharpoonright \!\! S_{f,n} \times S_{f,n} \|_\infty$ where $S_{f,n} = \{
	(y,x) : (x,y) \in S_n\}$ and $S_n$ is the support of measure corresponding to the capped random variables.
\begin{prop} \label{prop:regression_ball}
Let $(\X \times \R, \mathfrak{T}, \mathcal{A},P)$ be a topological measure space such that $P$ is a Radon probability measure which has support $S$. Let $k$ be continuous bounded kernel function defined on $\X$ such that the corresponding RKHS is finite dimensional and does not contain the constant functions. Let $(X_1,Y_1), \ldots, (X_n,Y_n)$ be i.i.d. random variables with law $P$, and assume that $Y_i = f_0(X_i) + \epsilon_i$, for all $i \leq n$, where $f_0$ is a measurable and bounded function and $\epsilon_1,\ldots, \epsilon_n$ are i.i.d. sub-Gaussian random variables with variance $0 < \sigma^2$ which are independent of $X_1,\ldots, X_n$. Consider the kernel function 
$\tau = \kappa_y + \rho$ on $\R \times \X$ and let $\bar \lambda_{\star,\tilde S}$ be the smallest eigenvalue of the covariance operator $\Cov^{\tilde S}_k$ corresponding to the kernel function $k\!\upharpoonright\! \tilde S \times \tilde S$, $\tilde S = \overline{\{x: (x,y) \in S\}}$. Furthermore, let $\bar \lambda_\star$ be the smallest eigenvalue of $\tilde C_{c,\tau}^S$, then $\bar \lambda_\star>0$. Choose $q \in (0,1)$ and define the sequence $\{r_n\}_{n\geq 1}$ in the following way: choose $r_1$ as in \eqref{def:r1_regression} and for $n >2$ let $r_n = r_1 \vee\sqrt{2} \sigma \vee \sqrt{2} \sigma \log^{1/2}(16 n^{1/2}  \sigma^2 \|k\|_\infty/  q) $. For the smallest eigenvalue $\wideparen \lambda_{\star}^{(n)}$ of $\CovCapntau$ it holds that $\wideparen \lambda_{\star}^{(n)} \geq  \bar \lambda_\star/2 \geq \sigma^2 \bar \lambda_{\star,\tilde S}/2 > 0$, where $S_n$ is the support of the law of $P^{(n)}$ corresponding to $(X,\wideparen{Y}^{(n)})$. There exists a ball of radius $\wideparen \delta^{(n)} = \wideparen \lambda_{\star}^{(n)} / 2 \|\tau \|_{S_{f,n},\infty}^{1/2} \geq \wideparen \lambda_{\star}^{(n)} / 2 (\|k\|_\infty  + (\|f_0\|_\infty + r_n)^{1/2} \|k\|_\infty^{1/2})$ around $\wideparen \mean_\tau$ within the affine space spanned by $\wideparen{C}_\tau = \{ \tau((y,x),\cdot) : (x,y) \in S_n\}$ as a subset of  $\cH_\tau$. Furthermore, for any $q \in (0,1)$ and whenever $n$ is (strictly) greater than 
\[
	\left(
\frac{8 \|\tau\|_{S_{f,n},\infty} (\sqrt{2\log(12/q)} + 192\|\tau\|_{S_{f,n},\infty} /\wideparen{\lambda}_\star^{(n)})}{(\wideparen{\lambda}_\star^{(n)})^2} \right)^{2} \vee 
		\left(\frac{16\|\tau\|_{S_{f,n},\infty}^{1/2} +  \sqrt{288\log(4/q)}}{\wideparen{\delta}^{(n)}}\right)^{2}
\]
then with probability $1-q$ there exists a ball of radius $\wideparen \delta^{(n)}/4$ around $\wideparen{\mean}_{\tau,n}$ in $\wideparen{C}_{\tau,n}$ within the affine subspace spanned by $\wideparen{C}_\tau$ and  
\[ 
\|\wideparen{\mean}_{\tau,n} - \mean_{\tau,n}\|_\tau \leq n^{-1/2}.
\]
\end{prop}
\begin{proof}
Most of the statement has already been derived. Just note that 
$\|\tau\|_{S_{f,n},\infty}^{1/2}
\leq (\|k^2\|_\infty + (\|f_0\|_\infty + r_n) \|k\|_\infty )^{1/2} \leq \|k\|_\infty + (\|f_0\|_\infty + r_n)^{1/2} \|k\|_\infty^{1/2}$. For the definition of $r_n$ and the bound on the difference between the mean and the capped mean we could use 
essentially the same bound as in Proposition \ref{prop:unbounded}. Instead we use here a slightly different bound to demonstrate how the arguments can be varied: when $r_n \geq \sqrt{2} \sigma$ it follows from \eqref{eq:direct_sum_mean} that with probability $\tilde q$,
\[\| \mean_{\tau,n} - \wideparen \mean_{\tau,n} \|^2_\tau \leq   \sqrt{8} \sigma \|k\|^{1/2}_\infty  e^{-r_n^2/4 \nu}  
\tilde q^{-1/2}.
\]
Setting the right side equal to $n^{-1/4}$, setting $\tilde q = q/2$ and solving for $r_n$ gives
\[
r_n =\sqrt{2} \sigma \log^{1/2}(16 n^{1/2}  \sigma^2 \|k\|_\infty/  q).
\]
\end{proof}

\subsection{Rescaling the kernel function does not affect compression}
We finish this section by studying the effect of modifying the kernel function, or the involved convex sets, on the approximation of $\mean$. Given that the smallest eigenvalues of the covariance operator and the width of $C$ control the approximation of $\mean$ it is natural to try to increase these. One way to do so is to scale the kernel function by a constant factor $\alpha > 0$, i.e. replace the kernel function $k$ on $\X$ by $\alpha k$. However, due to \cite[Prop.5.20]{PAUL16} the inner products corresponding to the two spaces are scaled versions of each other (for all $x,y \in \X$,  $\langle k(x,\cdot), k(y,\cdot) \rangle_k = (1/\alpha) \langle \alpha k(x,\cdot), \alpha k(y,\cdot) \rangle_{\alpha k}$) 
and the algorithms that we discuss in the next section are unaffected by this change. One might also wonder if the error bounds are affected and if we can, at least, optimize these by choosing an appropriate scaling. It turns out that the error bounds are also invariant to the scaling of the kernel function. Let use start by analyzing the width of the convex set $C$. The change in width is easy to quantify: given a kernel function $k: \X \times \X \to \R$ the width increases by $\alpha^{1/2}$ if we replace the kernel function by 
$\alpha k$, where $\alpha >0$. In detail, a function $h$, $\|h\|_k=1$, lies in $\cH_k$ if, and only if,  the function $\alpha^{1/2} h$ lies in $\cH_{\alpha k}$ and has norm $\|\alpha^{1/2} h\|_{\alpha k} = 1$ \cite[Prop.5.20]{PAUL16} and the width is 
\[
	\diam_{\alpha^{1/2} h, \cH_{\alpha k}}(C)  = \alpha^{1/2}(\sup_{x \in \X} h(x) - \inf_{x \in \X} h(x)) = \alpha^{1/2} \diam_{h, \cH_k}(C).
\]
Similarly, when the smallest non-zero eigenvalue of the centered covariance operator $\Cov_{c,k}$ is  $\bar \lambda_l$, $l\geq 1$,
it follows that the smallest non-zero eigenvalue of  $\Cov_{c,\alpha k}$ is $\alpha \bar \lambda_l$ : let $h \in \cH_k, \|h\|_k=1$, be the eigenfunction corresponding to $\bar \lambda_l$ then $\alpha^{1/2} h$ has unit norm in $\cH_{\alpha k}$ and 
\[
\langle \Cov_{c,\alpha k} \alpha^{1/2} h, \alpha^{1/2} h \rangle_{\alpha k} = \alpha (E(h^2(X)) - (E(h(X)))^2) = \alpha \bar \lambda_l.
\]
Recall that the relation between the width and the radius $\delta$ of the largest ball around $\mean$ is approximately $\delta \approx (\diam_h(C))^{d+1}$ 
when $\X$ is a subset of $\R^d$ and under suitable conditions on the density (end of Section \ref{sec:assumption_density}). In other words, a scaling of the kernel function by $\alpha$ increases $\delta$ approximately by a factor of $\alpha^{(d+1)/2}$. This increase has to be compared to the increase in the Lipschitz constant which results from this scaling. In Equation \eqref{eq:width_lower_bnd} the Lipschitz constant $L$ enteres into the lower bound on $\delta$ through $L^{-d}$. For a function $\alpha^{1/2} h \in \cH_{\alpha k}$ we have  that
\[
\frac{|\alpha^{1/2} h(x) - \alpha^{1/2} h(y)|}{\|x-y\|} \leq \alpha^{1/2} L \|h\|  = \alpha^{1/2} L \|\alpha^{1/2} h\|_{\alpha k}, 	
\]
for any $x,y \in \X, x \not = y,$ if the Lipschitz assumption \eqref{def:Lip_for_RKHS} holds with constant $L$ for $\cH$. Combining these we see that 
$\delta$ changes approximately by a factor of $\alpha^{(d+1)/2} L^{-d} = \alpha^{1/2}$. Hence, we can increase $\delta$ by increasing $\alpha$. In the error bounds in the next section we will see that the key quantity for controlling the approximation error is the ratio $\delta/\|k\|^{1/2}_\infty$, which has to be large to guarantee a good compression. Since $\|\alpha k\|_\infty = \alpha \|k\|_\infty$ this term is independent of $\alpha$ and rescaling does not change the rate  of  compression that is promised by the bounds. Also note that we are changing the norm by which we measure the error. For some approximation $\hat \mean$ of $\mean$ in $\cH_k$ we can observe that $\|\hat \mean - \mean\|_k = \alpha^{1/2} \|  \hat \mean - \mean \|_{\alpha k}$. This $\alpha^{1/2}$ factor is cancelled by the leading constant in the error bound of the algorithms, which is of the order $\|\alpha k\|_\infty^{1/2} = \alpha^{1/2} \|k\|^{1/2}_\infty$.

If we consider instead the bound on $\delta$ that is based on the covariance operator (Theorem \ref{thm:covariance}) then $\delta$ is of order $\bar \lambda_l/\|k\|_\infty^{1/2}$ and a scaling of $\alpha$ leads again to $\alpha \bar \lambda_l / \alpha^{1/2} \|k\|_\infty^{1/2}$ and the scaling does not affect the error bound.

\section{Compression using the CGM and related approaches} \label{sec:algorithms}
We discuss in this section two methods to compress $\mean_n$. The bottleneck in both algorithms is the computation of the vector 
\[
s = (\langle k(X_1,\cdot), \mean_n \rangle, \ldots, \langle k(X_n,\cdot), \mean_n \rangle)^\top.
\]
If $s$ is available then the remaining parts of the two algorithms that we analyze have a runtime of $n l$, where $l$ is the number of iterations the algorithms are run for. In particular, for large $n$, $l$ will be in the order of $\log(n)$ when using the classical CGM and of order $n^{1/2}$ when using the kernel herding algorithm. This then results in a runtime of $O(n \log(n))$ and $O(n^{3/2})$ respectively to gain a representation of $\mean_n$. The CGM achieves in this context a compression down to $\log(n)$ many points and the kernel herding algorithm down to $n^{1/2}$, that is, if we have a ball of sufficient size around $\mean_n$ in $C_n$.

A naive algorithm to compute $s$ has a runtime of $O(n^2)$. In fact,  a brute-force computation needs to compute all pairs $k(X_i,X_j)$, $i,j \leq d$, and the computational complexity is the same as the computational complexity of computing the kernel matrix itself (though the algorithm only needs $O(n)$ memory instead of $O(n^2)$). However, there is hope for faster algorithms. For instance, when we have a finite dimensional RKHS with dimension $d$ then we can represent $\mean_n$ as a linear combination of $d$ points and $s$ can be computed in $O(n d)$ time. Computing the representation of $\mean_n$ needs another $d^2$ steps. In practice this is not useful because we would derive an exact representation of $\mean_n$ based on $d$ data points to find an approximation of $\mean_n$ using more than $d$ data points. Ideally, we would hope for an algorithm that can compute, or approximate, $s$ in $n \log n$ steps independently of the dimension $d$ of the Hilbert space. Alternatively, we could try to modify the main algorithms itself to mitigate the complexity of computing $s$. There are some standard ways to deal with large scale data in the context of the CGM as summarized in \cite{BRAU22}. However, they do not lead to computational benefits in our particular setting. We discuss a promising alternative that is based on a divide and conquer approach in some detail below (Section \ref{sec:fast_preproc_kh} and \ref{sec:fast_preproc_cgm}) after analyzing  the standard algorithms. We also include short discussions on how to adapt these methods when aiming for compressing other quantities like the covariance and  how to use the CGM to compress the data in the case of kernel regression. 

 \subsection{Kernel herding and subset selection} 
Let us start by stating a version of the kernel herding algorithm \cite{CHEN10} for compressing the empirical measure. 

\begin{algorithm}[h!] \caption{(The kernel herding algorithm for compressing $\mean_n$)} \label{alg:KH} 
\begin{normalsize}
\begin{align*} 
&\text{Input: sample } X_1,\ldots, X_n, \text{ kernel } k, \text{ number of points in the coreset } T.   \notag \\
&\text{Initialise: let }  w_1 = k(X_1,\cdot) - \mean_n \text{ and } \iota(1) = 1, \text{ iterate through } t \geq 2:  \notag \\
&\text{choose } \textstyle i^\star \in \argmax_{i \leq n} 
\inner{w_t}{k(X_i,\cdot)} \notag \\
&\text{set } \iota(t)  = i^\star, w_{t+1} = w_t - (k(X_{i^\star},\cdot) - \mean_n), \text{ and } \wh \mean_t = \frac{1}{t} \sum_{u = 1}^t k(X_{\iota(u)},\cdot)
\notag \\
&\text{Stop when } t=T \text{ and return the approximation }  \textstyle {\wh \mean}_{T}.   \notag
\end{align*}
\end{normalsize}
\vspace{-0.4cm}
\end{algorithm}

The index function $\iota:\{1, \ldots, T\} \to \{1,\ldots, n\}$ tracks the samples that we include in the coreset and 
the elements $w_t$  measure the error between $\mean_n$ and $\wh \mean_t$ as $\|w_t\| = t \| \mean_n - \wh \mean_t\|$. The algorithm converges with a rate of $1/t$ if, and only if, the sequence of weights $w_t$ is bounded. In other words, if the sequence diverges then the algorithm converges with a slow rate. It is easy to show that $\|w_t\|$ stays bounded when a ball of radius $\delta>0$ exists around $\mean_n$ in $C_n$ and that $\|\mean_n - \wh \mean_t\| \leq \alpha_{KH} /\delta t$ for a constant $\alpha_{KH}$  and all $t$. In particular, we can choose
\[
	\alpha_{KH} =    8 \|k\|_\infty. 
\]
Also, notice that the same bound holds when a ball of radius $\delta$ exists around $\mean_n$ in the affine span of $C_n$. 

Instead of running the algorithm for $T$ iterations independent of the approximation error we can also use the approximation error as a stopping criterion. The approximation error  $(1/t) \|w_t\| = (\|\wh \mean_t \|^2 - 2 \langle \wh \mean_t, \mean_n \rangle + \|\mean_n\|^2)^{1/2}$ can be computed exactly in $O(n^2)$ if we prevent the algorithm from running for more than $n$ iterations. In detail, pre-computing $\|\mean_n\|$ and $s$ can be done in $O(n^2)$. Also, $\|\wh \mean_t\|$ can be computed in $O(t)$ given $\|\wh \mean_{t-1}\|$ by using that $\|\wh \mean_t\|^2 = \|\wh \mean_{t-1}\|^2 + 2 \langle k(X_{\iota(t)},\cdot), \wh \mean_{t-1} \rangle + k(X_{\iota(t)}, X_{\iota(t)})$. Similarly, $\langle \wh \mean_t, \mean_n \rangle$ can be easily gained from $\langle \wh \mean_{t-1}, \mean_n \rangle$ by using $\langle \wh \mean_t, \mean_n \rangle = \langle \wh \mean_{t-1},\mean_n \rangle + s_{\iota(t)}$. 
A natural stopping point for the algorithm is an approximation error of $n^{-1/2}$ which guarantees that $\|\hat \mean_t - \mean\|$ will be of the same order as $\|\mean_n -\mean\|$.

The compression of this algorithm is sub-optimal but it has the advantage that it returns a coreset. The CGM, which we discuss below, achieves a significantly better compression but does not return a coreset of samples.

It is easy to gain high probability guarantees for the approximation error of a compression that uses $n^{1/2}$ many points, under the conditions stated in Section \ref{sec:bringing_it_together}. With a bit more work it is also possible to control the expected approximation error. We summarize these in the following proposition under a Lipschitz assumption on the kernel function, assuming that we have a Mercer kernel and that the constant functions are in the RKHS; in particular, we assume that $k$ is a continuous kernel on $[0,1]^l$, which is a sufficient assumption for Mercer's theorem to hold. When discussing the CGM we give a similar proposition which uses instead an assumption on the covariance operator; the aim is to highlight how the various assumptions can be combined with the algorithms.   
\begin{proposition} \label{prop:KH_guarantees}
Let $X_1,\ldots, X_n$ be i.i.d. random variables on some probability space $(\Omega,\mathcal{A},P)$, which attain values in  $\X = [0,1]^l, l \geq 1,$ and let $k$ be a continuous kernel function on $\X$ such that the corresponding RKHS $\cH$ is $d$-dimensional, $1 \leq d < \infty$, functions $h \in \cH$ are Lipschitz continuous in the sense of \eqref{def:Lip_for_RKHS} with Lipschitz constant $L>0$, and $\bm 1 \in \cH$. Furthermore, assume that the law of $X_1$ has a density $p$ on $\X$ and $\inf_{x\in\X} p(x) \geq c >0$ for some constant $c$. Let $\tilde \lambda_d$ be the smallest eigenvalue of the Mercer decomposition. With probability $1-q, q \in (0,1)$, 
\[
\| \wh \mean_{\lceil n^{1/2} \rceil} - \mean_n \| \leq \frac{32\|k\|_\infty}{\delta} \, n^{-1/2}  
\]
whenever 
\[
	n > \left(\frac{\sqrt{2\log(6/q)} + 96 \|k\|_\infty^{1/2}/\delta}{c\beta_l(\delta/8L)^l}\right)^{2} \vee 
		\left(\frac{4\|k\|_\infty^{1/2} + 3 \sqrt{2\log(2/q)}}{\delta/4}\right)^{2}
\]
and where
\[
\delta = 2\tilde \lambda_d \wedge  \frac{2 c \tilde \lambda_d^{l+1}\beta_l}{(l+1) L^l}. 
\]
Furthermore, let $N = (16 \|k\|_\infty^{1/2} /\delta)^2\vee (96 (8L)^l \|k\|_\infty^{1/2}/c \beta_l\delta^{l+1})^2$ then 
for any $t \geq 1$ and whenever $n\geq N$, 
\[
E(\| \wh \mean_{t} - \mean_n \|)  \leq 32 \|k\|_\infty/t \delta + 4\|k\|^{1/2}_\infty \exp(-(c_1 n^{1/2} - c_2 )_+^2)/t^{1/2},
\]
where $c_1$ and $c_2$ can be chosen as
\[
c_1  = (c\beta_l (\delta/8L)^l /\sqrt{3}) \wedge (\delta/17) \text{ and } c_2 =  ((96/ \sqrt{3} \delta) \vee 1) \|k\|^{1/2}_\infty. 
\]
If the stopping criterion of the algorithm is an error of $\|\wh \mean_{\hat t} - \mean_n \| \leq n^{-1/2}$, i.e. $\hat t = \inf_{t \geq 1} \|\wh \mean_t - \mean_n\| \leq n^{-1/2}$, and if the infimum is greater than $n$, then $\hat t = n$ and $\wh \mean_{\hat t} = \mean_n$, then
\[
E(\hat t\,) \leq  \lceil 32 \|k\|_\infty n^{1/2} / \delta \rceil + 2 n \exp(-(c_1 n^{1/2} - c_2 )^2) 
\]
whenever $n \geq N$.
\end{proposition}
\begin{proof}
The first part follows directly from Theorem \ref{thm:ball_in_empirical} and the bound on the error of the kernel herding algorithm. For the second statement, observe that
\[
\Pr( \|\mean_n - \mean\| \geq \delta/4) \leq \exp\left(-\frac{1}{288} \left(n^{1/2}\delta - 16 \|k\|_\infty^{1/2}   \right)^2 \right)
\]
whenever $n \geq (16 \|k\|_\infty^{1/2} /\delta)^2$, 
follows by the same argument as in Theorem \ref{thm:ball_in_empirical}. Similarly, there is a ball of radius $\delta/2$ around $\mean$ in $C_n$ in the affine span of $C_n$ with probability at least
\[
1 - (1/2) \exp\bigl(-\bigl(n^{1/2} c \beta_l (\delta/8L)^l   - 96 \|k\|_\infty^{1/2} / \delta   \bigr)^2 \bigr)
\]
whenever $n \geq (96 (8L)^l \|k\|_\infty^{1/2}/c \beta_l\delta^{l+1})^2$. Also, notice that even when there is no ball around $\mean_n$, for $t \geq 1$ it holds that
\[
\|w_{t+1} \|^2 = \|k(X_{\iota(t)},\cdot) - \mean_n\|^2  + \| w_t \|^2 - 2 \langle k(X_{\iota(t)},\cdot) - \mean_n, \mean_n \rangle.
\]
Now, $\langle k(X_{\iota(t)}, \cdot) , \mean_n \rangle \geq
\|\mean_n\|^2$  since $\mean_n$ lies in $C_n$ which has extremes $k(X_1,\cdot), \ldots,$ $ k(X_n,\cdot)$, and 
\[
\|w_{t+1} \|^2 \leq \sum_{i=1}^t \|k(X_{\iota(t)},\cdot) - \mean_n\|^2.  
\]
Hence, $\|\wh \mean_t - \mean_n\|^2 \leq (4/t) \|k\|_\infty $.

Combining these, we find that
\begin{align*}
E(\|\wh \mean_t - \mean_n \|) \leq 32 \|k\|_\infty /t \delta + 4\|k\|^{1/2}_\infty  \exp(-(c_1 n^{1/2} - c_2 )_+^2) t^{-1/2} 
\end{align*}
whenever $n$ is large enough and with $c_1, c_2$ as in the theorem statement.  

The third statement follows along similar lines. In the event that we have a ball of size $\delta/4$ it follows that 
$\|\wh \mean_t - \mean_n \| \leq 32 \|k\|_\infty/t \delta$. Setting the right hand side to $n^{-1/2}$ leads to
$\hat t \leq  \lceil 32 \|k\|_\infty n^{1/2} / \delta \rceil$. If this event does not occur then 
$\|\wh \mean_t - \mean_n \| \leq 2\|k\|_\infty^{1/2} t^{-1/2}$ and $\hat t \leq \lceil 4 \|k\|_\infty n \rceil$, but the algorithm stops when $\hat t > n$ and the trivial upper bound $\hat t \leq n$ is more useful. Combining these we find that
\[
E(\hat t) \leq \lceil 32 \|k\|_\infty n^{1/2} / \delta \rceil + 2 n \exp(-(c_1 n^{1/2} - c_2 )_+^2) 
\]
when $n$ is large enough.
\end{proof}

\subsubsection{Avoiding the explicit computation of $s$} \label{sec:fast_preproc_kh}
There are various ways one can try to reduce the computation time. For instance, the stochastic conditional gradient method seems like a promising candidate. An alternative way  
to mitigate the cost of computing $s$ is to split the data into batches of size about $n^{1/2}$, which implies that for each batch the corresponding vector can be computed in $O(n)$. Algorithm \ref{alg:KH_non_quad} implements such a version of kernel herding. There are a number of interesting observations that can be made when following this route. We will discuss a few such observations in this section and in Section \ref{sec:fast_preproc_cgm} below.   

\begin{algorithm}[ht] \caption{ (A version of kernel herding that avoids the explicit computation of $s$)} \label{alg:KH_non_quad} 
\begin{normalsize}
\begin{align*} 
&\text{Input: sample } X_1,\ldots, X_n, \text{ kernel } k, \alpha > 0.   \notag \\
&\text{Initialise: let }  \ell = \lfloor n^{1/2} +1/2 \rfloor. \notag \\
&\text{Split sample into  } \ell \text{ disjoint batches: } X_{j1}, \ldots, X_{j\ell_j}, j \leq \ell, \\
&\hspace{6cm}\text{ with each } \ell_j \in \{\ell-1, \ell, \ell +1,\ell+2\}.   \\
&\text{Apply Algorithm \ref{alg:KH} with } T = \lceil n^{1/4 + \alpha} \rceil \text{ to each batch to get } \wh \mean_1,\ldots, \wh \mean_\ell. \\
&\text{Compute } \|\tilde \mean_n\| \text{ and } \wh s =  (\langle \wh \mean_1,\tilde \mean_n \rangle, \ldots, \langle \wh \mean_\ell, \tilde \mean_n \rangle)^\top \text{, where } 
\tilde \mean_n = \frac{1}{n} \sum_{i=1}^\ell \ell_i \wh \mean_i. \\
&\text{Apply a version of Algorithm \ref{alg:KH} with } T = \lceil n^{1/4 + \alpha} \rceil \text{ to } \tilde \mean_n  \text{ using } \|\tilde \mean_n\| \text{ and } \wh s \text{ to get } \iota. \\
&\text{Return the approximation } \frac{1}{T} \sum_{i=1}^T \wh \mean_{\iota(i)}. \\  
\end{align*}
\end{normalsize}
\vspace{-0.4cm}
\end{algorithm}

 In terms of Algorithm \ref{alg:KH_non_quad}, notice that the number of samples per batch can always be chosen in $\{\ell-1, \ell, \ell+1, \ell+2 \}$ to guarantee that $n = \sum_{i=1}^\ell \ell_i$ because 
$\ell (\ell-1) \leq (n^{1/2} + 1/2)(n^{1/2} - 1/2) < n \leq (\ell + 1/2)^2 \leq  \ell(\ell +2)$ and by a version of Algorithm \ref{alg:KH} we mean the obvious modification where instead of $k(X_1,\cdot),\ldots, k(X_n,\cdot)$ we use $\wh \mean_1,\ldots, \wh \mean_\ell$ to approximate $\mean_n$. The algorithm works by specifying the number of iterations for the kernel herding algorithm. Alternatively, it makes sense to run the  first $\ell$ optimization algorithms as well as the last optimization step until an error of $n^{1/4}$ is attained.

The runtime of the algorithm can be computed in the following way: observe that, initially, the standard kernel herding algorithm is applied $\ell$ times to about $\ell$ many samples and the overall order of runtime for the first part is $O(\ell^3) = O(n^{3/2})$.  Also, observe that, given $\wh \mean_1,\ldots, \wh \mean_\ell$,
an approximation of $\|\mean_n\|$ can be computed in   $\ell \lceil n^{1/4 + \alpha}\rceil^2$, which is of order $O(n^{1+2 \alpha})$, by using the approximation 
$\tilde \mean_n = (1/n) (\ell_1 \wh \mean_1 + \ldots + \ell_\ell \wh \mean_\ell)$. Similarly, 
$\wh K = ( \langle \wh \mean_i, \wh \mean_j \rangle )_{i,j \leq \ell}$ can be computed in $\ell^2 \lceil n^{1/4+\alpha} \rceil^2$ which is of order
$O(n^{3/2 + 2 \alpha})$ and, given $\wh K$, the vector $\wh s = (\langle \wh \mean_1,\tilde \mean_n \rangle, \ldots, \langle \wh \mean_\ell, \tilde \mean_n \rangle)^\top$ can be computed in $\ell^2$ steps, which is of order $O(n)$. Given $\wh s$ and $\|\tilde \mean_n\|$ the second application of the kernel herding algorithm can be run in $\lceil n^{1/4 + \alpha} \rceil \ell$, which is of order $O(n^{3/4 + \alpha})$, and Algorithm \ref{alg:KH_non_quad} has an overall order of $O(n^{3/2+2 \alpha})$ 

Quantifying the approximation error of this algorithm is more difficult and in the following we only highlight some of the challenges that one has to address to control the approximation error. For $n$ large enough the difference 
$\|\wh \mean_i - (1/\ell_i) \sum_{j=1}^{\ell_i} k(X_{ij}, \cdot) \|$ is with high probability of order $n^{-1/4}$ for all $i \leq \ell$. Furthermore,
\begin{align*}
\| \mean_n - \tilde \mean_n\|^2 = \| \mean_n -  \frac{1}{n} \sum_{i=1}^\ell \ell_i \wh \mean_i \|^2 = \| \frac{1}{n} \sum_{i=1}^\ell \sum_{j=1}^{\ell_i} (k(X_{ij},\cdot) - \wh \mean_i)\|^2. 
\end{align*}
Notic Bochner integral  
\[
E( \sum_{j=1}^{\ell_i} (k(X_{ij},\cdot) - \wh \mean_i) ) = \ell_i ( \mean - E(\wh \mean_i)).	
\]
Furthermore, observe that the $\wh \mean_1, \ldots, \wh \mean_\ell$ are independent random variables since they are functions of separate samples and that $E(\langle \bm X, \bm Y \rangle) = \langle E(\bm X), E(\bm Y) \rangle$ for independent random variables in $\mathcal{L}^2(P;\cH)$. Hence, 
\begin{align*}
	E( \| \mean_n -  \frac{1}{n} \sum_{i=1}^\ell \ell_i \wh \mean_i \|^2 ) =  \frac{1}{n^2} \sum_{i=1}^\ell &E( \| \sum_{j=1}^{\ell_i} (k(X_{ij},\cdot) - \wh \mean_i) \|^2) \\
&+ \frac{1}{n^2} \sum_{i\not= j}^{\ell} \ell_i \ell_j \langle \mean - E(\wh \mean_i), \mean - E(\wh \mean_j) \rangle \\
\approx n^{-3/2} &E( \| \sum_{j=1}^{\ell_1} (k(X_{1j},\cdot) - \wh \mean_1) \|^2) +   \| \mean - E(\wh \mean_1)\|^2,
\end{align*}
where we have an approximation in the last line since the $\ell_i$'s are not necessarily all equal. A first difficulty is to determine the bias $\|\mean - E(\wh \mean_1)\|$ that the kernel herding algorithm introduces. A simple bound on the bias is gained by using $\|\mean -  E(\wh \mean_1)\| \leq E( \| \mean_n - \wh \mean_1\|)$ and Proposition \ref{prop:KH_guarantees} can be used to bound this by about $n^{-1/4}$ which implies a bound on the squared bias of order $n^{-1/2}$. This bound is of no use since we need a bias of order $1/n$ or less.
The other term behaves approximately as $1/n$ if there is a ball of size $\delta>0$ around $\mean$ in $C$ and $n$ is large enough. In particular, under the conditions of Proposition \ref{prop:KH_guarantees},
\begin{align*}
E( \| \frac{1}{\ell_1} \sum_{j=1}^{\ell_1} (k(X_{1j},\cdot) - \wh \mean_1) \|^2)
&\leq 16 \alpha^2/n^{1/2} \delta^2 + 8\|k\|_\infty \exp(-(c_1 n^{1/4} - c_2 )^2)n^{-1/4}.
\end{align*}
Up to the exponential term on the right, we have that   
\begin{align*}
n^{-3/2}  &E( \| \sum_{j=1}^{\ell_1} (k(X_{1j},\cdot) - \wh \mean_1) \|^2) \lesssim n^{-3/2} \ell_1^2   n^{-1/2}  \approx 1/n. 
\end{align*}
If the bias is also of order $O(1/n)$ then 
\[
\Pr( \| \mean_n -  \frac{1}{n} \sum_{i=1}^\ell \ell_i \wh \mean_i \| \geq t ) \lesssim \frac{1}{n t^2}.   
\]
In particular, for any $\beta > 0 $,
\[
\Pr( \|\mean_n - \tilde \mean_n\|  \geq n^{-1/2 + \beta})  = \Pr( \| \mean_n -  \frac{1}{n} \sum_{i=1}^\ell \ell_i \wh \mean_i \| \geq n^{-1/2 + \beta} ) \lesssim  n^{-2\beta}.   
\]
To summarize, if the bias is of order $O(1/n)$ we will have with high probability an approximation of $\mean_n$ that has an error of order $n^{-1/2 + \beta}$ and this approximation consists of approximately $T \ell \approx n^{3/4}$ many points.

The second application of the kernel herding algorithm aims to compress this further. In particular, if with high probability there is a ball around $\mean_n$ in the convex set
$\ch\{\wh \mean_i : i \leq \ell\}$, then we can hope that $n^{1/4}$ many of the $\wh \mean_i$ are sufficient to approximate $\mean_n$ with an error of order $n^{-1/2 + \beta}$. This would imply that an approximation with $n^{1/2}$ many elements is sufficient. However, since $\wh \mean_i$ converges to $\mean_n$ as $n$ goes to infinity, 
the size of such a ball has to be a function of $n$ and will shrink with $n$. This  itself does not imply that the algorithm will converge slowly since the smaller $\delta$ might be set-off by a smaller size of the convex set. In any case, a detailed analysis of the interplay between $\ch\{\wh \mean_i : i\leq \ell \}$ and $\mean_n$ is necessary to understand the compression that can be achieved by this algorithm and variations thereof.

Let us conclude our discussion of these algorithms with a final simple observation. The elements $\bar \mean_i = (1/\ell_i) \sum_{j=1}^{\ell_i} k(X_{ij},\cdot)$, which we are approximating with $\wh \mean_i$, can be interpreted as a sequence of independent and identically distributed (up to differences in the $\ell_i$'s) random variables whose second moment is 
given by 
\begin{align*}
	E(\|\bar \mean_i\|^2) &= \frac{1}{\ell_i^2} \sum_{j=1}^{\ell_i} E (k(X_{ij},X_{ij})) + \frac{1}{\ell_i^2} \sum_{u \not = v}^{\ell_i} \langle E(k(X_{iu},\cdot)), E(k(X_{iv},\cdot)) \rangle \\
			      &= \frac{1}{\ell_i}   + \frac{\ell_i -1 }{\ell_i} \|\mean\|^2, 
\end{align*}
whenever $k(x,x) = 1$ for all $x\in \X$. Hence,
\begin{align*}
	E(\|\bar \mean_i - \mean\|^2) = E(\|\bar \mean_i\|^2) - \|\mean\|^2  \approx n^{-1/2} \text{\, and \,}  
	E(\|\bar \mean_i - \mean\|) \lesssim n^{-1/4}. 
\end{align*}




\subsection{Better compression with the CGM}
A significantly better compression can be attained by using the CGM. The downside of using the CGM is that no coreset of datapoints is generated but some convex combination of the images of the data points in $\cH$ that approximates $\mean_n$ well. The standard CGM for compressing $\mean_n$ is given below. 
\begin{algorithm}[h!] \caption{(The CGM for compressing $\mean_n$.) } \label{alg:CGM} 
\begin{normalsize}
\begin{align*} 
&\text{Input: sample } X_1,\ldots, X_n, \text{ kernel } k, \text{ number of iterations } T.   \notag \\
&\text{Initialise: let }  \wh \mean_1 = k(X_1,\cdot), \alpha_{11} = 1 \text{ and } \iota(1) = 1, \text{ iterate through } t \geq 2:  \notag \\
&\text{choose } \textstyle i^\star \in \argmax_{i \leq n} 
\inner{k(X_i,\cdot) }{\wh \mean_{t-1} - \mean_n}, \notag \\
&\text{let } \alpha^\star = \frac{\langle k(X_{i^\star},\cdot) - \wh \mean_{t-1}, \wh \mean_{t-1} - \mean_n \rangle}{\|k(X_{i^\star},\cdot) - \wh \mean_{t-1} \|^2} \wedge 1, \\  
&\text{set } \iota(t)  = i^\star, \alpha_{tt} = \alpha^\star \text{ and  for all } u \leq t-1, \alpha_{tu} = (1-\alpha^\star) \alpha_{t-1,u},    \\
&\text{and let } \wh \mean_t = \sum_{u = 1}^t \alpha_{tu} k(X_{\iota(u)},\cdot).
\notag \\
&\text{Stop when } t=T \text{ and return the approximation }  \textstyle {\wh \mean}_{T}.   \notag
\end{align*}
\end{normalsize}
\vspace{-0.4cm}
\end{algorithm}

Notice that $\alpha^\star \geq 0$ since $k(X_{i^\star},\cdot)$ maximizes the inner product between any element in $C_n$ and $\wh \mean_{t-1} - \mean_n$.
This algorithms guarantees that the error is bounded by  
\[
\|\mean_n - \wh \mean_t\| \leq 2 \|k\|_\infty^{1/2} \exp\Biggl(- \frac{\delta(t-1)}{ 6 \|k\|_\infty^{1/2}} \Biggr),
\]
when a ball of size $\delta$ exists around $\mean_n$ in $C_n$ within the affine subspace spanned by $C_n$ \cite[Prop.3.2]{BECK04} and with $S$ denoting the support of the law of $X_1$.

The run-time of this algorithm is again dominated by the $O(n^2)$ run-time cost needed to compute $s$. When $s$ is available the run-time reduces to $O(T n)$: the $\arg\max$ step can be performed in $O(n)$ given $s$ and when the inner products $\langle k(X_i,\cdot), \wh \mean_{t-1} \rangle$ are available. Similarly, if $s$, the inner products $\langle k(X_i,\cdot), \wh \mean_{t-1} \rangle$, $\|\wh \mean_{t-1}\|$, $\langle \wh \mean_{t-1}, \mean_n \rangle$ and $\|k(X_{i^*},\cdot) - \wh \mean_{t-1}\|$ are available, it is possible to compute $\alpha^\star$ in $O(1)$. The norm term in the denominator can be computed in $O(1)$ from $\|\wh \mean_{t-1}\|$ and the inner products $\langle k(X_i,\cdot), \wh \mean_{t-1} \rangle$. The coefficients $\alpha_{tu}$ can be computed in $O(T^2)$. Updating the elements $\langle k(X_i,\cdot), \wh \mean_{t-1} \rangle$  to 
\[
	\langle k(X_i,\cdot), \wh \mean_{t} \rangle = (1-\alpha^\star) \langle k(X_i,\cdot),  \wh \mean_{t-1} \rangle + \alpha^\star \langle k(X_i,\cdot), k(X_{i^*},\cdot) \rangle
\]
can be done in $O(n)$. Furthermore, $\|\wh \mean_t\|^2 = (1-\alpha^\star)^2 \|\wh \mean_{t-1}\|^2 + (\alpha^\star)^2 k(X_{i^*},X_{i^*}) + 2 \alpha^\star (1-\alpha^\star) \langle k(X_{i^*},\cdot), \wh \mean_{t-1} \rangle$ and $\langle \wh \mean_t,\mean_n \rangle = (1-\alpha^\star) \langle \wh \mean_{t-1}, \mean_n \rangle + \alpha^\star \langle k(X_{i^*},\cdot) , \mean_n \rangle $ can both be updated in $O(1)$. In particular, if we aim for a compression down to $T = \log(n)$ elements then the run-time of the algorithm is $O(n\log(n))$, if $s$ is available.

As for the kernel herding algorithm, it is easy to bound, with high probability, the approximation error, as well as the expected error and the number of data points that are needed for the approximation when the stopping criterion is a pre-specified error.  
In the following proposition, we bound the approximation error given that the algorithm is run for $\lceil 12 \|k\|_\infty^{1/2} \log(n)/ \delta \rceil$ many iterations. 
Alternatively, it is possible to use $\lceil \log^\gamma (n) \rceil$, with $\gamma >1$, as a stopping criterion that does not depend on the unknown quantity $\delta$. For large enough $n$, $\lceil 12 \|k\|_\infty^{1/2} \log(n)/ \delta \rceil \leq  \lceil \log^\gamma (n) \rceil$ and the guarantees will carry over to that setting. 
\begin{prop}
Let $(\X,\mathcal{A},P)$ be some probability space with $P$ being a topological measure that is $\tau$-additive, and 
with measurable kernel function $k$ defined on $\X$ such that the corresponding RKHS  $\cH$ is finite dimensional.  Furthermore, let $X_1,\ldots, X_n$ be i.i.d. random variables attaining values in $\X$ and with law $P$. Assume that $\|k\|_\infty < \infty$, and that the centered covariance operator $\tilde{\mathfrak{C}}_c$ has an eigen-decomposition with smallest non-zero eigenvalue being $\bar \lambda_d$. Let $\beta = 3 \|k\|_\infty^{1/2} / \delta$ then with probability $1- q$, $q \in (0,1)$, 
\[
	\| \wh \mean_{\lceil 2 \beta \log(n) \rceil} - \mean_n \| \leq 2 \|k\|_\infty^{1/2}  n^{-1/2} 
\]
whenever $n$ is (strictly) greater than
\[
	\left(
\frac{8 \|k\|_\infty (\sqrt{2\log(6/q)} + 192\|k\|_\infty /\bar \lambda_d)}{\bar \lambda_d^2} \right)^{2} \vee 
		\left(\frac{16\|k\|_\infty^{1/2} +  \sqrt{288\log(2/q)}}{\delta}\right)^{2}
\]
and where $\delta = \bar \lambda_d/2 \|k\|_\infty^{1/2}$. Let $N = (16 \|k\|_\infty^{1/2} /\delta)^2 \vee (1536 \|k\|_\infty^2 / \bar \lambda_d^3)^2$
then for any $t \geq 1$ and whenever $n\geq N$,  
\[
E(\| \wh \mean_{t} - \mean_n \|)  \leq  \exp(- \delta (t-1)/24 \|k\|_\infty^{1/2})  + 
6\|k\|_\infty^{1/2} \exp(- (c_3 n^{1/2} - c_4)_+^2) 
,
\]
where $c_3 = (\bar \lambda_d^2/ 8 \sqrt{2}\|k\|_\infty) \wedge (\delta/17)$ and $c_4 = (192\|k\|_\infty / \sqrt{2} \bar\lambda_d) \vee \|k\|_\infty^{1/2}$ are possible choices.   

If the stopping criterion of the algorithm is an error of $\|\wh \mean_{\hat t} - \mean_n \| \leq n^{-1/2}$, i.e. $\hat t = \inf_{t \geq 1} \|\wh \mean_t - \mean_n\| \leq n^{-1/2}$, and if the infimum is greater than $n$, then $\hat t = n$ and $\wh \mean_{\hat t} = \mean_n$, then
\[
E(\hat t\,) \leq  \lceil 1 + 12 \|k\|_\infty^{1/2} \log(n)/\delta  \rceil + 3 n \exp(-(c_3 n^{1/2} - c_4 )^2) 
\]
whenever $n \geq N$.
\end{prop}
\begin{proof}
The first part follows directly from Theorem \ref{thm:covariance} and the bound on the error of the CGM. 
For the other statements let us consider the space $\cH_S$ corresponding to the kernel $k_S  = k\!\!\upharpoonright\!\! S \times S$, where $S$ is the support of the law $P$, and with corresponding objects $\mean_S, \mean_{S,n}$ and $\tilde{\mathfrak{C}}^S_c$.  As in the proof of Proposition \ref{prop:KH_guarantees}, we have that  
$\Pr( \|\mean_{S,n} - \mean_S\|_S \geq \delta/4) \leq \exp(-(1/288) (n^{1/2}\delta - 16 \|k\|_\infty^{1/2}  )^2 )$
whenever $n \geq (16 \|k\|_\infty^{1/2} /\delta)^2$. Furthermore, there is a ball or radius $\delta/2$ around $\mean_S$ in $C_n$ (as a subset of the affine span of $C_S$) with probability at least  
\[
1- 2 \exp\Bigl( - \frac{1}{2} \Bigl(  \frac{n^{1/2}\bar \lambda_d^2}{8\|k\|_\infty }  - \frac{192 \|k\|_\infty}{\bar \lambda_d}   
\Bigr)^2 \Bigr),
\]
whenever $n \geq (1536 \|k\|_\infty^2 / \bar \lambda_d^3)^2$. Hence, with probability at least 
\[
	1 - 3 \exp\bigl(- (c_3 n^{1/2} - c_4)_+^2  \bigr) 
\]
there is a ball or radius $\delta/4$ around $\mean_{S,n}$ in $C_n$ (as a subset of the affine span of $C_S$). The second result follows since the CGM reduces the error in each step and the initial error is bounded by $\|\wh \mean_1 -  \mean_n \| \leq 2 \|k\|^{1/2}_\infty$.

The third statement follows along similar lines. In the event that we have a ball of size $\delta/4$ it follows that 
$\|\wh \mean_t - \mean_n \| \leq \exp(- \delta(t-1)  / 24 \|k\|_\infty^{1/2})$. 
Setting the right hand side to $n^{-1/2}$ leads to $\hat t \leq  \lceil 1 + 12 \|k\|_\infty^{1/2} \log(n)/\delta   \rceil$. 
\end{proof}

\subsubsection{Compression for kernel regression}
We can also apply Algorithm \ref{alg:CGM} to compress the data for kernel regression. The only thing that we need to do is to use
the kernel function $\tau((y,x),(y',x')) = (\kappa_y + \rho)((y,x), (y',x')) = \kappa(x,x') + yy' k(x,x')$ that we used in Section \ref{sec:simult_risk} and where $k$ is some kernel function on the space $\X$, and to cap the response variables $Y$. We state the corresponding result for the compression of the mean element in high probability below. One can obviously also derive bounds on the deviation in expectation and the expected number of points in the core-set. 
\begin{proposition}
Let $(\X \times \R, \mathfrak{T}, \mathcal{A},P)$ be a topological measure space such that $P$ is a Radon probability measure which has support $S$. Let $k$ be continuous bounded kernel function defined on $\X$ such that the corresponding RKHS is finite dimensional and does not contain the constant functions. Let $(X_1,Y_1), \ldots, (X_n,Y_n)$ be i.i.d. random variables with law $P$, and assume that $Y_i = f_0(X_i) + \epsilon_i$, for all $i \leq n$, where $f_0$ is a measurable and bounded function and $\epsilon_1,\ldots, \epsilon_n$ are i.i.d. sub-Gaussian random variables with variance $0 < \sigma^2$ which are independent of $X_1,\ldots, X_n$. Let $\bar \lambda_{\star,\tilde S}$ be the smallest eigenvalue of the covariance operator $\Cov^{\tilde S}_k$ corresponding to the kernel function $k\!\upharpoonright\! \tilde S \times \tilde S$, $\tilde S = \overline{\{x: (x,y) \in S\}}$.
Chose $q \in (0,1)$ and define the sequence $\{r_n\}_{n\geq 1}$ in the following way: 
\[
r_1 = 1 \vee 2 \sigma \log^{1/2}\Bigl(\frac{22 \|k\|_\infty (\sigma + \|f_0\|_\infty + \|k\|_\infty^{1/2}) }{\sigma^2 \bar \lambda_{\star,\tilde S}}
\Bigr) 
\]
and for $n\geq 2$, let
\[
r_n = r_1 \vee\sqrt{2} \sigma \vee \sqrt{2} \sigma \log^{1/2}(16 n^{1/2}  \sigma^2 \|k\|_\infty/  q). 
\]
Define $\wideparen Y^{(n)} = (Y \wedge (r_n + \|f_0\|_\infty)) \vee - (r_n + \|f_0\|_\infty)$, let $\wideparen \mean_{\tau,n}$ be the empirical mean element corresponding to the kernel $\tau$ and the data $(X_1, \wideparen Y_1^{(n)}), \ldots, (X_n, \wideparen Y_n^{(n)})$ and let $\wh{\wideparen \mean}_{\tau,t}$ be the output of the algorithm when applied to the capped data and  $\wideparen \mean_{\tau,n}$. Let 
$\beta = 48  \|\tau \|_{S_{f,n},\infty} /\sigma^2 \bar \lambda_{\star,\tilde S}$ 
then with probability $1-q$
\[
\| \wh{\wideparen \mean}_{\tau,\lceil \beta \log(4 n \|\tau \|_{S_{f,n},\infty}) \rceil +1} - \mean_{\tau,n} \| \leq  2 n^{-1/2} 
\]
whenever $n$ is (strictly) greater than
\begin{align*}
	&\left(
\frac{16 \|\tau\|_{S_{f,n},\infty} (\sqrt{2\log(12/q)} + 384 \|\tau\|_{S_{f,n},\infty} /\sigma^2 \bar \lambda_{\star,\tilde S})}{
\sigma^4 \bar \lambda^2_{\star, \tilde S}} \right)^2 \\
	&\vee \left(\frac{64\|\tau\|_{S_{f,n},\infty} +  68 \|\tau\|_{S_{f,n},\infty}^{1/2} \log^{1/2}(4/q)}{\sigma^2 \bar \lambda_{\star,\tilde S}}\right)^{2}\!\!.
\end{align*}
\end{proposition}
\begin{proof}
The statement follows from Proposition \ref{prop:regression_ball}. In particular, under the states conditions and with probability at least $1-q$, simultaneously $\|\wideparen \mean_{\tau,n} - \mean_{\tau,n}\|_\tau \leq n^{-1/2}$ and 
\[
\| \wh{\wideparen \mean}_{\tau,t} - \wideparen{\mean}_{\tau,n} \|_\tau \leq 2 \|\tau\|_{S_{f,n},\infty}^{1/2}
\exp \left( - \frac{\wideparen \delta^{(n)}}{4} \frac{t-1}{6 \|\tau \|_{S_{f,n},\infty}^{1/2}}
\right).
\]
Setting the right side of the last equation equal to $n^{-1/2}$ yields
\[
t-1 = \left\lceil \frac{12 \|\tau \|_{S_{f,n},\infty}^{1/2} \log(4 n \|\tau \|_{S_{f,n},\infty} )}{\wideparen \delta^{(n)}} 
	\right\rceil. 
\]
Replacing $\wideparen \delta^{(n)}$ by its lower bound $\sigma^2 \bar \lambda_{\star,\tilde S}/4 \|\tau \|_{S_{f,n},\infty}^{1/2}$ gives the constant $\beta$ stated in the proposition. 

\end{proof}
\begin{remark}
The eigenvalue in the definition of the sequence $r_n$ can be replaced in that definition by a lower bound on this eigenvalue. Similarly, the term $\|f_0\|_\infty$ in the definition of $\wideparen Y$ can be replaced by an upper bound. We also used here the lower bound $\sigma^2 \bar \lambda_{\star,\tilde S}$ on $\wideparen \lambda_\star^{(n)}$ instead of using $\wideparen \lambda_\star^{(n)}$ directly. This affects, in particular, the number $n$ from which point onward the compression results apply. 
\end{remark}

\subsubsection{Avoiding the explicit computation of $s$} \label{sec:fast_preproc_cgm}
Mitigating the cost of computing $s$ is more difficult when the CGM is used. The main problem is that we are aiming for a run-time of $O(n \log(n))$ and there is not much leeway in each iteration. For instance, if, like for kernel herding, we split the data into $\sqrt{n}$ batches of size $\sqrt{n}$ then we have an overall run-time of $\sqrt{n} \times (\sqrt{n})^2 = n^{3/2}$ because computing $s$ per batch incurs a quadratic cost in the sample size. One way to reduce that compuational cost is to make the quadratic term smaller but then we have many batches. For example, if we aim for a $\log(n)$ batch size then we have $n/\log(n)$ many batches and the reduction in sample size is minuscle. In particular, we could not just run the CGM directly on the $n/\log(n)$ many approximations since that would result in an $n^2/\log^2(n)$ run-time cost. One way around this problem is to apply the process iteratively: in the first iteration use about $n/\log(n)=: T_1$ many batches and compute $\wh \mean_{11}, \ldots \wh \mean_{1T_1}$. This can be done in about $(n/\log(n)) \times \log^2(n) = n \log(n)$ time, resulting in approximations that consist of $\log\log(n)$ many elements each. If we want to allow a run-time of $O(n\log(n))$ per iteration then in the second iteration we can use  $T_2 := T_1 / \log^2(n)$ many batches since $(T_1 / \log^2(n)) \times \log^4(n) = n \log(n)$  (ignoring the $\log\log(n)$ terms). Continuing this process, at iteration $3$, we have $T_3 = T_2 / \log^4(n)$ many batches, and, more generally, for $i \geq 2$, we have $T_i = T_{i-1} / \log^{2^{i-1}}(n)$ many batches. We can stop the iterations when we are down to $\sqrt{n}$ many batches since we can apply the CGM then directly. To get down to $\sqrt{n}$ many batches we need about
\[
	\ell = \log \left(\frac{\log(n)}{\log \log(n)} \right) \approx \log \log(n)
\]
many iterations since
\[
\sqrt{n} \approx T_i = \frac{n}{\prod_{i=1}^\ell \log^{2^{i-1}}(n)} = \frac{n}{\log^{2^\ell}(n)}.   
\]
This then implies an overall run-time of this algorithm of about $O(n\log(n) \log\log(n))$.

A major concern with this algorithm is that we have many optimization problems that have to be solved simultaneously and we need to be lucky in each case to have a ball of sufficient size around the corresponding $\mean_n$ in $C_n$. It seems rather unlikely that we can guarantee for each of these optimization problems the existence of such a ball. A better way to approach this compression problem might be to work instead with fixed error bounds that have to be achieved in each optimization problem. The hope with this approach is that we can then guarantee a sufficient compression but the number of sample points needed might be larger than $\log(n)$. \textit{Algorithm \ref{alg:CGM_s}} implements this idea.

\begin{algorithm}[h!] \caption{(A compression algorithm for $\mean_n$ that uses the CGM and avoids the explicit computation of $s$.)} \label{alg:CGM_s} 
\begin{normalsize}
\begin{align*} 
&\text{Input: sample } X_1,\ldots, X_n, \text{ kernel } k.   \notag \\
&\text{In the follow let \,} \ell = \left \lceil\frac{1}{\log(2)} \log\left(\frac{\log(n) }{\log\log(n)}\right) -1 \right \rceil. \\ 
&\text{Split the sample into } T_1 := \lceil n/\log(n) \rceil \text{ batches}.  \\ 
&\text{Let } \mathcal{I}_{11}, \ldots, \mathcal{I}_{1T_1} \text{ be the corresponding indices of the sample points.} \\
&\text{Apply the CGM to each batch to approximate } \mean_{1,1}, \ldots, \mean_{1,T_1}
\text{ by using sample points} \\
&\text{\quad indexed by } \mathcal{I}_{1,1},\ldots, \mathcal{I}_{1,T_1}
\text{ until the error of all approximations }  j \leq T_1 \\
&\text{\quad is below }  \varepsilon_{1,j} = |\mathcal{I}_{1,j}|^{-1/2}. \\
&\text{Store the approximations in } \wh \mean_{1,1}, \ldots, \wh \mean_{1,T_1} \text{ and let } M_{1,1} = |\mathcal{I}_{1,1}|,\ldots,  M_{1,T_1} = |\mathcal{I}_{1,T_1}|. \\
&\text{Iterate through } i =  2, \ldots, \ell:  \\
&\text{\quad Split the approximations } \wh \mean_{i-1,1}, \ldots, \wh \mean_{i-1,T_{i-1}} \text{ into } T_i := \lceil T_{i-1}/ \log^{2^{i-1}}(n) \rceil \text{ batches}. \\
&\text{\quad Let } \mathcal{I}_{i1}, \ldots, \mathcal{I}_{iT_i} \text{ be the corresponding indices and for all } j \leq T_i \text{ let } \\
&\quad\quad\quad M_{i,j} = \sum_{u\in \mathcal{I}_{i,j}} M_{i-1,u}. 
\\
&\text{\quad For each batch } j \leq T_i \text{ average the old approximations  } \\
&\quad\quad\quad \tilde \mean_{i,j} := \frac{1}{M_{i,j}} \sum_{u \in \mathcal{I}_{i,j}} M_{i-1,u} \wh \mean_{i-1,u}. \\
&\text{\quad Apply the CGM to each batch } j \leq T_i, \text{ approximating } \tilde \mean_{i,j} \text{ by convex combinations } \\
&\text{\quad \quad of the elements } \wh \mean_{i-1,u}, u \in \mathcal{I}_{i,j}, \text{ with an error of at most  }  \varepsilon_{i,j} = M_{i,j}^{-1/2}. \\
&\text{\quad Store the approximations in } \wh \mean_{i,1}, \ldots, \wh \mean_{i,T_i}. \\
&\text{Apply the CGM a final time to } \wh \mean_{\ell,1}, \ldots, \wh \mean_{\ell, T_\ell} \text{ to compress } \frac{1}{T_\ell} \sum_{j=1}^{T_\ell} \wh \mean_{\ell,j} \\
&\text{\quad with an approximation error of at most } n^{-1/2} \text{ and return the approximation}.
\notag 
\end{align*}
\end{normalsize}
\vspace{-0.4cm}
\end{algorithm}

In the algorithm  $\mean_{1,1},\ldots, \mean_{1,T_1}$ denote the mean elements corresponding to the
initial $T_1$ batches. For the analysis of the algorithm it is useful to also have the mean elements corresponding to all the samples entering into the $j$'th batch in iteration $i$; denote this element by $\mean_{i,j}$. The idea of the algorithm is to approximate $\mean_{i,j}$ in iteration $i$ and batch $j$. Working directly with $\mean_{i,j}$ is not possible if we try to stay around $n \log(n)$ computation time per iteration since $\mean_{i,j}$ will consist eventually of about $\sqrt{n}$ many samples in each batch which implies a cost of $n$ per batch. Therefore, we 
approximate $\mean_{i,j}$ first by $\tilde \mean_{i,j}$ which will consist, under suitable conditions, of far fewer sample points. The approximation $\tilde \mean_{i,j}$ is then further compressed into $\wh \mean_{i,j}$ which consists of even fewer sample points. The variables $M_{i,j}$ keep track over how many sample points  $\mean_{i,j}$ is averaged. Hence, $\mean_{i,j} = (1/M_{i,j}) \sum_{u \in \mathcal{I}_{i,j}} M_{i-1,u} \mean_{i-1, u}$ for all $j \leq T_i$ and $2 \leq i \leq  \ell$. 

\textit{We left out a few details in the algorithm.} In particular, the usual vector $s$ that consists of inner products between $k(X_i,\cdot)$ and $\mean_n$ has to be replaced by vectors with entries of the form $\langle \wh \mean_{i-1,u}, \tilde \mean_{i,j} \rangle, u \in \mathcal{I}_{i,j}$, when $i\geq 2$. The element $\tilde \mean_{i,j}$ corresponds to an average over the $\wh \mean_{i-1,u}$ terms and there are $|\mathcal{I}_{i,j}|$ many terms over which this average is taken. The quantity $|\mathcal{I}_{i,j}|$ is not of major concern when bounding the computational complexity. The computational complexity of calculating these inner product vectors is rather dominated by how many points are contained in the approximations $\wh \mean_{i,j}$. Another point worth noting is that the final approximation will ideally by given in terms of convex combinations of the original sample points $k(X_1,\cdot), k(X_2,\cdot),\ldots$. Roughly speaking, this convex combination can be computed by multiplying the weights in the different iterations. Finally, observe that we can keep track of how well $\tilde \mean_{i,j}$ is approximated if $\wh \mean_{i,j}$ does not consist of too many points since the $|\mathcal{I}_{i,j}|$ are chosen small enough that we can compute and store  the corresponding kernel matrices 
\[
	(\langle \wh \mean_{i-1,u}, \wh \mean_{i-1,v} \rangle )_{u,v \in \mathcal{I}_{i,j}}\]
and from these kernel matrices we can compute the approximation errors.

\textit{Bounding the size of the set which is used in the resulting approximation} in high probability or expectation is a major challenge that we will not address here. However, it is easier to say something about the resulting \textit{approximation error} by refining the analysis of the kernel herding algorithm: the philosophy of the algorithm is to guarantee in high probability in each iteration that $\mean_n$ is approximated with an error of $n^{1/2}$. In detail,  observe that for any $1 \leq i \leq \ell$, $\mean_n = (1/n) \sum_{j=1}^{T_i} M_{i,j} \mean_{i,j}$, where we use that $\sum_{j =1}^{T_i} M_{i,j} = n$. We can  use the link between $\wh \mean_{i,j}$ and $\mean_{i,j}$ to measure in each iteration the error when approximating  $\mean_n$ by $(1/n) \sum_{j=1}^{T_i} M_{i,j} \wh \mean_{i,j}$. The naive approach of using the triangular inequality does not lead to useful results 
since  
\[
\| \mean_n - (1/n) \sum_{j=1}^{T_i} M_{i,j} \wh \mean_{i,j} \| \leq \frac{1}{n} \sum_{j=1}^{T_i} M_{i,j} \|\mean_{i,j} - \wh \mean_{i,j} \| \leq \varepsilon_i 
\]
and we would need to set $\varepsilon_i$ to $n^{-1/2}$ to guarantee a low enough approximation error. But aiming in each batch for an error of $n^{-1/2}$ when only $\log(n)$ sample points are in each batch is not useful. As for the kernel herding analysis, a better approach might be to consider the variance of the error and to make use of the independence of the sample points.  Let us first look at the case $i=1$,
\begin{align*}
&E(\|\mean_n - \frac{1}{n} \sum_{j=1}^{T_i} M_{i,j} \wh \mean_{i,j}\|^2)  = \frac{1}{n^2} E( \| \sum_{j=1}^{T_i} M_{i,j} (\mean_{i,j} - \wh \mean_{i,j})\|^2) \\
&= \frac{1}{n^2} \sum_{j=1}^{T_i} E( \| M_{i,j} (\mean_{i,j} - \wh \mean_{i,j})\|^2) + \frac{1}{n^2} \sum_{j_1=1}^{T_i} \sum_{j_2\not=j_1} M_{i,j_1} M_{i,j_2} E( \langle \mean_{i,j_1} - \wh \mean_{i,j_1}, \mean_{i,j_2} - \wh \mean_{i,j_2} \rangle) \\
&= \frac{1}{n^2} \sum_{j=1}^{T_i} M^2_{i,j} E( \| (\mean_{i,j} - \wh \mean_{i,j})\|^2) + \frac{1}{n^2} \sum_{j_1=1}^{T_i} \sum_{j_2\not=j_1} M_{i,j_1} M_{i,j_2} \langle \mean - E(\wh \mean_{i,j_1}), \mean - E(\wh \mean_{i,j_2}) \rangle. 
\end{align*}
As for the kernel herding algorithm we can control the bias term in a crude manner by using that 
\[
|\langle \mean - E(\wh \mean_{i,j_1}), \mean - E(\wh \mean_{i,j_2}) \rangle| \leq \max_{j\leq T_i} \|\mean - E(\wh \mean_{i,j})\|^2 
\leq \max_{j\leq T_i} E(\|\mean_{i,j} - \wh \mean_{i,j}\|^2) \leq \varepsilon_i^2. 
\]
However, this is not leading to an improvement since  in the first iteration
\begin{align*}
&E(\|\mean_n - \frac{1}{n} \sum_{j=1}^{T_1} M_{1,j} \wh \mean_{1,j}\|^2) \leq \frac{\varepsilon_1^2 \max_{j' \leq T_1} M_{1,j'}^2 (T_1 + T_1^2)}{n^2} \approx \varepsilon_1^2  (1+ 1/n) 
\end{align*}
and 
\[
\Pr(\|\mean_n - \frac{1}{n} \sum_{j=1}^{T_1} M_{1,j} \wh \mean_{1,j}\| \geq n^{-1/2}) \leq \varepsilon_1^2(n + 1). 
\]
implies that $\varepsilon_1$ would have to be of order $n^{-1/2}$. A central question at this point is \textit{of what order is the bias term}. In particular, is the upper bound of $\varepsilon_1^2$ for the squared bias term overly pessimistic? A natural threshold for the error in each batch is $\log^{-1/2}(n)$ in the first iteration since there are about $\log(n)$ many samples in each batch. For $\log^{1/2}(n)$ to be sufficiently low we need a bound on the bias term of about  $c n^{-1/2}$, $ c \in (0,1)$, since then  
\[
\Pr(\|\mean_n - \frac{1}{n} \sum_{j=1}^{T_1} M_{1,j} \wh \mean_{1,j}\| \geq n^{-1/2}) \leq c + \varepsilon_1^2 \log(n) \leq 2 c, 
\]
when a threshold of $c^{1/2} \log^{-1/2}(n)$ is used in the optimization. In other words, the bias term has to fall exponentially fast to allow for a threshold that is proportional to the sample size, i.e. the bias has to be below $\exp(-m/2)$, where $m = \log(n)$ is the sample size in each batch in the first iteration.  

\textit{The error in the successive approximations} can be treated in a similar way and since there are only about $\log\log(n)$ many iterations a simple union bound argument suffices to control the error simultaneously over all iterations. To demonstrate how the error evolves consider $i=2$, then 
\begin{align*}
&E(\|\mean_n - \frac{1}{n} \sum_{j=1}^{T_2} M_{2,j} \wh \mean_{2,j}\|^2)  \\
&\leq  3(E(\|\mean_n - \frac{1}{n} \sum_{j=1}^{T_1} M_{1,j} \wh \mean_{1,j}\|^2)  + E(\|\frac{1}{n} \sum_{j=1}^{T_1} M_{1,j} \wh \mean_{1,j} - \frac{1}{n} \sum_{j=1}^{T_2} M_{2,j} \wh \mean_{2,j}\|^2) ) 
\end{align*}
and
\begin{align*}
&n^2 E(\|\frac{1}{n} \sum_{j=1}^{T_1} M_{1,j} \wh \mean_{1,j} - \frac{1}{n} \sum_{j=1}^{T_2} M_{2,j} \wh \mean_{2,j}\|^2) = E(\| \sum_{j=1}^{T_2} (M_{2,j} \tilde \mean_{2,j} -  M_{2,j} \wh \mean_{2,j}) \|^2) \\
&=  \sum_{j=1}^{T_2} M^2_{2,j} E(\| \tilde \mean_{2,j} -   \wh \mean_{2,j} \|^2)
+ 2 \sum_{j=1}^{T_2} \sum_{u\not = j} M_{i,j} M_{i,u} \langle E(\tilde \mean_{2,j} -   \wh \mean_{2,j}), E(\tilde \mean_{2,u} -   \wh \mean_{2,u}) \rangle   
\end{align*}
where we can move the expectation inside the inner produce since $\tilde \mean_{2,j}$ and $\wh \mean_{2,j}$ are independent of $\tilde \mean_{2,u}$ and $\wh \mean_{2,u}$. The bias term that is now important is $\|E( \tilde \mean_{2,j})  - E(\wh \mean_{2,j})\|$ and we need a similar fast decay of the bias as for $i=1$. The other term is easier to deal with,
\begin{align*}
\frac{1}{n^2} \sum_{j=1}^{T_2} M^2_{2,j} E(\| \tilde \mean_{2,j} -   \wh \mean_{2,j} \|^2)
\leq \frac{1}{n^2} \sum_{j=1}^{T_2} M_{2,j} =\frac{1}{n}
\end{align*}
by the choice of $\varepsilon_{i,j}$ in the algorithm.

\textit{There are a few open problems} concerning this algorithm, and variations thereof. The algorithm is set up to enforce tighter and tighter error bounds in each iteration, i.e. the error threshold changes  approximately from $\log^{-1/2}(n)$ in the first iteration to $\log^{-3/2}(n)$ in the second iteration and $\log^{-7/2}(n)$ in the third iteration. The hope is that good approximations in the first iteration allow us to get even better approximations in the second round and so forth. But it is by no means obvious that this intuition is correct and in all likelihood these choices are not optimal. 

The next major obstacles in controlling the error of the algorithm are obviously the bias terms. If there is an exponential decrease in the bias then we are in a very fortunate situation and can control the approximation error. If the bias term decreases slower then it might be worth to consider alternatives of the CGM which incorporate bias reduction techniques and are not focusing solely on the approximation error.

The biggest challenge when studying this algorithm is in all likelihood the problem of controlling the size of the ball around the various elements $\mean_{i,j}$ simultaneously over all iterations and batches. In fact, a uniform bound might even be suboptimal for analyzing the performance of the algorithm since small ball sizes can be compensated for by batches that have a larger ball around their corresponding $\mean_{i,j}$ and which need less sample points than suggested by a worst case bound. In other words, we might need to control the fluctuations or the distribution of the ball sizes.

\section{Applications} \label{sec:applications}
In the following, we look at how these techniques can be combined with machine learning methods. In particular, we are looking at the two sample problem,  at kernel ridge regression and at kernel PCA. Since it is currently unclear what compression rates can be achieved when avoiding the upfront cost of $O(n^2)$, we formulate the runtime statements as functions of $\psi_{\text{comp}}(n)$ and $\psi_{\text{size}}(n)$, where $O(\psi_{\text{comp}}(n))$ is the computational cost for calculating the compression and $O(\psi_{\text{size}}(n))$ is order of the number of points that are needed in the compression to guarantee, with high probability, that the compression is no more than $c n^{-1/2}$, $c >0$, away from the mean element that corresponds to the empirical measure. In the finite dimensional settings that we consider and when using the standard algorithms, we can use $\psi_{\text{comp}}(n) = n^2$ and $\psi_{\text{size}}(n) = \log(n)$. Generally, the hope is that these can be changed to something of the form  $\psi_{\text{comp}}(n) = n \log^\alpha(n)$ and $\psi_{\text{size}}(n) = \log^\alpha (n)$, $\alpha \geq 1$.

\subsection{Two Sample Test}
In the two sample test problem  i.i.d. data $X_1,\ldots, X_n$ and $Y_1, \ldots, Y_m$ attaining values in $\X$ are given, the $X_i$'s are furthermore independent from the $Y_i$'s but it is unknown if the $X_i$'s have the same distribution as the $Y_i$'s. The null-hypothesis is that the distributions are equal. One way to build a test statistic for this hypothesis testing problem is to consider $\| \mean_{X,n} - \mean_{Y,m}\|$, where $k$ is a kernel function on $\X$, $k(X,\cdot), k(Y,\cdot) \in \mathcal{L}^1(P)$, $\mean_{X,n} = (1/n) \sum_{i=1}^n k(X_i,\cdot)$ and $\mean_{Y,m} = (1/m) \sum_{i=1}^m k(Y_i,\cdot)$. Calculating the norm can be done in $O((n\vee m)^2)$ by using that
\[
\|\mean_{X,n} - \mean_{Y,m}\|^2 = \frac{1}{n^2}\sum_{i,j=1}^n k(X_i,X_j) - \frac{2}{nm} \sum_{i=1}^n \sum_{j=1}^m k(X_i,Y_j)  + \frac{1}{m^2} \sum_{i,j=1}^m k(Y_i,Y_j).
\]
When using one of the compression approaches this turns into a run-time of the order
$O( (\psi_{\text{comp}}(n) \vee \psi_{\text{comp}}(m)) \vee (\psi_{\text{size}}(n) \vee \psi_{\text{size}}(m))^2)$. In particular, we can simply replace $\mean_{Y,m}$ and $\mean_{X,n}$ by their approximations. 
Furthermore, with high probability, the rate of convergence of $\|\mean_{X,n} - \mean_{Y,n}\|$ to 
$\|\mean_X - \mean_Y\|$, where $\mean_X = \int k(X,\cdot) \, dP$ and $\mean_Y = \int k(Y,\cdot) \, dP$,
will be preserved when moving to the compression.


\subsection{Kernel ridge regression}
Let us consider now the regression problem with data $(X_1,Y_1), \ldots, (X_n,Y_n)$, where we assume that the pairs are independent and that the $Y_i$ are bounded.
When the conditional gradient method is used to approximate $\mathfrak{C}_y$, $\mean_y$ and $(1/n) \sum_{i\leq n} Y_i$  simultaneously we get a single index function $\iota:\{1,\ldots,l\} \to \{1,\ldots,n\}$ and corresponding approximations  
$\hat{\mathfrak{C}}_{y,l} = \sum_{i=1}^l w_i \kappa(X_{\iota(i)}, \cdot)$ and 
$\hat{\mean}_{y,l} = \sum_{i=1}^l w_i \langle Y_{\iota(i)}, \cdot\rangle \otimes k(X_{\iota(i)},\cdot)$ with strictly positive $w_i$'s such that $w_1+\ldots + w_l = 1$. The approximation of the least-squares error for a function $h \in \cH$ is  
\begin{align*}
	\sum_{i=1}^l w_i (Y_{\iota(i)} - h(X_{\iota(i)}))^2 =  \langle \hat{\mathfrak{C}}_{y,l}, \tilde h \rangle_{\widehat{\cH \odot \cH}}
	- 2 \langle \hat \mean^{\otimes}_{y,l}, \check{h} \rangle_{\dR \otimes \cH} 
	+ \sum_{i=1}^l w_i Y_{\iota(i)}^2, 
\end{align*}
where we denote the function $(x,y) \mapsto h^2(x)$ with $\tilde h$.
Due to the representer theorem  we can write the solution  to the ridge regression problem in the form $h_\star = \sum_{i=1}^l \alpha_i k(X_{\iota(i)}, \cdot)$ for suitable $\alpha_i \in \mathbb{R}$. Substituting this into the  equation for the least-squares error and ignoring the last term  (which is irrelevant for finding the solution) leads to
\begin{align*}
\langle \hat{\mathfrak{C}}_{y,l}, \tilde{h}_\star \rangle_{\widehat{\cH \odot \cH}} - 2 \sum_{i=1}^l \alpha_i \langle \hat{\mean}^\otimes_{y,l}, \hat k(X_{\iota(i)},\cdot) \rangle_{\dR \otimes \cH}.
\end{align*}
Let $C_l$ be an $l\times l$ matrix with the entry in row $i$ and column $j$ being  
\begin{align*}
	&
	\sum_{u=1}^l w_u \langle k(X_{\iota(u)},\cdot)\otimes k(X_{\iota(u)},\cdot),k(X_{\iota(i)},\cdot)\otimes k(X_{\iota(j)},\cdot) \rangle_\otimes  \\
	 &= \sum_{u=1}^l w_u k(X_{\iota(u)},X_{\iota(i)}) k(X_{\iota(u)},X_{\iota(j)}). 
\end{align*}
then 
\[
	\langle \hat{\mathfrak{C}}_{y,l}, \tilde{h}_\star \rangle_{\widehat{\cH \odot \cH}} 
	= \alpha^\top C_l \alpha.
\]

Also, let $K_l$ be the  kernel matrix for samples $X_{\iota(1)},\ldots, X_{\iota(l)}$ and let $m_l$ be an $l$-dimensional vector with entry $i$ being 
\[
	\langle \hat \mean^\otimes_{y,l}, \hat k(X_{\iota(i)},\cdot) \rangle_{\dR \otimes \cH} = \sum_{u=1}^l w_u Y_{\iota(u)} k(X_{\iota(i)}, X_{\iota(u)}).
\]
With these in place the solution of the ridge-regression problem with regularization parameter $\lambda >0$ is found by minimizing 
\[
	\alpha^\top C_l \alpha - 2 \alpha^\top m_l + \lambda \alpha^\top K_l \alpha
\]
with respect to $\alpha \in \mathbb{R}^l$. Taking the gradient with respect to $\alpha$ and setting it to zero yields
\begin{equation} \label{eq:ridge_reg_zero}
2 C_l \alpha - 2 m_l + 2 \lambda K_l \alpha = 0.
\end{equation} 
Observe that $C_l = K_l W K_l$, where $W$ is a diagonal matrix with $W_{uu} =w_u$ for all $u\leq l$. Similarly,
$m_l = K_l W y = K_l W (K_l W)^\dagger K_l W y$, where $y$ is a vector with entries $y_u = Y_{\iota(u)}$ for all $u\leq l$. Hence, we can rewrite \eqref{eq:ridge_reg_zero} as
\[
	K_l ( (WK_l + \lambda I_l)  \alpha - W (K_l W)^\dagger K_l W y)=0,
\]
where $I_l$ is the $l\times l$ identity matrix. Since $W$ has strictly positive entries on the diagonal we can rewrite this as  
\begin{equation} \label{eq:Tikh_normal_eq}
	K_lW ( (K_l + \lambda W^{-1})  \alpha - (K_l W)^\dagger K_l W y)=0,
\end{equation}
for which a solution is given by 
\begin{equation} \label{eq:Tikh_minimal}
	\alpha = (K_l + \lambda W^{-1})^{-1} (K_l W)^\dagger K_l Wy.
\end{equation}
The inverse is well defined because $K_l$ is p.s.d. and $W^{-1}$ is (strictly) positive definite; the sum of a p.s.d. and strictly positive definite matrix is strictly positive definite and, therefore, has an inverse. Also $(K_l W)^\dagger = (K_l W)^{-1}$ whenever $K_l$ is of full rank and in this case
\begin{equation} \label{eq:Tikh_suboptimal}
	\alpha =  (K_l + \lambda W^{-1})^{-1} y. 
\end{equation}
This $\alpha$ is also a solution to \eqref{eq:Tikh_normal_eq} in the general case when $K_l$ is not full rank
since $(K_l W)^\dagger K_l W y$ can be replaced by $y$ in this equation.

In terms of the runtime, if we use an $O(n^3)$ algorithm for deriving the inverse then, after compression, the runtime is $O(\psi_{\text{comp}}(n)  \vee (\psi_{\text{size}}(n))^3)$. For example, if we work with a finite dimensional RKHS and use the standard CGM then we attain a runtime of $O(n^2)$. Beside the reduction in runtime the storage demand also goes down since only a matrix of size $\psi_{\text{size}}(n) \times \psi_{\text{size}}(n)$ has to be stored for calculating $\alpha$, and this can  be as small as $\log(n) \times \log(n)$. The CGM itself needs memory in the order of $O(n)$. 

\subsection{Kernel PCA}
The plug-in estimator of an eigenfunction of the 
covariance operator has a large bias when working in infinite dimensionsal RKHSs and does not achieve the minimax optimal rate of convergence \cite{KOLT18}. However, in finite dimensional RKHSs this is not of major concern and we can use the eigenfunction of  $\tilde{\mathfrak{C}}_n \in L(\cH,\cH)$ as an estimate of the eigenfunctions of the covariance operator. 
In this context we want to approximate $\tilde{\mathfrak{C}}_n$, which is
given by 
\[
\frac{1}{n} \sum_{i=1}^n k(X_i,\cdot) \wh{\otimes} k(X_i,\cdot),
\]
by using the CGM. 
As discussed in Section \ref{sec:cov_op_approx}
 we can apply the CGM to the RKHS with the kernel function $\kappa(x,y) = k^2(x,y)$ to approximate $\mathfrak{C}_n$ with some convex combination $\wh{\mathfrak{C}}_t = \sum_{i=1}^t \alpha_i \kappa(X_{\iota(i)},\cdot)$, where $\alpha_i \geq 0$ for all $i\leq t$, $\alpha_1 + \ldots + \alpha_t = 1$, and $\iota$ is some selection of data points. The element $\mathfrak{C}_n$ is closely related to the operator $\tilde{\mathfrak{C}}_n$ and 
a natural approximation of $\tilde{\mathfrak{C}}_n$ is $\wh{\tilde{\mathfrak{C}}}_t = \sum_{i=1}^t \alpha_i k(X_{\iota(i)}, \cdot) \wh{\otimes} k(X_{\iota(i)}, \cdot)$. Note that for any $f,g \in \cH$, 
\[
\langle \wh{\mathfrak{C}}_t, (f \otimes g) \circ \psi \rangle_\odot
= \sum_{i=1}^t \alpha_i f(X_{\iota(i)}) g(X_{\iota(i)}) 
= \langle  \wh{\tilde{\mathfrak{C}}}_t f, g \rangle, 
\]
where $\psi:\X \to \X  \times \X$,   $\psi(x) = (x,x)$.
The operator $\wh{\tilde{\mathfrak{C}}}_t$ is clearly symmetric  and, hence, self-adjoint since the RKHS is finite dimensional. Furthermore, all eigenvalues are non-negative since if $e \in \cH$ is an eigenfunction of $\wh{\tilde{\mathfrak{C}}}_t$ then
\[
\langle \wh{\tilde{\mathfrak{C}}}_t e, e \rangle =  \sum_{i=1}^t \alpha_i e^2(X_{\iota(i)}) \geq 0.  
\]
The main question is now if we can quantify the difference between eigenfunctions of $ 
\wh{\tilde{\mathfrak{C}}}_t$ and $\tilde{\mathfrak{C}}_n$. Let us assume that there are no multiple eingenvalues and that $\wh \lambda_1 \geq \ldots \geq \wh \lambda_d >0$ are the eigenvalues of  $\wh{\tilde{\mathfrak{C}}}_t$ and $e_1,\ldots, e_d$ are the corresponding eigenfunctions. Similarly, let $\mu_1,\ldots, \mu_d > 0$ be the eigenvalues of $\tilde{\mathfrak{C}}_n$ and $f_1,\ldots, f_d$ the corresponding eigenfunctions. Furthermore, assume that the CGM is run until 
\[
\| \wh{\mathfrak{C}}_t - \mathfrak{C}_n\|_\odot \leq \epsilon.
\]
Since 
\[
\|\wh{\tilde{\mathfrak{C}}}_t - \tilde{\mathfrak{C}}_n \|_{op} = \sup_{\|h\|=1} 
\|\wh{\tilde{\mathfrak{C}}}_t h - \tilde{\mathfrak{C}}_n h  \| = \sup_{\|h\|=1} \sup_{\|g\|=1} 
\langle (\wh{\tilde{\mathfrak{C}}}_t - \tilde{\mathfrak{C}}_n)(h), g \rangle 
= \langle  \wh{\mathfrak{C}}_t - \mathfrak{C}_n, (h \otimes g) \circ \psi \rangle_\odot
\]
and $\|(h \otimes g) \circ \psi \|_\otimes \leq \|h \otimes g\|_\otimes = \|h\|\|g\| \leq 1$ it follows from the Cauchy-Schwarz inequality that 
\[
\|\wh{\tilde{\mathfrak{C}}}_t - \tilde{\mathfrak{C}}_n \|_{op} \leq  \| \wh{\mathfrak{C}}_t - \mathfrak{C}_n\|_\odot \leq \epsilon.
\]
From this bound on the operator norm it follows right away that 
$\|(\wh{\tilde{\mathfrak{C}}}_t - \tilde{\mathfrak{C}}_n) (e_i)\|$ and 
$\|(\wh{\tilde{\mathfrak{C}}}_t - \tilde{\mathfrak{C}}_n) (f_i)\|$ are less than $\epsilon$ for all $i\leq d$. In particular,
\[
|\lambda_i - \langle e_i, \tilde{\mathfrak{C}}_n e_i \rangle| \leq \epsilon \text{\quad and \quad}
|\langle f_i, \wh{\tilde{\mathfrak{C}}}_t f_i \rangle - \mu_i| \leq \epsilon
\]
for all $i\leq d$. In particular,
\[
\lambda_1 \leq \langle e_1, \tilde{\mathfrak{C}}_n e_1 \rangle + \epsilon \leq  
\sup_{\|h\|=1} \|\tilde{\mathfrak{C}}_n h\| \epsilon \leq \mu_1 +\epsilon.     
\]
By symmetry of the argument it follows that $|\lambda_1 - \mu_1| \leq \epsilon$. The difference between $e_1$ and $f_1$ can now also be controlled: let $a_1,\ldots, a_d \in \R$ be such that $ e_1 = a_1 f_1 + \ldots + a_d f_d$ then $1 = \|e_1\|^2 = a_1^2 + \ldots a_d^2$,  \[
\mu_1 - 2\epsilon \leq \langle e_1, \tilde{\mathfrak{C}}_n e_1 \rangle = 
\sum_{i=1}^d a_i^2 \mu_i 
\]
and from $(1-a_1^2) \mu_1 -2 \epsilon \leq (1-a_1^2) \mu_2$ we can infer that for sufficiently small $\epsilon >0$,
\[
	a_1^2 \geq 1 -  \frac{2 \epsilon}{\mu_1 - \mu_2} \text{\quad and \quad}
	\|e_1 - f_1\|^2 = 2 - 2 a_1 \leq 2 -2 \Bigl(1- \frac{2\epsilon}{\mu_1 - \mu_2}\Bigr)^{1/2}\!\!\!\!\!\!. 
\]
The other eigenfunctions can be treated in a similar way by moving to the subspaces that are orthogonal to the already covered eigenfunctions $e_1,\ldots, e_l$, $l\leq d$. 

The computational complexity of the eigendecomposition is $O(n^3)$ as for kernel ridge regression. By compressing the data this goes down to $O(\psi_{\text{comp}}(n)  \vee (\psi_{\text{size}}(n))^3)$.






\section{Example: Slow rate of convergence in infinite dimensions}

The last section of this paper is dedicated to the construction of the example for which the kernel herding algorithm performs strictly worse than in finite dimensions when the density function of the data distribution has a density that is bounded away from zero. The corresponding theorem is the following.
 
\begin{theorem} \label{prop:counter-main}
There exists an initialization, a continuous kernel, and a Borel probability measure on $[0,1]$ which assigns non-zero probability to open intervals for which the kernel herding algorithm does not converge fast, i.e. there exists no constant $b$ such that $\|\mean_t- \mean \| \leq b/t$ for all $t \geq 1$. 
\end{theorem}

The proof of Proposition \ref{prop:counter-main} is split into two parts. In the first part, we construct a Hilbert space, a map $\phi: [0,1]\to \cH$, and an element $\mean \in \cH$ such that the algorithm does not converge fast. We then use this Hilbert space to construct an RKHS for which the algorithm behaves in exactly the same way as when acting on the Hilbert space, and, consequently, the algorithm does not converge fast when applied to the RKHS.

\paragraph{The construction idea.}
Before getting into the technical details we like to outline the basic intuition of the construction: let the mean element $\mathfrak{m} = 0$. Then, given an infinite dimensional Hilbert space $\cH$, choose an orthonormal sequence 
$\{e_n\}_{n\geq 1}$ and elements $\{a_n\}_{n\geq 1}$ in $\cH$ such that each $a_n$ is a multiple of $e_n$. Initialize the algorithm with an element $c \in \cH$ which is of small magnitude compared to the $a_n$ and has a positive inner product with each $a_n$. The idea is that the different $a_n$'s will be chosen at one point by the algorithm and will add to the (rescaled) approximation error $w_t$ of $\mathfrak{m}$ ($t$ is the iteration number of the algorithm). In fact, we like to show that its norm will diverge to infinity. 

This initial construction has a few problems which have to be addressed to make this construction work.  The first problem with this construction is that $\inner{a_n}{c}$ is positive. In fact, $\inner{a_n}{e_n}, \inner{c}{e_n} > 0$ for all $n\geq 1$. But, we want the mean element $\mathfrak{m}$ to be $0$. Hence, we will need probability mass on the negative side to counter the mass accumulated by the $a_n$ and $c$. We can achieve this by introducing another set of elements $\{b_n\}_{n\geq 1}$ which are lying opposite to the $a_n$. Therefore, each $b_n$ is a negative multiple of $e_n$. These $b_n$ need to be further constraint in magnitude. If they are of a similar order like the $a_n$ then they can cancel the weight added to $w_t$ by the $a_n$'s. We are using here sequences with values in the order of $1/\ln(n+1)$ for $a_n$ and $-2^{-n}$ for  $b_n$.

Even though the $b_n$'s are of small magnitude compared to the $a_n$'s  it is not directly obvious why these $b_n$'s should not be chosen many times by the algorithm to cancel step-by-step the weight accumulated by the $a_n$'s. Here is an argument why this does not happen: the $a_n$'s are constructed such that each $a_n$ is chosen exactly once and they are selected in order by the algorithm. At a given iteration there is then an element $a_m$ which has not yet been chosen and our construction assures that in this case $\inner{a_m}{w_t}$ equals the initial value $\inner{a_m}{c}$, which is of magnitude $1/(m\ln(m+1))$. Since the algorithm chooses the element $h \in \phi[\X]$ that maximizes the inner product with $w_t$ we can infer that this inner product must be larger than $1/(m\ln(m+1))$. Or put differently, an element $b_n$ will only be chosen if $\inner{b_n}{w_t} \geq 1/(m\ln(m+1))$, that is $\inner{e_n}{w_t} \geq 2^{n}/(m\ln(m+1))$. If, in fact, the algorithm chooses, in this case, $b_n$ then we are at least assured that $\inner{e_n}{w_{t+1}} \geq 2^{n}/(m\ln(m+1)) - 2^{-m}$  (Figure \ref{fig:CounterEx} on page \pageref{fig:CounterEx} visualizes these bounds for different $m$). We do not need this extra scaling of $2^n$ and we use in the proof only that there are sufficiently many $e_n$ for which $\abs{\inner{e_n}{w_t}}$ is larger than 
$1/\ln(m+1)$. The number of elements for which the inner product is at least of this size grows in $m$ and the sum over these inner products gives us a diverging number that approaches infinity in $m$. This is then sufficient to show that the norm of $w_t$ diverges.  

\subparagraph{Interlacing.} In the above discussion we assume $\mathfrak{m} = 0$. However, constructing the probability measure such that  $\mathfrak{m} = 0$ is not straightforward. The problem is that the scaling on the positive side (the $a_n$'s and the $c$) is exponentially larger than the scaling on the negative side  (the $b_n$'s).  To get $\mathfrak{m} = 0$ we would need the probability mass for the $a_n$'s and $c$ times the magnitude of these elements to be scaled so that it equals the probability mass of the $b_n$'s times the scale of the $b_n$'s. The exponential difference in scale implies that the probability mass of the $b_n$'s needs to grow exponentially in $n$ and the sum of all this mass has to add up to infinity.

By closer inspection, one can observe that the $a_n$'s pose no serious problem since one can just downscale the probability assigned to them by an exponential factor. However, the $c$ poses a more serious problem. Let $p > 0$ be the probability corresponding to $c$. We use  
$c= \sum_{n=1}^\infty n^{-1} e_n$ and we thus have a factor of $p/n$ pulling the mean element towards the positive direction in dimension $n$. Hence, we will need a probability of $p_n = p 2^n/n$ for the $b_n$ elements to counter this pull. Since $p$ does not change with $n$ we are left with $p_n$'s that grow rapidly in $n$.   

Using an initialization $c$ is in a way too rigid and does not allow us to assign lower probability mass as $n$ increases. One way to overcome this problem is to break the initialization up and add probability mass to the different dimensions while the algorithm is running. We do this by replacing the single $c$ with infinitely many elements, one for each dimension $e_n$. Since we do not want to alter the overall behavior of the algorithm these different elements will need to be of a low scale and we need to sum multiple elements to regain the $1/n$ value that $c$ would have assigned. Therefore, for each dimension $e_n$, we are left with a finite sequence of elements $c_{n,1}, c_{n,2},\ldots$ which takes the role of the original $c$. 

The question is then how we can guarantee that all these $c_{n,i}$ elements are chosen to simulate the initialization through $c$ before the algorithm proceeds as usual. We guarantee this by introducing dimensions $\tilde e_{n,i}$ which are orthogonal to all the $e_n$. These dimensions are used to force the algorithm to choose $c_{n,i+1}$  
after $c_{n,i}$ until the final element of the sequence is chosen and we have a weight of $1/n$ in dimension $e_n$.

We still have not addressed the problem of assigning different probabilities to the different dimensions. But, since $c$ is now broken into many small pieces, it is easy to `lose' probability in $n$. 

\begin{proposition}  \label{Thm:CounterExample}
For any infinite-dimensional Hilbert space $\cH$ there exists a continuous function $\phi:\X \rightarrow \cH$, $\X:=[0,1]$, a probability measure $P$ on $\mathscr{B}_{[0,1]}$  which assigns positive measure to any open subset of $\X$, and an initialization $w_1 \in \phi[\X]$ such that the kernel herding algorithm when applied to $\int \phi(x) \,dP(x)$ generates a sequence $\{w_t\}_{t\geq 1}$ that is unbounded and the algorithm does not converge with a $1/t$ rate to $\mean = \int \phi(x) \,dP(x) \in \cH$.
\end{proposition}
\begin{proof}
\textbf{(a) Definition of the convex set:}
Let $\{N_i\}_{i=1}^\infty$ be a set of natural numbers to be defined below, pick a countable infinite orthonormal sequence $\{e_n'\}_{n \geq 1}$ in $\cH$  and split this sequence into 
$\{e_n\}_{n \geq 1} $ and the sequences $\tilde e_{n,1}, \ldots, \tilde e_{n,N_n}$ where $n$ goes through $2,3 \ldots $. This can be done since these are countable many sequences of $N_n +1$ elements and since countable unions of countable sets are again countable.
 Furthermore, define the sequences $\{a_n\}_{n \geq 1}, \{b_n\}_{n \geq 1}, \{c_{n,m} : 1 \leq n, 1 \leq m \leq N_n \}, \{d_n : 2 \leq n \}  \subseteq \cH$ by
{\allowdisplaybreaks
\begin{align*}
a_n &:=  \left(a_n' + \frac{1}{n}\right) e_n
\text{ with }  a_n' := C \left\lceil  \frac{2^n}{\ln(n+1)}\right\rceil 2^{-n} \text{ and } C = 4 \left\lceil 3 + \frac{4 \ln(9)}{\ln(2)} \right\rceil  = 64, \\
 b_n &:= - 2^{-n} e_n, \\
 N_1 &:= 1 \text{ and for } n \geq 2, N_n := \left\lceil \frac{2}{n \inner{-b_n}{e_n}} \right\rceil, \\ 
c_{1,1} &:= e_1 +  \alpha_{2,1} \tilde e_{2,1}, \text{ and for } n\geq 2: \\
c_{n,1} &:= \beta_n e_n + \alpha_{n,1} \tilde e_{n,1} - \alpha_{n,2} \tilde e_{n,2}, \, \quad
\ldots \\
c_{n,N_n-1} &:= \beta_n e_n + \alpha_{n,N_n-1} \tilde e_{n,N_n-1} -  \alpha_{n,N_n} \tilde e_{n,N_n},  \\
c_{n,N_n} &:= \beta_n e_n + \alpha_{n,N_n} \tilde e_{n,N_n} -  \alpha_{n+1,1} \tilde e_{n+1,1}, \\
d_2 &:= -(1/2) \alpha_{2,1} \tilde e_{2,1} \text{ and for all } 2 \leq n \text{ let } \\
d_n &:= (1/2) \alpha_{n,1} \tilde e_{n,1}, \\
\beta_n &:= - \frac{1}{n N_n}, \text{ for which } -\beta_n < \inner{-b_n}{e_n} \text{ holds,} \\
\alpha_{n,1} &:= \sqrt{\frac{\inner{e_n}{a_n}}{n}}, \text{ and } \alpha_{n,i} := \sqrt{\alpha_{n,1}^2 + (i-1) \beta^2_n} \text{ for } 2 \leq i \leq N_n.  
\end{align*}}
$- \beta_n$ is smaller than $\inner{-b_n}{e_n}$ because 
$-\beta_n = 1/(nN_n) \leq \inner{-b_n}{e_n}/2$. Also observe that the sequence $a_n'$ is non-increasing in $n$ since 
\begin{align*}
 \left\lceil  \frac{2^{n+1}}{\ln(n+2)}\right\rceil \frac{1}{2^{n+1}} \leq \left\lceil 2 \frac{ 2^{n}}{\ln(n+1)}\right\rceil \frac{1}{2^{n+1}} &\leq \left\lceil 2 \left\lceil \frac{ 2^{n}}{\ln(n+1)}\right\rceil \right\rceil \frac{1}{2^{n+1}} 
= \left\lceil \frac{ 2^{n}}{\ln(n+1)}\right\rceil \frac{1}{2^{n}},
\end{align*}
where we used that the function $\lceil \cdot \rceil$ is monotonically increasing.

\textbf{(b) Construction of a continuous map $\phi$:}
We construct a continuous function \\$\phi:[0,1] \rightarrow \cH$ which goes through the points $\{a_n\}_{n\in \mathbb{N}}, \{b_n\}_{n\in \mathbb{N}}$, $\{c_{n,i} : 1 \leq n, 1 \leq i \leq N_n \}$
and $\{d_n : 2 \leq n \}$. We split the construction into three separate functions, $\phi_1$ for the $a_n,b_n$ elements, $\phi_2$ for the $c_{n.i}$ and $\phi_3$ for the $d_n$ elements.

For ease of reading let $y_n = 1/(n+1)$ and 
$z_n = ( y_n + y_{n+1})/2$ for all $n\geq 1$. Define $\phi_1:[0,1] \rightarrow \cH$, with n going through $1, 2, 3 \ldots$, by
$$ \phi_1(x) := \begin{cases} 
  \frac{1 - x}{1-y_1} \, a_1  & \text{if } y_1 < x \leq 1,  \\
\frac{x - \xi }{y_n- \xi} \, a_n 		& \text{if } \xi:= \frac{y_n + z_n}{2} < x  \leq y_n, \\
\frac{\xi - x  }{\xi - z_n}	\, b_n	& \text{if } z_n < x \leq \frac{y_n + z_n}{2} =: \xi, \\
\frac{x - \xi }{z_n- \xi} \, b_n 		& \text{if } \xi:= \frac{y_{n+1} + z_n}{2} < x \leq z_n, \\
\frac{\xi - x  }{\xi - y_{n+1}} \, a_{n+1}		& \text{if } y_{n+1} < x \leq \frac{y_{n+1} + z_n}{2} =: \xi, \\
0 & \text{if } x = 0.
\end{cases}
$$
The function is continuous on $(0,1]$ as it is piecewise linear and the 
end points of the lines are connected. The only critical point is $0$. For continuity at $0$ it suffices that 
for any $\epsilon >0$ we can pick a $\delta$ such that $x < \delta$ implies
$\norm{\phi(x)} < \epsilon$. We restrict the search for a $\delta$ to points $1/n$, $n \in \mathbb{N}$. For such a $\delta$ the maximum of $\phi(x)$ in an interval $[0,\delta]$ is either attained on an $a_n$ or a $b_n$. As we have that
$\lim_{n\rightarrow \infty} \norm{a_n} = \lim_{n\rightarrow \infty} \norm{b_n} = 0$
there is for every $\epsilon>0$ an $N \in \mathbb{N}$ such that for all $n> N$ 
we have $\max(\norm{a_n},\norm{b_n}) < \epsilon$ and, consequently for
$\delta = 1/(N+1)$ we have that $\norm{\phi(x)} < \epsilon$ for any $0 \leq x \leq \delta$.

In the following let $N_1 :=1$.  Furthermore, let $\tilde y_n := 1/n$, $\Delta_n := (\tilde y_n - \tilde y_{n+1})/N_n$, $u_{n,m} := \tilde y_n -  m \Delta_n$, $u_{n,0} := \tilde y_n$ and let 
$\tilde z_{n,m} := (u_{n,m-1} + u_{n,m})/2$, for all $n\geq 1$ and $1\leq m \leq N_n-1$.  
With $n$ going through all of $1,2,\ldots$, define $\phi_2:[0,1]\rightarrow \cH$ by
\[\phi_2(x) := \begin{cases} 
 \frac{\tilde y_n - x }{\tilde y_n - \tilde z_{n,1}} \, c_{n,1} & \text{if } \tilde z_{n,1} < x \leq \tilde y_n, \\
 \frac{ x - u_{n,m}}{\tilde z_{n,m} - u_{n,m}} \, c_{n,m}     & \text{if } u_{n,m} < x \leq \tilde z_{n,m},1\leq m \leq N_n -1,\\
 \frac{u_{n,m} - x }{u_{n,m} -\tilde z_{n,m+1}} \, c_{n,m+1}     & \text{if } \tilde z_{n,m+1} < x \leq  u_{n,m},  1\leq m \leq N_n -1,\\
 \frac{x -\tilde  y_{n+1}}{\tilde z_{n,N_n} - \tilde y_{n+1}} \, c_{n,N_n}     & \text{if } \tilde y_{n+1} < x \leq \tilde z_{n,N_n}, \\
 0 & \text{if } x=0.
\end{cases}
\]
Similarly, by going through all $n\geq 2$ define 
$\phi_3:[0,1]\rightarrow \cH$ by 
\[\phi_3(x) := \begin{cases} 
 0 & \text{if } 1/2 < x \leq 1, \\
 \frac{\tilde y_n - x }{\tilde y_n - \tilde z_{n,1}} \, d_n & \text{if } \tilde z_{n,1} < x \leq \tilde y_n, \\
 \frac{ x -  u_{n,1}}{\tilde z_{n,1} - u_{n,1}} \, d_n     & \text{if }  u_{n,1} < x \leq \tilde z_{n,1}, \\
0 & \tilde y_{n+1} \leq x \leq u_{n,1}, \\
 0 & \text{if } x=0.
\end{cases}
\]
With the same reasoning as for $\phi_1$ one can infer that $\phi_2$ and $\phi_3$ are continuous.
Define 
\[\phi(x) := \begin{cases} 
\phi_1(3 x -2) & \text{if } 2/3 < x \leq 1, \\
\phi_2(3x - 1) & \text{if } 1/3 < x \leq 2/3, \\
 \phi_3(3x) & \text{if } 0 \leq x \leq 1/3. \\
\end{cases}
\]
The function $\phi$ is continuous since $\phi_1,\phi_2,\phi_3$ are continuous and $\phi_1(0) =  \phi_2(1)
=\phi_2(0)=\phi_3(1)=0$. This implies also that $\phi: [0,1] \rightarrow \cH$ is measurable. It is also Bochner integrable with respect to any probability measure defined on the Borel sets of $\mathbb{R}$ as $\norm{\phi(\cdot)} : [0,1] \rightarrow \mathbb{R}$ is continuous and, hence, bounded, i.e. $\int \norm{\phi(x)} \, dP(x) < \infty$.    

\textbf{(c) Definition of the probability measure:} 
We construct a Borel measure by defining a density $p$ on $[0,1]$.
Using the variables defined for $\phi_1,\phi_2,\phi_3$, constants $\mathfrak a_1,\mathfrak b_1, \ldots$ and $n$ going through $1, 2, \ldots$ we set
\[ p_1(x) = 
\begin{cases} 
   \mathfrak a_1  & \text{if } y_1 < x \leq 1,  \\
\mathfrak a_n 		& \text{if } \frac{y_n + z_n}{2} < x  \leq y_n, \\
\mathfrak b_n	& \text{if } \frac{y_{n+1} + z_n}{2} < x \leq \frac{y_n + z_n}{2}, \\
\mathfrak a_{n+1}		& \text{if } y_{n+1} < x \leq \frac{y_{n+1} + z_n}{2}, \\
0 & \text{if } x = 0.
\end{cases}
\]
and using constants $\mathfrak c_n$  we furthermore define
\[p_2(x) := \begin{cases} 
 \mathfrak c_n, & \text{if } \tilde y_{n+1} <  x \leq \tilde y_n, \\
 0 & \text{if } x=0.
\end{cases}
\]
Finally, going through all $n\geq 2$, $1 \leq m \leq N_n$ and with the constants $\mathfrak d_{n,m}$ let 
\[p_3(x) := \begin{cases} 
 1 & \text{if } 1/2 < x \leq 1, \\
 \mathfrak d_n & \text{if } u_{n,1} \leq x \leq \tilde y_n, \\
1 & \text{if }  \tilde y_{n+1} \leq x \leq u_{n,1}\\ 
 0 & \text{if } x=0
\end{cases}
\]
and combine these to define the density $p$ by
\[p(x) := \begin{cases} 
p_1(3 x -2) & \text{if } 2/3 < x \leq 1, \\
p_2(3x - 1) & \text{if } 1/3 < x \leq 2/3, \\
p_3(3x) & \text{if } 0 \leq x \leq 1/3. \\
\end{cases}
\]
Now, $\mathfrak m = 0$ iff $\langle e_n, \mathfrak m\rangle = \E \inner{e_n}{\phi} = 0
=\langle e_{n,i}', \mathfrak m\rangle = \E \langle e_{n,i}',\phi \rangle
$ for all $n\geq 1$,  $1 \leq i \leq N_n$. 

Observe that in general, if $a,b \in [0,1]$, $a \leq b$, the density $p$ is  constant  on $[a,b]$ 
with value $\mu \in [0,\infty)$, $h \in \cH$ and $\psi:[0,1] \rightarrow \cH$  is defined by
\[
\psi(x) = \begin{cases} 
(x-a)/(b-a) \mu h  & x \in [a,b], \\
0 & \text{otherwise}
\end{cases}
\] 
then for any $e_n$ (and $e_{n,i}'$) 
\[ \inner{e_n}{\E \psi} = \E \inner{e_n}{\psi} = \int_{[a,b]} \frac{x-a}{b-a} \mu \inner{e_n}{h} 
= (1/2) \mu \inner{e_n}{h} (b-a) 
\] 
and if 
\[
\psi(x) = \begin{cases} 
(b-x)/(b-a) \mu h & x \in [a,b], \\
0 & \text{otherwise}
\end{cases}
\] 
then  
\[ \inner{e_n}{\E \psi} = (1/2)  \mu \inner{e_n}{h} (b-a).
\]
So, 
\[
\inner{e_1}{\mathfrak m} = \frac{1}{6}  \inner{e_1}{\left(1-\frac{y_1+z_1}{2} \right) \mathfrak a_1 a_1 + 
\left(\frac{y_1 - y_2}{2} \right)\mathfrak b_1 b_1 +(\tilde y_1 - \tilde y_2) \mathfrak c_1 c_{1,1}} 
\]
will be zero by setting  
\[
\mathfrak b_1 := 2\cdot 6 \left(\left(1-\frac{y_1+z_1}{2} \right)\mathfrak a_1 
\frac{\inner{a_1}{e_1}}{\inner{-b_1}{e_1} } +  \left(1 - \frac{1}{2} \right)  \mathfrak c_1\frac{\inner{c_{1,1}}{e_1}}{\inner{-b_1}{e_1} } \right)
\]
and $\inner{e_n}{\mathfrak m} = 0$ by setting
\[
\mathfrak b_n = 2(n+1)(n+2) \left(
\frac{1}{4} \left(\frac{1}{n} - \frac{1}{n+2}\right) \mathfrak a_n \frac{ \inner{a_n}{e_n}}{\inner{-b_n}{e_n}}
+ \left(\frac{1}{n} - \frac{1}{n+1} \right) \frac{\mathfrak c_n \beta_n}{\inner{-b_n}{e_n}}
\right).
\]
Also, for any $n \geq 1$ we have that $\mathfrak b_n > 0$ if $\mathfrak a_n,\mathfrak c_n >0$. 
Let, $\mathcal{N}$ be a normalising constant to be defined below and let
\[ \mathfrak a_1 :=    \left. \frac{\mathcal{N}\inner{-b_1}{e_1} }{24 \inner{a_1}{e_1}}\middle/  \left(1-\frac{y_1+z_1}{2} \right)\right. > 0 \quad \text{and} \quad \mathfrak c_1 :=\frac{\mathcal{N}\inner{-b_1}{e_1} }{12 \inner{c_{1,1}}{e_1}} >0
\]
such that $\mathfrak{b}_1 = \mathcal{N}$. Also set for all $n\geq 2$ 
\[ \mathfrak a_n :=     \frac{n}{n+1} \frac{\mathcal{N}\inner{-b_n}{e_n} }{\inner{a_n}{e_n}} >0 \quad \text{and} \quad \mathfrak c_n :=\frac{n}{2(n+2)}\frac{\mathcal{N}\inner{b_n}{e_n} }{\beta_n} \frac{\Delta_{n+1}}{\Delta_{n}} >0
\]
which makes $\mathfrak b_n = \mathcal{N} (1- \Delta_{n+1}/\Delta_{n}) $ and all $\inner{e_n}{\mathfrak m} =0$. 
For the elements $\tilde e_{n,i}$ we have that 
\[
6 \inner{\tilde e_{2,1}}{\mathfrak m} = \Delta_1 \mathfrak c_1 \inner{c_{1,1}}{\tilde e_{2,1}} + \Delta_2 \mathfrak c_2    
 \inner{c_{2,1}}{\tilde e_{2,1}} + \Delta_2 \mathfrak d_2 \inner{d_{2}}{\tilde e_{2,1}}
\]
and we set 
\[
\mathfrak d_2 := \frac{\Delta_1}{\Delta_2} \mathfrak c_1 \frac{\inner{c_{1,1}}{\tilde e_{2,1}}}{\inner{-d_2}{\tilde e_{2,1}}}+ \mathfrak c_2 \frac{\inner{c_{2,1}}{\tilde e_{2,1}}}{\inner{-d_2}{\tilde e_{2,1}}} >0.
\]
Furthermore, for all $n>2$ let
\[
\mathfrak d_n :=  \frac{\Delta_{n-1}}{\Delta_n}
\mathfrak c_{n-1} \frac{\inner{-c_{n-1,N_{n-1}}}{\tilde e_{n,1}}}{\inner{d_n}{\tilde e_{n,1}}}- \mathfrak c_n \frac{\inner{c_{n,1}}{\tilde e_{n,1}}}{\inner{d_n}{\tilde e_{n,1}}}.
\]
$\mathfrak d_n > 0$ for $n>2$ since 
\begin{align*}
\frac{2 \inner{-d_n}{\tilde e_{n,1}}}{\alpha_{n,1} \mathcal{N}} \mathfrak{d}_n =&
\frac{(n-1)^2}{n+1} \inner{-b_{n-1}}{e_{n-1}} \left\lceil \frac{2}{(n-1) \inner{-b_{n-1}}{e_{n-1}}}\right\rceil \\
& - \frac{\Delta_{n+1}}{\Delta_{n}}\frac{n^2}{n+2} \inner{-b_{n}}{e_{n}} \left\lceil \frac{2}{n \inner{-b_{n}}{e_{n}}}\right\rceil \\
\geq& 2 \frac{n-1}{n+1} - 3\frac{n}{n+2} \frac{\Delta_{n+1}}{\Delta_n} 
= 2 \frac{n-1}{n+1} - 3\frac{n^2}{(n+2)^2}  \frac{N_n}{N_{n+1}} \\
\geq& 2 \frac{n-1}{n+1} - \frac{3}{2} \frac{n^2 (n+1)}{(n+2)^2} \left(\frac{1}{n} + \frac{1}{2^{n+1}} \right)
\end{align*}
which is strictly greater zero if 
\begin{align*}
4(n-1)(n+2)^2 - 4 n (n+1)^2 =4 (n^2 + n -4)
\end{align*}
is. But this is obvious for $n\geq 3$. For all remaining $n,i\geq 2$ we can observe that 
\[\inner{\tilde e_{n,i}}{\mathfrak m} = (1/6) \Delta_n \mathfrak c_n \inner{c_{n,i-1} - c_{n,i}}{\tilde e_{n,i}} =0
\]
for all $m\geq2$. Hence, $\mathfrak{m} = 0$ and the density is strictly greater $0$ on all but three points. It remains to set $\mathcal{N}$ such that the density integrates to one. We have for any $\mathcal{N}>0$ that
\begin{align*}
0 < \int_{[0,1]} p =&\mathfrak a_1 (1-(y_1 + z_1)/2)/3  +  (1/3) \sum_{n=2}^\infty  \mathfrak a_n ((y_n - z_{n-1}) - (y_n - z_n))/2   \\
&+  (1/3) \sum_{n=1}^\infty \mathfrak b_n ((y_n - z_{n}) - (y_{n+1} - z_n))/2   \\  
& +\sum_{n=1}^\infty  \frac{1}{3n(n+1)} \mathfrak c_n
+ 1/6 +  (1/3) \sum_{n=1}^\infty \Delta_n \mathfrak d_n +  (1/3) \sum_{n=1}^\infty (u_{n,1} - \tilde y_{n+1}).
\end{align*}
The first sum is a finite multiple of $\mathcal{N}$ since $\mathfrak a_n \approx 2^{-n}$ and the sum over 
$((y_n - z_{n-1}) - (y_n - z_n))/2$ is bounded by $1$. Similarly, the $\mathfrak c_n$ sum is bounded since 
$\mathfrak c_n$ itself is upper bounded by $\mathcal N$ and the rest is quadratic in $n$. Furthermore, 
$\mathfrak b_n$ is upper bounded by $\mathcal N$ and the sum of the intervals cannot exceed $1$.  Finally,
$\mathfrak d_n$ is upper bounded since
\[
\frac{2 \inner{-d_n}{\tilde e_{n,1}}}{\alpha_{n,1} \mathcal{N}} \mathfrak{d}_n \leq
\frac{2(n-1)}{n+1} 
\]
$\alpha_{n,1}$ is bounded and so is $\inner{-d_n}{\tilde e_{n,1}}$. Hence, the sum is a finite multiple of $\mathcal{N}$ and we have in total a term that is a finite multiple of $\mathcal{N}$ plus a constant that is smaller 
than $1/2$. Therefore, we can choose $\mathcal{N}$ such that $\int p = 1$.
\\

\textbf{(d) Behaviour of the algorithm:} 
Initialize the algorithm with $w_1 := c_{1,1} \in \phi[X]$ and let $x_t$ be the element which is chosen at stage $t$. The algorithm behaves as follows: 
\begin{enumerate}
\item[(1)] For any $t \geq 1$, if $w_t \not = 0$ then $x_t \in  \{a_n\}_{n \geq 1} \cup \{b_n\}_{n \geq 1} \cup \{c_{n,m} : 1 \leq n, 1 \leq m \leq N_n\}
\cup \{d_{n,m} : 2 \leq n, 1 \leq m \leq N_n \} 
$.






\item[(2)] Let $n = \min\{m : a_m \text{ has not been chosen in steps } 1 \ldots t-1\}$. If $t \geq 2$ then 
either the smallest element of $\{(m,j) : c_{m,j} \text{ has not been chosen in steps } 1 \ldots t-1\}\backslash\{c_{1,1}\}$
in the lexicographic order is $(n,i)$ with $1 \leq i \leq N_n$ and
$$w_t = - \gamma_1 e_1 -  \ldots - \gamma_{n-1} e_{n-1} + \gamma_n e_n +  \alpha_{n,i} \tilde e_{n,i},$$
where
$$\gamma_j = (2^j a_j' - l) 2^{-j},l \in \mathbb{N}  \text{\quad  and  \quad} a_j'\geq \gamma_j \geq \min\left\{a_j' ,\max\left\{\frac{2^j \inner{a_n}{e_n}}{n} - 2^{-j},0 \right\} \right\}$$ 
for $1 \leq j \leq n-1$ and $\gamma_n = -(i-1) \beta_n $ (first case),  
or the smallest element is $(n+1,1)$ and 
$$w_t = - \gamma_1 e_1 -  \ldots - \gamma_{n-1} e_{n-1} + \gamma_n e_n +  \alpha_{n+1,1} \tilde e_{n+1,1},$$
with $\gamma_1, \ldots, \gamma_{n-1}$ like above and $\gamma_n = 1/n$ (second case). 
In particular $w_t \not = 0$.

\item[(3)] Let $N(n) := \lceil 1 + \log_2( n \ln(n+1))\rceil$ then $n -1 \geq N(n)$ for all $n\geq 7$. If $n$ is the smallest index of an $a_n$ which has not been chosen yet and if this $n \geq 7$ then   
for any $i$ with $n-1 \geq i \geq N(n)$ 
$$ \inner{e_i}{w_t}   \leq - \frac{1}{\ln(n+1)}. $$

\item[(4)] For each $n \geq 1$ there exists a step $t \geq 1$ with  $x_t = a_n$.
\end{enumerate}
\quad\\
\pStart
($\alpha$) (1) is saying that no point on the line from $0$ to an $a_n$, $b_n, c_{n,m}$ or $d_{n,m}$ is chosen that differs from $a_n, b_n, c_{n,m}$ and $d_{n,m}$. To see this first observe  that only points  $\phi(x)$ will be chosen at any 
stage $t$ for which $\langle w_t, \phi(x) \rangle > 0$: By assumption $w_t \not= 0$. If there exists an $e_n$ with $\inner{e_n}{w_t} \not = 0$ then either $\inner{a_n}{w_t}$ or
$\inner{b_n}{w_t}$ is strictly positive. Also, if there is an $\tilde e_{n,m}$,  $(n,m) \not = (2,1)$, such that $\inner{\tilde e_{n,m}}{w_t} > 0 $ then
$\inner{d_{n,m}}{w_t}$ is strictly positive. Similarly, if $\inner{\tilde e_{2,1}}{w_t} < 0$ then $\inner{d_{2,1}}{w_t} >0$. Assuming that none of these cases apply we have that 
either  $\inner{\tilde e_{2,1}}{w_t} > 0$ or there is an $\tilde e_{n,m}$,  $(n,m) \not = (2,1)$, with $\inner{\tilde e_{n,m}}{w_t} < 0 $. 
In the first case $\inner{c_{1,1}}{w_t}>0$.  
In the latter case and with $\inner{\tilde e_{2,1}}{w_t} = 0$ 
let $(n',m') := \min \{ (n,m) :  \inner{\tilde e_{n,m}}{w_t} < 0 \}$ where the minimum is taken wrt. the lexicographic ordering. We have
$\inner{c_{n',m'-1}}{w_t} = \alpha_{n',m'-1} \inner{\tilde e_{n',m'-1}}{w_t}
- \alpha_{n',m'} \inner{\tilde e_{n',m'}}{w_t} > - \alpha_{n',m'} \inner{\tilde e_{n',m'}}{w_t} >0 $, if $m'>1$, and if $m'=1$ then  \[\langle c_{n'-1,N_{n'-1}},w_t\rangle = 
\alpha_{n'-1,N_{n'-1}} \langle\tilde e_{n'-1,N_{n'-1}},w_t \rangle 
- \alpha_{n',1} \inner{\tilde e_{n',1}}{w_t} > - \alpha_{n',1} \langle \tilde e_{n',1}, w_t\rangle > 0.\]

If the chosen $\phi(x)$ is on the line from $0$ to an $a_n$ then $a_n = \xi \phi(x)$ with $\xi \geq 1$ and 
$0 < \inner{\phi(x)}{w_t} \leq \xi \inner{\phi(x)}{w_t} = \inner{a_n}{w_t}$ and $\phi(x) = a_n$. The same argument applies to $b_n, c_{n,m}$ and $d_{n,m}$. 

($\beta$) We \textit{prove by induction} over $t \geq 2$ that (2) holds.
We start with the \textit{induction basis}. $w_1 = c_{1,1} = e_1 + \alpha_{2,1} \tilde e_{2,1}$ and 
we have $\inner{w_1}{b_n} \leq 0$,$\inner{w_1}{d_{n}} \leq 0$ for all $n$, $\inner{w_1}{a_n} = 0$ for all $n\geq 2$, 
$\inner{w_1}{c_{n,i}} = 0$ if either $n > 2$ or ($n=2$ and $i>2$). Furthermore,  
\begin{align*}
\inner{w_1}{c_{1,1}}= \norm{c_{1,1}}^2 = 1 + \alpha_{2,1}^2=  1 + \frac{\inner{e_2}{a_2}}{2}  &= 1 + \frac{C}{8} \left\lceil \frac{4}{\ln(3)}\right\rceil + \frac{1}{4} \leq 1 + \frac{C}{4} \left\lceil \frac{2}{\ln(2)} \right\rceil + \frac{1}{4} \\
 &< 1 + \frac{C}{2} \left\lceil \frac{2}{\ln(2)} \right\rceil  = \inner{w_1}{a_1}
 \end{align*}  
since $C\geq 1$, $\lceil 2/\ln(2) \rceil = 3$  and hence 
$$\frac{C}{4} \left\lceil \frac{2}{\ln(2)}\right\rceil >\frac{1}{4}.$$
Also,
$ \inner{w_1}{c_{2,1}} = \alpha_{2,1}^2 < \inner{w_1}{c_{1,1}} < \inner{w_1}{a_1}$
and $x_1 = a_1$. Therefore,
$$ w_2 = w_1 - a_1 = e_1 + \alpha_{2,1} \tilde e_{2,1} - (a_1' + 1) e_1 = - a_1' e_1 + \alpha_{2,1} \tilde e_{2,1}$$
and $w_2$ has the promised form.

($\gamma$) Next, we address the \textit{induction step}. 
\textbf{(i)}
Assuming $w_t$ has the given form in step $t$ we can observe that 
$$ \inner{x_t}{w_t} \geq \frac{\inner{a_n}{e_n}}{n} > 0$$
since in the first case 
$$\inner{c_{n,i}}{w_t} = 
\beta_n \gamma_n  + \alpha_{n,i}^2  =
-(i-1) \beta_n^2 +  (\alpha_{n,1}^2 + (i-1) \beta_n^2) =
 \frac{\inner{e_n}{a_n}}{n}
 $$ 
in case that $i>1$ or, for $i=1$, 
$$ \inner{c_{n,i}}{w_t} = 
\alpha_{n,1}^2  =
 \frac{\inner{e_n}{a_n}}{n}.
 $$ 
In the second case, 
$$\inner{a_n}{w_t} \geq \gamma_n \inner{e_n}{a_n} = \frac{\inner{e_n}{a_n}}{n}.$$ 
\textbf{(ii)} For the $b_j$ \textbf{(i)} implies that, first, no $j\geq n$ will have been chosen in $t$ since for these $\inner{b_j}{w_t} \leq 0$ holds in both cases. Also, if for a $j$, $1 \leq j \leq n-1$,  $\gamma_j < 2^j \inner{a_n}{e_n}/n $ then 
$\inner{b_j}{w_t} = \gamma_j 2^{-j} < \inner{a_n}{e_n}/n$ and $b_j \not = x_t$. 
On the other hand, if $\gamma_j  \geq 2^j \inner{a_n}{e_n}/n$ and $x_t = b_j$ then 
the coefficient changes by $-2^{-j}$, i.e. the new coefficient is
$$ \gamma_j - 2^{-j} \geq  2^j \frac{\inner{a_n}{e_n}}{n} - 2^{-j}.$$
The coefficient is also always non-negative since $\gamma_j$ is a multiple of $2^{-j}$ and $b_j$ will not be selected if $\gamma_j = 0$. In total, all cases are consistent with our induction hypothesis and we are safe against any application of $b_j$. 

\textbf{(iii)} In terms of $a_j$, we can directly observe that $\inner{a_j}{w_t} = 0$ if $j > n$ and $\inner{a_j}{w_t} = - \gamma_j \inner{e_j}{a_j} \leq 0 $ if $j<n$. So only 
$a_n$ might have been chosen at time $t$. However, in the first case we have 
that 
$$\inner{a_n}{w_t} = \gamma_n \inner{a_n}{e_n}  = - (i-1) \beta_n \inner{a_n}{e_n} 
\leq \frac{N_n -1}{n N_n}  \inner{a_n}{e_n}  <  \frac{\inner{a_n}{e_n}}{n} $$ 
and $x_t \not = a_n$. In the second case, if $\gamma_n = 1/n$, then
$$\inner{a_n}{w_t} = \frac{\inner{a_n}{e_n}}{n}.$$
Thus $a_n$ might be chosen, and, in case it is, then the new coefficient is
$\gamma_n - a_n' - 1/n = - a_n'$ which is consistent with the induction hypothesis.

\textbf{(iv)} Turning to the $c_{n',i'}$ elements we can observe that for $(n,i) > (2,1)$ we have $\inner{w_t}{c_{1,1}}\leq 0$. 
For $(n,i) = (2,1)$ we are in the first case since $N_2 = 4$ and 
\begin{align*}
\inner{w_t}{c_{1,1}} &= \inner{-\gamma_1 e_1 + \gamma_2 e_2 + \alpha_{2,1} \tilde e_{2,1}}{e_1 + \alpha_{2,1} \tilde e_{2,1}} = - \gamma_1 + (1/2) \inner{e_2}{a_2} \\
&\leq - (1/2) \inner{e_2}{a_2} + 1/2 + (1/2) \inner{e_2}{a_2}.
\end{align*}
But, $ 1/2 < (1/2) \inner{a_2}{e_2}$ and $c_{1,1}$ will never be chosen.  For the remaining $c_{n',i'}$ elements we have in the first case in step $t$ that no $c_{n',i'}$ will be chosen if $n' \not= n$ or $i' \not = i$ since
in these cases
$$   \inner{w_t}{c_{n',i'}} 
\begin{cases}
=0 & \text{if } n' > n, \\
= \gamma_n \beta_n = -(i-1) \beta_n^2 \leq 0 & \text{if } n'=n \text{ and } ( i' > i  \text{ or } i'< i-1), \\
= \gamma_n \beta_n - \alpha_{n,i}^2 < 0 & \text{if } n'=n \text{ and } i' + 1 = i, \\
 \leq 0 & \text{if } 2 \leq n' = n-1 \text{ and } i' = N_{n'} \text{ and } i=1, \quad (*) \\
= - \gamma_{n'} \beta_{n'} < 
\inner{w_t}{b_{n'}}  & \text{if } n' < n \text{ and } (*) \text{ does not apply.} 
\end{cases} $$
(*) follows  for $n\geq 3$  from 
\begin{align*} &\inner{w_t}{c_{n-1,N_{n-1}}} = \inner{w_t}{\beta_{n-1} e_{n-1} + \alpha_{n-1,N_{n-1}} \tilde e_{n-1,N_{n-1}} - \alpha_{n,1} \tilde e_{n,1}} \\
&= -\gamma_{n-1} \beta_{n-1} - \alpha^2_{n,1} = \frac{\gamma_{n-1}}{(n-1)N_{n-1}} - \alpha^2_{n,1} \\
&\leq \gamma_{n-1} 2^{-n} - \frac{\inner{e_n}{a_n}}{n} \leq 2^{-n} a_{n-1}' - \frac{a_n'}{n} \\
&= C \left(2^{-n} \left\lceil\frac{2^{n-1}}{\ln(n)} \right\rceil 2^{-(n-1)} 
-\frac{1}{n} \left\lceil\frac{2^{n}}{\ln(n+1)} \right\rceil 2^{-n}
\right)\\
 &\leq C n^{-1} 2^{-2n-1} \left(n \left\lceil\frac{2^{n}}{\ln(n^2)} \right\rceil 
- 2^{n-1} \left\lceil\frac{2^{n}}{\ln(n+1)} \right\rceil \right) \leq 0.
\end{align*}
If now $x_t = c_{n,i}$ then, in case $i<N_n$, 
\begin{align*}
 w_{t+1} &= w_t - c_{n,i} = - \sum_{i=1}^{n-1} \gamma_i e_i -(i-1) \beta_n e_n + \alpha_{n,i} \tilde e_{n,i} 
- \beta_n e_n - \alpha_{n,i} \tilde e_{n,i} + \alpha_{n,i+1} \tilde e_{n,i+1} \\
&= - \sum_{i=1}^{n-1} \gamma_i e_i - i \beta_n e_n + \alpha_{n,i+1} \tilde e_{n,i+1} 
\end{align*}
which has the desired form. Similarly, in case that $i= N_n$
\begin{align*}
 w_{t+1} &= w_t - c_{n,N_n} \\
 &= - \sum_{i=1}^{n-1} \gamma_i e_i -(N_n-1) \beta_n e_n + \alpha_{n,N_n} \tilde e_{n,N_n} 
- \beta_n e_n - \alpha_{n,N_n} \tilde e_{n,N_n} + \alpha_{n+1,1} \tilde e_{n+1,1} \\
&= - \sum_{i=1}^{n-1} \gamma_i e_i - N_n \beta_n e_n + \alpha_{n+1,1} \tilde e_{n+1,1} 
\end{align*}
which has the form of the second case since $- N_n \beta_n = 1/n$.

In the second case no element $c_{n',i'}$ will be chosen since
$$   \inner{w_t}{c_{n',i'}} = 
\begin{cases}
0 & \text{if } n' > n + 1 \text{ or } (n' = n+1 \text{ and } i'>1), \\
\alpha_{n+1,1}^2 
< \inner{e_n}{a_n}/n & \text{if } n'=n+1 \text{ and } i'=1, \\
\beta_n / n - \alpha_{n+1,1}^2 < 0 & \text{if } n' = n \text{ and } i'= N_n, \\
\beta_n / n  < 0 & \text{if } n' = n \text{ and } i' < N_n, \\
- \gamma_{n'} \beta_{n'}   <  \inner{w_t}{b_{n'}}  \text{ or } =0& \text{if } n' < n. \\
\end{cases} $$
\textbf{(v)} We turn to the $d$ elements. First case: If $(n,i) = (2,1)$ then $\inner{w_t}{d_{2}} = -(1/2) \alpha_{2,1}^2 <0$ and otherwise 
$\inner{w_t}{d_{2}} = 0$. In either way $d_{2}$ will not be chosen. For any other $n'$ we have that
$\inner{w_t}{d_{n'}} = 0$ if $(n,i) \not = (n',1)$. Otherwise
\begin{align*}
&\inner{w_t}{d_{n'}} = (1/2) \alpha_{n,1}^2 
= \frac{1}{2} 
\frac{\inner{e_n}{a_n}}{n}  
\end{align*}
and 
$\inner{w_t}{d_{n'}} < \inner{a_n}{e_n}/n$ and $d_{n'}$ will never be chosen.

Second case: 
 $\inner{w_t}{d_{2}} \leq 0$ and  all $d_{n'} $ with $n' \not = n+1$ are zero. Finally
\[
\inner{w_t}{d_{n+1}} = (1/2) \alpha_{n+1,1}^2 = \frac{\inner{e_{n+1}}{a_{n+1}}}{n+1} <
\frac{\inner{e_{n}}{a_{n}}}{n}
\]
since the sequence  $\inner{e_{n}}{a_{n}}$ is non-increasing.

So in step $t+1$ the element $w_{t+1}$ will have the right form, and, certainly, $w_{t+1} \not = 0$.

$(\delta)$ Next we prove \textbf{(3)}. Since the smallest index $n$ for which $a_n$ has not been chosen in rounds $1$ to $t$ is assumed to be larger than $7$ we can assume that $t\geq 2$.  

For each $m$, $\inner{a_m}{e_m}/m \geq a_m'/m \geq 1/(m\ln(m	+1))$. Hence, from (2) we conclude for all $i$ with 
$n-1 \geq i \geq N(n)$
\begin{align*}
-\langle e_i, w_{t+1} \rangle &= \gamma_i \geq \min\left\{ a_i',\max\left\{ 2^i \frac{\inner{a_n}{e_n}}{n} - 2^{-i},0 \right\} \right\} \\
&\geq  \min\left\{ \frac{1}{\ln(i+1)}, \max\left\{ \frac{2^i}{n \ln(n+1)} -2^{-i},0\right\} \right\}. 
\end{align*}  
Using the assumption $ i \geq N(n) $ we observe that 
\[ i \geq 1 + \log_2 (n \ln(n+1)) = \log_2( 2 n \ln(n+1))  \Rightarrow 
2^i \geq 2n \ln(n+1) \Rightarrow   \frac{2^i}{n \ln(n+1)} - 1 \geq 1\]
and since $i>1$ 
\[
\frac{2^i}{n \ln(n+1)} - 2^{-i}  \geq \frac{2^i}{n \ln(n+1)} - 1 \geq \frac{1}{\ln(i+1)} \geq \frac{1}{\ln(n+1)} .
\]
So, for $n -1 \geq i\geq N(n)$,  
$$  -\langle e_i, w_{t+1} \rangle \geq   \frac{1}{\ln(n+1)}.$$

\textbf{($\epsilon$)}  \textbf{(4)} also follows from (2). 
First, if $a_n$ has been chosen in any step $t$ then
for all $t' > t $ (2) tells us that $\inner{a_n}{w_{t'}} \leq 0$. Consequently, $x_{t'} \not = a_n$ and  $a_n$ will 
be chosen at most ones. Also, if $a_n$ is the element with minimal index which has not been chosen yet in time $t$ then
either $\inner{w_t}{a_n} = \inner{w_t}{a_m} = 0$ or $\inner{w_t}{a_n} > \inner{w_t}{a_m}$ for all $m>n$. In the first case no element $a_{n'}$ will be chosen and in the second case, if an $a_{n'}$ 
will be chosen it will be the one with the smallest index in the set of elements which have not been chosen yet. So the elements $a_n$ will be chosen in order and no element will be skipped.

Let us now assume that $\{a_m: a_m \not = x_t \text{ for all } t \geq 1 \}$ is not empty and let 
$a_n$ be the element with the smallest index in this set. 

The argument in ($\gamma$) shows us that no $c_{m,j}$ with $m>n$ will be chosen. 
Also, if $c_{m,j}$, $m\leq n$, has been chosen in any step $t$  then for all $t' > t$ we again infer from (2) 
and the argument in ($\gamma$) that $c_{m,j}$ will not be chosen in $t'$ and, hence, each $c_{m,j}$ is not chosen more than ones.  Also none of the $d_{n,i}$ elements will be ever chosen.

So the only way that an $a_n$ is never chosen is that infinite many $b_m$ elements are selected. Yet,
no $b_m$ with $m \geq n$ will be chosen since the inner product with the weight vector is less or equal to zero. Also  
each $b_m$, $m<n$ can only be chosen finite many times before the weight vector in direction $e_m$ becomes $0$ and the inner product with $b_m$ becomes $0$ too (at which point it will certainly not be chosen any more). So only finite many applications of $b_m$'s are possible with a contradiction that $a_n$ will not be chosen.  
 \pEnd

\textbf{(e) Unboundedness:}
d.3 and d.4 allow us now to show that the sequence $\{\norm{w_t}\}_{t \geq 1}$ is unbounded. Assume that at stage $t$ the element $n$ is the smallest index such that
$\inner{a_n}{w_t}$ is positive.

For $n\geq 7$ we know from $d.3$ that $\abs{\inner{e_i}{w_t}} \geq 1/ \ln(n+1)$ for all $i$, $N(n) \leq i < n$
Hence, for $n \geq 7$,    
\begin{align*}
\norm{w_t}^2 &= \sum_{i=1}^\infty\abs{\langle e_i, w_t \rangle}^2 \geq \sum_{i= N(n)}^{n-1} \frac{1}{(\ln(n+1))^2} \geq \frac{n-1-N(n)}{(\ln(n+1))^2}  \\
&= \frac{n-1 - \lceil 1 + \log_2 (n \ln(n+1))\rceil}{(\ln(n+1))^2} 
\geq \frac{n-3}{(\ln(n+1))^2} - \frac{\log_2 (n\ln(n+1))}{(\ln(n+1))^2}.
\end{align*}
Furthermore, 
$$ \frac{\log_2 (n\ln(n+1))}{(\ln(n+1))^2} = \frac{\ln(n) + \ln(\ln(n+1)))}{\ln(2) (\ln(n+1))^2} 
\leq \frac{2}{\ln(2)}, 
$$
since $\ln(x) \leq x$ for all $x > 0$ and $\ln(n+1) > 0$. Hence, 
$$\norm{w_t}^2 \geq \frac{n-3}{(\ln(n+1))^2} -  \frac{2}{\ln(2)}. $$
The right side goes to infinity in $n$ and, since for every $n$ there is a $t$ at which $a_n$ is chosen 
due to d.4, the norm of $w_t$ crosses any boundary at one time $t$. 
\end{proof}

\begin{corollary}
There exists a continuous kernel on $[0,1]$, a Borel probability measure on $[0,1]$ which assigns positive measure to open subsets of $[0,1]$ and an initialization for which the algorithm does not converge with a $1/t$ rate to $\mean$.
\end{corollary}
\begin{proof}
We consider the Hilbert space $(\cH,\inner{\cdot}{\cdot})$  from Proposition \ref{Thm:CounterExample} with the corresponding feature map $\phi:[0,1] \rightarrow \cH$. Define the continuous kernel function $k(x,y) := \inner{\phi(x)}{\phi(y)}$ on $[0,1]$ and let the corresponding RKHS be $(\mathcal{K},(\cdot,\cdot))$.
The geometry of the two spaces is closely related. We have for scalars $a_i, b_j$ and $x_i, y_j \in [0,1]$, $i=1,\ldots,n, j=1, \ldots, m$, that
\begin{align*}
\inner{\sum_{i=1}^n a_i \phi(x_i)}{\sum_{j=1}^m b_j \phi(y_j)} = \sum_{i=1}^n \sum_{j=1}^m k(x_i,y_j) = 
\left( \sum_{i=1}^n a_i k(x_i,\cdot),\sum_{j=1}^m b_j k(y_j,\cdot)\right).
\end{align*}
Furthermore, we know that the Bochner-integral $\mathfrak{m} \in \cH$ lies in $\cch \phi[\X]$ which equals the closure of $\ch \phi[\X]$ \cite{SIM11}[Thm. 5.2, p.71] 
and there exists a sequence $\{n_i\}_{i \in \mathbb{N}}$, $n_i \in \mathbb{N}$, elements $x_{ij} \in [0,1]$, and non-negative weights $a_{ij}$ with $\sum_{j=1}^{n_i} a_{ij} =1$ such that the sequence 
$\{s_i = \sum_{j=1}^{n_i} a_{ij} \phi(x_{ij})\}_{i\in \mathbb{N}}$ converges to $\mathfrak m$ in norm, i.e. $\norm{\mathfrak{m} - s_i} \rightarrow 0$ for $i\rightarrow \infty$. The corresponding sequence $\{\bar s_i = \sum_{j=1}^{n_i} a_{ij} k(x_{ij},\cdot))\}_{i\in \mathbb{N}}$ is a Cauchy sequence in $\mathcal{K}$ since 
\[ 
\norm{\bar s_i - \bar s_j}_\mathcal{K} = \norm{s_i - s_j}_\cH
\]
and has a limit $\mathfrak n \in \mathcal{K}$ because  $\mathcal{K}$ is complete. In particular,  for any $x \in \X$ 
\begin{align*}
 \abs{ \left(\mathfrak n, k(x,\cdot)\right) - \inner{\mathfrak m}{\phi(x)}}&= 
\abs{(\mathfrak{n} - \bar s_i,k(x,\cdot)) + (\bar s_i,k(x,\cdot)) - \inner{s_i}{\phi(x)} + \inner{s_i - \mathfrak{m}}{\phi(x)}} \\
&\leq k(x,x) \norm{\mathfrak{n} - \bar s_i}_\mathcal{K} + \norm{\phi(x)}_\cH \norm{\mathfrak m - s_i}_\cH \rightarrow 0 \text{ (in } i)
\end{align*}
and $(\mathfrak{n},k(x,\cdot)) = \inner{\mathfrak{m}}{\phi(x)}$ for every $x \in X$. 
Furthermore, for arbitrary $l$ points $x_1,\ldots, x_l \in X$ and scalars $a_1,\ldots,a_l$ it holds that  $(\mathfrak{n},\sum_{i=1}^l a_i k(x_i,\cdot))
=\langle \mathfrak{m},\sum_{i=1}^l a_i \phi(x_i)\rangle$ and
\begin{align*}
\abs{\norm{\mathfrak n}_\mathcal{K} - \norm{\mathfrak m}_\cH} &\leq |\norm{\mathfrak n}_\mathcal{K} - \norm{\bar s_i}_\mathcal{K}\!| 
+ |\norm{\bar s_i}_\mathcal{K} - \norm{s_i}_\cH \!| 
+ |\norm{s_i}_\cH - \norm{\mean}_\cH \!|
\end{align*}
which also goes to $0$ in $i$ and therefore $\norm{\mathfrak{n}}_\mathcal{K} = \norm{\mathfrak{m}}_\cH$.

The function $k(x,\cdot): X \rightarrow \cH$ is continuous and therefore Bochner-integrable with respect to $P$. Denote the Bochner integral with $\mathfrak{n}' = \int k(x,\cdot) \, dP$.  
For any $x \in X$
\[ 
\mathfrak n(x) = (\mathfrak n, k(x,\cdot)) = \inner{\mathfrak{m}}{\phi(x)} = \E \inner{\phi(\cdot)}{\phi(x)} = 
\E k(x,\cdot) = (\mathfrak{n}',k(x,\cdot)) = \mathfrak n'(x)
\]
and $\mathfrak n = \mathfrak n'$.

Now if the algorithm is applied in $(\mathcal{K},(\cdot,\cdot))$ with the initialization $k(x_0,\cdot)$, where $x_0$ is the element in $\X$ that maps to the initialization $\phi(x_0)$ that we use in Proposition \ref{Thm:CounterExample}, then sequences of elements $x_t$ and of weights $w_t$ are generated. The weights $w_t$ are of the form $k(x_0,\cdot) + \sum_{i=1}^t k(x_i,\cdot) - t \mathfrak n$. The sequence $x_1,x_2, \ldots$ also 
maximizes the  objective in $(\cH,\inner{\cdot}{\cdot})$. This can be seen by an induction over the weights $\tilde w_t \in \cH$ that are generated by the algorithm. The induction step is the following. 
\begin{align*}
&\max_{x \in [0,1]} \inner{\tilde w_t}{\phi(x)} = \max_{x \in [0,1]} \left( \inner{\phi(x_0)}{\phi(x)} + \sum_{i=1}^t \inner{\phi(x_i)}{\phi(x)} - t \inner{\mathfrak{m}}{\phi(x)}\right ) \\
&=  \max_{x \in [0,1]} \left( (k(x_0,\cdot),k(x,\cdot)) + \sum_{i=1}^t (k(x_i,\cdot),k(x,\cdot)) - t(\mathfrak{n},k(x,\cdot))\right) 
= (w_t,k(x_{t+1},\cdot)) \\
&=  (k(x_0,\cdot),k(x_{t+1},\cdot)) + \sum_{i=1}^t (k(x_i,\cdot),k(x_{t+1},\cdot)) - t(\mathfrak{n},k(x_{t+1},\cdot)) \\
&= \Bigl\langle \phi(x_0) + \sum_{i=1}^t \phi(x_i) -t\mathfrak{m},\phi(x_{t+1})\Bigr\rangle = \inner{\tilde w_t}{\phi(x_{t+1})}.
\end{align*}
From Proposition \ref{Thm:CounterExample} we can now infer  that the sequence $\{\norm{\tilde w_t}_\cH\}_{t \in \mathbb{N}} = \{\norm{w_t}_\mathcal{K}\}_{t \in \mathbb{N}}$ is unbounded and the algorithm does not converge with the fast rate in $\mathcal{K}$. 
\end{proof}

\bibliographystyle{plainnat}

\begin{thebibliography}{33}
\providecommand{\natexlab}[1]{#1}
\providecommand{\url}[1]{\texttt{#1}}
\expandafter\ifx\csname urlstyle\endcsname\relax
  \providecommand{\doi}[1]{doi: #1}\else
  \providecommand{\doi}{doi: \begingroup \urlstyle{rm}\Url}\fi

\bibitem[Agarwal et~al.(2005)Agarwal, Har-Peled, and Varadarajan]{AGAR05}
P.~K. Agarwal, S.~Har-Peled, and K.R. Varadarajan.
\newblock Geometric approximation via coresets.
\newblock \emph{Combinatorial and Computational Geometry}, 2005.

\bibitem[Aronszajn(1950)]{ARON50}
N.~Aronszajn.
\newblock Theory of reproducing kernels.
\newblock \emph{Transactions of the American Mathematical Society}, 68, 1950.

\bibitem[Bach et~al.(2012)Bach, Lacoste-Julien, and Obozinski]{BACH12}
F.~Bach, S.~Lacoste-Julien, and G.~Obozinski.
\newblock On the equivalence between herding and conditional gradient
  algorithms.
\newblock In \emph{Proceedings of the International Conference on Machine
  Learning}, 2012.

\bibitem[Badoiu et~al.(2002)Badoiu, Har-Peled, and Indyk]{BAD02}
M.~Badoiu, S.~Har-Peled, and P.~Indyk.
\newblock Approximate clustering via coresets.
\newblock In \emph{Proc. 34th Annu. ACM Sympos. Theory Comput.}, 2002.

\bibitem[Bartlett and Mendelson(2002)]{BART02}
P.~Bartlett and S.~Mendelson.
\newblock Rademacher and gaussian complexities: Risk bounds and structural
  results.
\newblock \emph{Journal of Machine Learning Research}, 2002.

\bibitem[Beck and Teboulle(2004)]{BECK04}
A.~Beck and M.~Teboulle.
\newblock A conditional gradient method with linear rate of convergence for
  solving convex linear systems.
\newblock \emph{Math. Meth. Op. Res.}, 2004.

\bibitem[Beer(1993)]{BEER93}
G.~Beer.
\newblock \emph{Topologies on Closed and Closed Convex Sets}.
\newblock Springer, 1993.

\bibitem[Braun et~al.(2022)Braun, Carderera, Combettes, Hassani, Karbasi,
  Mokhtari, and Pokutta]{BRAU22}
G.~Braun, A.~Carderera, C.~W. Combettes, H.~Hassani, A.~Karbasi, A.~Mokhtari,
  and S.~Pokutta.
\newblock Conditional gradient methods.
\newblock \emph{arXiv}, 2022.

\bibitem[Campbell and Meyer(2009)]{CAMP09}
S.L. Campbell and C.D. Meyer.
\newblock \emph{Generalized Inverses of Linear Transformations}.
\newblock Classics in Applied Mathematics. SIAM, 2009.

\bibitem[Chen et~al.(2010)Chen, Welling, and Smola]{CHEN10}
Y.~Chen, M.~Welling, and A.~Smola.
\newblock Supersamples from kernel-herding.
\newblock In \emph{Proceedings of the Conference on Uncertainty in Artificial
  Intelligence}, 2010.

\bibitem[Defant and Floret(1992)]{DEF92}
A.~Defant and K.~Floret.
\newblock \emph{Tensor Norms and Operator Ideals}.
\newblock North-Holland Mathematics Studies, 1992.

\bibitem[Diestel and Uhl(1977)]{DIES77}
J.~Diestel and J.J. Uhl.
\newblock \emph{Vector measures}.
\newblock American Mathematical Soc., 1977.

\bibitem[Dudley(2014)]{DUDLEY14}
R.M. Dudley.
\newblock \emph{Uniform Central Limit Theorems}.
\newblock Cambridge University Press, 2nd edition, 2014.

\bibitem[Dwivedi and Mackey(2021)]{DWI21}
R.~Dwivedi and L.~Mackey.
\newblock Kernel thinning.
\newblock \emph{arXiv}, 2021.

\bibitem[Engelking(1989)]{ENG89}
R.~Engelking.
\newblock \emph{General Topology}.
\newblock Heldermann Verlag Berlin, 1989.

\bibitem[Floyd and Warmuth(1995)]{FLOYD95}
S.~Floyd and M.~Warmuth.
\newblock Sample compression, learnability, and the vapnik-chervonenkis
  dimension.
\newblock \emph{Machine Learning}, 1995.

\bibitem[Frank and Wolfe(1956)]{FW56}
M.~Frank and P.~Wolfe.
\newblock An algorithm for quadratic programming.
\newblock \emph{Naval Res. Logist. Quart.}, 1956.

\bibitem[Fremlin(2001)]{FREM}
D.H. Fremlin.
\newblock \emph{Measure Theory}.
\newblock Torres Fremlin, 2001.

\bibitem[Gin\'e and Nickl(2016)]{GINE15}
E.~Gin\'e and R.~Nickl.
\newblock \emph{Mathematical Foundations of Infinite-dimensional Statistical
  Models}.
\newblock Cambridge University Press, 2016.

\bibitem[Harvey and Samadi(2014)]{HARV14}
N.~Harvey and S.~Samadi.
\newblock Near-optimal herding.
\newblock In \emph{Conference on Learning Theory}, 2014.

\bibitem[Huggins et~al.(2016)Huggins, Campbell, and Broderick]{HUG16}
J.~H. Huggins, T.~Campbell, and T.~Broderick.
\newblock Coresets for scalable bayesian logistic regression.
\newblock In \emph{Proceedings of the 30th International Conference on Neural
  Information Processing Systems}, 2016.

\bibitem[Koltchinskii(2018)]{KOLT18}
V.~Koltchinskii.
\newblock Asymptotic efficiency in high-dimensional covariance estimation.
\newblock In \emph{Proc. ICM 2018}, 2018.

\bibitem[Ledoux and Talagrand(2013)]{TAL13}
M.~Ledoux and M.~Talagrand.
\newblock \emph{Probability in Banach Spaces: Isoperimetry and Processes}.
\newblock Classics in Mathematics. Springer Berlin Heidelberg, 2013.

\bibitem[Littlestone and Warmuth(1986)]{LITTLE86}
N.~Littlestone and M.~K. Warmuth.
\newblock Relating data compression and learnability.
\newblock Technical report, University of California at Santa Cruz, 1986.

\bibitem[Murphy(1990)]{MUR90}
G.J. Murphy.
\newblock \emph{$C^*$-Algebras and Operator Theory}.
\newblock Academic Press, 1990.

\bibitem[Paulsen and Raghupathi(2016)]{PAUL16}
V.~I. Paulsen and M.~Raghupathi.
\newblock \emph{An Introduction to the Theory of Reproducing Kernel Hilbert
  Spaces}.
\newblock Cambridge Studies in Advanced Mathematics. Cambridge University
  Press, 2016.

\bibitem[Phillips and Tai(2020)]{PHIL20}
J.M. Phillips and W.M. Tai.
\newblock Near-optimal coresets of kernel density estimates.
\newblock \emph{Discrete \& Computational Geometry}, 2020.

\bibitem[Reed and Simon(1972)]{REED72}
M.~Reed and B.~Simon.
\newblock \emph{Functional Analysis}, volume~I.
\newblock Academic Press, 1972.

\bibitem[Rockafellar(1972)]{ROCK72}
R.T. Rockafellar.
\newblock \emph{Convex Analysis}.
\newblock Princeton University Press, 2nd edition, 1972.

\bibitem[Simon(2011)]{SIM11}
B.~Simon.
\newblock \emph{Convexity: an analytic viewpoint}.
\newblock Cambridge university press, 2011.

\bibitem[Steinwart and Christmann(2008)]{STEIN08}
I.~Steinwart and A.~Christmann.
\newblock \emph{Support Vector Machines}.
\newblock Springer, 2008.

\bibitem[Temlyakov(2011)]{TEMLY11}
V.~Temlyakov.
\newblock \emph{Greedy Approximation}.
\newblock Cambridge University Press, 2011.

\bibitem[Werner(2002)]{Wer02}
D.~Werner.
\newblock \emph{Funktionalanalysis}.
\newblock Springer, 4th edition, 2002.

\end{thebibliography}

\end{document}